\documentclass{article}

\usepackage{wrapfig,enumitem}
\usepackage[nonatbib,preprint]{neurips_2022}

\usepackage{amssymb,amsthm,graphicx}

\usepackage{amsfonts}
\usepackage{amsmath}
\usepackage{amsthm}
\usepackage{amssymb} % NOT COMPATIBLE WITH svjour3

\usepackage{hyperref}

\usepackage[capitalize,noabbrev]{cleveref}

\usepackage{nicefrac}
\usepackage{mathtools}

\usepackage{empheq}

\newcommand{\R}{\mathbb{R}} % Reals
\newcommand{\cD}{{\cal D}}

\newcommand{\cO}{{\cal O}}

\newcommand{\mH}{{\bf H}}
\newcommand{\mI}{{\bf I}}

\newcommand{\mM}{{\bf M}}
\newcommand{\mN}{{\bf N}}

\newcommand{\mP}{{\bf P}}

\usepackage[colorinlistoftodos,bordercolor=orange,backgroundcolor=orange!20,linecolor=orange,textsize=scriptsize]{todonotes}

\newcommand{\eqdef}{\coloneqq}

\newcommand{\dotprod}[1]{\left< #1\right>} % product
\newcommand{\norm}[1]{ \left\| #1 \right\|}      % norm

\newcommand{\E}[1]{\mathbb{E}\left[#1\right] }
\DeclareMathOperator{\Null}{Null}  % nullpsace
\DeclareMathOperator{\Range}{Range}     % range
\DeclareMathOperator{\argmin}{argmin}        % argmin
\newcommand{\negquad}{\mkern-18mu}

\usepackage{microtype}
\usepackage{graphicx}
\usepackage{subfigure}
\usepackage{booktabs} 

\usepackage{hyperref}
\graphicspath{{./figures/}}

\usepackage{amsmath}
\usepackage{amssymb}
\usepackage{mathtools}
\usepackage{amsthm}

\usepackage{mdframed}
\usepackage{thmtools}

\definecolor{shadecolor}{gray}{0.90}
\declaretheoremstyle[
headfont=\normalfont\bfseries,
notefont=\mdseries, notebraces={(}{)},
bodyfont=\normalfont,
postheadspace=0.5em,
spaceabove=1pt,
mdframed={
  roundcorner=30pt,
  skipabove=8pt,
  skipbelow=8pt,
  hidealllines=true,
  backgroundcolor={shadecolor},
  innerleftmargin=4pt,
  innerrightmargin=4pt}
]{shaded}

\declaretheorem[style=shaded,within=section]{definition}
\declaretheorem[style=shaded,sibling=definition]{theorem}
\declaretheorem[style=shaded,sibling=definition]{proposition}
\declaretheorem[style=shaded,sibling=definition]{assumption}
\declaretheorem[style=shaded,sibling=definition]{corollary}

\declaretheorem[style=shaded,sibling=definition]{lemma}

\usepackage[utf8]{inputenc} % allow utf-8 input
\usepackage[T1]{fontenc}    % use 8-bit T1 fonts
\usepackage{hyperref}       % hyperlinks
\usepackage{url}            % simple URL typesetting
\usepackage{booktabs}       % professional-quality tables
\usepackage{amsfonts}       % blackboard math symbols
\usepackage{nicefrac}       % compact symbols for 1/2, etc.
\usepackage{microtype}      % microtypography
\usepackage{xcolor}         % colors

\title{\texttt{SP2} : A Second Order Stochastic Polyak Method}

\author{%
   Shuang Li \\
   UCLA \\
   Los Angeles, USA \\
   \texttt{shuangli@math.ucla.edu} \\
   \And
   William J. Swartworth \\
   UCLA \\
   Los Angeles, USA \\
   \texttt{wswartworth@gmail.com} \\
   \And  
   Martin Tak\'a\v{c}\\
    MBZUAI \\ Abu Dhabi,
UAE \\ 
   \texttt{Takac.MT@gmail.com} \\
   \And
   Deanna Needell \\
   UCLA
   \\
   Los Angeles, USA \\
   \texttt{deanna@math.ucla.edu} \\
   \And
   Robert M. Gower \\
   Center for Computational Mathematics\\ Flatiron Institute, Simons Foundation \\ New York, NY
10010 USA \\
   \texttt{rgower@flatironinstitute.org} \\
}

\begin{document}

\setlength{\abovedisplayskip}{3pt}
\setlength{\belowdisplayskip}{3pt}

\maketitle

\begin{abstract}
  Recently the \texttt{SP} (Stochastic Polyak step size) method has emerged as a competitive adaptive method for setting the step sizes of SGD.  \texttt{SP} can be interpreted as a method specialized to interpolated models, since it solves the \emph{interpolation equations}. \texttt{SP} solves these equation by using local linearizations of the model.  We take a step further and develop a method for solving the interpolation equations that uses the local second-order approximation of the model. Our resulting method \texttt{SP2} uses Hessian-vector products to speed-up the convergence of \texttt{SP}. Furthermore, and rather uniquely among second-order methods, the design of SP2 in no way relies on positive definite Hessian matrices or convexity of the objective function. We show \texttt{SP2} is very competitive on matrix completion, non-convex test problems and logistic regression. We also provide a convergence theory on sums-of-quadratics. 
  
% \SL{We only mentioned \texttt{SP2} in the abstract, how about other methods?} 
\end{abstract}

\section{Introduction}
\vspace{-2mm}

Consider the problem
\begin{equation} \label{eq:main}
w^* \in \argmin_{w \in \R^d} \left\{f(w):=\frac{1}{n} \sum_{i=1}^n f_i(w) \right\},
\end{equation}
where $f$ is twice continuously differentiable, and the set of minimizers % $\cW^* \subset \R^d$
is nonempty.
Let the optimal value of~\eqref{eq:main} be $f^* \in \R$, and $w^0$ be a given initial point. Here each $f_i(w)$ is the loss of a model parametrized in $w\in \R^d$ over an $i$-th data point. Our discussion, and forth coming results, also hold for a loss given as an expectation $f(w) = \mathbf{E}_{\xi \sim \cD}\left[ f_{\xi}(w)\right]$, where $\xi \sim \cD$ is the data generating process and $f_{\xi}(w)$ the loss over this sampled data point. But for simplicity we use the $f_i(w)$ notation.

Contrary to classic statistical modeling, there is now a growing trend of using overparametrized models that are able to interpolate the data~\cite{MaBB18}; that is, models for which the loss is minimized over every data point as described in the following assumption.
\begin{assumption}\label{ass:interpolate}
We say that the interpolation condition holds when the loss is nonnegative, $f_i(w) \geq 0$, and 
\begin{equation}\label{eq:interpolatei}
\exists~ w^* \in \R^d \quad \mbox{such that} \quad  f(w^*) =0.
\end{equation}
Consequently, 
%
%\vspace{-0.5cm}
%\begin{equation} \label{eq:interpolate}
 $f_i(w^*) =0$ for $i=1, \ldots, n$.
%\end{equation}
\end{assumption}
Overparameterized deep neural networks are the most notorious example of models that satisfy Assumption~\ref{ass:interpolate}.
Indeed, with sufficiently more parameters than data points, we are able to simultaneously minimize the loss over all data points.

%  The most notorious example of models that satioften holds in many overparameterized deep neural networks where one has enough parameters in the model to fit all the data points. 

% This  result~\eqref{eq:interpolate} follows because $f_i(w^*)$ is non-negative and their average $f(w^*) = \sum_{i=1}^n f_i(w^*) =0$ which implies that each $f_i(w^*)$ must be zero.

If we admit that our model can interpolate the data, then we have that our optimization problem~\eqref{eq:main} is equivalent to solving the system of nonlinear equations
\begin{equation}\label{eq:fizero}
f_i(w) \;= \;0, \quad \mbox{for }i=1,\ldots, n.
\end{equation}
Since we assume $f_i(w) \geq 0$ any solution to the above is a solution to our original problem.

Recently, it was shown in~\cite{TASPS} that the Stochastic Polyak step size (\texttt{SP}) method~\cite{ALI-G,SPS,poliak1987introduction} directly solves the interpolation equations. Indeed, at each iteration \texttt{SP} samples a single $i$-th equation from~\eqref{eq:fizero}, then  projects the current iterate $w^t$ onto the linearization of this constraint, that is \\[-0.3cm]
%, making it an ideal method for minimizing overparametrized models.  
\begin{align}
w^{t+1} =  \argmin_{w\in\R^d}%\tfrac{1}{2}
\norm{w-w^t}^2 \quad
\mbox{ s.t. }f_i(w^t) + \dotprod{\nabla f_i(w^t), w-w^t} =0. \label{eq:SPSproj}
\end{align}
The closed form solution to~\eqref{eq:SPSproj} is given by
\begin{eqnarray}\label{eq:SP}
w^{t+1} =w^t - \tfrac{f_i(w^t)}{\norm{\nabla f_i(w^t)}^2}\nabla f_i(w^t).
\end{eqnarray}
% The resulting update~\eqref{eq:SP} is SGD with an adaptive step size.
Here we take one step further, and instead of projecting onto the linearization of $f_i(w)$ we use the local quadratic expansion. That is,   as a proxy of setting $f_i(w)=0$ we set the quadratic expansion  of $f_i(w)$ around  $w^t$  to zero
\begin{align} \label{eq:quadeqzero}
 f_i(w^t) + \dotprod{\nabla f_i(w^t), w-w^t}  +  \tfrac{1}{2} \dotprod{\nabla^2 f_i(w^t) (w-w^t), w-w^t}=0.
\end{align}
The above quadratic constraint could have infinite solutions, a unique solution or no solution at all\footnote{Or even two solutions in the 1d case.}.
Indeed, for example if $\nabla^2 f_i(w^t)$ is positive definite, there may exist no solution, which occurs when $f_i$ is convex, and is the most studied setting for second order methods. But if the loss is positive $f_i$ and the Hessian has at least one negative eigenvalue, then~\eqref{eq:quadeqzero} always has a solution. 

If~\eqref{eq:quadeqzero} has solutions, then analogously to the \texttt{SP} method, we can choose one using a projection step
% Next we characterize when~\eqref{eq:quadeqzero} has a solution,
% and how to choose one of the solutions using a projection step given by
\footnote{Note that there could be more than one solution to this projection and we can choose one either with least norm or arbitrarily.}
\begin{align}
w^{t+1} \; \in \;  &\argmin_{w\in\R^d}\tfrac{1}{2} \norm{w-w^t}^2 \nonumber \\
&\mbox{ s.t. }f_i(w^t) + \dotprod{\nabla f_i(w^t), w-w^t}  + \tfrac{1}{2} \dotprod{\nabla^2 f_i(w^t) (w-w^t), w-w^t}=0. \label{eq:SP2proj}
\end{align}
%\rob{Point out can be more than one, choose one with least norm or arbitrarily}
% \topic{Argue that solving quadratics makes sense. Modal more accurate, allows for larger step sizes. Also second order method that benefits from non-convexity. Indeed, if non-convex, there is always a solution}
We refer to~\eqref{eq:SP2proj} as the \texttt{SP2} method.
%\SL{But we use \texttt{SP2} to denote the specific method in Lemma~\ref{lma:GLM}, right?}
Using a quadratic expansion has several advantages. First, quadratic expansions are more accurate than linearizations, which will allow us to take larger steps. Furthermore, using the quadratic expansion will lead to convergence rates which are \emph{independent} on how well conditioned the Hessian matrices are, as we show later in Proposition~\ref{prop:convergence}.

Our \texttt{SP2} method occupies a unique position in the literature of stochastic second order method since it is incremental and in no way relies on convexity or positive semi-definite Hessian matrices. Indeed, as we will show in our non-convex experiments in~\ref{sec:nonconvextoy} and matrix completition~\ref{sec:matcom}, the \texttt{SP2} excels at minimizing non-convex problems that satisfy interpolation. In contrast, Newton based methods often converge to stationary points other than the global minima.

%\DN{Would it make more sense to have a small Organization section so the reader can locate this more easily? Especially since they may not see Section 5 coming. (Maybe should also add to abstract?)}
%In Section~\ref{sec:slack} we then
We also
relax the interpolation assumption, and develop analogous quadratic methods for finding $w$ and the smallest possible $s\in \R$ such that 
\begin{align}\label{eq:slack}
    f_i(w) \; \leq \; s, \quad \mbox{for } i=1,\ldots, n.
\end{align}
We refer to this as the \emph{slack interpolation} equations, which were introduced in~\cite{Crammer06} for linear models.
If the interpolation assumption holds then $s =0$ and the above is equivalent to solving~\eqref{eq:fizero}. When interpolation does not hold, then~\eqref{eq:slack} is still a upper approximation of~\eqref{eq:main}, as detailed in~\cite{polyakslack}.

The rest of this paper is organized as follows. We introduce some related work in Section~\ref{sec:related_work}. We present the proposed \texttt{SP2} methods in Section~\ref{sec:sp2} and corresponding convergence analysis in Section~\ref{sec:convergence}. In Section~\ref{sec:slack}, we relax the interpolation condition and develop a variety of quadratic methods to solve the slack version of this problem. We test the proposed methods with a series of experiments in Section~\ref{sec:simulations}. Finally, we conclude our work and discuss future directions in Section~\ref{sec:conclusion}.

\section{Related Work}
\label{sec:related_work}
\vspace{-2mm}
%\subsection{Related Work on Quadratic Projections.}
%\label{sec:RelatedWork}

Since it became clear that Stochastic Gradient Descent (SGD), with appropriate step size tuning, was an efficient method for solving the training problem~\eqref{eq:main}, there has been a search for an efficient second order counter part. The hope being, and our objective here, is to find a second order stochastic method that is \emph{incremental}; that is, it can work with mini-batches,  requires little to \emph{no tuning} since it would 
depend less on how well scaled or conditioned the data is, and finally, would also apply to
\emph{non-convex} problems. 
%\DN{Perhaps we should say something more concrete than `makes sense' here?}
To date there is a vast literature on stochastic second order methods, yet none that achieve all of the above.
% that we grouped based on two categories 1: The incremental or large sampled based method and 2: Convex or non-convex. 

 The subsampled Newton methods such as~\cite{Roosta-Khorasani2016,Bollapragada2018,Newton-MR,Erdogdu2015nips,KohlerL17,jahani2017distributed} 
require large batch sizes in order to guarantee that the subsampled Newton direction is close to the full Newton direction in high probability. As such are not incremental.
Other examples of large sampled based methods include the  Stochastic quasi-Newton methods~\cite{Byrd2011,Mokhtari2014,moritz2016linearly,GowerGold2016,wang2017stochastic,berahas2016multi}, stochastic cubic Newton~\cite{Tripuraneni-stoch-cubic-2018},   
SDNA~\cite{Qu2015}, Newton sketch~\cite{Pilanci2015a} and Lissa~\cite{Lissa}, since these require a large mini-batch or full gradient evaluations.

The only incremental second order methods we aware of are  \emph{IQN} (Incremental Quasi-Newton)~\cite{mokhtari2018iqn}, \emph{SNM} (Stochastic Newton Method)~\cite{SNM,pmlr-v48-rodomanov16} and very recently \emph{SAN} (Stochastic Average Newton)~\cite{SAN}.
IQN and SNM enjoy a fast local convergence, but 
their computational and memory costs per iteration, is of $\cO(d^2)$ making them prohibitive in large dimensions.  

% Non-convexity brings a challenge

Handling non-convexity in second order methods is particularly challenging because most second order methods rely on convexity in their design. For instance, the classic Newton iteration is the minima of the local quadratic approximation if this approximation is convex. If it is not convex, the Newton step can be meaningless, or worse, a step uphill. Quasi-Newton methods maintain positive definite approximation of the Hessian matrix, and thus are also problematic when applied to non-convex problems~\cite{wang2017stochastic}  for which the Hessian is typically indefinite.  Furthermore the incremental Newton methods IQN, SNM and SAN methods 
rely on the convexity of $f_i$ in their design. Indeed,  without convexity, the iterates of IQN, SNM and SAN are not well defined.

In contrast, our approach of finding roots of the local quadratic approximation~\eqref{eq:SP2proj} in no way relies on convexity, and relies solely on the fact that the local quadratic approximation around $w^t$ is good if we are not far from $w^t.$ 
%\DN{Maybe comment here when and how we address this?}
But our approach does introduce a new problem: the need to solve a system of quadratic equations. We propose a series of methods to solve this in Sections~\ref{sec:sp2} and~\ref{sec:slack}. 
% Consequently, designing second-order methods for stochastic non-convex problem requires a departure from the classic Newton and quasi-Newton methods. 

Solving quadratic equations has been heavily studied. 
%\paragraph{Quadratic equality with non-PSD matrix.} 
%
%Even the following special case has dedicated methods, that is
There are even dedicated methods for solving
\begin{align}
    w^\star=\operatorname*{argmin}_{w\in\R^d} \tfrac{1}{2}\|w-\bar{w}\|^2~\mbox{ s.t. }Q(w)=0, \text{ where } Q({w}) = \tfrac{1}{2} {w}^\top \mH {w} + b^\top {w} + c
%    \tag{QP}
    \label{problem:qp}
\end{align}
for a given $\bar{w}$, where
$ \mH $ is a
 nonzero symmetric ({\em not necessarily PSD}) matrix,
and the level set $\{w:Q(w)=0\}$ is nonempty. Note that since $Q(w)$ is a quadratic function, the problem \eqref{problem:qp} is nonconvex. 
Yet despite this non-convexity, so long as there exists a feasible point, the projection~\eqref{problem:qp} can be solved in polynomial time by re-writing the projection as a semi-definite program, or by using the S-procedure, which involves computing the eigenvalue decomposition of $\mH$ and using a line search as proposed in ~\cite{park2017general}, and detailed here in Section~\ref{sec:projquad}. But this approach is too costly when the dimension $d$ is large.

% \rob{***HERE***}

% There also exist specialized iterative methods for solving~\eqref{problem:qp}, see ~\cite{park2017general,sosa2020algorithm}.
%\SL{We may want to run some experiments and compare our proposed iterative method with the one proposed in \cite{sosa2020algorithm}.}
An alternative iterative method is proposed in~\cite{sosa2020algorithm}, but only asymptotic convergence is guaranteed.
In~\cite{dai2006fast}, 
the authors consider a similar problem by projecting a point onto a general ellipsoid, which is again a problem of solving quadratic equations. However, they require the matrix $\mH$ to be a positive definite matrix.

The problem~\eqref{eq:SP2proj} and~\eqref{problem:qp} are  also an instance of a quadratic constrained quadratic program (QCQP).
Although the QCQP in~\eqref{eq:SP2proj} has no closed form solution in general, we show in the next section that there is a closed form solution for Generalized linear models (GLMs), that holds for convex and non-convex GLMs alike. For general non-linear models we propose in Section~\ref{sec:SPS+} an approximate solution to~\eqref{eq:SP2proj} by iteratively linearizing the quadratic constraint and projecting onto the linearization.

\section{The SP2 Method }
\label{sec:sp2}
\vspace{-2mm}

Next we give a closed form solution to~\eqref{eq:SP2proj} for GLMs.
We then provide an approximate solution to~\eqref{eq:SP2proj} for more general models.

\subsection{\texttt{SP2}$^+$ - Generalized Linear Models}
\label{sec:GLMs}
\vspace{-1mm}

% \topic{The quadratic equation can be solved exactly for generalized linear models. Interesting new method.}

Consider when $f_i$ is the loss over a linear model with
%Here we consider loss functions of Generalized Linear Models (GLMs), that is
\begin{equation}\label{eq:GLMs}
    f_i(w) =  \phi_i(x_i^\top w - y_i),
\end{equation} 
where $\phi_i: \R \rightarrow \R$ is a loss function, and $(x_i,y_i)\in \R^{d+1}$ is an input-output pair.
Consequently
\begin{align}
    \nabla f_i(w) = \phi_i'(x_i^\top w - y_i) x_i \eqdef a_i x_i, \qquad
    \nabla^2 f_i(w) = \phi_i''(x_i^\top w - y_i) x_ix_i^\top \eqdef h_i x_i x_i^\top.\label{eq:GLMsgradhess}
\end{align} 
The quadratic constraint problem~\eqref{eq:SP2proj} can be solved exactly for GLMs~\eqref{eq:GLMs} as we show next.

% \begin{minipage}[l]{0.65\textwidth}
\begin{lemma}(SP2)\label{lma:GLM}
Assume 
%\DN{Should introduce what ai and hi are here?}
$f_i(w)$ is the loss of a generalized linear model~\eqref{eq:GLMs} and is non-negative.
Let $f_i=f_i(w^t)$ for short. Let $a_i \eqdef \phi_i'(x_i^\top w - y_i)$ and $h_i\eqdef \phi_i''(x_i^\top w - y_i)$.
If \begin{equation}\label{eq:conditionGLM} 
    a_i^2 - 2h_i f_i \geq 0
\end{equation} then the optimal solution of  
\eqref{eq:SP2proj} 
is as follows
\begin{align}\label{eq:GLMsol}
w^{t+1} &= w^t -  \frac{a_i}{h_i} \left( 1-\frac{\sqrt{a_i^2 - 2 h_i f_i }}{|a_i|}  \right) \frac{x_i}{\|x_i\|^2}. 
% \\
% &=  w^t -  \frac{|a_i|-\sqrt{a_i^2 - 2 h_i f_i }}{h_i|a_i|}   \frac{\nabla f_i(w^t)}{\norm{x_i}^2}.
\end{align}
Alternatively if~\eqref{eq:conditionGLM} does not hold, since $f_i \geq 0 $ we have necessarily that  $h_i >0$, and consequently a Newton step will give the minima of the local quadratic, that is
\begin{align}
\label{eq:newtonStep}
w^{t+1} &= w^t -\frac{a_i}{h_i } \frac{x_i}{\norm{x_i}^2}.
\end{align}
% Finally if ~\eqref{eq:conditionGLM} does not hold and $h_i<0$
% \tau^* = \begin{cases}
%   \frac{-a_i+ \sqrt{a_i^2 - 2 h_i f }  }
%          {h_i},& \mbox{if}\ -a_i< 0\\
%     \frac{-a_i- \sqrt{a_i^2 - 2 h_i f }  }
%          {h_i},& \mbox{otherwise}.     
%   \end{cases}
% $$
\end{lemma}
% \end{minipage}
%% Rob: Alternative placement for figure
% \begin{minipage}[r]{0.35\textwidth}
% \centering
% \includegraphics[width=0.9\textwidth]{concave_funcs}\\
% {\footnotesize{{\bf Figure (a)} Non-convex loss func. for which condition~\eqref{eq:conditionGLM} holds.} }
% \end{minipage}
The proof to the above lemma, and all subsequent missing proofs can be found in the appendix.
% The above Lemma is proved in Appendix~\ref{proof_lma:GLM}.
Lemma~\eqref{lma:GLM} establishes a condition~\eqref{eq:conditionGLM} under which we should not take a full Newton step. Interestingly, this condition~\eqref{eq:conditionGLM}  holds when the square root of the loss function has negative curvature, as we show in the next lemma.

% which is proved in Appendix~\ref{proof_lem:phi2}.

\begin{minipage}[l]{0.65\textwidth}
\begin{lemma}\label{lem:phi2}
Let $\phi$ be a non-negative function which is twice differentiable at all $t$ with $\phi(t)\neq 0$.  The condition~\eqref{eq:conditionGLM}, in other words \[\phi'(t)^2 \geq 2 \phi(t) \phi''(t)\] holds  when $\sqrt{\phi(t)}$ is concave away from its roots, i.e. when $\frac{d^2}{dt^2} \sqrt{\phi(t)} \leq 0$ for all $t$ with $\phi(t) \neq 0.$
\end{lemma}
\end{minipage}
\begin{minipage}[r]{0.35\textwidth}
\vspace{-0.5cm}
\centering
\includegraphics[width=0.8\textwidth]{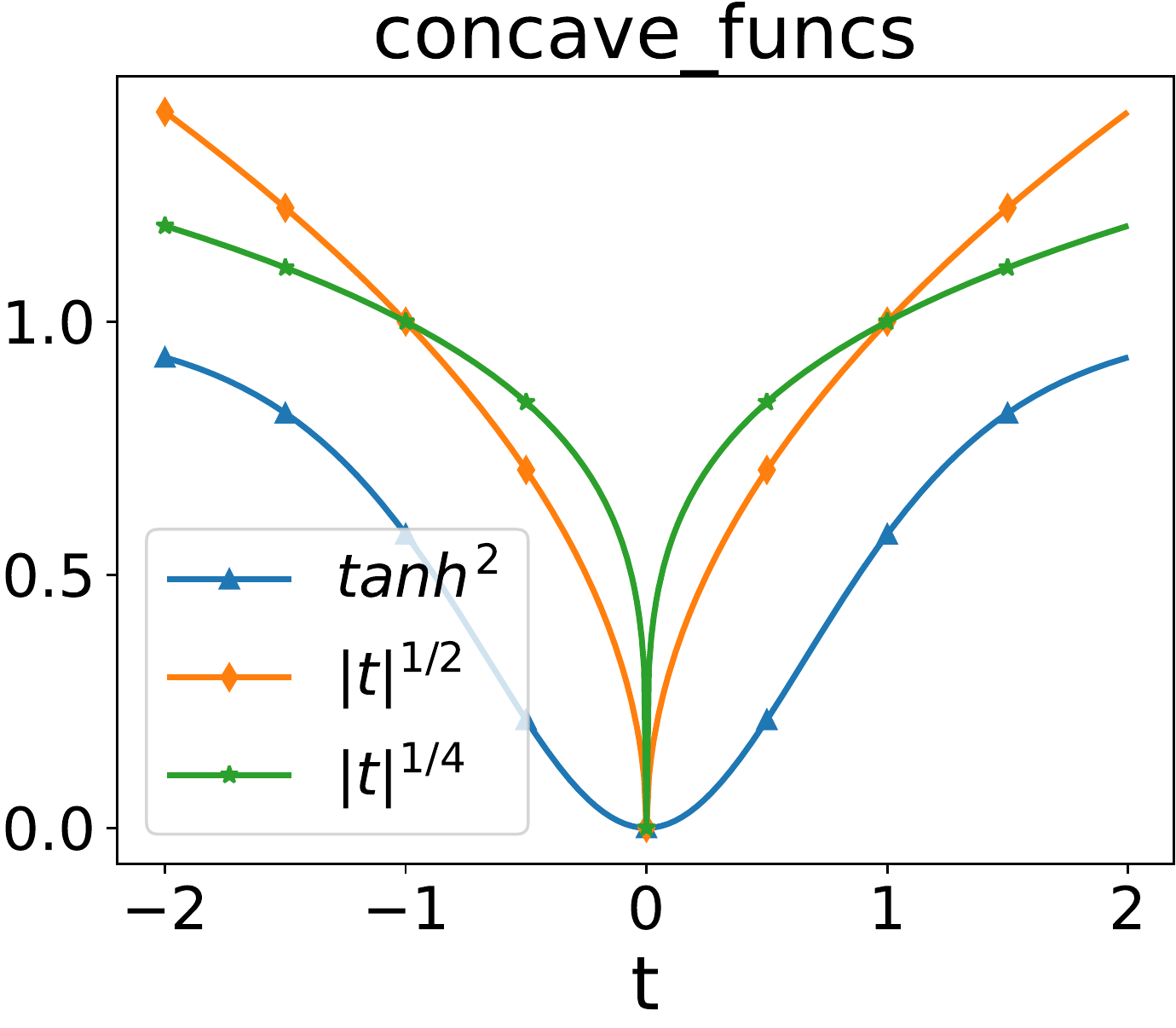}\\
% \vspace{-0.25cm}
{\footnotesize{{\bf Figure (a)} Non-convex loss func. for which condition~\eqref{eq:conditionGLM} holds.} }
% \label{fig:nonconvloss}
% \vspace{-0.5cm}
\end{minipage}
% \begin{proof}
% If $\phi(t)=0$ then the condition holds trivially. For $t$ such that $\phi(t) \neq 0$, $\sqrt{\phi(t)}$ is differentiable, and we have
% \begin{align*}
% \frac{d^2}{dt^2} \sqrt{\phi(t)} &= -\tfrac14\phi(t)^{-3/2}\phi'(t)^2 + \tfrac12 \phi(t)^{-1/2}\phi''(t)= \tfrac14 \phi(t)^{-3/2}(-\phi'(t)^2 + 2\phi(t)\phi''(t)),
% \end{align*}
% which is negative precisely when $\phi'(t)^2 \geq 2\phi(t)\phi''(t).$
% \end{proof}
Examples of loss functions include $\phi(t)=\tanh^2{(t)}$ 
and $\phi(t)= t^{p}$ with $0 \leq p \leq 2$ (see Figure~{\bf (a)}).

% \begin{figure}
%     \centering
%     \includegraphics[width =4cm]{concave_funcs}
%     \caption{Non-convex loss func. for which condition~\eqref{eq:conditionGLM} holds. }
%     \label{fig:nonconvloss}
%     \todo[inline]{the figure has a bit weird title, shall we remove it OR crop the image? I would also put $\tanh^2(t)$, hence $t$ is missing}
% \end{figure}
In conclusion to this section, \texttt{SP2} has a closed form solution for GLMs, and this closed form solution includes many non-convex loss functions.

\subsection{\texttt{SP2}$^+$ - Linearizing and Projecting}
\label{sec:SPS+}
\vspace{-1mm}

In general, there is no closed form solution to~\eqref{eq:SP2proj}. Indeed, there may not even exist a solution.
% models, solving problems with quadratic constraints is expensive and sometimes impossible.
Inspired by the fact that computing a Hessian-vector product can be done with a single backpropagation at the same cost as computing a gradient~\cite{Christianson:1992}, we will make use of the cheap Hessian-vector product to derive an approximate solution to~\eqref{eq:SP2proj}.

Instead of solving~\eqref{eq:SP2proj} exactly, here we propose to 
take two steps towards solving~\eqref{eq:SP2proj} by projecting onto the linearized constraints.
To describe this method let
\begin{align} \label{eq:quadeqzero2}
 q(w) \eqdef f_i(w^t) + \dotprod{\nabla f_i(w^t), w-w^t}  + \tfrac{1}{2} \dotprod{\nabla^2 f_i(w^t) (w-w^t), w-w^t}.
\end{align}
In the first step we linearize the quadratic  constraint~\eqref{eq:quadeqzero2} around $w^t$ and project onto this linearization:  
\begin{align}
w^{t+1/2}  =  &\argmin_{w\in\R^d}\tfrac{1}{2} \norm{w-w^t}^2  \quad
\mbox{ s.t. }f_i(w^t) + \dotprod{\nabla f_i(w^t), w-w^t}  =0. \label{eq:polyakproj}
\end{align}
The closed form update of this first step is given by
\begin{eqnarray}
w^{t+1/2} 
& = & w^t - \frac{f_i(w^t)}{\norm{\nabla f_i(w^t)}^2}\nabla f_i(w^t),\label{eq:polyakup}
\end{eqnarray}
which is a Stochastic Polyak step~\eqref{eq:SP}. For the second step, we once again linearize the quadratic  constraint~\eqref{eq:quadeqzero2}, but this time around the point $w^{t+1/2}$ and set this linearization to zero, that is
%only update should be numbered
\begin{align}
w^{t+1} =  &\argmin_{w\in\R^d}\tfrac{1}{2}\! \norm{w-w^{t+1/2}}^2  ~~
\mbox{s.t. }q(w^{t+1/2}) + \dotprod{\nabla q(w^{t+1/2}), w-w^{t+1/2}}  =  0.
\label{eq:polyakproj2}
\end{align}
The closed form update of this second step is given by
\begin{eqnarray}
w^{t+1} =  w^{t+1/2} - \frac{q(w^{t+1/2})}{\norm{\nabla q(w^{t+1/2}))}^2}\nabla q(w^{t+1/2}).\label{eq:polyakup2}
\end{eqnarray}
We refer to the resulting proposed method as the SP2$^+$ method, summarized in the following.
%We summarize the proposed SP2$^+$ method in the following lemma.

 \begin{lemma}{(SP2$^+$)}\label{lem:SPS2plus} %For the sake of brevity, 
Let $g_t \equiv \nabla f_i(w^t)$ and $\mH_t \equiv \nabla^2 f_i(w^t)$.  Update~\eqref{eq:polyakup} and~\eqref{eq:polyakup2} is given by
\begin{align}
% v^{t+1} &= \left(\mI - \mH_t\frac{f_i(w^t)}{\norm{g_t}^2}\right)g_t, \nonumber \\
w^{t+1}  \!=  w^t \!- \frac{f_i(w^t)}{\norm{g_t}^2}g_t  \!- \frac{1}{2}\frac{f_i(w^t)^2}{\norm{g_t}^4}\frac{ \dotprod{\mH_t g_t, g_t} }{\norm{v^{t+1}}^2}
v^{t+1}, \; \mbox{where} \;v^{t+1} 
\;= \left(\mI - \mH_t\frac{f_i(w^t)}{\norm{g_t}^2}\right)g_t,
\label{eq:SP2plus}
\end{align}

\end{lemma}
In~\eqref{eq:SP2plus} we  can see that \texttt{SP2}$^+$ applies a second order correction to the \texttt{SP} step. 

\texttt{SP2}$^+$  is equivalent to two steps of a Newton Raphson method applied to finding a root of $q(w)$. If we apply multiple steps of the Newton Raphson method, as opposed to two, the resulting method converges to the root of $q$, see Theorem~\ref{theo:qconvNR} in the appendix.
% if $q$ is star-convex, see Theorem~\ref{theo:qconvNR}. Star-convexity is a generalization of convexity that includes several non-convex functions~\cite{hinder2019near}.
 Theorem~\ref{theo:qconvNR} shows that this multi-step version of \texttt{SP2}$^+$ converges when $q$ belongs to a large class of non-convex functions known as the star-convex functions. Star-convexity, which is a generalization of convexity,  includes several non-convex loss functions~\cite{hinder2019near}.

\section{Convergence Theory}
\label{sec:convergence}
\vspace{-2mm}

Here we provide a convergence theory for SP2 and SP2$^{+}$ for when $f(w)$ is an average of  quadratic functions. Let $w^*\in \argmin_{w\in\R^d} f(w)$ and let the loss over the $i$-th data point be given by
\begin{eqnarray}\label{eq:quadexe}
f_i(w) =  %\dotprod{\nabla f_i, w-w^*} + 
\dotprod{\mH_i (w-w^*), w-w^*}, 
\end{eqnarray}
where  $\mH_i\in \R^{d\times d}$ is a symmetric positive semi-definite matrix for $i=1,\ldots, n.$ Consequently $ f_i(w^*) = 0 =f(w^*) = \min_{w\in \R^d} f(w), $
thus the interpolation condition holds. 

\begin{proposition}\label{prop:convergence}
Consider the loss functions given in~\eqref{eq:quadexe}. The \texttt{SP2} method~\eqref{eq:SP2proj}  converges  linearly\begin{equation}\label{eq:rateSP2}
\E{\norm{w^{t+1}-w^*}^2} \leq \rho\; \E{\norm{w^{t}-w^*}^2},
\quad
\text{where}
\;
%\end{eqnarray}
%where 
%\vspace{-1.0cm}
%\begin{equation}\label{eq:rateSP2} %\E{\mH_i\mH_i^+} =
   \rho = \lambda_{\max}\left( \mI -\frac{1}{n}\sum_{i=1}^n \mH_i \mH_i^+\right) <1.
\end{equation}   
\end{proposition}

The rate of convergence of \texttt{SP2} in~\eqref{eq:rateSP2} can be orders of magnitude better  than SGD. Indeed, since~\eqref{eq:quadexe} is convex, smooth and interpolation holds, we have from~\cite{SGDstruct} that SGD converges at a rate of 
\begin{equation}\label{eq:SGDrate}
    \rho_{SGD} = 1-\frac{1}{2n}\frac{ \lambda_{\min}(\sum_{i=1}^n\mH_i) }{\max_{i=1,\ldots, n} \lambda_{\max}(\mH_i)}.
\end{equation}
To compare~\eqref{eq:SGDrate} to $\rho$ rate in Proposition~\ref{prop:convergence}, consider the case where all $\mH_i$ are invertible. In this case $\mH_i \mH_i^{+} = \mI$ and thus $\rho = 0$ and \texttt{SP2} converges in one step. Indeed, even if a single $\mH_i$ is invertible, after sampling $i$ the \texttt{SP2} will converge. In contrast, the SGD method is still at the mercy of the spectra of the $\mH_i$ matrices and depend on how well conditioned these matrices are. Even in the extreme case where all $\mH_i$ are well conditioned, for example $\mH_i = i \times \mI$, the rate of convergence of SGD can be very slow, for instance in this case we have  $\rho_{SGD}~=~1~-~\frac{1}{2n^2}.$

\begin{proposition}\label{prop:convergenceSPplus}
Consider the loss functions in~\eqref{eq:quadexe}.  The \texttt{SP2}$^+$ method~\eqref{eq:SP2plus}  converges linearly\begin{equation}\label{eq:rhoSP2+}
\E{\norm{w^{t+1}-w^*}^2} \leq \rho_{SP2^+}^2\; \E{\norm{w^{t}-w^*}^2},\;
%\end{eqnarray}
\text{where}\;
%\vspace{-1.2cm}
%\begin{eqnarray}\label{eq:rhoSP2+}
\rho_{SP2+} = 1 - \frac{1}{2n} \sum_{i=1}^n\frac{\lambda_{\min}(\mH_i)}{\lambda_{\max}(\mH_i)}.  
\end{equation}  
\end{proposition}
The rate of convergence of \texttt{SP2}$^+$ now depends on the average condition number of the $\mH_i$ matrix. Yet still, the rate of convergence  in~\eqref{eq:rhoSP2+} is always better than that of SGD. Indeed, this follows because of the maximum over the index $i$ in~\eqref{eq:SGDrate} and since 
\[ \frac{1}{2n} \sum_{i=1}^n\frac{\lambda_{\min}(\mH_i)}{\lambda_{\max}(\mH_i)}  \geq  \frac{1}{2n} \sum_{i=1}^n\frac{\lambda_{\min}(\mH_i)}{\max_{i=1,\ldots, n}\lambda_{\max}(\mH_i)} . \]

Note also that the rate of \texttt{SP2}$^+$ appears squared in~\eqref{eq:rhoSP2+} and the rate $\rho_{SGD}$ of SGD is not squared. But this difference accounts for the fact that each step of \texttt{SP2}$^+$ is at least twice the cost of SGD, since each step of \texttt{SP2}$^+$  is comprised of two gradient steps, see~\eqref{eq:polyakup} and~\eqref{eq:polyakup2}. Thus we can neglect the apparent advantage of the rate $\rho_{SP2+}$ being squared.

\section{Quadratic with Slack}
\label{sec:slack}
\vspace{-2mm}
% \rob{This form of introducing slack first appeared in~\cite{Crammer06}. Need to decide if we keep this or not.}

Here we depart from the interpolation Assumption~\ref{ass:interpolate} and design a variant of \texttt{SP2}$^+$ that can be applied to models that are \emph{close} to interpolation.
Instead of trying to set all the losses to zero, we now will try to find the smallest \emph{slack variable} $s>0$ for which
\[ f_i(w)  \leq s, \quad \mbox{ for } i=1,\ldots, n.\]
If interpolation holds, then $s =0$ is a solution. Outside of interpolation, $s$ may be non-zero. 

\subsection{L2 slack formulation}
\vspace{-1mm}

To make  $s$ as small as possible,  we can solve the following problem
\begin{align}\label{eq:slackL2}
    \min_{s \in \R, w \in \R^s} \frac{1}{2}s^2  \mbox{ subject to } f_i(w)  \leq s, \mbox{ for } i=1,\ldots, n, 
\end{align}
which is called the {\em L2 slack formulation}.
This type of slack problem was introduced in~\cite{Crammer06} to derive  variants of the passive-aggressive method that could be applied to linear models on non-separable data, in other words, that do not interpolate the data.

To solve~\eqref{eq:slackL2} we will again project onto a local quadratic approximation of the constraint. 
Let 
\begin{align}
q_{i,t}(w) \eqdef f_i(w^t) + \dotprod{\nabla f_i(w^t), w-w^t} 
+ \tfrac{1}{2} \dotprod{\nabla^2 f_i(w^t) (w-w^t), w-w^t}.
\end{align}
and let $\Delta_{t} = \norm{w-w^t}^2 + (s-s^t)^2.$ 
Consider the iterative  method given by
\begin{align}
w^{t+1} , s^{t+1}  =  &\underset{s \geq 0, \; w\in\R^d}{\argmin}\frac{1-\lambda}{2}\Delta_{t}
 + \frac{\lambda}{2} s^2 \quad
\mbox{ s.t. }q_{i,t}(w)
\leq s, \label{eq:newtraph2ndslack}
\end{align}
% \rob{Also, does this have a closed form solution in the case of Generalized Linear Models?}
where $\lambda \in [0, \;1]$ is a regularization parameter that trades off between having a small $s$, and using the previous iterates as a regularizor.  The resulting projection problem in~\eqref{eq:newtraph2ndslack} has a quadratic inequality, and thus in most cases has no closed form solution, despite always being feasible\footnote{For instance $w =w^t$ and $s=f_i(w^t)$ is feasible}.

So instead of solving~\eqref{eq:newtraph2ndslack} exactly, we propose an approximate solution by iteratively linearizing and projecting onto the constraints. Our approximate solution has
two steps, the first step being
\begin{align}
w^{t+1/2},s^{t+1/2}  =  &\underset{s \geq 0, \; w\in\R^d}{\argmin}\frac{1-\lambda}{2}\Delta_{t} + \frac{\lambda}{2} s^2 \quad
\mbox{ s.t. }q_{i,t}(w^t) +\dotprod{\nabla q_{i,t}(w^t), w-w^t} 
\leq s. \label{eq:newtraph2ndslackhalf}
\end{align}

The second step is given by projecting $w^{t+1/2}$ onto the linearization around $w^{t+1/2}$
as follows
\begin{align}
w^{t+1}, s^{t+1} \!= \!\! &\underset{s \geq 0, \; w\in\R^d}{\argmin}\!\!\frac{1\!-\!\lambda}{2}\Delta_{t+\frac{1}{2}} \!\!+\! \frac{\lambda}{2} s^2 ~~\mbox{s.t. }q_{i,t}(w^{t+1/2})\!+\! \dotprod{\!\nabla q_{i,t}(w^{t+1/2}), w\!-\!w^{t+1/2}\!} 
\!\leq \!s. \label{eq:newtraph2ndslackhalfhalf}
\end{align}
%The closed form solution to our two step method is given in the following lemma, which is proven in Appendix~\ref{proof_lma:quadslackproj}. We refer to this method as \texttt{SP2L2}$^+$.

The closed form solution to our two step method is given in Lemma~\ref{lem:quadslackproj} of Appendix~\ref{proof_lma:quadslackproj}. We refer to this method as \texttt{SP2L2}$^+$.

\subsection{L1 slack formulation}
\vspace{-1mm}

To make  $s$ as small as possible,  we can also solve the following {\em L1 slack formulation}
\begin{align}
    \min_{s \geq 0, w \in \R^d} s \quad  \mbox{ s.t. } f_i(w)  \leq s, \mbox{ for } i=1,\ldots, n. \nonumber
\end{align}
Similarly, we can again project onto a local quadratic approximation of the constraint and consider the iterative method given by
\begin{align}
w^{t+1} , s^{t+1}  =  &\underset{s \geq 0, \; w\in\R^d}{\argmin}\frac{1-\lambda}{2}\Delta_{t}
 + \frac{\lambda}{2} s \quad
\mbox{ s.t. }q_{i,t}(w)
\leq s, \label{eq:newtraph2ndslack_l1}
\end{align}
where $\lambda \in [0, \;1]$ is a regularization parameter that trades off between having a small $s$, and using the previous iterates as a regularizor.

% We also summarize the closed form solution to~\eqref{eq:newtraph2ndslack_l1} for the GLMs in the following lemma, which is proved in Appendix~\ref{proof_lma:GLM_L1}.
% \begin{lemma}\label{lma:GLM_L1}
% Assume 
% $f_i(w)$ is the loss of a generalized linear model~\eqref{eq:GLMs} and is non-negative.
% Let $f_i=f_i(w^t)$ for short.
% Then the optimal solution of  
% \eqref{eq:newtraph2ndslack_l1} 
% is as follows
% \begin{align*}
%     w^{t+1} &= xxxxx,\\
%     s^{t+1}& = xxxxx.
% \end{align*}

% \end{lemma} 
% \SL{Could not find a closed-form solution from the KKT conditions.}

%\subsubsection{SP2L1$^+$: An approximate two step method}
%\rob{Probably move this to the appendix, since it is an almost repeat of the previous section.}
To approximately solve~\eqref{eq:newtraph2ndslack_l1}, we again propose an approximate two step method similar to~\eqref{eq:newtraph2ndslackhalf} and~\eqref{eq:newtraph2ndslackhalfhalf}. 
% \begin{align}
% w^{t+1/2},s^{t+1/2} & =  \underset{s \geq 0, \; w\in\R^d}{\argmin}\frac{1-\lambda}{2}\Delta_{t} + \frac{\lambda}{2} s\label{eq:newtraph2ndslackhalf_l1} \\
% &\mbox{ s.t. }q_{i,t}(w^t) +\dotprod{\nabla q_{i,t}(w^t), w-w^t} 
% \leq s. \nonumber
% \end{align}
% \begin{align}
% w^{t+1}, s^{t+1} =  &\underset{s \geq 0, \; w\in\R^d}{\argmin}\frac{1-\lambda}{2}\Delta_{t+\frac{1}{2}} + \frac{\lambda}{2} s\label{eq:newtraph2ndslackhalfhalf_l1} \\
% &\quad\mbox{s.t. }q_{i,t}(w^{t+1/2})+ \dotprod{\nabla q_{i,t}(w^{t+1/2}), w-w^{t+1/2}} 
% \leq s. \nonumber
% \end{align}
The closed form solution to the  two step method is given in Lemma~\ref{lem:quadslackproj_l1} of Appendix~\ref{proof_lma:quadslackproj_l1}. We refer to this method as \texttt{SP2L1}$^+$.
%\begin{lemma}(\texttt{SP2L1}$^+$) \label{lem:quadslackproj_l1}
%The $w^{t+1}$ and $s^{t+1}$ update
%of~\eqref{eq:newtraph2ndslackhalf_l1}--\eqref{eq:newtraph2ndslackhalfhalf_l1}  
%is given by
%\begin{align*}
%    w^{t+1} & =  w^t - (\Gamma_4+\Gamma_6) \nabla f_i(w^t) + \Gamma_6 \Gamma_4\nabla^2 f_i(w^t) \nabla f_i(w^t),   \\
%    s^{t+1} 
 %   &= \left(\!\! \left( \!\! s^t - \frac{\lambda}{2(1-\lambda)}\! +  \Gamma_3   \!   \right)_+ \!\!- \frac{\lambda}{2(1-\lambda)} +  \Gamma_5   \!   \right)_+ \!,
%\end{align*}
%where 

%\vspace{-1.2cm}
%\begin{align*}
%  \Gamma_3 &=   \frac{\left( f_i(w^t)- \left(s^t - \frac{\lambda}{2(1-\lambda)} \right)  \right)_+ }{1+\| \nabla f_i(w^t) \|^2  },\quad
%  \Gamma_4 = \min\left\{  \Gamma_3, \frac{f_i(w^t)}{\| \nabla f_i(w^t) \|^2  }     \right\},\\
 % \Gamma_5 &= \frac{\left( \Lambda_1- \left(s^t - \frac{\lambda}{2(1-\lambda)} \right)  \right)_+ }{1+\| \nabla f_i(w^t) - \Gamma_4\nabla^2 f_i(w^t) \nabla f_i(w^t) \|^2  },\\
%  \Gamma_6 &= \min\left\{  \Gamma_5, \frac{\Lambda_1}{\| \nabla f_i(w^t) - \Gamma_4\nabla^2 f_i(w^t) \nabla f_i(w^t) \|^2  }\right\},\\
 % \Lambda_1 &=  f_i(w^t )  -  \Gamma_4  \norm{\nabla f_i(w^t )}^2     +  \frac 1 2 \Gamma_4^2   \left \langle  \nabla^2 f_i(w^t ) \nabla  f_i(w^t ) ,  \nabla  f_i(w^t )    \right \rangle .
%\end{align*}
%\end{lemma}

\subsection{Dropping the Slack Regularization}
\vspace{-1mm}

Note that the objective function in~\eqref{eq:newtraph2ndslack_l1} contains a regularization term $(s-s^t)^2$, which forces $s$ to be close to $s^t$. If we allow $s$ to be far from $s^t$, we can instead solve the following unregularized problem
\begin{align}
w^{t+1} , s^{t+1}  =  &\underset{s \geq 0, \; w\in\R^d}{\argmin}\frac{1-\lambda}{2}\norm{w-w^t}^2
 + \frac{\lambda}{2} s \quad
\mbox{ s.t. }q_{i,t}(w)
\leq s, \label{eq:newtraph2ndslack_l1_unreg}
\end{align}
where $\lambda \in [0, \;1]$ is again a regularization parameter that trades off between having a small $s$, and using the previous iterates as a regularizor.
We call the resulting method in~\eqref{eq:newtraph2ndslack_l1_unreg} the \texttt{SP2max} method since it is a second order variant of the \texttt{SPmax} method~\cite{SPS,polyakslack}.
The advantage of \texttt{SP2max} is that it has a closed form solution for GLMs~\eqref{eq:GLMs} as shown in the following lemma
%We summarize the SP2max method in the following lemma, 
which is proved in Appendix~\ref{proof_lma:GLM_L1_unreg}.

% In GLMs, the unregularized problem~\eqref{eq:newtraph2ndslack_l1_unreg} can be solved exactly and the closed form solution is given in the following lemma, which is proved in Appendix~\ref{proof_lma:GLM_L1_unreg}.
% We refer this method as \texttt{SP2max}.
\begin{lemma}(\texttt{SP2max})\label{lma:GLM_L1_unreg}
Consider the GLM model given in~\eqref{eq:GLMs} and~\eqref{eq:GLMsgradhess}.
If the loss $f_i=f_i(w^t)$ is non-negative, then 
the iterates of~\eqref{eq:newtraph2ndslack_l1_unreg} have a closed form solution given by
% \begin{align*}
%     w^{t+1} &=  w^t -\frac{\lambda a_i }{2(1-\lambda)+ \lambda h_i   }  \frac{x_i}{ \|x_i\|^2},\\
%     s^{t+1}& = \max\left\{ f_i  +\frac{\lambda a_i^2 \|x_i\|^2}{2}\frac{ \lambda  h_i \|x_i\|^2-4(1-\lambda)}{ (\lambda h_i \|x_i\|^2 +2(1-\lambda))^2 },0\right\}. 
% \end{align*}
%The proof is in Appendix~\ref{proof_lma:GLM_L1_unreg}.
\begin{align*}
    w^{t+1} = w^t + c^\star x_i , \quad \quad
    s^{t+1} = \max\left\{ \widetilde{s},~0 \right\},
\end{align*}
and where $\widetilde{s} = f_i  -\frac{\widetilde{\lambda} a_i^2 \ell}{1+ \widetilde{\lambda} h_i \ell  }  +  \frac{ h_i\widetilde{\lambda}^2 a_i^2 \ell^2}{2(1+ \widetilde{\lambda} h_i \ell )^2 }$, $\ell = \|x_i\|^2$, $\widetilde{\lambda} = \frac{\lambda}{2(1-\lambda)}$, and
\begin{align*}
    c^\star = \begin{cases}
    0, \quad &\text{if } f_i = 0\\
    -\frac{\widetilde{\lambda} a_i }{1+ \widetilde{\lambda} h_i \ell  }, &\text{if } f_i>0 \text{ and }  \widetilde{s}\geq 0,\\
\frac{-a_i + \sqrt{a_i^2  - 2h_i  f_i}  }{h_i\ell}, & \text{otherwise.}
    \end{cases}
\end{align*}
\end{lemma}

%\SL{Note that this is true only when $h_i \geq 0$.}

%\subsubsection{SP2max$^+$: An approximate two step method}
%\label{sec:sp2maxplus}
% \rob{Again this is repetitive for the main paper, and can probably be moved into the appendix.}
To approximately solve~\eqref{eq:newtraph2ndslack_l1_unreg} in general, we again propose an approximate two step method.
%similar to~\eqref{eq:newtraph2ndslackhalf_l1} and~\eqref{eq:newtraph2ndslackhalfhalf_l1}: 
% \begin{align}
% w^{t+1/2},s^{t+1/2} & =  \underset{s \geq 0, \; w\in\R^d}{\argmin}\frac{1-\lambda}{2}\norm{w-w^t}^2+ \frac{\lambda}{2} s\label{eq:newtraph2ndslackhalf_l1_unreg} \\
% &\mbox{ s.t. }q_{i,t}(w^t) +\dotprod{\nabla q_{i,t}(w^t), w-w^t} 
% \leq s. \nonumber
% \end{align}
% \begin{align}
% &w^{t+1}, s^{t+1} =  \underset{s \geq 0, \; w\in\R^d}{\argmin}\frac{1-\lambda}{2}\norm{w-w^{t+1/2}}^2+ \frac{\lambda}{2} s\label{eq:newtraph2ndslackhalfhalf_l1_unreg} \\
% &\quad\mbox{s.t. }q_{i,t}(w^{t+1/2})+ \dotprod{\nabla q_{i,t}(w^{t+1/2}), w-w^{t+1/2}} 
% \leq s. \nonumber
% \end{align}
The closed form solution to the  two step method is given in Lemma~\ref{lem:quadslackproj_l1_unreg} of Appendix~\ref{proof_lma:quadslackproj_l1_unreg}. We refer to this method as \texttt{SP2max}$^+$.
%\begin{lemma}(\texttt{SP2max}$^+$) \label{lem:quadslackproj_l1_unreg}
%The $w^{t+1}$ and $s^{t+1}$ update 
%of~\eqref{eq:newtraph2ndslackhalf_l1_unreg}--\eqref{eq:newtraph2ndslackhalfhalf_l1_unreg}  
%is given by
%\begin{align*}
%    w^{t+1} &= w^t - \left(\Gamma_1 + \Gamma_3\right) \nabla f_i(w^t)+ \Gamma_3 \Gamma_1\nabla^2 f_i(w^t) \nabla f_i(w^t),\\
%    s^{t+1} &= \max\left\{\Gamma_2  - \frac{\lambda}{2(1-\lambda)} \norm{\nabla f_i(w^t) - \Gamma_1\nabla^2 f_i(w^t) \nabla f_i(w^t)}^2,0  \right\},
%\end{align*}
%where 

%\vspace{-1.2cm}
%\begin{align*}
%\Gamma_1 &= \min \left\{ \frac{f_i(w^t)}{\norm{\nabla f_i(w^t)}^2}, \frac{\lambda}{2(1-\lambda)}  \right\},\\
%\Gamma_2 & = f_i(w^t) -\Gamma_1 \norm{\nabla f_i(w^t)}^2  
%+ \tfrac{1}{2} \Gamma_1^2 \dotprod{\nabla^2 f_i(w^t)  \nabla f_i(w^t),  \nabla f_i(w^t)},\\
%    \Gamma_3 &= \min \left\{ \frac{\Gamma_2}{\norm{\nabla f_i(w^t) - \Gamma_1\nabla^2 f_i(w^t) \nabla f_i(w^t)}^2}, \frac{\lambda}{2(1-\lambda)}  \right\}.
%\end{align*}

%\end{lemma}

\section{Experiments }
\label{sec:simulations}
\vspace{-2mm}

\subsection{Non-convex problems}
\label{sec:nonconvextoy}
\vspace{-1mm}
To emphasize how our new \texttt{SP2} methods can handle non-convexity, we have tested \texttt{SP2}~\eqref{eq:SP2proj}, \texttt{SP2$^+$}~\eqref{eq:SP2plus} on the non-convex problems
PermD$\beta^+$, Rastrigin and Levy N. 13, Rosenbrock~\cite{testfuncs}\footnote{
We used the Python Package \texttt{pybenchfunction} available on github \url{Python_Benchmark_Test_Optimization_Function_Single_Objective}. We 
also note that the PermD$\beta^+$ implemented in this package is a modified version of the 
 $\texttt{PermD}\beta$ function, as we detail in 
Section~\ref{sec:PermD}.}, see Figures~\ref{fig:perm},~\ref{fig:rastrigin} in the main text, and Figures~\ref{fig:levy} and~\ref{fig:rosenbrock} in the appendix. The two experiments with the function Levy N. 13 and Rosenbrock
are detailed in the appendix in Section~\ref{secapp:nonconvextoy}.

All of these functions are sums-of-terms of the format~\eqref{eq:main} and satisfy the interpolation Assumption~\ref{ass:interpolate}. 
 To compute the \texttt{SP2} update we used ten steps of Newton's Raphson method as 
 detailed in Section~\ref{sec:SP2multistep}. We consistently find across these  non-convex problems that \texttt{SP2} and \texttt{SP2}$^+$ are very competitive, with \texttt{SP2} converging in under $10$ epochs. Here we can clearly see that \texttt{SP2} converges to a high precision solution (like most second order methods), and different than other second order methods is not attracted to local maxima or saddle points.  In contrast,  \texttt{Newtons} method converges to a local maxima on all problems excluding the Rosenbrock function in Figure~\ref{fig:rosenbrock} in the appendix. For instance on the right of Figure~\ref{fig:rastrigin} we can see the red dot of \texttt{Newton} stuck on a local maxima. The iterates of \texttt{Newton} do not appear in the middle of Figure~\ref{fig:rastrigin} since they are outside of plotted region.
%  which was a function designed to show the benefits of Newton's method over gradient descent. 

\definecolor{citrine}{rgb}{0.88, 0.66, 0.37}
% \definecolor{citrine}{rgb}{0.89, 0.82, 0.04}
% \definecolor{darkgreen}{rgb}{0.0, 0.2, 0.13}
\definecolor{darkgreen}{rgb}{0.31, 0.47, 0.26}

\begin{figure}
\centering
\includegraphics[width = 0.3\textwidth]{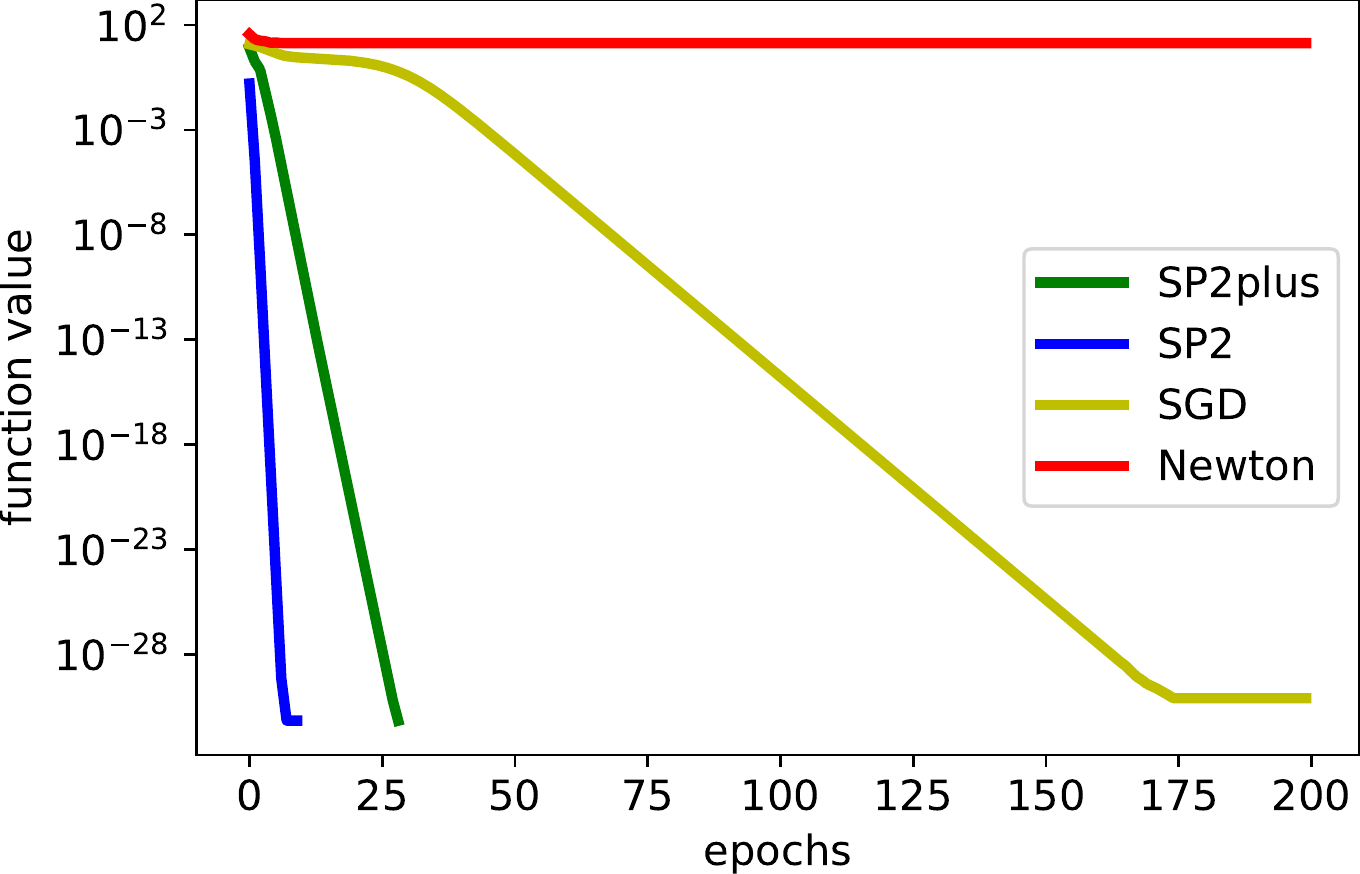}
\includegraphics[width = 0.3\textwidth]{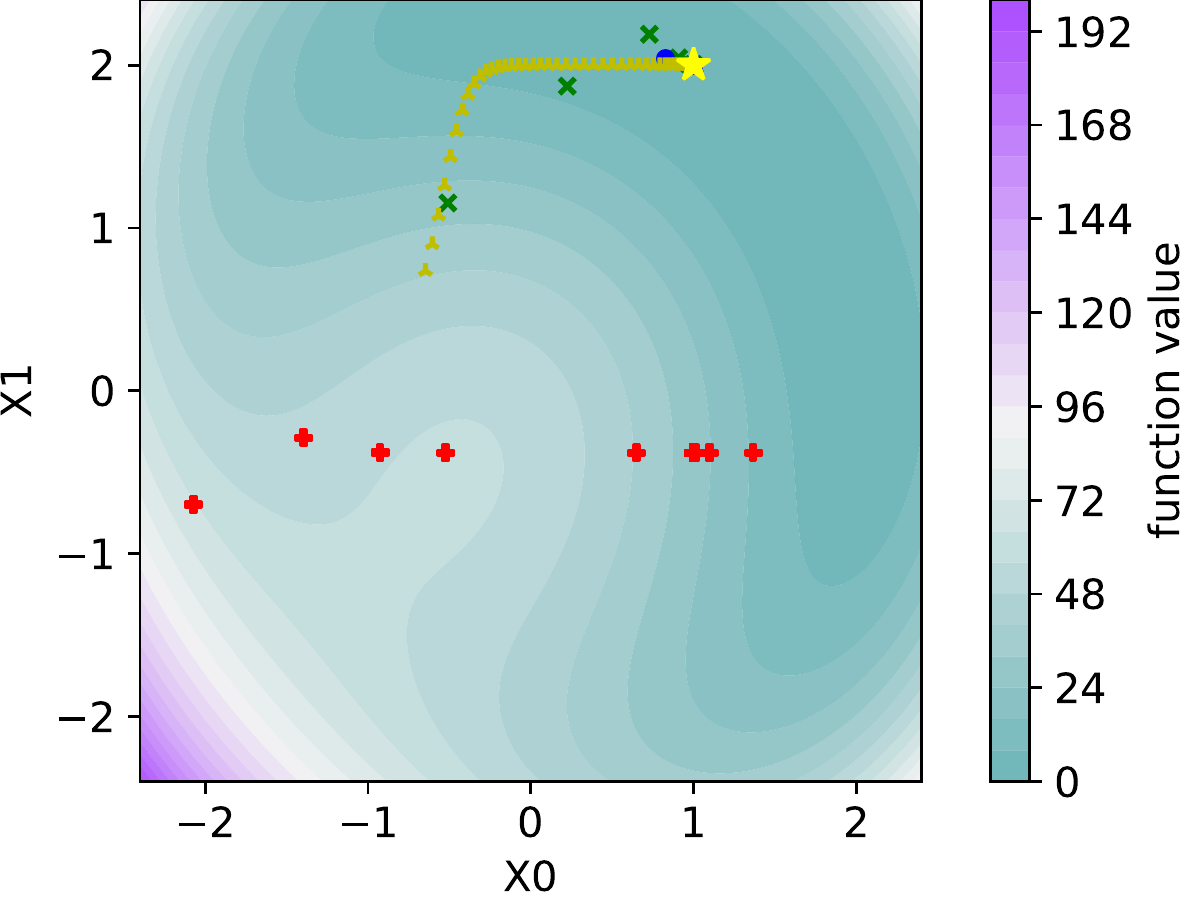}
\includegraphics[width = 0.3\textwidth]{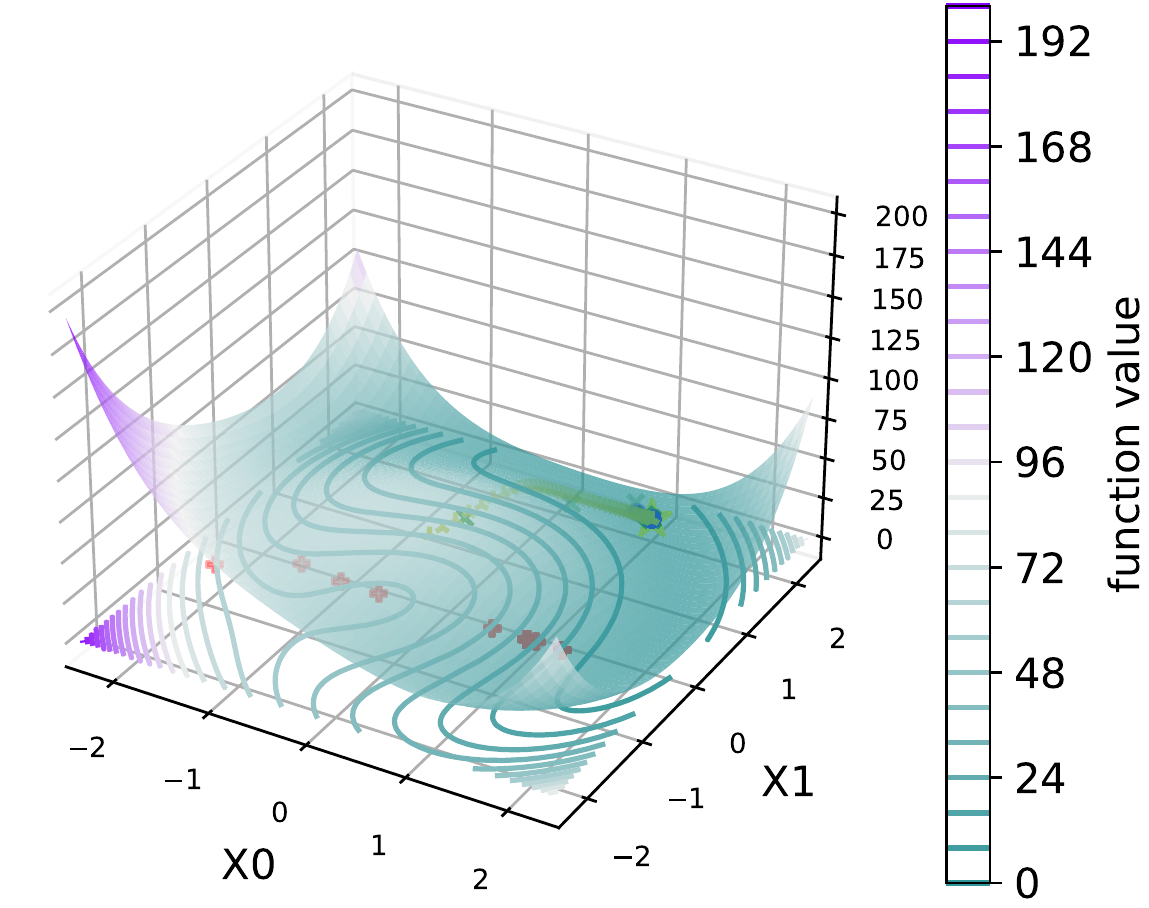}
\caption{The PermD$\beta^+$ function with $\beta =0.5$ where Left: we plot $f(x)$ across epochs Middle: level set plot, Right: Surface plot.  \texttt{SP2} is in \textcolor{blue}{blue }, \texttt{SP2$^+$} is in \textcolor{darkgreen}{green },  \texttt{SGD} is in \textcolor{citrine}{yellow } and  \texttt{Newton} in \textcolor{red}{red}.}
\label{fig:perm}
\end{figure}

\begin{figure}
\centering
\includegraphics[width = 0.3\textwidth]{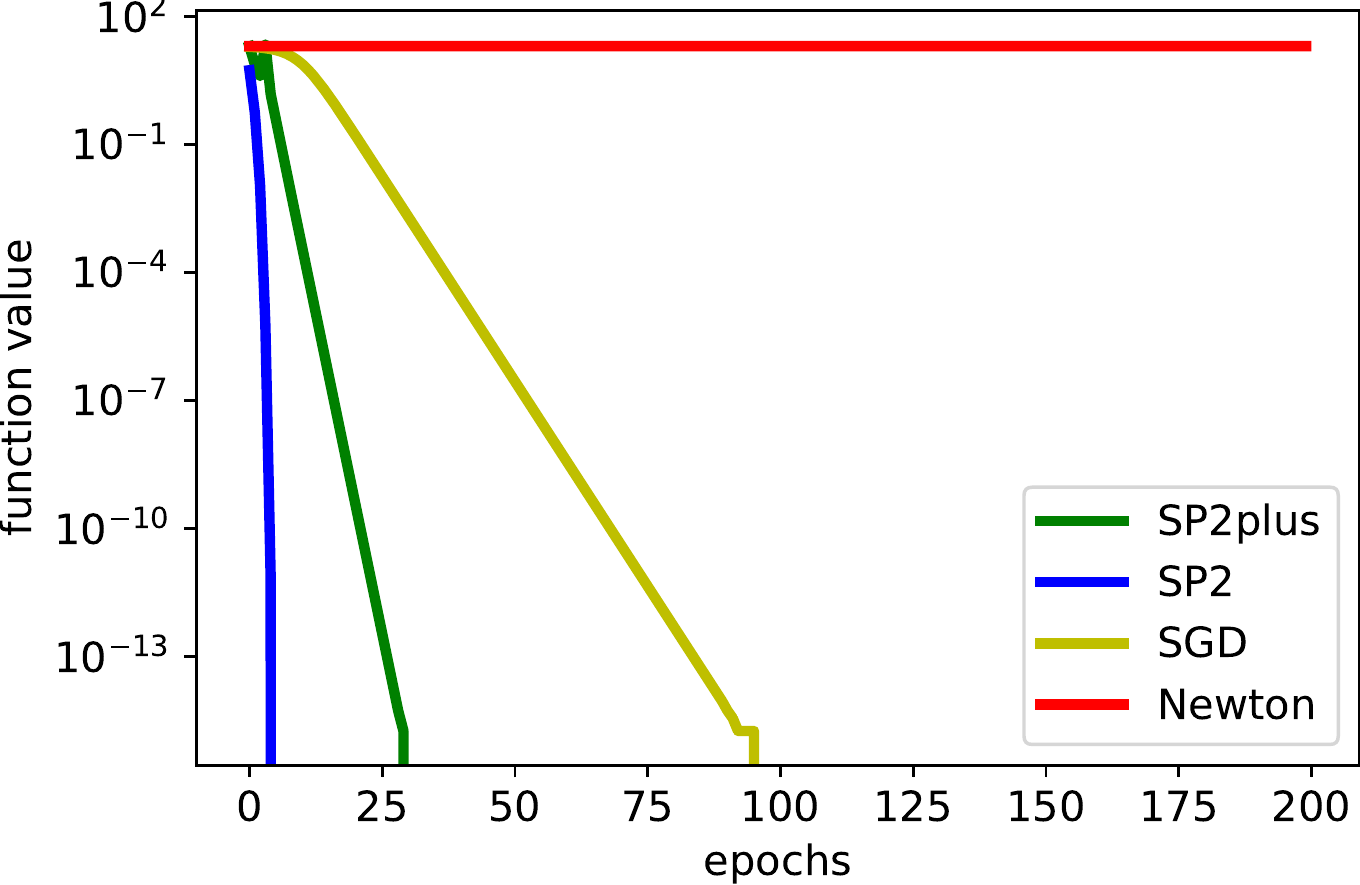}
\includegraphics[width = 0.3\textwidth]{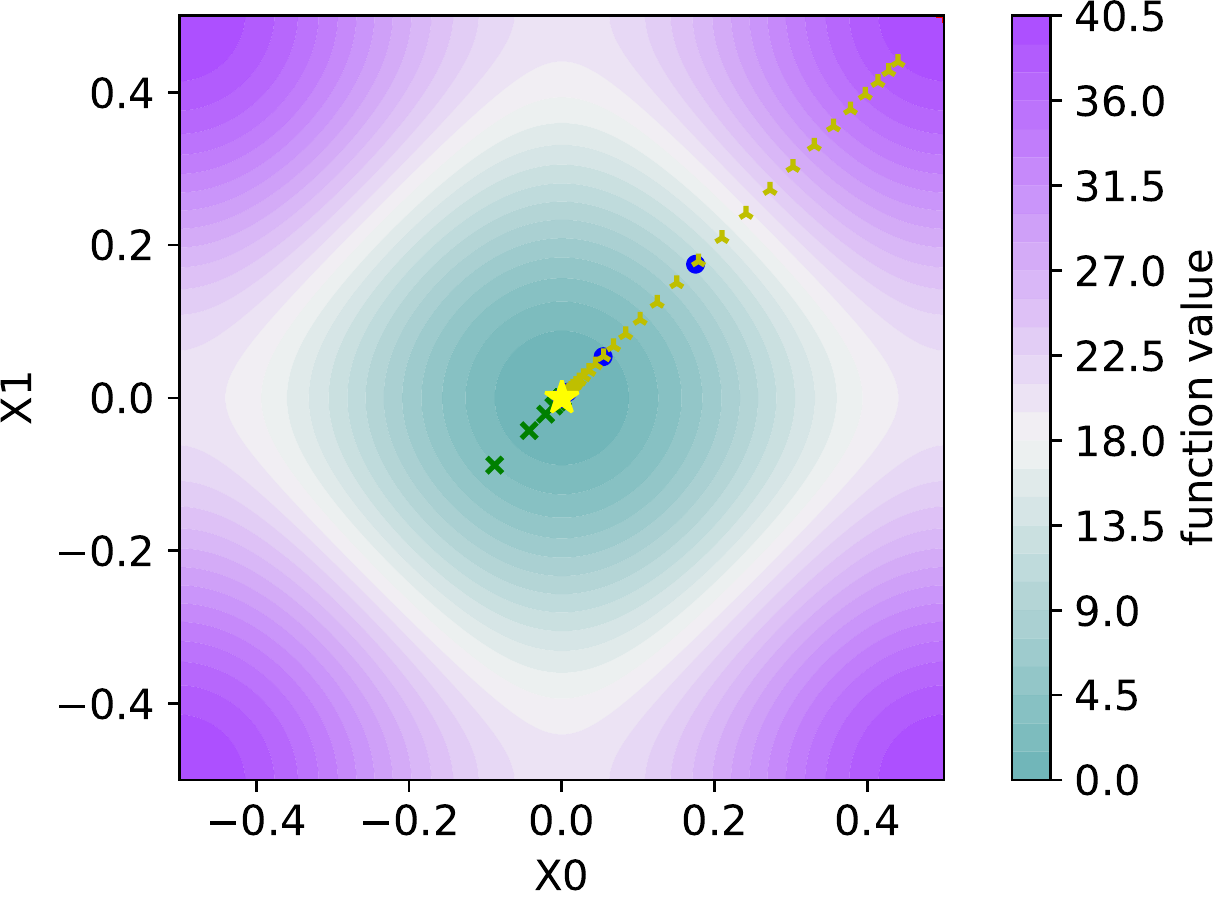}
\includegraphics[width = 0.3\textwidth]{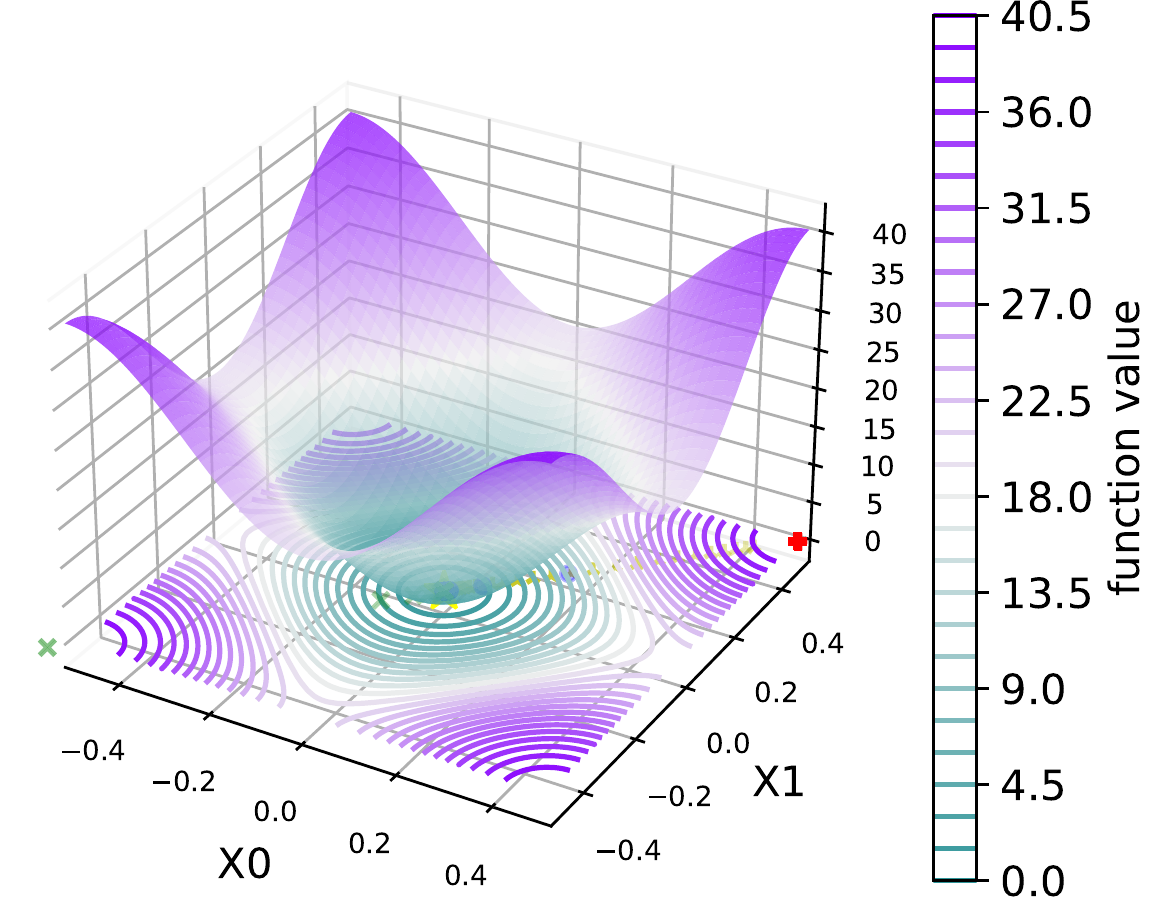}
\caption{The Rastrigin function where Left: we plot $f(x)$ across epochs Middle: level set plot, Right: Surface plot.  \texttt{SP2} is in \textcolor{blue}{blue }, \texttt{SP2$^+$} is in \textcolor{darkgreen}{green },  \texttt{SGD} is in \textcolor{citrine}{yellow } and  \texttt{Newton} is in \textcolor{red}{red}.}
\label{fig:rastrigin}
\end{figure}

% \begin{figure}
% \centering
% \includegraphics[width = 0.3\textwidth]{Rosenbrock-funcs}
% \includegraphics[width = 0.3\textwidth]{Rosenbrock-2d}
% \includegraphics[width = 0.3\textwidth]{Rosenbrock-3d}
% \caption{The Rosenbrock function where Left: we plot $f(x)$ across epochs Middle: level set plot, Right: Surface plot.  \texttt{SP2} is in \textcolor{blue}{blue }, \texttt{SP2$^+$} is in \textcolor{darkgreen}{green },  \texttt{SGD} is in \textcolor{citrine}{yellow } and  \texttt{Newton} is in \textcolor{red}{red }}
% \label{fig:rosenbrock}
% \end{figure}

\subsection{Matrix Completion}
\vspace{-1mm}
Assume a set of known values
$\{a_{i,j}\}_{(i,j)\in \Omega}$
where $\Omega$ is a set of known elements of the matrix, and we want to determine the missing elements.   One approach is solving the \emph{matrix completion} problem
\begin{equation}\label{eq:matcomp}
  \min_{U,V}  \sum_{(i,j) \in \Omega}
     \frac12 (u_i^T v_j - a_{i,j})^2,
\end{equation}
where $U = [u_i]_{(i,j)\in \Omega}$ and $V = [v_j]_{(i,j)\in \Omega}$. With the solution to~\eqref{eq:matcomp}, we then use $U^\top V$ as an approximation to the complete matrix $A = [a_{i,j}]_{i,j =1, \ldots, n}.$

Depite~\eqref{eq:matcomp} being a non-convex problem, if there exists an \emph{interpolating} solution to~\eqref{eq:matcomp}, that is one where 
$ u_i^T v_j  = a_{i,j},~ \mbox{for }(i,j) \in \Omega$,
then the SP2 method can solve~\eqref{eq:matcomp}. Indeed, the SP2 can be applied to~\eqref{eq:matcomp} by sampling a single pair $(i,j) \in \Omega$ uniformly, then projecting onto the quadratic
\begin{align}\label{eq:SP2matcom}
    u_i^{k+1}, v_j^{k+1} &=\; \argmin_{u,v} \frac12 \norm{u-u_i^k}^2 + \frac12\norm{v-v_j^k}^2 \mbox{ subject to } u^\top v  = a_{i,j}.
\end{align}
This projection can be solved as we detail in Theorem~\ref{THM:matcom} in Appendix~\ref{sec:matcom}.

We compared our method~\ref{eq:SP2matcom} to a specialized variant of \texttt{SGD} for online matrix completion described in \cite{jin2016provable}, see Figure~\ref{fig:MC}. To compare the two methods we generated a rank $k=2$ matrix $A\in \R^{100 \times 50}$.
We selected a subset entries with probability $p =0.1, 0.2 $ or $0.3$ to form our set $\Omega_{init}$ 
that was used to obtain an initial estimate $U_0, V_0$ using rank-k SVD method as described in \cite{jin2016provable}.
We extensively tuned the step size of \texttt{SGD} using a grid search, and the method labelled Non-convex \texttt{SGD} is the resulting run of \texttt{SGD} with the best step size. We also show how sensitive \texttt{SGD} is to this step size, by including the run of SGD with step sizes that were only a factor of $2$ to $4$ away from the optimal, which greatly degrades the performance of \texttt{SGD}.
In contrast, \texttt{SP2} worked with no tuning, and matches the performance of SGD with the optimal step size in the $p=0.1$ experiment, and outperforms \texttt{SGD} in the experiments with more measurements as can be seen in the $p=0.2$ and $p=0.3$ figures.

%  As is shown in Figure~\ref{fig:MC}, our \texttt{SP2} method significantly outperforms the SGD method, especially when there are enough measurements.

\begin{figure}
    \centering

\includegraphics[width=1.75in]{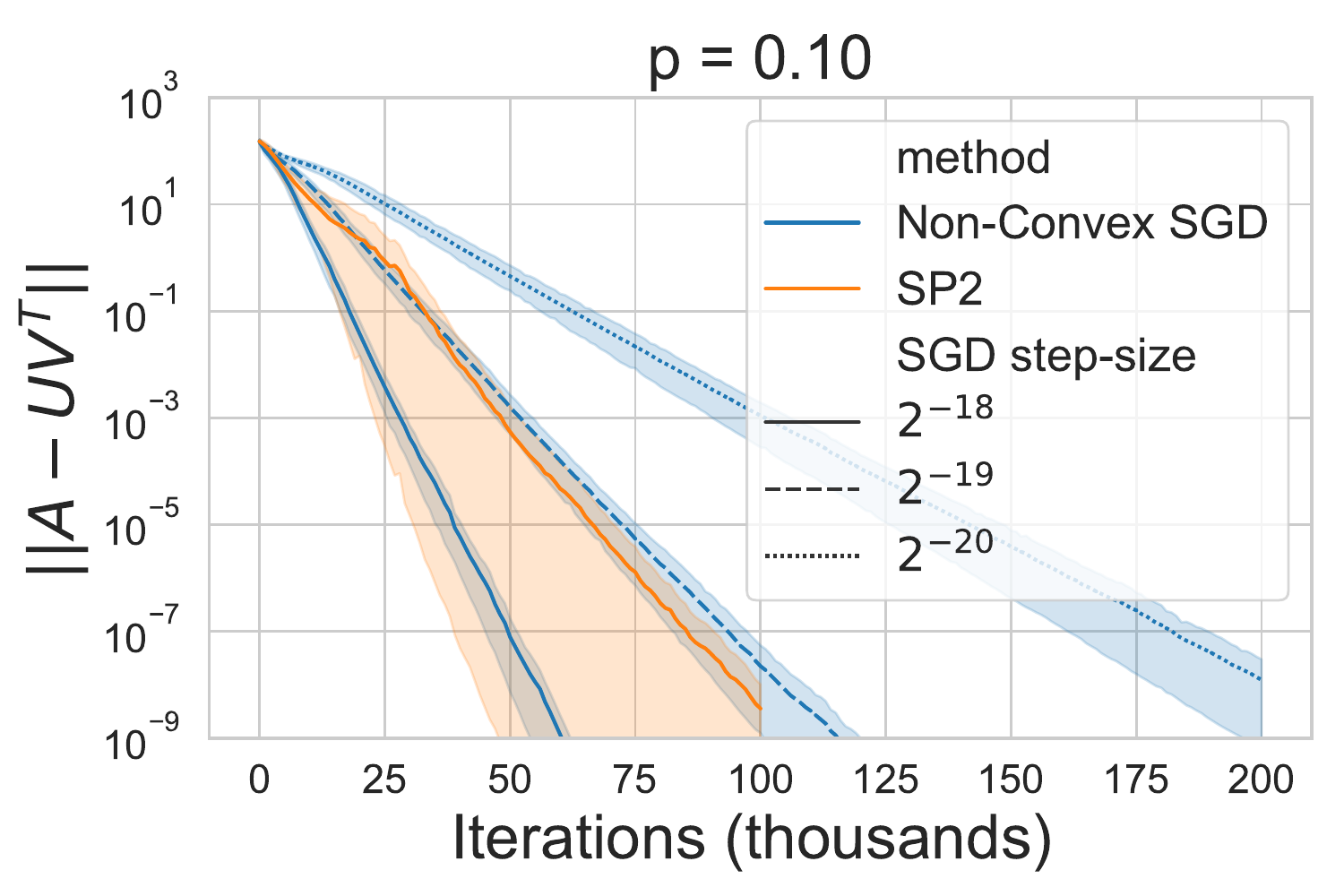}
\includegraphics[width=1.75in]{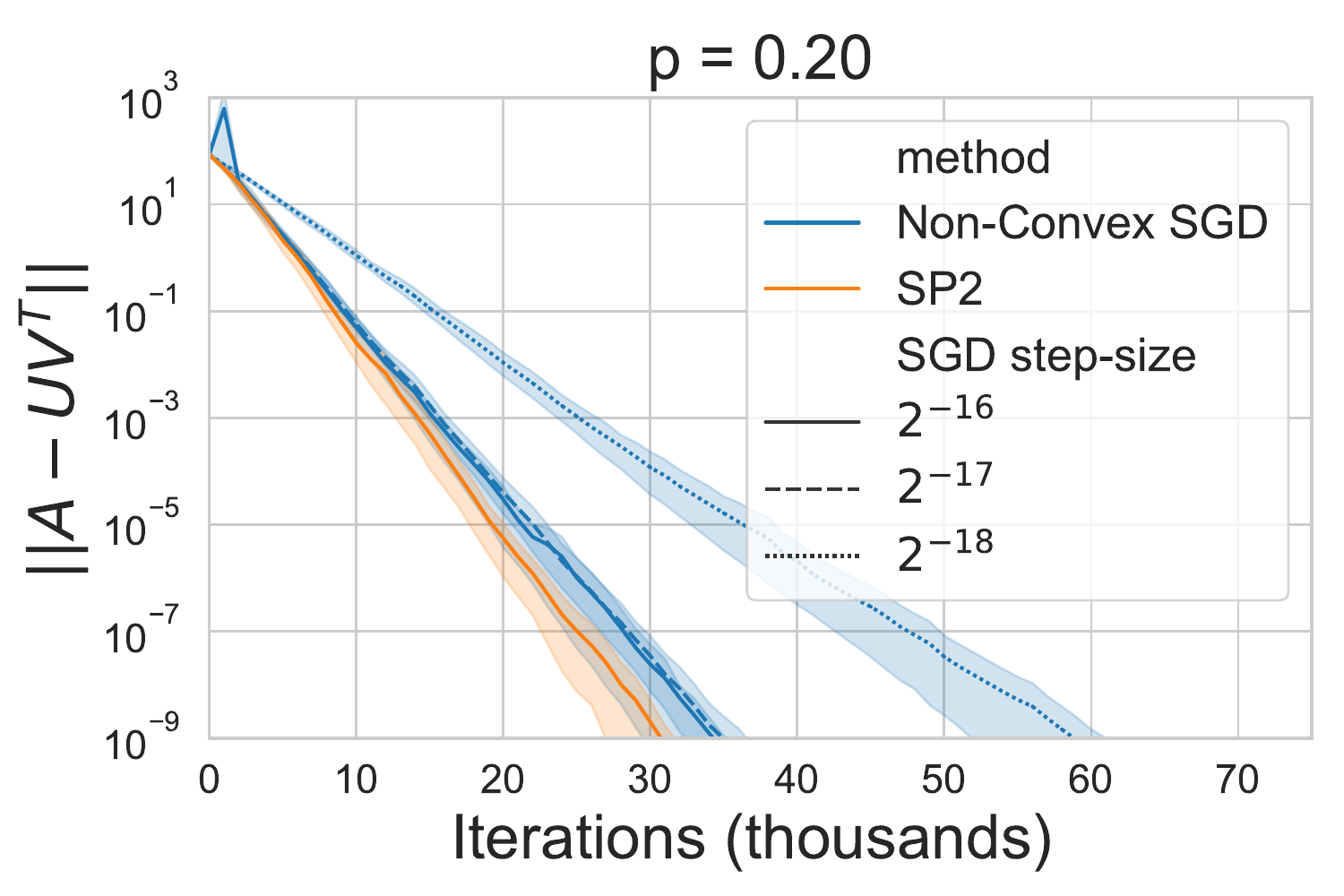}
\includegraphics[width=1.75in]{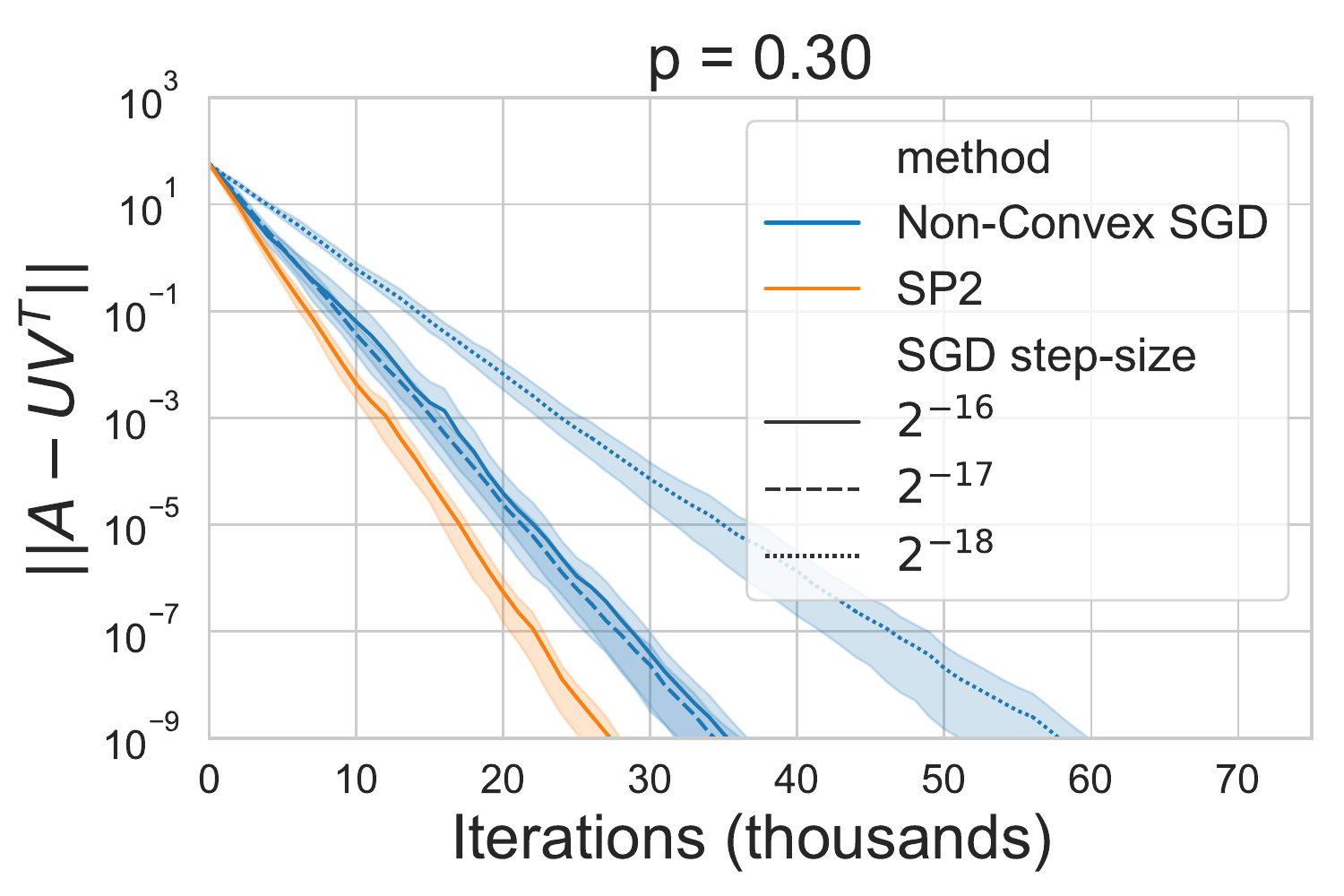}

    \caption{Recovery error for matrix completion. Left: Using 10\%, Middle: using 20\% and Right: using 30\% of the entries of $A$ to form $\Omega$, respectively. Shaded region corresponds to $5$ repeated runs. }
    \label{fig:MC}
\end{figure}

\subsection{Convex classification}
\label{sec:convexexp}
\vspace{-1mm}
In this experiment, we test the proposed methods on a logistic regression problem and compare them with some state-of-the-art methods (e.g., SGD, SP, and ADAM). In particular, we consider the problem of minimizing the following loss function
$
%\begin{align*}
    f(w) = \frac 1 n \sum_{i=1}^n f_i(w) + \frac{\sigma}{2}\|w\|_2^2,
%\end{align*}
$
where $f_i(w) = \phi_i(x_i^\top w)$ with $\phi_i(t) = \ln(1+e^{-y_i t})$. Here, $\{(x_i,y_i)\in \R^{d+1}\}_{i=1}^n$ stands for the feature-label pairs and $\sigma>0$ is the regularization parameter. We can control how far each problem is from interpolation by increasing $\sigma$. When $\sigma >0$ the problem cannot interpolate, and thus we expect to see a benefit of the slack methods in Section~\ref{sec:slack} over \texttt{SP2$^+$}.

% As is illustrated in Table 1 of paper~\cite{gower2021stochastic}, $\sigma=0$ corresponds to the case when $f^*=0$, or equivalently, the interpolation condition holds, while $\sigma \neq 0$ corresponds to the case when we are away from the interpolation condition.    
We used two data sets: colon-cancer~\cite{alon1999broad} and mushrooms~\cite{west2001predicting}, both of which interpolate when $\sigma =0.$

We compare the proposed methods \texttt{SP2}$^+$~\eqref{eq:SP2plus}, \texttt{SP2L2}$^+$ (Lemma~\ref{lem:quadslackproj}),  \texttt{SP2L1}$^+$ (Lemma~\ref{lem:quadslackproj_l1}), and \texttt{SP2max}$^+$ (Lemma~\ref{lem:quadslackproj_l1_unreg})  with \texttt{SGD}, \texttt{SP}~\eqref{eq:SP}, and \texttt{ADAM} on both data sets with three regularizations $\sigma \in  \{0, \;0.001, \;0.008\}$ and with momentum set to $0.3$. 
For the \texttt{SGD} method, we use a learning rate $L_{\max}/\sqrt{t}$ in the $t$-th iteration, where $L_{\max} = \frac{1}{4}\max_{i=1,\ldots, n} \| x_i \|^2$ denotes the smoothness constant of the loss function. We chose $\lambda$ for \texttt{SP2L2}$^+$, \texttt{SP2L1}$^+$, and \texttt{SP2max}$^+$ using a grid search of 
$\lambda \in \{0.1,0.2, \ldots, 0.9\}$, the details are in Section~\ref{sec:add_exp}.
    
The gradient norm and loss evaluated at each epoch are presented in Figures~\ref{fig:colon_Gradnorm_loss_M05} and \ref{fig:mush_Gradnorm_loss_M05} (see Appendix~\ref{sec:add_exp}).   
We see that \texttt{SP2} methods converge much faster than  classical methods (e.g., \texttt{SGD}, \texttt{SP}, \texttt{ADAM}) and need fewer epochs to achieve the tolerance when $\sigma$ is small (left and middle plots). However, they can all fail when the problem is far from interpolation, e.g., when $\sigma = 8\times 10^{-3}$.
The running time used for each algorithm to achieve either the tolerance or maximum number of epochs for both data sets is presented in Figure~\ref{fig:runtime} (see Appendix~\ref{sec:add_exp}).
%, which also indicates that the \texttt{SP2} methods outperform the classical methods when . 

%We also test the ``exact'' version methods \texttt{SP2} and \texttt{SP2max} for GLMs under interpolation and  compare against the SGD and SP methods, which are known to be fast under interpolation ($\sigma = 0$). The gradient norm and loss evaluated at each epoch are presented in Figure~\ref{fig:exact}. It can be seen that the proposed second order methods are faster.  

\begin{figure}[t]
\begin{minipage}{0.32\linewidth}
\centering
\includegraphics[width=1.7in]{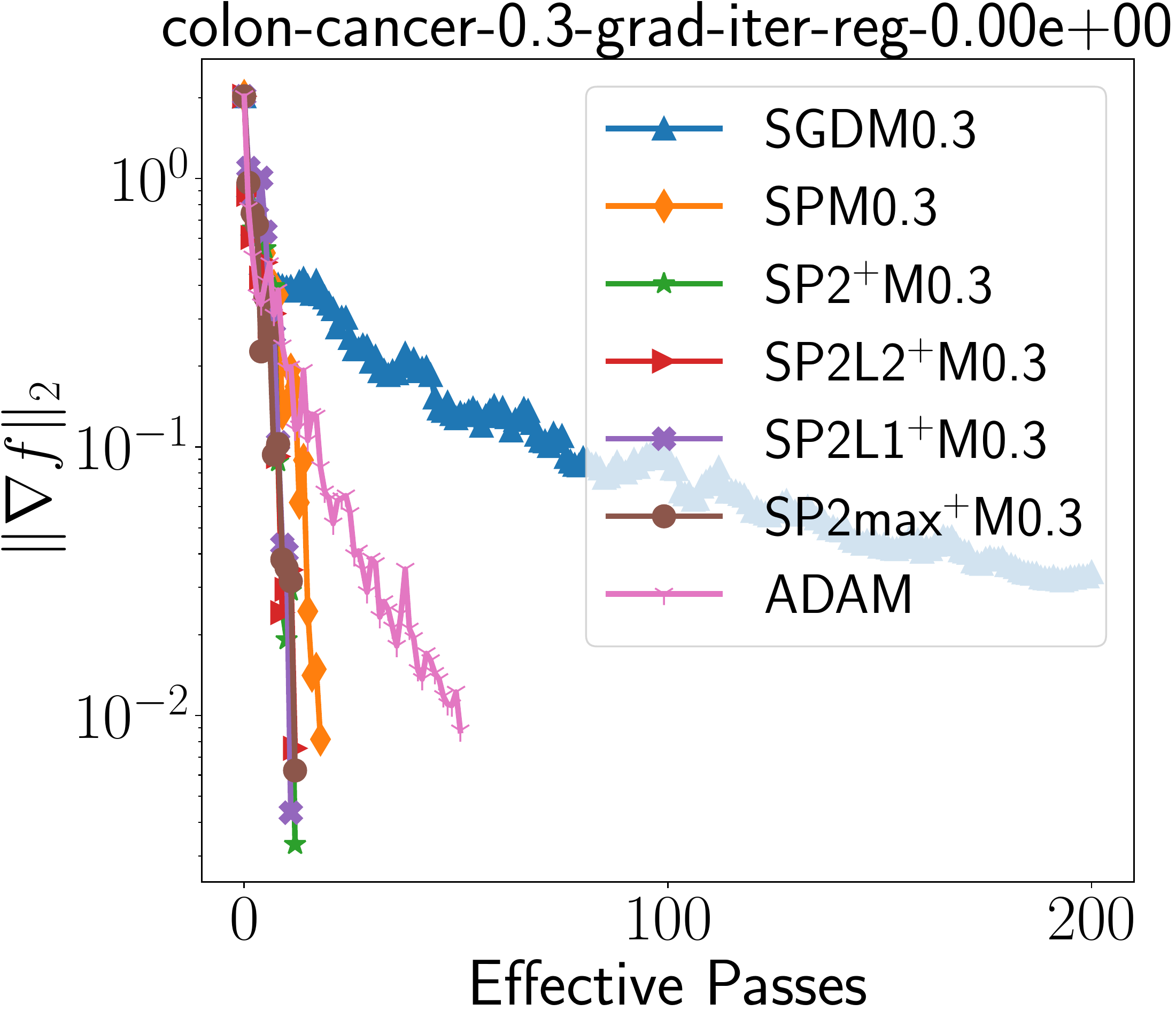}
%\centerline{\small{(a)}}
\end{minipage}
\hfill
\begin{minipage}{0.32\linewidth}
\centering
\includegraphics[width=1.7in]{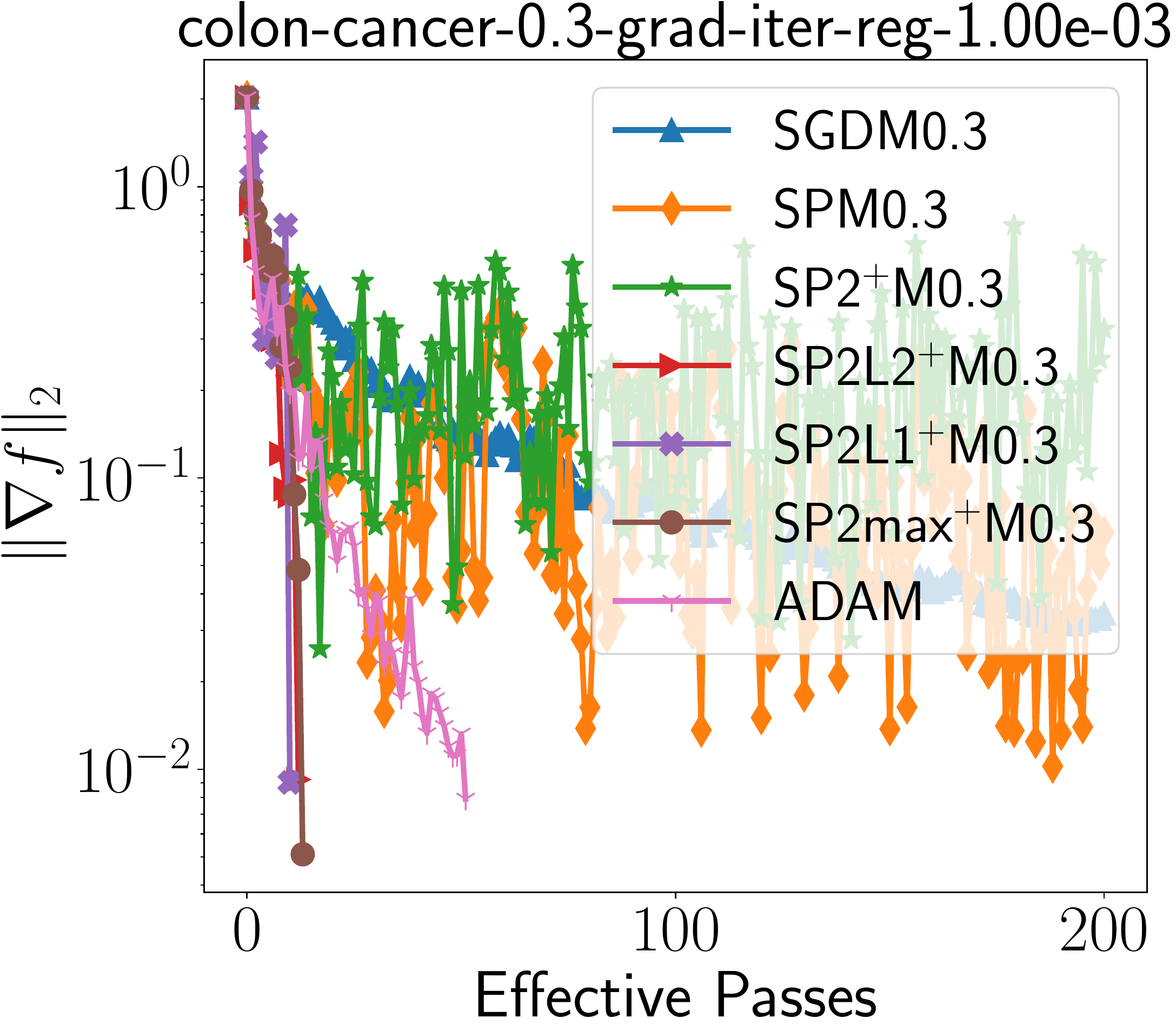}
%\centerline{\small{(a)}}
\end{minipage}
\hfill
\begin{minipage}{0.32\linewidth}
\centering
\includegraphics[width=1.7in]{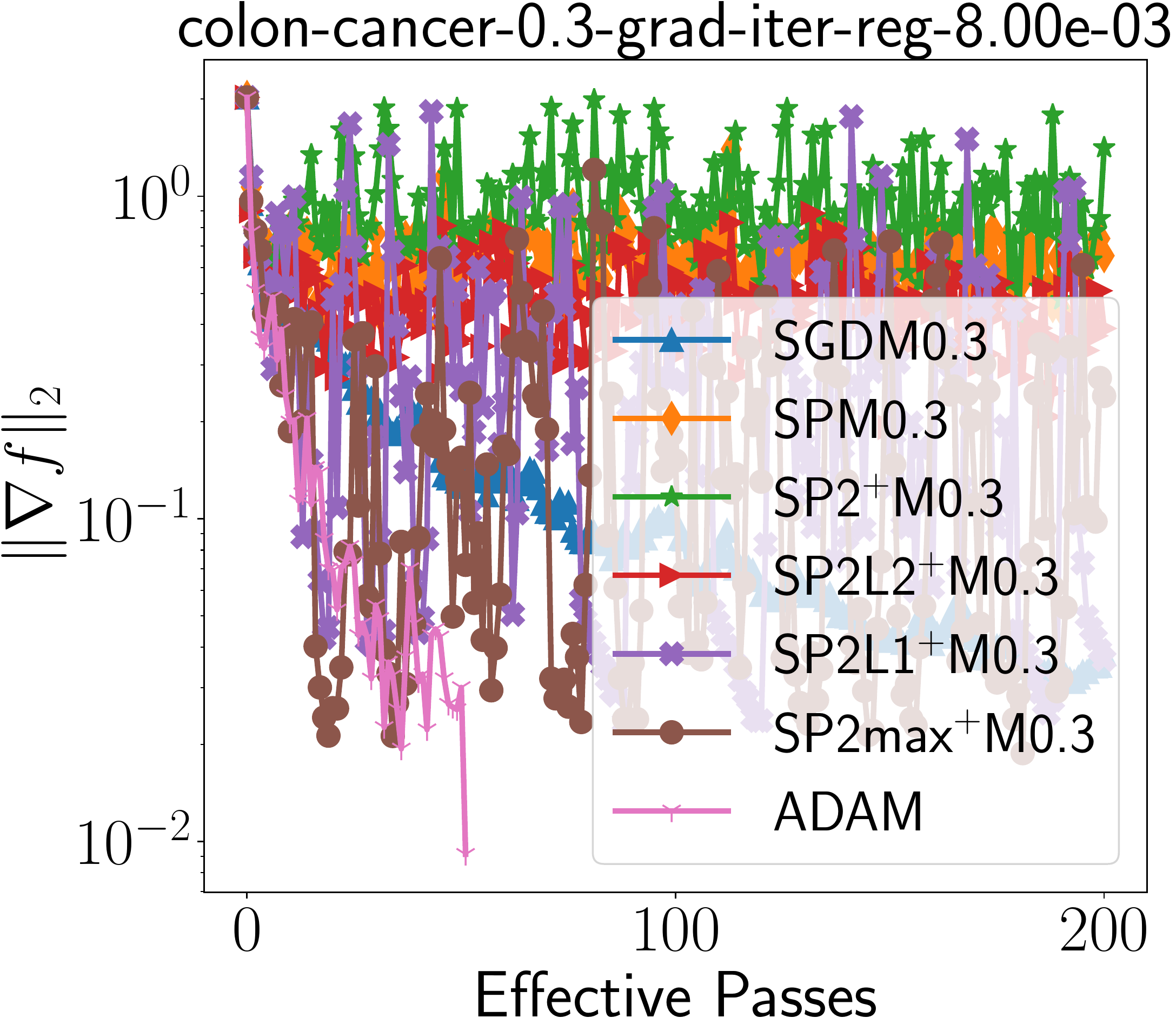}
%\centerline{\small{(b)}}
\end{minipage}
\\
\begin{minipage}{0.32\linewidth}
\centering
\includegraphics[width=1.7in]{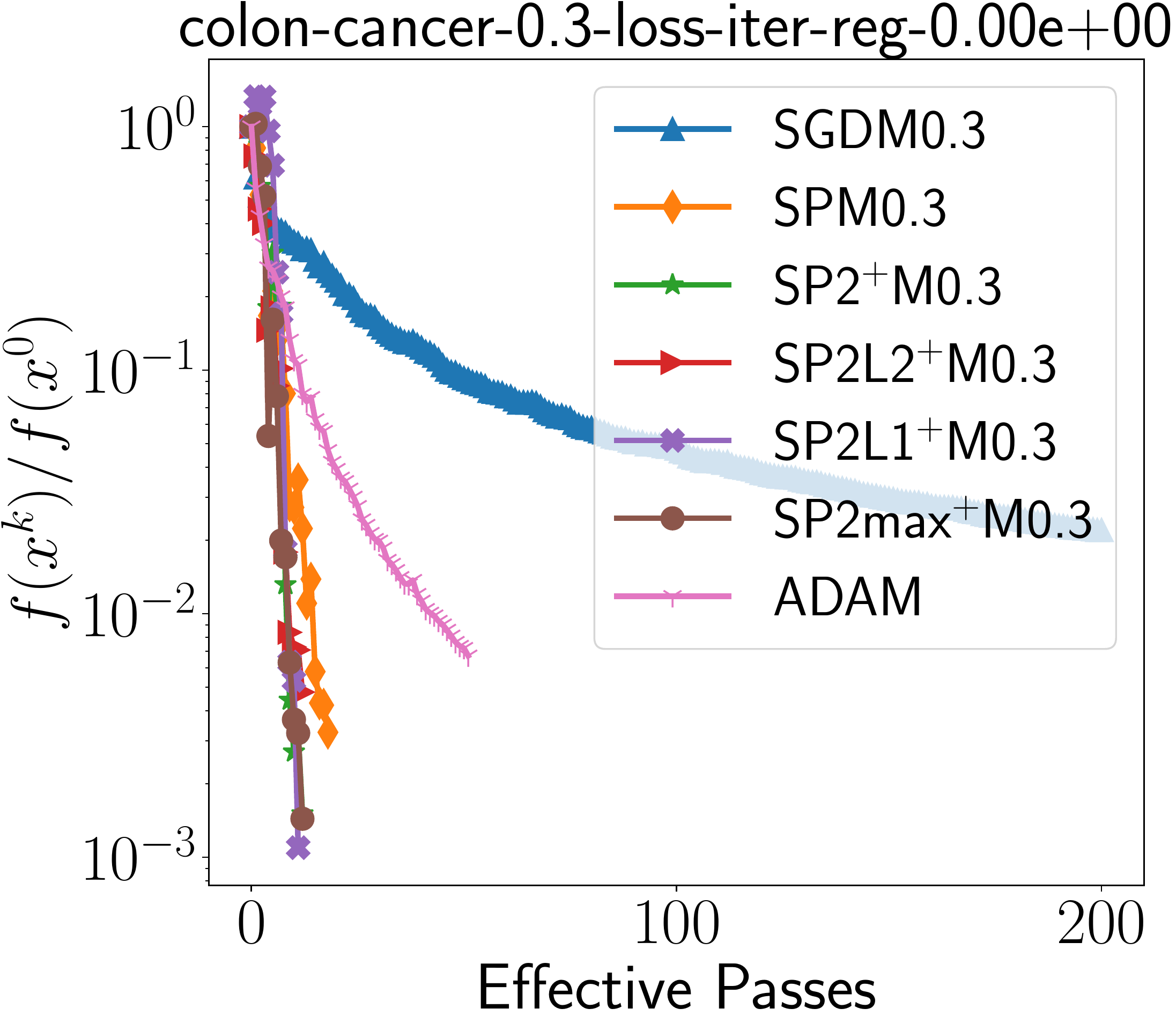}
%\centerline{\small{(c)}}
\end{minipage}
\hfill
\begin{minipage}{0.32\linewidth}
\centering
\includegraphics[width=1.7in]{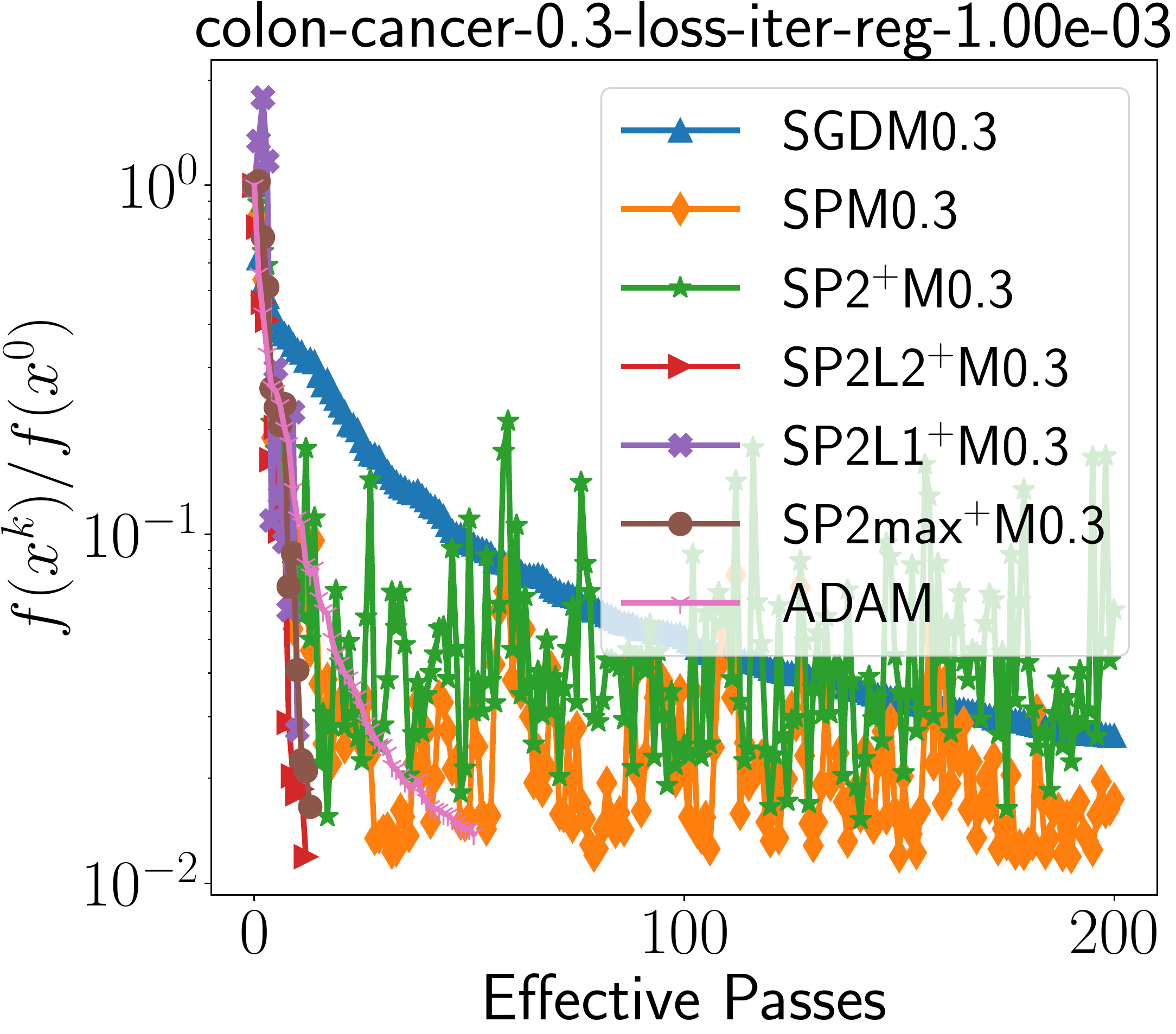}
%\centerline{\small{(c)}}
\end{minipage}
\hfill
\begin{minipage}{0.32\linewidth}
\centering
\includegraphics[width=1.7in]{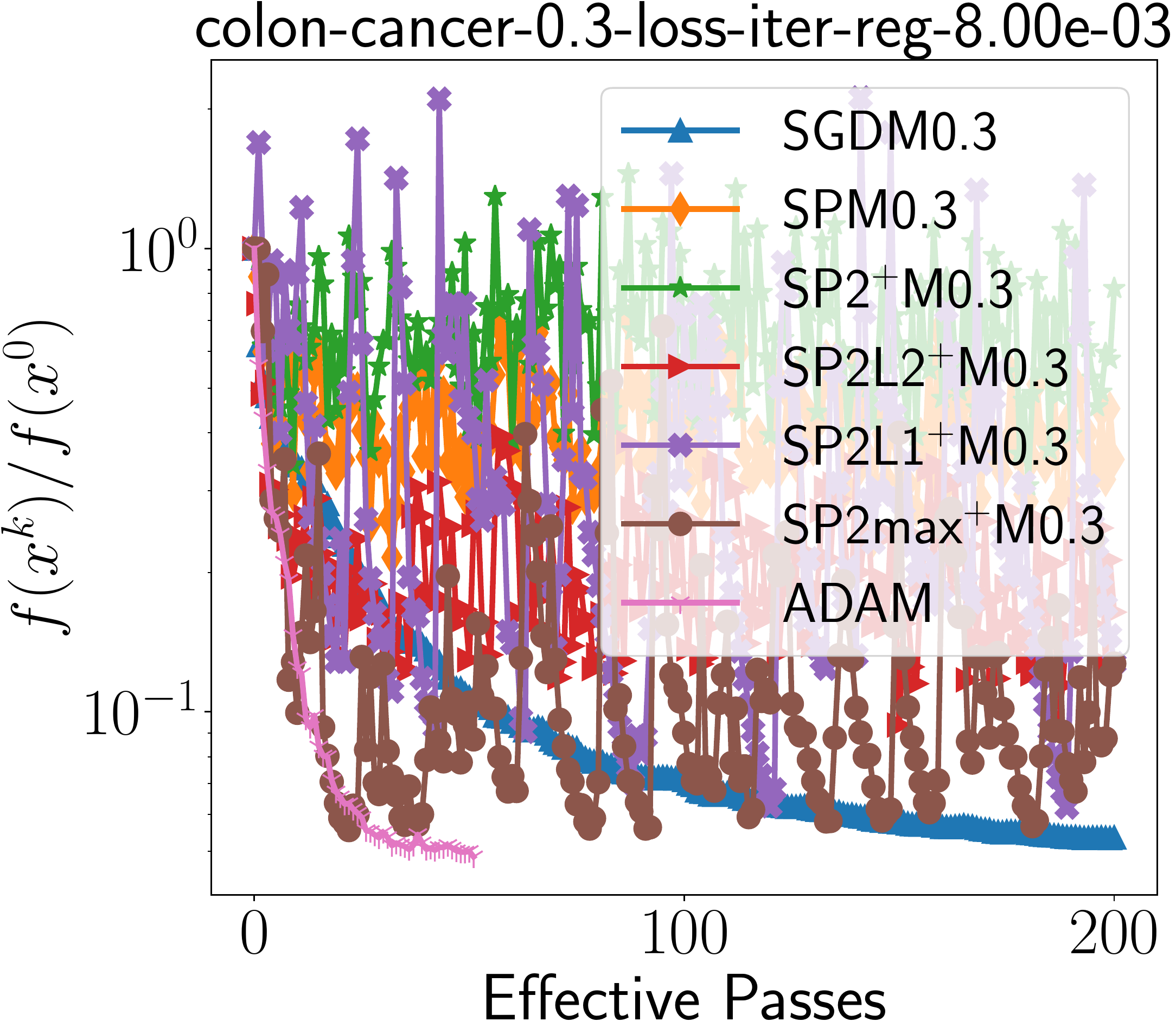}
%\centerline{\small{(d)}}
\end{minipage}
\caption{Colon-cancer: gradient norm and loss at each epoch with momentum being $0.3$. Left: $\sigma =0,$ Middle: $\sigma = 0.001$ and Right $\sigma = 0.008$.
%\DN{May be better to remove the current titles and add brief description (with a), b), ...) to this caption. And if we could remove the legend from all but one fig to make the others easier to see?}
}
\label{fig:colon_Gradnorm_loss_M05}
\end{figure}

\section{Conclusion}
\label{sec:conclusion}
\vspace{-2mm}

%> New way forward for designing incremental second order methods that is well suited to non-convex problems that interpolate, or are close to interpolation. Very promising results on non-convex problems.

%> Future direction to develop specialized variants for optimizing Deep Neural Networks, which is possible since gradient and Hessian-vector products can be computed using backpropagation.

%> Slack variants of the matrix completion problem?  

Here we have proposed new incremental second order methods that exploit models that interpolate the data, or are close to interpolation.
What sets our methods apart from most previous incremental second order methods is that they do not rely on convexity in their design. Quite the opposite, the \texttt{SP2} method can benefit from the Hessian having at least one negative eigenvalue. Consequently the \texttt{SP2} method excels at minimizing non-convex models that interpolate, as can be seen in Sections~\ref{sec:nonconvextoy} and~\ref{sec:matcom}. We also provided an indicative convergence in Theorem~\ref{theo:qconvNR} that shows  that \texttt{SP2} and its approximation \texttt{SP2}${^+}$ enjoy a significantly faster rate of convergence as compared to \texttt{SGD} for sums-of-quadratics.

We then developed  second order methods that can solve a relaxed version of the interpolation equations that allows for some slack in Section~\ref{sec:slack}, and showed that these methods still perform well on problems that are close to interpolation in Section~\ref{sec:convexexp}.

%  We design incremental second order methods that are well suited to non-convex problems that interpolate or are close to interpolation. Moreover, our proposed \texttt{SP2} method is proven to have a faster rate of convergence than the classic SGD method. It also speeds up the convergence of \texttt{SP} due to the use of Hessian and gradient vector products. We also relax the interpolation assumption, and develop analogous quadratic methods for the slack variant of this problem. 
In future work,  it would be interesting to develop specialized variants of \texttt{SP2} for optimizing (DNNs) Deep Neural Networks. DNNs are particularly well suited since they can interpolate, are non-convex and  since  gradients and Hessian vector products can be computed efficiently using back-propagation.

\bibliographystyle{abbrv}
\bibliography{references}

\clearpage

\newpage
 
\appendix

\section{Appendix}

\subsection{Projecting onto Quadratic}
\vspace{-1mm}
\label{sec:projquad}
This following projection lemma is based on Section B in~\cite{park2017general}. What we do in addition to~\cite{park2017general} is to clarify how to compute the resulting projection, and add further details on the proof.
\begin{lemma}\label{lem:quadeqproj}
Let $w\in \R^d$ and let $\mP \in \R^{d\times d}$ is a symmetric matrix.
Consider the projection
\begin{align}
w' \in &\argmin_{w\in\R^d}\tfrac{1}{2} \norm{w-z}^2 \nonumber \\
&\mbox{ s.t. }r+ \dotprod{q, w-z}  + \tfrac{1}{2} \dotprod{\mP (w-z), w-z}=0. \label{eq:quadproj}
\end{align}
Let 
\[\lambda_1 \leq \lambda_2 \leq \cdots \leq \lambda_d\]
be the eigenvalues of $\mP$ and
let $Q\Lambda Q^\top = \mP$ be the eigenvalue decomposition of $\mP$, where $\Lambda = \mbox{diag}(\lambda_i)$ and $QQ^\top =I.$ Let $\hat{q} = Q^\top q$.
If the quadratic constraint in~\eqref{eq:quadproj} is feasible, then there exists a solution to~\eqref{eq:quadproj}.
Now we give the three candidate solutions.
\begin{enumerate}
\item If $r=0$, then the solution is given by
\[ w= z.\]
    \item 
Now assuming $r \neq 0.$ Let 
\begin{eqnarray}
    \nu &= &\max_{i \;: \; \lambda_i \neq 0} \left\{ -\frac{1}{\lambda_1}, -\frac{1}{\lambda_d } \right\} \label{eq:quadprojnumax}\\
    i^* &\in& \arg\max_{i \;: \; \lambda_i \neq 0} \left\{ -\frac{1}{\lambda_1}, -\frac{1}{\lambda_d } \right\}\\
    N & =& \{i \; : \; \lambda_i \neq \lambda_{i^*} \}.
\end{eqnarray} 
Let
\begin{eqnarray}
 x^* & = &-(\mI +\nu \Lambda)^\dagger \nu \hat{q}. 
\end{eqnarray}
If
\begin{equation}\label{eq:condition}
     2\nu r +\nu\dotprod{\hat{q},x^*}-\norm{x^*}^2+\frac{\nu^2}{4}\sum_{i\in N}\hat{q}_i^2 \geq 0 
\end{equation}
then the solution is given by
\begin{eqnarray}
w' =z+ Q (x^* + n),
\end{eqnarray}
where $n \in \R^d$ and
\begin{align*}
 n_i & = \frac{\nu}{2} \hat{q}_i + \frac{1}{\sqrt{|N|}}\sqrt{2\nu r +\nu\dotprod{\hat{q},x^*}-\norm{x^*}^2+\frac{\nu^2}{4}\sum_{i\in N}\hat{q}_i^2}, \quad \mbox{for }i \in N \\
 n_i & = 0, \quad \mbox{for }i \not\in N.
 \end{align*}
 
\item
 
Alternatively if~\eqref{eq:condition} does not hold, then the solution is given by 
\begin{eqnarray}
w' &=& z -(\mI +\nu \Lambda)^\dagger \nu q
\end{eqnarray}
where $\nu$ is the solution to the nonlinear equation
\begin{eqnarray}\label{eq:nueq}
\frac{\nu}{2}\sum_i \frac{\hat{q}_i^2(2+\nu \lambda_i)}{(1+\nu \lambda_i)^2}&= & r.
\end{eqnarray}
\end{enumerate}
\end{lemma}
\begin{proof}
First note that there exists a solution to~\eqref{eq:quadproj} since the constraint is a closed feasible set.

Let $Q\Lambda Q^\top = \mP$ be the SVD of $\mP$, where $QQ^\top =I.$ By changing variables $x= Q^\top (w -z)$ we have that~\eqref{eq:quadproj} is equivalent to
\begin{align}
&\argmin_{\hat{x}\in\R^d}\tfrac{1}{2} \norm{x}^2 \nonumber \\
&\mbox{ s.t. }r+ \dotprod{\hat{q}, x}  + \tfrac{1}{2} \dotprod{\Lambda x, x}=0, \label{eq:quadproj2}
\end{align}
where $\hat{x} = Q^\top (w-z)$ and $\hat{q} = Q^\top q.$
The Lagrangian of~\eqref{eq:quadproj2} is given by
\begin{align}
L(x,\nu) & = \;  \tfrac{1}{2} \norm{x}^2 + \nu\big(r+ \dotprod{\hat{q}, x}  + \tfrac{1}{2} \dotprod{\Lambda x, x}\big)\nonumber \\
&= \;\tfrac{1}{2}x^\top (\mI +\nu \Lambda)x + \nu\big(r+ \dotprod{\hat{q}, x}\big).
\end{align}
Thus the KKT conditions are given by
\begin{align}
    \nabla_{x} L(x,\nu) & = (\mI +\nu \Lambda)x  + \nu \hat{q} =0 \label{eq:KKTquadprojw}\\
    \nabla_{\nu}  L(x,\nu)& = r+ \dotprod{\hat{q}, x}  + \tfrac{1}{2} \dotprod{\Lambda x, x}=0.\label{eq:KKTquadprojnu}
\end{align}
Since we are guaranteed that the projection has a solution, we have that as a necessary condition that the solution satisfies $$\nabla_{x}^2 L(x,\nu) = (\mI +\nu \Lambda) \succeq 0,$$ 
see Theorem 12.5 in~\cite{wright1999numerical}. Consequently either $(\mI +\nu \Lambda) \succ 0$ or $(\mI +\nu \Lambda) $ has a zero eigenvalue.

Consider the case where $ (\mI +\nu \Lambda) \succ 0.$ From~\eqref{eq:KKTquadprojw} we have that
\begin{eqnarray}\label{eq:tmeaasomer}
x  = -  \nu  (\mI +\nu \Lambda)^{-1}\hat{q}.
\end{eqnarray}
Now note that if $\nu =0$ then $x =0$ and by the constraint we must have $r=0$. Otherwise, if $r \neq0,$ then $\nu \neq 0.$ Assume now $\nu \neq 0$ and substituting the above into~\eqref{eq:KKTquadprojnu} and letting $\Lambda = \mbox{diag}(\lambda_i)$ gives
\begin{align*}
    \nabla_{\nu}  L(x,\nu)&=r+ \dotprod{\hat{q}, x}  + \tfrac{1}{2\nu} \dotprod{\nu \Lambda x, x} \\
   &= r+\dotprod{\hat{q}, x}  + \tfrac{1}{2\nu} \dotprod{(\mI+\nu \Lambda)x, x}  - \frac{1}{2\nu}\norm{x}^2 \\
    &=   r+\frac{1}{2}\dotprod{\hat{q}, x}   - \frac{1}{2\nu}\norm{x}^2 & \mbox{Using~\eqref{eq:KKTquadprojw}}\\
        &=   r-\frac{\nu}{2}\dotprod{\hat{q},   (\mI +\nu \Lambda)^{-1}\hat{q}}   - \frac{\nu}{2}\norm{  (\mI +\nu \Lambda)^{-1}\hat{q}}^2 & \mbox{Using~\eqref{eq:tmeaasomer} }\\
        &= r - \frac{\nu}{2}\sum_i\left( \frac{\hat{q}_i^2}{1+\nu \lambda_i}+\frac{\hat{q}_i^2}{(1+\nu \lambda_i)^2}\right).
\end{align*}
Thus
\begin{eqnarray}
\frac{\nu}{2}\sum_i \frac{\hat{q}_i^2(2+\nu \lambda_i)}{(1+\nu \lambda_i)^2}&= & r.
\end{eqnarray}
Upon finding the solution $\nu$ to the above, we have that our final solution is given by $w' = z+Qx,$ that is
\begin{eqnarray}
w' &=& z -Q(\mI +\nu \Lambda)^\dagger \nu \hat{q} \nonumber \\
& =& z -(\mI +\nu \Lambda)^\dagger \nu Q\hat{q} \nonumber \\
&= & z -\nu(\mI +\nu \Lambda)^\dagger  q
\end{eqnarray}

Alternatively, suppose that $ (\mI +\nu \Lambda) \succeq 0$ is  non-singular. The positive definiteness implies that
\begin{equation}\label{eq:tempzne9ze}
     \nu  \; \geq \; -\frac{1}{\lambda_i}, \quad \mbox{for } i=1,\ldots, d.
\end{equation}
For $(\mI +\nu \Lambda)$ to be non-singular, at least one of the above inequalities will hold to equality. To ease notation, 
let us arrange the eigenvalues in increasing order so that 
\[\lambda_1 \leq \lambda_2 \leq \cdots \leq \lambda_d.\]
For one of the~\eqref{eq:tempzne9ze} inequalities  to hold to equality we need that 
$$ \nu \;  = \; \max_{i \;: \; \lambda_i \neq 0} -\frac{1}{\lambda_i}  = \max_{i \;: \; \lambda_i \neq 0} \left\{ -\frac{1}{\lambda_1}, -\frac{1}{\lambda_d } \right\}.$$ 
Since $(\mI +\nu \Lambda)$ is now singular
with this $\nu$, we have that the solution to~\eqref{eq:KKTquadprojw} is given by
\begin{equation}
 x \;=\; -(\mI +\nu \Lambda)^\dagger \nu \hat{q} + n \; \eqdef \; x^* +n, \quad \mbox{where } \dotprod{x^*,n} =0, \label{eq:tempnxo848os}
\end{equation} 
where $\dagger$ denotes the pseudo-inverse and where $n$ is in the kernel of $(\mI +\nu \Lambda)$, in other words  $(\mI +\nu \Lambda)n =0.$  It remains to determine $n$, which we can do with~\eqref{eq:KKTquadprojnu}. Indeed, substituting~\eqref{eq:tempnxo848os} into~\eqref{eq:KKTquadprojnu} gives
\begin{align*}
    \nabla_{\nu}  L(x,\nu)
    &=   r+\frac{1}{2}\dotprod{\hat{q}, x}   - \frac{1}{2\nu}\norm{x}^2 & \mbox{Using~\eqref{eq:KKTquadprojw}}\\
        &=  r+\frac{1}{2}\dotprod{\hat{q}, x^* +n}   - \frac{1}{2\nu}\norm{n}^2  - \frac{1}{2\nu}\norm{x^*}^2. & \mbox{Using~\eqref{eq:tempnxo848os}}.
\end{align*}
Setting to zero and completing the squares in $n$ 
% and letting $$C \eqdef  r +\frac{1}{2}\dotprod{\hat{q},x^*}-\frac{1}{2\nu}\norm{x^*}^2+\frac{\nu}{8}\norm{\hat{q}}^2,
% % =  r -\frac{1}{2\nu}\norm{x^*-\frac{\nu}{2}\hat{q}}^2+\frac{\nu}{4}\norm{\hat{q}}^2,
% $$ 
we have that
\begin{eqnarray}\label{eq:zn;oeinznz}
\frac{1}{2\nu}\norm{n-\frac{\nu}{2}\hat{q}}^2 = r +\frac{1}{2}\dotprod{\hat{q},x^*}-\frac{1}{2\nu}\norm{x^*}^2+\frac{\nu}{8}\norm{\hat{q}}^2.
\end{eqnarray}
To characterize the solutions in $n$ of the above, first note that $n$ will only have a few non-zero elements. To see this,  let 
 $i^* \in \arg\max_{i \;: \; \lambda_i \neq 0} \left\{ -\frac{1}{\lambda_1}, -\frac{1}{\lambda_d } \right\}$, and note that $(\mI +\nu \Lambda)$ has as many zeros on the diagonal as the multiplicity of the eigenvalue $\lambda_{i^*}.$ That is, it has zeros elements on the  indices in 
 $$I \;= \; \{i \; :\;  \lambda_i = \lambda_{i^*} \} .$$ Thus the non-zero elements of $n$ are in the set $$N \;=\;  \{i \; :\;  \lambda_i \neq \lambda_{i^*} \}.$$ Because of this observation we further re-write~\eqref{eq:zn;oeinznz} as
 
\begin{eqnarray}
\sum_{i\in N}\left(n_i - \frac{\nu}{2} \hat{q}_i\right)^2
% &=& 2\nu C - \sum_{i\in I}\frac{\nu^2}{4} \hat{q}^2_i \nonumber \\
&=& 2\nu r +\nu\dotprod{\hat{q},x^*}-\norm{x^*}^2+\frac{\nu^2}{4}\norm{\hat{q}}^2- \sum_{i\in I}\frac{\nu^2}{4} \hat{q}^2_i \nonumber \\
&= & 2\nu r +\nu\dotprod{\hat{q},x^*}-\norm{x^*}^2+\frac{\nu^2}{4}\sum_{i\in N}\hat{q}_i^2.
\end{eqnarray}
Consequently, if the above is positive, then there exists solutions to the above of which
\begin{eqnarray}
n_i = \frac{\nu}{2} \hat{q}_i + \frac{1}{\sqrt{|N|}}\sqrt{2\nu r +\nu\dotprod{\hat{q},x^*}-\norm{x^*}^2+\frac{\nu^2}{4}\sum_{i\in N}\hat{q}_i^2}, \quad \mbox{for }i \in N.
\end{eqnarray}
is one. Consequently, the final solution is given by $w = z+ Q(x^* +n)$ where $x^*$ is given in by~\eqref{eq:tempnxo848os}.
\end{proof}

\begin{corollary}
If $r >0$ and $\mP$ has at least one negative eigenvalue, there always exists a solution to the projection~\eqref{eq:quadproj}.
\end{corollary}
\begin{proof}
We only need to prove that there exists a solution to the quadratic equation in~\eqref{eq:quadproj}, after which Lemma~\ref{lem:quadeqproj} guarantees the existance of a solution.
\end{proof}

\section{Matrix Completion}
\label{sec:matcom}

% Assume we have a set of known values
% $\{a_{i,j}\}_{(i,j)\in \Omega}$
% where $\Omega$ is a set of known elements of the matrix, and we want to determine the missing elements.   One way to do this is to solve the \emph{matrix completion} problem
% \begin{equation}\label{eq:matcomp}
%   \min_{U,V}  \sum_{(i,j) \in \Omega}
%      \frac12 (u_i^T v_j - a_{i,j})^2,
% \end{equation}
% where $U = [u_i]_{(i,j)\in \Omega}$ and $V = [v_j]_{(i,j)\in \Omega}$. With the solution to~\eqref{eq:matcomp}, we then use $U^\top V$ as an approximation to the complete matrix $A = [a_{i,j}]_{i,j =1, \ldots, n}.$

% If there exists an \emph{interpolating} solution to~\eqref{eq:matcomp}, that is one where 
% $$ u_i^T v_j  = a_{i,j}, \quad \mbox{for }(i,j) \in \Omega$$
% then the SP2 method can solve~\eqref{eq:matcomp}. Indeed, the SP2 can be applied to~\eqref{eq:matcomp} by sampling a single pair $(i,j) \in \Omega$ uniformly, then projecting onto the quadratic
% \begin{align}\label{eq:SP2matcom}
%     u_i^{k+1}, v_i^{k+1} &=\; \argmin_{u,v} \frac12 \norm{u-u_i^k}^2 + \frac12\norm{v-v_i^k}^2 \mbox{ subject to } u^\top v  = a_{i,j}.
% \end{align}
The projection~\eqref{eq:SP2matcom} can be solved as we shown in the following theorem.
\begin{theorem}\label{THM:matcom}
The solution to ~\eqref{eq:SP2matcom} is given by one of the following cases.
\begin{enumerate}
    \item If $(u_j^k)^\top v_j^k = a_{i,j}$ then $u= u^k_i$ and $v = v^k_j.$
    \item Alternatively if $u_j^k = v_j^k$ and 
    \[ (u_j^k)^\top v_j^k \geq 4 a_{i,j}\]
    then
    \begin{align}\label{eq:vsol2}
v &= \frac{1}{2}v_j^k  +\frac{1}{2}\frac{v_j^k}{\norm{v_j^k}} \sqrt{\norm{v_j^k}^2-4a_{i,j} }\\
u &= -\frac{1}{2}u_j^k  +\frac{1}{2}\frac{u_j^k}{\norm{u_j^k}} \sqrt{\norm{u_j^k}^2-4a_{i,j} }
\end{align} 
   \item Finally, if none of the above holds then 
\begin{align}
     u &= \frac{u_i^k-\gamma v_j^k}{1-\gamma^2} \\
     v &= \frac{v_j^k-\gamma u_i^k}{1-\gamma^2}, 
\end{align}
where $\gamma \in (-1, 1)$ and is the solution to the depressed quartic equation
\begin{align}
    (1+\gamma^2)\dotprod{u_i^k,v_j^k}- \gamma( \norm{u_i^k}^2+\norm{v_j^k}^2)  &= (1-\gamma^2)^2 a_{i,j}.
\end{align}
\end{enumerate}
\end{theorem}
\begin{proof}
% \rob{Actually this method is not exactly SP2. Instead, the constraint should be the local quadratic approximation to $(u^T v  - a_{i,j})^2$ around $(u_i^k,v_i^k).$ But this method~\eqref{eq:SP2matcom} just makes more sense if we know interpolation holds.
% }
The Lagrangian of~\eqref{eq:SP2matcom} is given by
$$L(u,v,\gamma) \; = \; \tfrac{1}{2}\norm{u-u_i^k}^2 + \tfrac{1}{2} \norm{v-v_j^k}^2 + \gamma(u^\top v  - a_{i,j}), $$
where $\gamma \in \R$ is the unknown Lagrangian Multiplier.
Thus the KKT equations are
\begin{align}
u- u_i^k + \gamma v &=0 \\
v- v_j^k + \gamma u &=0  \\
    u^\top v   &= a_{i,j} \label{eq:KKTMatComp}
\end{align}
Subtracting $\gamma$ times the the second equation from the first equation, and  analogously, subtracting the first equation from the second gives
\begin{align}\label{eq:temposoiunsin}
    (1-\gamma^2)u-u_i^k +\gamma v_j^k &=0\\
    \label{eq:temposoiunsin2}
(1-\gamma^2)v-v_j^k +\gamma u_i^k &=0.\end{align}
If $\gamma =1$, then necessarily $u^k_i = v_j^k$ and furthermore
from the  first equation in~\eqref{eq:KKTMatComp} we have that
\begin{equation}
    u = u_i^k-v = v_j^k -v.
\end{equation}
Substituting $u$ out in the original projection problem~\eqref{eq:SP2matcom} we have that 
% \begin{align}\label{eq:SP2matcomvonly}
%     \; \min_{v} \frac12 \norm{v}^2 + \frac12\norm{v-v_j^k}^2 \mbox{ subject to } (v_j^k )^\top v - \norm{v}^2= a_{i,j}.
% \end{align}
\begin{align}\label{eq:SP2matcomvonly}
    \; \min_{v} \norm{v}^2 -v^\top v_j^k \;\mbox{ subject to } \;v^\top v_j^k  - \norm{v}^2= a_{i,j}.
\end{align}
Consequently, for every $v$ that satisfies the constraint we have that the objective value is invariant and equal to $-a_{i,j}.$
Consequently there are infinite solutions. To find one such solution, we complete the squares of the constraint and find
\begin{eqnarray}
\norm{v-\tfrac{1}{2} v_j^k}^2& =& \frac{1}{4}\norm{v_j^k}^2-a_{i,j}.
\end{eqnarray}
The above only has solutions if $\frac{1}{4}\norm{v_j^k}^2-a_{i,j} \geq 0.$ One solution to the above is given by~\eqref{eq:vsol2}.

Alternatively, if $\gamma \neq 1$ then by isolating $u$ and $v$ in~\eqref{eq:temposoiunsin} and~\eqref{eq:temposoiunsin2}, respectively, gives
\begin{align}
     u &= \frac{u_i^k-\gamma v_j^k}{1-\gamma^2} \\
     v &= \frac{v_j^k-\gamma u_i^k}{1-\gamma^2} 
\end{align}
To figure out $\gamma$, we use the third constraint in~\eqref{eq:KKTMatComp} and the above two equations, which gives
\begin{align*}
    u^\top v &= \frac{(u_i^k-\gamma v_j^k)^\top}{1-\gamma^2}\frac{v_j^k-\gamma u_i^k}{1-\gamma^2}  \\
    &= \frac{(1+\gamma^2)\dotprod{u_i^k,v_j^k}- \gamma( \norm{u_i^k}^2+\norm{v_j^k}^2) }{(1-\gamma^2)^2}  \;=\; a_{i,j}.
\end{align*}
% Re-arranging we have that
% \[ (1-\gamma)^2\dotprod{u_i^k,v_j^k}-\gamma\norm{u_i^k -  v_j^k}^2 = (1-\gamma^2)^2 a_{i,j}. \]
Let
\[\phi(\gamma) = (1+\gamma^2)\dotprod{u_i^k,v_j^k}- \gamma( \norm{u_i^k}^2+\norm{v_j^k}^2)  - (1-\gamma^2)^2 a_{i,j}.\]
Can we now find an interval which will contain the solution in $\gamma$? Note that 
\begin{align*}
     \phi(-1) &= 2\dotprod{u_i^k,v_j^k} +\norm{u_i^k}^2+\norm{v_j^k}^2 =\norm{u_i^k+v_j^k}^2 \geq 0 \\
     \phi(1) &=  2\dotprod{u_i^k,v_j^k} - \norm{u_i^k}^2-\norm{v_j^k}^2 = -\norm{u_i^k-v_j^k}^2\leq 0.
\end{align*}
Thus it suffices to search for $\gamma \in (-1, \; 1),$ which can be done efficiently with bisection.

% Furthermore we have that
% \[\phi'(\gamma) = -2(1-\gamma)\dotprod{u_i^k,v_j^k}-\norm{u_i^k -  v_j^k}^2 + 4\gamma(1-\gamma^2) a_{i,j}.\]

\end{proof}

\section{Proofs of Important Lemmas}

\subsection{Proof of Lemma \ref{lma:GLM}}
\label{proof_lma:GLM}
\vspace{-1mm}
%\SL{To make the notations be consistent, I think it would be better to remove the subscript ``$i$'' on $a_i$ and $h_i$ and explain this. Or we could add subscript ``$i$'' on the other parameters like $f$ and $x$.}

Let us first describe the set of solutions for given constraint. We need to have
\begin{align}
    f_i + a_ix_i^T \Delta + \frac12 h_i \Delta^T x_i x_i^T \Delta = 0,  \label{eq:linearModelConstraint}
\end{align}
where $\Delta = w - w^t$ is unknown.
If we denote by $\tau_i = x_i^T \Delta$ then
\eqref{eq:linearModelConstraint}
will reduce to
\begin{align}
\label{eq:linModel1d}
    f_i + a_i\tau_i  + \frac12 h_i \tau_i^2  = 0. 
\end{align}
% The solution now depends on the sign of $h_i$. If $h_i \geq 0,$
This quadratic equation ~\eqref{eq:linModel1d} 
 has solution if
\begin{equation}\label{eq:conditiontemp}
a_i^2 - 2 h_i f_i \geq 0.
\end{equation}
If the condition above holds then we have that the solution for $\tau$ is in this set
\begin{align}
    T^* :=  
    \left\{ \frac{-a_i+ \sqrt{a_i^2 - 2 h_i f_i }  }
         {h_i},
         \frac{-a_i- \sqrt{a_i^2 - 2 h_i f_i }  }
         {h_i}\right\}.
\end{align}
Recall that the problem \eqref{eq:SP2proj} 
now reduces into
\begin{equation}\label{linearModelSolution}
    \min_\Delta \left\{ \|\Delta\|^2,\ \mbox{such that}\ x_i^T\Delta \in T^*\right\}.
\end{equation}
Note that because we want to minimize $\|\Delta\|^2$,
we want to choose the constraint with smallest possible absolute value, hence the problem
\eqref{linearModelSolution}
is equivalent to
\begin{equation}\label{linearModelSolution1}
    \min_\Delta \left\{ \|\Delta\|^2,\ \mbox{such that}\ x_i^T\Delta = \tau_i^* \right\},
\end{equation}
where 
$$
\tau_i^* = \begin{cases}
  \frac{-a_i+ \sqrt{a_i^2 - 2 h_i f_i }  }
         {h_i},& \mbox{if}\ a_i> 0\\
    \frac{-a_i- \sqrt{a_i^2 - 2 h_i f_i }  }
         {h_i},& \mbox{otherwise}.     
  \end{cases}
$$
In other words,
 \[\tau_i^* =  -\frac{a_i}{h_i}+\mbox{sign}(a_i)\frac{\sqrt{a_i^2 - 2 h_i f_i }  } {h_i} = \frac{a_i}{h_i} \left( \frac{\sqrt{a_i^2 - 2 h_i f_i }}{|a_i|} -1 \right)\]
The final solution is hence
$$
\Delta^* = \frac{\tau_i^*}{\|x_i\|^2} x_i
$$
and therefore
$$
w^* = w^t +  \frac{\tau_i^*}{\|x_i\|^2} x_i
$$

In case when \eqref{eq:conditiontemp}
is not satisfied, and because we assumed that the loss function is non-negative, we necessary have
$h_i > 0$.
Then natural choice of $\tau_i$ is the one that would minimize the
$$f_i+a_i \tau_i + \frac12 h_i \tau_i^2.$$
From first order optimality conditions we obtain that
$$   \tau_i^* = -\frac{a_i }{h_i}$$
which leads to \eqref{eq:newtonStep}.

\subsection{Proof of Lemma \ref{lem:phi2}}
\label{proof_lem:phi2}
\vspace{-1mm}

\begin{proof}
If $\phi(t)=0$ then the condition holds trivially. For $t$ such that $\phi(t) \neq 0$, $\sqrt{\phi(t)}$ is differentiable, and we have
\begin{align*}
\frac{d^2}{dt^2} \sqrt{\phi(t)} &= -\tfrac14\phi(t)^{-3/2}\phi'(t)^2 + \tfrac12 \phi(t)^{-1/2}\phi''(t)= \tfrac14 \phi(t)^{-3/2}(-\phi'(t)^2 + 2\phi(t)\phi''(t)),
\end{align*}
which is negative precisely when $\phi'(t)^2 \geq 2\phi(t)\phi''(t).$
\end{proof}

\subsection{Proof of Lemma~\ref{lem:SPS2plus} }
\vspace{-1mm}

Note that
\begin{align*}
  q(w^{t+1/2})  &\overset{\eqref{eq:polyakup}+\eqref{eq:quadeqzero2}} {=} f_i(w^t) - \dotprod{\nabla f_i(w^t), \frac{f_i(w^t)}{\norm{\nabla f_i(w^t)}^2}\nabla f_i(w^t)}  \\
&\quad\quad  + \frac{1}{2} \dotprod{\nabla^2 f_i(w^t) \frac{f_i(w^t)}{\norm{\nabla f_i(w^t)}^2}\nabla f_i(w^t), \frac{f_i(w^t)}{\norm{\nabla f_i(w^t)}^2}\nabla f_i(w^t)}  \\
 &\quad=  \frac{1}{2}\frac{f_i(w^t)^2}{\norm{\nabla f_i(w^t)}^4} \dotprod{\nabla^2 f_i(w^t) \nabla f_i(w^t), \nabla f_i(w^t)}. 
\end{align*}
Furthermore
\begin{eqnarray*}
\nabla q(w^{t+1/2}) &\overset{\eqref{eq:quadeqzero2}} {=} & 
\nabla f_i(w^t) + \nabla^2 f_i(w^t)(w^{t+1/2}-w^t) \\
&\overset{\eqref{eq:polyakup}} {=} & 
\left(\mI - \nabla^2 f_i(w^t)\frac{f_i(w^t)}{\norm{\nabla f_i(w^t)}^2}\right)\nabla f_i(w^t).
\end{eqnarray*}
Thus the second step~\eqref{eq:polyakup2}  is given by
\begin{multline}
w^{t+1} =  w^{t+1/2}  - \frac{1}{2}\frac{f_i(w^t)^2}{\norm{\nabla f_i(w^t)}^4}\frac{ \dotprod{\nabla^2 f_i(w^t) \nabla f_i(w^t), \nabla f_i(w^t)} }{\norm{\left(\mI - \nabla^2 f_i(w^t)\frac{f_i(w^t)}{\norm{\nabla f_i(w^t)}^2}\right)\nabla f_i(w^t)}^2} 
\cdot\left(\mI - \nabla^2 f_i(w^t)\frac{f_i(w^t)}{\norm{\nabla f_i(w^t)}^2}\right)\nabla f_i(w^t).\label{eq:polyakup2sub}
\end{multline}
Putting the first~\eqref{eq:polyakup} and second~\eqref{eq:polyakup2sub} updates together gives~\eqref{eq:SP2plus}.

%\begin{subequations}
%\begin{empheq}[box=\widefbox]{align}
%w^{t+1} 
%& =  w^t - \frac{f_i(w^t)}{\norm{\nabla f_i(w^t)}^2}\nabla f_i(w^t) \nonumber \\
%& \quad - \frac{1}{2}\frac{f_i(w^t)^2}{\norm{\nabla f_i(w^t)}^4}\frac{ \dotprod{\nabla^2 f_i(w^t) \nabla f_i(w^t), \nabla f_i(w^t)} }{\norm{\left(\mI - \nabla^2 f_i(w^t)\frac{f_i(w^t)}{\norm{\nabla f_i(w^t)}^2}\right)\nabla f_i(w^t)}^2}\left(\mI - \nabla^2 f_i(w^t)\frac{f_i(w^t)}{\norm{\nabla f_i(w^t)}^2}\right)\nabla f_i(w^t) 
%\end{empheq}\label{eq:SP2plus}
%\end{subequations}

This gives a second order correction of the Polyak step that only requires computing a single Hessian-vector product that can be done efficiently using an additional backwards pass of the function. We call this method \texttt{SP2}.

% \begin{eqnarray}
% w^{t+1/2} & = & \argmin_{w\in\R^d}\frac{1}{2} \norm{w-w^t}^2 \nonumber \\
% &&\mbox{ s.t. }f_i(w^t) + \dotprod{\nabla f_i(w^t), w-w^t}  =0. \label{eq:polyak2ns}
% \end{eqnarray}
% The closed form update of this first step is given by
% \begin{eqnarray}
% w^{t+1/2} 
% & = & w^t - \frac{f_i(w^t)}{\norm{\nabla f_i(w^t)}^2}\nabla f_i(w^t).
% \end{eqnarray}

\subsection{Convergence of multi-step \texttt{SP2}$^+$}
\label{sec:SP2multistep}
\vspace{-1mm}
If we apply multiple steps of the \texttt{SP2}$^+$, as opposed to two steps, the method converges to the solution of~\eqref{eq:quadeqzero2}. This follows because each step of \texttt{SP2}$^+$ is a step of NR Newton Raphson's method applied to solving the nonlinear equation
\[  q(w) \eqdef f_i(w^t) + \dotprod{\nabla f_i(w^t), w-w^t}  
+ \tfrac{1}{2} \dotprod{\nabla^2 f_i(w^t) (w-w^t), w-w^t}. \]

Indeed, starting from $w^0 = w^t$, the iterates of the \texttt{NR} (Newton Raphson) method are given by
\begin{eqnarray}\label{eq:NRforquadeq}
w^{i+1} = w^i -  \left(\nabla q(w^i)^\top\right)^{\dagger}  q(w^i) = w^i - \frac{q(w^i)}{\norm{\nabla q(w^i)}^2}\nabla q(w^i),
\end{eqnarray}
where $\mM^\dagger$ denotes the pseudo-inverse of the matrix $\mM.$

 The \texttt{NR} iterates in~\eqref{eq:NRforquadeq} can also be written in a variational form given by
\begin{align}\label{eq:NRforquadeqvar}
w^{i+1} = & \argmin_{w\in\R^d}%\tfrac{1}{2}
\norm{w-w^i}^2 \nonumber \\
&\mbox{ s.t. }q(w^i) + \nabla q(w^i)( w-w^i) =0. 
\end{align}
Comparing the above to the first~\eqref{eq:polyakproj} and second step~\eqref{eq:polyakproj2} are indeed two steps of the \texttt{NR} method. Further, we can see that~\eqref{eq:NRforquadeqvar} is indeed the multi-step version of \texttt{SP2}$^+$.

This method~\eqref{eq:NRforquadeq} is also known as gradient descent with a Polyak Step step size, or \texttt{SP} for short. It is this connection we will use to prove the convergence of~\eqref{eq:NRforquadeq} to a root of $q(w).$

% The equivalance follows from the fact that the pseudo-inverse solution~\eqref{eq:NRforquadeq} is the least
We assume that $q(w)$ has at least one root.
Let $w_q^* \in \R^d$ be a least norm root of $q(w),$ that is 
\begin{align}\label{eq:wqstar}
w_q^* &= \argmin \norm{w}^2 \quad \mbox{subject to }\; q(w) = 0.
\end{align}
It follows from Theorem 3.2 of~\cite{sosa2020algorithm} that the above optimization~\eqref{eq:wqstar} has solution if and only if the following matrix 
\begin{align}
B &= (\nabla f_i(w^t) - \nabla^2 f_i(w^t) w^t )(\nabla f_i(w^t) - \nabla^2 f_i(w^t) w^t )^\top +\nonumber \\ & \qquad  2\left( -f_i(w^t) 
+ \nabla f_i(w^t)^\top w^t - \frac 1 2 {w^t}^\top \nabla^2f_i(w^t) w^t   \right)\nabla^2 f_i(w^t) 
\label{eq:def_B}
\end{align}
has at least a
non-negative eigenvalue.  

\begin{theorem}\label{theo:qconvNR}
Assume that the matrix $B$ defined in~\eqref{eq:def_B} has at least a non-negative eigenvalue.
If $q(w)$ is star-convex with respect to $w_q^*$, that is if
\begin{eqnarray}\label{eq:qistarconvexhess}
(w^i -w_q^*)^\top\nabla^2 f_i(w^t)(w^i -w_q^*) \geq 0, \quad \mbox{for all }i,
\end{eqnarray}
then it follows that 
\begin{eqnarray}
\min_{i=0,\ldots, T-1} q(x^i) \; \leq \; \frac{ \sigma_{\max}(\nabla^2 f_i(w^t))}{2 T} \norm{w^0 -w_q^*}^2.
\end{eqnarray}
\end{theorem}
\begin{proof}
The proof follows by applying the convergence Theorem 4.4 in~\cite{SGDstruct} or equivalently Corollary D.3 in~\cite{TASPS}. This result first appeared in Theorem 4.4 in~\cite{SGDstruct}, but we apply Corollary D.3 in~\cite{TASPS} since it is a bit simpler. 

To apply this Corollary D.3 in~\cite{TASPS}, we need to verify that $q$ is an $L$--smooth function and star-convex. To verify if it is smooth, we need to find $L>0$ such that
\begin{eqnarray}
q(w) \; \leq \; q(y)+\dotprod{\nabla q(y), w-y}+ \frac{L}{2}\norm{ w-y}^2,
\end{eqnarray}
which holds with $L = \sigma_{\max}(\nabla^2 q(y)) = \sigma_{\max}(\nabla^2 f_i(w^t) $ since $q$ is a quadratic function. Furthermore, for $q$ to be star-convex along the iterates $w^i$, we need to verify if 
\begin{equation}\label{eq:lo9zo444s}
      q(w^*_q) \; \geq \; q(w^i) + \dotprod{\nabla q(w^i), w^* -w^i}. 
\end{equation}
Since $q$ is a quadratic, we have that 
\[ q(w^*_q) =q(w^i) + \dotprod{\nabla q(w^i), w^* -w^i} + \dotprod{\nabla^2 q(w^i)(w^*-w^i), w^*-w^i}. \]
Using this in~\eqref{eq:lo9zo444s} gives that
\begin{equation}\label{eq:lo9zo444s}
      0\; \geq \; \dotprod{\nabla^2 q(w^i)(w^*-w^i), w^*-w^i} = \dotprod{\nabla^2 f_i(w^t)(w^*-w^i), w^*-w^i} ,
\end{equation}
which is equivalent to our assumption~\eqref{eq:qistarconvexhess}. We can now apply the result in Corollary D.3 in~\cite{TASPS} which states that
\[ \min_{i=0,\ldots, T-1} (q(x^i)-q(w_q^*)) \; \leq \; \frac{ L}{2 T} \norm{w^0 -w_q^*}^2.\] 
Finally using  $q(w_q^*) =0$ and that $L =\sigma_{\max}(\nabla^2 f_i(w^t))$ gives the result.
\end{proof}

% \rob{The below is currently not used.}
To simplify notation, we will omit the dependency on $w^t$ and denote $c = f_i(w^t)$, $g=\nabla f_i(w^t)$ and $\mH = \nabla^2 f_i(w^t),$ thus
\begin{align}
    q(w) &= c + \dotprod{g, w-w^t } + \frac{1}{2}  \dotprod{\mH (w-w^t), w-w^t} \nonumber \\
    \nabla q(w) & = g +\mH (w-w^t)\nonumber  \\
    \nabla^2 q(w) &= \mH \label{eq:qsimple}
\end{align}

\begin{lemma}\label{lem:rangeqwis}
If $g \in \Range(\mH)$ and $w^0 \in \Range(\mH)$ then $w^i, \nabla q(w^i)\in \Range(\mH)$ for all $i$ and $w_q^*\in \Range(\mH)$.
\end{lemma}
\begin{proof}
First, note that since $g \in \Range(\mH)$ and since $\nabla q(w) = q + \mH(w-w^t)$ (see~\eqref{eq:qsimple}) we have that $\nabla q(w) \in \Range(\mH)$ for all $w$.
Consequently
by induction if  $w^i \in \Range(\mH)$ then by~\eqref{eq:NRforquadeq} we have that $w^{i+1} \in \Range(\mH)$ since it is a combination of $\nabla q(w^i)$ and $w^i$. 

Finally, let $w_q^* = w^t + w_{\mH} +  w_{\mH}^\perp$ where $w_{\mH}  \in \Range(\mH)$ and $w_{\mH}^\perp  \in \Range(\mH)^\perp.$ It follows that 
\[q(w_q^* ) = q( w^t + w_{\mH}).\]
Furthermore, by orthogonality and Pythagoras' Theorem
\[\norm{w_q^* } = \norm{w^t + w_{\mH}} + \norm{ w_{\mH}^\perp}\]
Consequently, since $w_q^*$ is the least norm solution, we must have that $w_{\mH}^\perp =0$ and thus $w_q^* \in \Range(\mH).$
\end{proof}

\subsection{Proof of Proposition~\ref{prop:convergence}}
\vspace{-1mm}
First we repeat the proposition for ease of reference.
\begin{proposition}\label{prop:convergence-ap}
Consider the loss functions given in~\eqref{eq:quadexe}. The \texttt{SP2} method converges~\eqref{eq:SP2proj} converges  according to \begin{eqnarray}\label{eq:convSP2-ap}
\E{\norm{w^{t+1}-w^*}^2} \leq \rho \; \E{\norm{w^{t}-w^*}^2},
\end{eqnarray}
where 
\begin{eqnarray}\label{eq:rhoSP2-ap}
\rho = \lambda_{\max}\left( \mI - \frac{1}{n}\sum_{i=1}^n\mH_i \mH_i^+\right) \; < \; 1.
\end{eqnarray}

\end{proposition}
\begin{proof}
First consider the first iterate of \texttt{SP2} which applied to~\eqref{eq:quadexe} are given by
\begin{align*}
   w^{t+1} =  & \min_{w\in\R^d} \norm{w-w^t}^2\\
   &\mbox{ s.t. } \norm{w-w^*}_{\mH_i}^2 =0.
\end{align*}
Thus every solution to the constraint set must satisfy
\begin{eqnarray}\label{eq:tempmliuzenr}
 w \in  w^* + \mN_i \alpha,
\end{eqnarray}
where $\mN_i\in\R^{d\times d}$ is a basis for the null space of $\mH_i.$
where $\alpha \in \R^d.$ Substituting into the objective we have the resulting linear least squares problem given by
\[
\min_{\alpha\in\R^d} \norm{w^* + \mN_i \alpha-w^t}^2 \]
The minimal norm solution in $\alpha$ is thus
\[\alpha  =\mN_i^+ (w^t -w^*)\]
which when substituted into~\eqref{eq:tempmliuzenr} gives
\begin{eqnarray}\label{eq:tempionorx}
w^{t+1} &=&w^* + \mN_i\mN_i^+ (w^t -w^*).
\end{eqnarray}
Note that $\mP_i \eqdef \mN_i \mN_i^+$ is the orthogonal projector onto $\Null(\mH_i).$ 
Subtracting $w^*$ from both sides of~\eqref{eq:tempionorx} and applying the squared norm we have that
\begin{eqnarray}
\norm{w^{t+1} -w^*}^2&=&  \norm{\mP_i (w^t -w^*)}^2  \nonumber \\
&=& \dotprod{\mP_i(w^t -w^*),(w^t -w^*)}\label{eq:tempionorx2}
\end{eqnarray}
where we used that $\mP_i \mP_i = \mP_i$ because it is a projection matrix. Now taking expectation conditioned on $w^t$ we have
\begin{eqnarray*}
\E{\norm{w^{t+1} -w^*}^2 \; | \; w^t}&=&  
 \dotprod{\E{\mP_i}(w^t -w^*),(w^t -w^*)} \\
& \leq & \lambda_{\max}(\E{\mP_i}) \norm{w^t -w^*}^2.
\end{eqnarray*}
Since the null space is orthogonal to the range of adjoint, we have that
\[\mP_i = \mI - \mH_i \mH_i^+.\]
Thus taking expectation again gives the result~\eqref{eq:convSP2-ap}.

Finally, the rate of convergence $\rho$ in~\eqref{eq:rhoSP2-ap} is always smaller than one because, due Jensen's inequality and that $\lambda_{\max}$ is convex over positive definite matrices we have that
\begin{equation}\label{eq:tempsonnns4}
    0 < \lambda_{\max}(\E{\mH_i\mH_i^*}) \leq \E{\lambda_{\max}(\mH_i\mH_i^*)} =1, 
\end{equation} 
where the greater than zero follows since there must exist $\mH_i \neq 0,$ otherwise the result still holds and the method converges in one step (with $\rho =0$). Now multiplying~\eqref{eq:tempsonnns4} by $-1$ then adding $1$ gives
\begin{equation}\label{eq:tempsonnns42}
    1 > \lambda_{\max}(\mI-\E{\mH_i\mH_i^*}) \geq 0 . 
\end{equation} 

\end{proof}

\subsection{Proof of Proposition~\ref{prop:convergenceSPplus}}
\vspace{-1mm}

For convenience we repeat the statement of the proposition here.

\begin{proposition}\label{prop:convergenceSPplus=ap}
Consider the loss functions given in~\eqref{eq:quadexe}. The \texttt{SP2}$^+$ method converges~\eqref{eq:SP2plus} converges  according to \begin{eqnarray}
\E{\norm{w^{t+1}-w^*}^2} \leq \rho_{SP2^+}^2\; \E{\norm{w^{t}-w^*}^2},
\end{eqnarray}
where 
\begin{eqnarray}\label{eq:rhoSP2+-ap}
\rho_{SP2+} = 1 - \frac{1}{2n} \sum_{i=1}^n\frac{\lambda_{\min}(\mH_i)}{\lambda_{\max}(\mH_i)}  
\end{eqnarray}  
\end{proposition}
\begin{proof}
% \rob{Note: try direct proof. It seems in this case, we may get the exact same result as SPS2!}
The proof follows simply by observing that for quadratic function the \texttt{SP2}$^+$ is equivalent to applying two steps of the \texttt{SP} method~\eqref{eq:SPSproj}. 
Indeed in Section~\ref{sec:SPS+} the \texttt{SP2}$^+$ applies two steps of the \texttt{SP} method to the local quadratic approximation of the function we wish to minimize. But in this case, since our function is quadratic, it is itself equal to it's local quadratic.

Consequently we can apply the convergence theory of \texttt{SP} for smooth, strongly convex functions that satisfy the interpolation condition, such as Corollary 5.7.I in~\cite{TASPS}, which states that  \texttt{SP}  converges at a rate of~\eqref{eq:rhoSP2+-ap}
\end{proof}

\subsection{Proof of Lemma~\ref{lem:quadslackproj}}
\label{proof_lma:quadslackproj}
\vspace{-1mm}

The following lemma gives the two step update for \texttt{SP2L2}$^+$. 

\begin{lemma}(\texttt{SP2L2}$^+$) \label{lem:quadslackproj}
The $w^{t+1}$ and $s^{t+1}$ update of~\eqref{eq:newtraph2ndslackhalf}--\eqref{eq:newtraph2ndslackhalfhalf}  is given by
\begin{align} 
w^{t+1} & = w^t - (\Gamma_1+\Gamma_2) \nabla f_i(w^t)  + \Gamma_2   \Gamma_1 \nabla^2 f_i(w^t) \nabla f_i(w^t),  \nonumber\\
s^{t+1} & = (1-\lambda) \left((1-\lambda)(s^t + \Gamma_1  ) + \Gamma_2   \right), \nonumber
\end{align}
where

\vspace{-1.3cm}
\begin{align*}
    ~~~~~~~~~\Gamma_1 &\eqdef \frac{(f_i(w^t)-(1-\lambda)s^t)_+}{ 1-\lambda + \norm{\nabla f_i(w^t)}^2},\\
%     \Gamma_2 &\eqdef  \tfrac{\left(f_i(w^t) - \Gamma_1  \norm{\nabla f_i(w^t)}^2 + \frac{1}{2} \Gamma_1^2 \left \langle \nabla^2 f_i(w^t) \nabla f_i(w^t) ,\nabla f_i(w^t)   \right \rangle  -(1-\lambda)^2(s^t + \Gamma_1  )\right)_+}{( 1-\lambda + \norm{\nabla f_i(w^t) - \Gamma_1 \nabla^2 f_i(w^t) \nabla f_i(w^t)}^2)}
% \\
 \Gamma_2 &\eqdef  
    \left(
    \tfrac{f_i(w^t) - \Gamma_1  \norm{\nabla f_i(w^t)}^2 
    -(1-\lambda)^2(s^t + \Gamma_1  )
      }
    {  1-\lambda + \norm{\nabla f_i(w^t) - \Gamma_1 \nabla^2 f_i(w^t) \nabla f_i(w^t)}^2  }
    + \tfrac{
      \frac{1}{2} \Gamma_1^2 \left \langle \nabla^2 f_i(w^t) \nabla f_i(w^t) ,\nabla f_i(w^t)   \right \rangle  }
    {  1-\lambda + \norm{\nabla f_i(w^t) - \Gamma_1 \nabla^2 f_i(w^t) \nabla f_i(w^t)}^2  }\right)_+,
\end{align*}
%\rob{Here's me guess:}
%\begin{align*}
%      \Gamma_2 &\eqdef  \tfrac{\left(f_i(w^t) - \Gamma_1  \norm{\nabla f_i(w^t)}^2 + \frac{1}{2} \Gamma_1^2 \left \langle \nabla^2 f_i(w^t) \nabla f_i(w^t) ,\nabla f_i(w^t)   \right \rangle  -(1-\lambda)^2(s^t + \Gamma_1  )\right)_+}{( 1-\lambda + \norm{\nabla f_i(w^t) - \Gamma_1 \nabla^2 f_i(w^t) \nabla f_i(w^t)}^2\bigr)}
%\end{align*}

\vspace{-0.5cm}
where we denote $(x)_+ =\begin{cases}  x & \mbox{ if } x \geq 0 \\ 0 & \mbox{otherwise} \end{cases}$. 
\end{lemma}

We will use the following lemma to prove Lemma~\ref{lem:quadslackproj}, which has been proven in Lemma C.2 of~\cite{polyakslack}.
\begin{lemma}[L2 Unidimensional Inequality Constraint] \label{lem:slackL2ineqconst}
Let $\delta >0, c\in \R$ and $w,w^0,a\in \R^d$ .
The closed form solution to 
\begin{align}
    \label{eq:slackL2ineqconstproj}
    w',s' =& \argmin_{w\in\R^d, s \in \R^b } \norm{w - w^0}^2 + \delta \norm{s-s^0}^2 \nonumber \\
    &\, \mbox{ s.t. } a^\top (w-w^0) +c \leq s \enspace,
\end{align}
is given by 
\begin{align} \label{eq:slackL2ineqconstprojsolw}
w' & =  w^0 - \delta \frac{(c-s^0)_+}{ 1 +\delta \norm{a}^2} a , \\
s' & = s^0+   \frac{(c-s^0)_+}{ 1 +\delta \norm{a}^2}, \label{eq:slackL2ineqconstprojsolb}
\end{align}
where we denote $(x)_+ =\begin{cases}  x & \mbox{ if } x \geq 0 \\ 0 & \mbox{otherwise} \end{cases} .$ 
\end{lemma}

% \begin{proof}
% The problem~\eqref{eq:slackL2ineqconstproj} is an L2 projection onto a halfspace. 
% The solution depends if the projected vector $(w,s) =(w^0,s^0)$ is in the halfspace.

% If $w = w^0$ and $s=s^0$ satisfies in the linear inequality constraint, that is if $c \leq s^0$, in which case the solution is simply $w' = w^0$ and $s' = s^0.$

% Else, $(w^0,s^0)$ is out of the feasible set, that is $c > s^0$, then we need to project $(w^0,s^0)$ onto the boundary of the halfspace, which means project onto 
% \[ \{ (w, s) \in \R^d \times \R^d \, | \, a^\top (w - w^0) + c = s \}  \enspace. \]
% In which case the solution is given in Lemma~\ref{lem:slackL2eqconst} (with $w^0 \leftarrow w^0$, $s^0 \leftarrow s^0$, $\delta \leftarrow \delta$, $\mA \leftarrow a$ and $c \leftarrow c$) in~\eqref{eq:slackL2eqconstprojsolw} and~\eqref{eq:slackL2eqconstprojsolb}.
% \SL{Rob, could you please fix the above "??"? Thanks!}

% \end{proof} 

We are now in the position to prove Lemma~\ref{lem:quadslackproj}.
Note that
\begin{align}
    \frac{1-\lambda}{2} (s-s^{t})^2  + \frac{\lambda}{2} s^2 & = 
    \frac{1-\lambda}{2} s^2 -(1-\lambda)s s^{t}  + \frac{\lambda}{2} s^2  
    +\frac{1-\lambda}{2}(s^t)^2 \nonumber \\
    &= \frac{1}{2} s^2  -(1-\lambda)s s^{t}  + \frac{1-\lambda}{2}(s^t)^2 \nonumber \\
    &= \frac{1}{2} (s -(1-\lambda)s^{t} )^2 
    +  \frac{\lambda-\lambda^2}{2}(s^t)^2.
\end{align}
Consequently \eqref{eq:newtraph2ndslackhalf} is equivalent to
\begin{eqnarray}
w^{t+1/2},s^{t+1/2} & = & \underset{s \geq 0, \; w\in\R^d}{\argmin}\norm{w-w^t}^2  +\frac{1}{1-\lambda} (s-(1-\lambda)s^t)^2  \nonumber \\
&&\negquad\mbox{ s.t. }q_{i,t}(w^t)  + \dotprod{\nabla q_{i,t}(w^t), w-w^t} 
\leq s. 
\end{eqnarray}
It follows from Lemma~\ref{lem:slackL2ineqconst} that the closed form solution is 
\begin{align} 
w^{t+1/2} & =  w^t - \frac{1}{1-\lambda} \frac{(q_{i,t}(w^t)-(1-\lambda)s^t)_+}{ 1 +\frac{1}{1-\lambda} \norm{\nabla q_{i,t}(w^t)}^2} \nabla q_{i,t}(w^t) , \\
s^{t+1/2} & = (1-\lambda)s^t+   \frac{(q_{i,t}(w^t)-(1-\lambda)s^t)_+}{ 1 +\frac{1}{1-\lambda} \norm{\nabla q_{i,t}(w^t)}^2}, 
\end{align}
where we denote 
\[(x)_+ =\begin{cases}  x & \mbox{ if } x \geq 0 \\ 0 & \mbox{otherwise} \end{cases}.\] Note that $q_{i,t}(w^t) = f_i(w^t)$ and $\nabla q_{i,t}(w^t) = \nabla f_i(w^t)$. 
To simplify the notation, we also denote 
\begin{align*}
    \Gamma_1 = \frac{1}{1-\lambda} \frac{(f_i(w^t)-(1-\lambda)s^t)_+}{ 1 +\frac{1}{1-\lambda} \norm{\nabla f_i(w^t)}^2}. 
\end{align*}
With this notation we have that
\begin{align} 
w^{t+1/2} & =  w^t - \frac{1}{1-\lambda} \frac{(f_i(w^t)-(1-\lambda)s^t)_+}{ 1 +\frac{1}{1-\lambda} \norm{\nabla f_i(w^t)}^2} \nabla f_i(w^t) \\
&= w^t - \Gamma_1 \nabla f_i(w^t) , \\
s^{t+1/2} & = (1-\lambda)s^t+   \frac{(f_i(w^t)-(1-\lambda)s^t)_+}{ 1 +\frac{1}{1-\lambda} \norm{\nabla f_i(w^t)}^2}\\
&= (1-\lambda)(s^t + \Gamma_1  ) .
\end{align}

%\SL{Should $\nabla q_{i,t+1/2}(w^{t+1/2})$ be $\nabla q_{i,t}(w^{t+1/2})$?}
In a completely analogous way, the closed form solution to~\eqref{eq:newtraph2ndslackhalfhalf} is
\begin{align} 
w^{t+1} & =  w^{t+1/2} - \frac{1}{1-\lambda} \frac{(q_{i,t}(w^{t+1/2})-(1-\lambda)s^{t+1/2})_+}{ 1 +\frac{1}{1-\lambda} \norm{\nabla q_{i,t}(w^{t+1/2})}^2} \cdot\nabla q_{i,t}(w^{t+1/2}) , \nonumber \\
s^{t+1} & = (1-\lambda)s^{t+1/2}   +   \frac{(q_{i,t}(w^{t+1/2})-(1-\lambda)s^{t+1/2})_+}{ 1 +\frac{1}{1-\lambda} \norm{\nabla q_{i,t}(w^{t+1/2})}^2}. 
\end{align}
Note that 
\begin{align*}
    q_{i,t}(w^{t+1/2}) & = f_i(w^t) - \left \langle \nabla f_i(w^t), \Gamma_1 \nabla f_i(w^t) \right\rangle + \frac 1 2 \left \langle \nabla^2 f_i(w^t) \Gamma_1 \nabla f_i(w^t), \Gamma_1 \nabla f_i(w^t)    \right \rangle\\
    & = f_i(w^t) - \Gamma_1  \norm{\nabla f_i(w^t)}^2  + \frac 1 2 \Gamma_1^2 \left \langle \nabla^2 f_i(w^t) \nabla f_i(w^t) ,\nabla f_i(w^t)   \right \rangle\\
\end{align*}
and 
\begin{align*}
    \nabla q_{i,t}(w^{t+1/2}) & = \nabla f_i(w^t) + \nabla^2 f_i(w^t)(w^{t+1/2} - w^t  )\\
    & = \nabla f_i(w^t) - \Gamma_1 \nabla^2 f_i(w^t) \nabla f_i(w^t).
\end{align*}
%Furthermore, we denote 
%\begin{align*}
%    \Gamma_2 &= \frac{1}{1-\lambda} \frac{(q_{i,t}(w^{t+1/2})-(1-\lambda)s^{t+1/2})_+}{ 1 +\frac{1}{1-\lambda} \norm{\nabla q_{i,t}(w^{t+1/2})}^2} \\
%    & = \frac{1}{1-\lambda} \frac{(f_i(w^t) - \Gamma_1  \norm{\nabla f_i(w^t)}^2 + \frac 1 2 \Gamma_1^2 \left \langle \nabla^2 f_i(w^t) \nabla f_i(w^t) ,\nabla f_i(w^t)   \right \rangle-(1-\lambda)^2(s^t + \Gamma_1  ))_+}{ 1 +\frac{1}{1-\lambda} \norm{\nabla f_i(w^t) - \Gamma_1 \nabla^2 f_i(w^t) \nabla f_i(w^t)}^2}.
%\end{align*}
Denoting $\Gamma_2$ as in the statement of the lemma we conclude that
\begin{align} 
w^{t+1} & =  w^{t+1/2} - \frac{1}{1-\lambda} \frac{(q_{i,t}(w^{t+1/2})-(1-\lambda)s^{t+1/2})_+}{ 1 +\frac{1}{1-\lambda} \norm{\nabla q_{i,t}(w^{t+1/2})}^2} \cdot \nabla q_{i,t}(w^{t+1/2}) \nonumber \\
&= w^t - \Gamma_1 \nabla f_i(w^t) - \Gamma_2  \left(\nabla f_i(w^t) - \Gamma_1 \nabla^2 f_i(w^t) \nabla f_i(w^t) \right) , \\
s^{t+1} & = (1-\lambda)s^{t+1/2} +   \frac{(q_{i,t}(w^{t+1/2})-(1-\lambda)s^{t+1/2})_+}{ 1 +\frac{1}{1-\lambda} \norm{\nabla q_{i,t}(w^{t+1/2})}^2}\nonumber\\
& = (1-\lambda) \left((1-\lambda)(s^t + \Gamma_1  ) + \Gamma_2   \right)
\end{align}

\subsection{Proof of Lemma~\ref{lem:quadslackproj_l1}}
\label{proof_lma:quadslackproj_l1}
\vspace{-1mm}

The following Lemma gives a closed form for the two-step update for  \texttt{SP2L1}$^+$.
\begin{lemma}(\texttt{SP2L1}$^+$) \label{lem:quadslackproj_l1}
The $w^{t+1}$ and $s^{t+1}$ update
%of~\eqref{eq:newtraph2ndslackhalf_l1}--\eqref{eq:newtraph2ndslackhalfhalf_l1}  
is given by
\begin{align*}
    w^{t+1} & =  w^t - (\Gamma_4+\Gamma_6) \nabla f_i(w^t) + \Gamma_6 \Gamma_4\nabla^2 f_i(w^t) \nabla f_i(w^t),   \\
    s^{t+1} 
    &= \left(\!\! \left( \!\! s^t - \frac{\lambda}{2(1-\lambda)}\! +  \Gamma_3   \!   \right)_+ \!\!- \frac{\lambda}{2(1-\lambda)} +  \Gamma_5   \!   \right)_+ \!,
\end{align*}
where 

\vspace{-1.2cm}
\begin{align*}
  \Gamma_3 &=   \frac{\left( f_i(w^t)- \left(s^t - \frac{\lambda}{2(1-\lambda)} \right)  \right)_+ }{1+\| \nabla f_i(w^t) \|^2  },\quad
  \Gamma_4 = \min\left\{  \Gamma_3, \frac{f_i(w^t)}{\| \nabla f_i(w^t) \|^2  }     \right\},\\
  \Gamma_5 &= \frac{\left( \Lambda_1- \left(s^t - \frac{\lambda}{2(1-\lambda)} \right)  \right)_+ }{1+\| \nabla f_i(w^t) - \Gamma_4\nabla^2 f_i(w^t) \nabla f_i(w^t) \|^2  },\\
  \Gamma_6 &= \min\left\{  \Gamma_5, \frac{\Lambda_1}{\| \nabla f_i(w^t) - \Gamma_4\nabla^2 f_i(w^t) \nabla f_i(w^t) \|^2  }\right\},\\
  \Lambda_1 &=  f_i(w^t )  -  \Gamma_4  \norm{\nabla f_i(w^t )}^2     +  \frac 1 2 \Gamma_4^2   \left \langle  \nabla^2 f_i(w^t ) \nabla  f_i(w^t ) ,  \nabla  f_i(w^t )    \right \rangle .
\end{align*}
\end{lemma}

To solve~\eqref{eq:newtraph2ndslack_l1}, we consider the following two-step method similar to~\eqref{eq:newtraph2ndslackhalf} and~\eqref{eq:newtraph2ndslackhalfhalf}:
\begin{align}
w^{t+1/2},s^{t+1/2} & =  \underset{s \geq 0, \; w\in\R^d}{\argmin}\frac{1-\lambda}{2}\Delta_{t} + \frac{\lambda}{2} s\label{eq:newtraph2ndslackhalf_l1} \\
&\mbox{ s.t. }q_{i,t}(w^t) +\dotprod{\nabla q_{i,t}(w^t), w-w^t} 
\leq s. \nonumber
\end{align}
\begin{align}
w^{t+1}, s^{t+1} =  &\underset{s \geq 0, \; w\in\R^d}{\argmin}\frac{1-\lambda}{2}\Delta_{t+\frac{1}{2}} + \frac{\lambda}{2} s\label{eq:newtraph2ndslackhalfhalf_l1} \\
&\quad\mbox{s.t. }q_{i,t}(w^{t+1/2})+ \dotprod{\nabla q_{i,t}(w^{t+1/2}), w-w^{t+1/2}} 
\leq s. \nonumber
\end{align}

Note that 
\begin{align*}
    \frac{1-\lambda}{2} (s-s^t)^2 + \frac{\lambda}{2} s = \frac{1-\lambda}{2} \left( s - \left( s^t - \frac{\lambda}{2(1-\lambda)}   \right) \right)^2 + \text{constants w.r.t. }w\text{ and }s.
\end{align*}
Then, \eqref{eq:newtraph2ndslackhalf_l1} is equivalent to solving 
\begin{align}
w^{t+1/2},s^{t+1/2} & =  \underset{s \geq 0, \; w\in\R^d}{\argmin} \|w-w^t\|^2 +  \left( s - \left( s^t - \frac{\lambda}{2(1-\lambda)}   \right) \right)^2  \label{eq:newtraph2ndslackhalf_l1_prf} \\
&\mbox{ s.t. }q_{i,t}(w^t) +\dotprod{\nabla q_{i,t}(w^t), w-w^t} 
\leq s. \nonumber
\end{align}
It follows from Lemma C.4 of \textcolor{red}{[cite]} that the closed form solution to~\eqref{eq:newtraph2ndslackhalf_l1} is  
\begin{align*}
    w^{t+1/2} & = w^t - \min\left\{  \frac{\left( q_{i,t}(w^t)- \left(s^t - \frac{\lambda}{2(1-\lambda)} \right)  \right)_+ }{1+\| \nabla q_{i,t}(w^t) \|^2  }, \frac{q_{i,t}(w^t)}{\| \nabla q_{i,t}(w^t) \|^2  }     \right\} \nabla q_{i,t}(w^t),   \\
    s^{t+1/2} & = \left( \left(s^t - \frac{\lambda}{2(1-\lambda)} \right) +  \frac{\left( q_{i,t}(w^t)- \left(s^t - \frac{\lambda}{2(1-\lambda)} \right)  \right)_+ }{1+\| \nabla q_{i,t}(w^t) \|^2  }      \right)_+,
\end{align*}
where we denote $(x)_+ =\begin{cases}  x & \mbox{ if } x \geq 0 \\ 0 & \mbox{otherwise} \end{cases}$. 

Note that $q_{i,t}(w^t) = f_i(w^t)$ and $\nabla q_{i,t}(w^t) = \nabla f_i(w^t)$. To simplify the notation, denote
\begin{align*}
  \Gamma_3 =   \frac{\left( f_i(w^t)- \left(s^t - \frac{\lambda}{2(1-\lambda)} \right)  \right)_+ }{1+\| \nabla f_i(w^t) \|^2  },
\end{align*}
and 
\begin{align*}
    \Gamma_4 = \min\left\{  \Gamma_3, \frac{f_i(w^t)}{\| \nabla f_i(w^t) \|^2  }     \right\}.
\end{align*}
Then, we have
\begin{align*}
    w^{t+1/2} & = w^t - \Gamma_4 \nabla f_i(w^t),   \\
    s^{t+1/2} & = \left( \left(s^t - \frac{\lambda}{2(1-\lambda)} \right) +  \Gamma_3      \right)_+.
\end{align*}

In a similar way, we can get the closed form solution to~\eqref{eq:newtraph2ndslackhalfhalf_l1}, which is given as
\begin{align*}
    w^{t+1} & = w^{t+1/2} - \min\left\{  \frac{\left( q_{i,t}(w^{t+1/2})- \left(s^{t+1/2} - \frac{\lambda}{2(1-\lambda)} \right)  \right)_+ }{1+\| \nabla q_{i,t}(w^{t+1/2}) \|^2  }, \frac{q_{i,t}(w^{t+1/2})}{\| \nabla q_{i,t}(w^{t+1/2}) \|^2  }     \right\} \nabla q_{i,t}(w^{t+1/2}),   \\
    s^{t+1} & = \left( \left(s^{t+1/2} - \frac{\lambda}{2(1-\lambda)} \right) +  \frac{\left( q_{i,t}(w^{t+1/2})- \left(s^t - \frac{\lambda}{2(1-\lambda)} \right)  \right)_+ }{1+\| \nabla q_{i,t}(w^{t+1/2}) \|^2  }      \right)_+.
\end{align*}
Note that 
\begin{align*}
    q_{i,t}(w^{t+1/2}) & = f_i(w^t) - \left \langle \nabla f_i(w^t), \Gamma_4 \nabla f_i(w^t) \right\rangle + \frac 1 2 \left \langle \nabla^2 f_i(w^t) \Gamma_4 \nabla f_i(w^t), \Gamma_4 \nabla f_i(w^t)    \right \rangle\\
    & = f_i(w^t) - \Gamma_4 \norm{\nabla f_i(w^t)}^2  + \frac 1 2 \Gamma_4^2 \left \langle \nabla^2 f_i(w^t) \nabla f_i(w^t) ,\nabla f_i(w^t)   \right \rangle\\
    & \triangleq \Lambda_1
\end{align*}
and 
\begin{align*}
    \nabla q_{i,t}(w^{t+1/2}) & = \nabla f_i(w^t) + \nabla^2 f_i(w^t)(w^{t+1/2} - w^t  )\\
    & = \nabla f_i(w^t) - \Gamma_4\nabla^2 f_i(w^t) \nabla f_i(w^t).
\end{align*}
Again, to simplify the notation, we denote
\begin{align*}
  \Gamma_5 &=   \frac{\left( q_{i,t}(w^{t+1/2})- \left(s^t - \frac{\lambda}{2(1-\lambda)} \right)  \right)_+ }{1+\| \nabla q_{i,t}(w^{t+1/2}) \|^2  } \\
  & = \frac{\left( \Lambda_1- \left(s^t - \frac{\lambda}{2(1-\lambda)} \right)  \right)_+ }{1+\| \nabla f_i(w^t) - \Gamma_4\nabla^2 f_i(w^t) \nabla f_i(w^t) \|^2  },
\end{align*}
and
\begin{align*}
    \Gamma_6 &= \min\left\{  \Gamma_5, \frac{q_{i,t}(w^{t+1/2})}{\| \nabla q_{i,t}(w^{t+1/2}) \|^2  }     \right\}\\
    &=\min\left\{  \Gamma_5, \frac{\Lambda_1}{\| \nabla f_i(w^t) - \Gamma_4\nabla^2 f_i(w^t) \nabla f_i(w^t) \|^2  }\right\}.
\end{align*}
Then, we have
\begin{align*}
    w^{t+1} & = w^{t+1/2} - \Gamma_6 \nabla q_{i,t}(w^{t+1/2})\\
    &=w^t - \Gamma_4 \nabla f_i(w^t) - \Gamma_6 \left(\nabla f_i(w^t) - \Gamma_4\nabla^2 f_i(w^t) \nabla f_i(w^t) \right)\\
    &= w^t - (\Gamma_4+\Gamma_6) \nabla f_i(w^t) + \Gamma_6 \Gamma_4\nabla^2 f_i(w^t) \nabla f_i(w^t),   \\
    s^{t+1} & = \left( \left(s^{t+1/2} - \frac{\lambda}{2(1-\lambda)} \right) +  \Gamma_5      \right)_+\\
    &= \left( \left(\left( \left(s^t - \frac{\lambda}{2(1-\lambda)} \right) +  \Gamma_3      \right)_+ - \frac{\lambda}{2(1-\lambda)} \right) +  \Gamma_5      \right)_+ .
\end{align*}

\subsection{Proof of Lemma~\ref{lem:quadslackproj_l1_unreg}}

The following lemma gives a closed form for two step method \texttt{SP2max}$^+$.

\begin{lemma}(\texttt{SP2max}$^+$) \label{lem:quadslackproj_l1_unreg}
The $w^{t+1}$ and $s^{t+1}$ update 
%of~\eqref{eq:newtraph2ndslackhalf_l1_unreg}--\eqref{eq:newtraph2ndslackhalfhalf_l1_unreg}  
is given by
\begin{align*}
    w^{t+1} &= w^t - \left(\Gamma_1 + \Gamma_3\right) \nabla f_i(w^t)+ \Gamma_3 \Gamma_1\nabla^2 f_i(w^t) \nabla f_i(w^t),\\
    s^{t+1} &= \max\left\{\Gamma_2  - \frac{\lambda}{2(1-\lambda)} \norm{\nabla f_i(w^t) - \Gamma_1\nabla^2 f_i(w^t) \nabla f_i(w^t)}^2,0  \right\},
\end{align*}
where 

\vspace{-1.2cm}
\begin{align*}
\Gamma_1 &= \min \left\{ \frac{f_i(w^t)}{\norm{\nabla f_i(w^t)}^2}, \frac{\lambda}{2(1-\lambda)}  \right\},\\
\Gamma_2 & = f_i(w^t) -\Gamma_1 \norm{\nabla f_i(w^t)}^2  
+ \tfrac{1}{2} \Gamma_1^2 \dotprod{\nabla^2 f_i(w^t)  \nabla f_i(w^t),  \nabla f_i(w^t)},\\
    \Gamma_3 &= \min \left\{ \frac{\Gamma_2}{\norm{\nabla f_i(w^t) - \Gamma_1\nabla^2 f_i(w^t) \nabla f_i(w^t)}^2}, \frac{\lambda}{2(1-\lambda)}  \right\}.
\end{align*}

\end{lemma}

\label{proof_lma:quadslackproj_l1_unreg}
\vspace{-1mm}
To solve~\eqref{eq:newtraph2ndslack_l1_unreg}, we again consider a two step method similar to~\eqref{eq:newtraph2ndslackhalf_l1} and~\eqref{eq:newtraph2ndslackhalfhalf_l1}: 
\begin{align}
w^{t+1/2},s^{t+1/2} & =  \underset{s \geq 0, \; w\in\R^d}{\argmin}\frac{1-\lambda}{2}\norm{w-w^t}^2+ \frac{\lambda}{2} s\label{eq:newtraph2ndslackhalf_l1_unreg} \\
&\mbox{ s.t. }q_{i,t}(w^t) +\dotprod{\nabla q_{i,t}(w^t), w-w^t} 
\leq s. \nonumber
\end{align}
\begin{align}
&w^{t+1}, s^{t+1} =  \underset{s \geq 0, \; w\in\R^d}{\argmin}\frac{1-\lambda}{2}\norm{w-w^{t+1/2}}^2+ \frac{\lambda}{2} s\label{eq:newtraph2ndslackhalfhalf_l1_unreg} \\
&\quad\mbox{s.t. }q_{i,t}(w^{t+1/2})+ \dotprod{\nabla q_{i,t}(w^{t+1/2}), w-w^{t+1/2}} 
\leq s. \nonumber
\end{align}

Note that \eqref{eq:newtraph2ndslackhalf_l1_unreg} is equivalent to solving 
\begin{align}
w^{t+1/2},s^{t+1/2} & =  \underset{s \geq 0, \; w\in\R^d}{\argmin}\frac{1}{2}\norm{w-w^t}^2+ \frac{\lambda}{2(1-\lambda)} s\label{eq:newtraph2ndslackhalf_l1_unreg_prf} \\
&\mbox{ s.t. }q_{i,t}(w^t) +\dotprod{\nabla q_{i,t}(w^t), w-w^t} 
\leq s. \nonumber
\end{align}
It follows from Lemma 3.1 in \textcolor{red}{[cite]} that the closed form solution to~\eqref{eq:newtraph2ndslackhalf_l1_unreg_prf} is 
\begin{align*}
    w^{t+1/2} & = w^t - \min \left\{ \frac{q_{i,t}(w^t)}{\norm{\nabla q_{i,t}(w^t)}^2}, \frac{\lambda}{2(1-\lambda)}  \right\} \nabla q_{i,t}(w^t) \\
    &= w^t - \min \left\{ \frac{f_i(w^t)}{\norm{\nabla f_i(w^t)}^2}, \frac{\lambda}{2(1-\lambda)}  \right\} \nabla f_i(w^t)\\
    &= w_t - \Sigma_1 \nabla f_i(w^t),\\
    s^{t+1/2} & = \max\left\{ q_{i,t}(w^t)  - \frac{\lambda}{2(1-\lambda)} \norm{\nabla q_{i,t}(w^t)}^2,0  \right\} \\
    &= \max\left\{ f_i(w^t)  - \frac{\lambda}{2(1-\lambda)} \norm{\nabla f_i(w^t)}^2,0  \right\},
\end{align*}
where we denote 
\begin{align*}
    \Gamma_1 = \min \left\{ \frac{f_i(w^t)}{\norm{\nabla f_i(w^t)}^2}, \frac{\lambda}{2(1-\lambda)}  \right\}. 
\end{align*}

Note that 
\begin{align*}
    q_{i,t}(w^{t+1/2}) & = f_i(w^t) -\Gamma_1 \norm{\nabla f_i(w^t)}^2  
+ \tfrac{1}{2} \Gamma_1^2 \dotprod{\nabla^2 f_i(w^t)  \nabla f_i(w^t),  \nabla f_i(w^t)}\\
&\eqdef \Gamma_2,
\end{align*}
and 
\begin{align*}
    q_{i,t}(w^{t+1/2}) & = \nabla f_i(w^t) + \nabla^2 f_i(w^t)(w^{t+1/2} - w^t  )\\
    & = \nabla f_i(w^t) - \Gamma_1\nabla^2 f_i(w^t) \nabla f_i(w^t).
\end{align*}

Similarly, we have the closed form solution to~\eqref{eq:newtraph2ndslackhalfhalf_l1_unreg} given as
\begin{align*}
    w^{t+1} & = w^{t+1/2} - \min \left\{ \frac{q_{i,t}(w^{t+1/2})}{\norm{\nabla q_{i,t}(w^{t+1/2})}^2}, \frac{\lambda}{2(1-\lambda)}  \right\} \nabla q_{i,t}(w^{t+1/2}) \\
    &= w^{t+1/2} - \min \left\{ \frac{\Gamma_2}{\norm{\nabla f_i(w^t) - \Gamma_1\nabla^2 f_i(w^t) \nabla f_i(w^t)}^2}, \frac{\lambda}{2(1-\lambda)}  \right\} \left( \nabla f_i(w^t) - \Gamma_1\nabla^2 f_i(w^t) \nabla f_i(w^t) \right)\\
    & = w^{t+1/2} - \Gamma_3 \left( \nabla f_i(w^t) - \Gamma_1\nabla^2 f_i(w^t) \nabla f_i(w^t) \right)\\
    &= w^t - \left(\Gamma_1 + \Gamma_3\right) \nabla f_i(w^t)+ \Gamma_3 \Gamma_1\nabla^2 f_i(w^t) \nabla f_i(w^t)\\
    s^{t+1} & = \max\left\{ q_{i,t}(w^{t+1/2})  - \frac{\lambda}{2(1-\lambda)} \norm{\nabla q_{i,t}(w^{t+1/2})}^2,0  \right\} \\
    &= \max\left\{\Gamma_2  - \frac{\lambda}{2(1-\lambda)} \norm{\nabla f_i(w^t) - \Gamma_1\nabla^2 f_i(w^t) \nabla f_i(w^t)}^2,0  \right\}   .
\end{align*}

\subsection{Proof of Lemma~\ref{lma:GLM_L1_unreg}}
\label{proof_lma:GLM_L1_unreg}
\vspace{-1mm}

In GLMs, the unregularized problem~\eqref{eq:newtraph2ndslack_l1_unreg} becomes
\begin{align}
w^{t+1} , s^{t+1}  =  &\underset{s \geq 0, \; w\in\R^d}{\argmin}\frac{1}{2}\norm{w-w^t}^2
 + \widetilde{\lambda} s\nonumber \\
&\mbox{ s.t. }f_i + \dotprod{a_i x_i, w-w^t}  
+ \tfrac{1}{2} \dotprod{h_i x_i x_i^\top (w-w^t), w-w^t}
\leq s, \label{eq:newtraph2ndslack_l1_unreg_glm}
\end{align}
where $\widetilde{\lambda} \eqdef \frac{\lambda}{2(1-\lambda)}$, and we denote $f_i=f_i(w^t)$, $a_i  \eqdef \phi_i'(x_i^\top w - y_i)$, $h_i  \eqdef \phi_i''(x_i^\top w - y_i)$ for short.

Denote $\triangle  \eqdef w-w^t$. Then, problem \eqref{eq:newtraph2ndslack_l1_unreg_glm} reduces to
\begin{align}
&\underset{s \geq 0, \; \triangle\in\R^d}{\min}\frac{1}{2}\norm{\triangle}^2
 + \widetilde{\lambda} s\nonumber \\
&~~~~~\mbox{ s.t. }f_i + a_i x_i^\top \triangle  
+ \tfrac{1}{2} h_i \triangle^\top x_i x_i^\top \triangle
\leq s. \label{eq:newtraph2ndslack_l1_unreg_glm_tri}
\end{align}
Note that we want to minimize $\|\triangle\|^2$. Together with the above constraint, we can conclude that $\triangle$ must be a multiple of $x_i$ since any other component will not help satisfy the constraint but increase $\|\triangle\|^2$. Let $\triangle = c x_i$ and $\ell = \|x_i\|^2$, then problem~\eqref{eq:newtraph2ndslack_l1_unreg_glm_tri} becomes 
\begin{align}
&\underset{s \geq 0, \; c\in\R}{\min}\frac{1}{2}c^2 \ell
 + \widetilde{\lambda} s\nonumber \\
&~~~~~\mbox{ s.t. }f_i + a_i  \ell c  
+ \tfrac{1}{2} h_i  \ell^2 c^2
\leq s. \label{eq:newtraph2ndslack_l1_unreg_glm_tri_c}
\end{align}
The corresponding Lagrangian function is then given as
%\footnote{Note that we ignore the $s\geq 0$ constraint here for simplicity.}
\begin{align*}
    L(s,c, \nu_1, \nu_2) = \frac{1}{2}c^2 \ell
 + \widetilde{\lambda} s + \nu_1(f_i + a_i  \ell c  
+ \tfrac{1}{2} h_i  \ell^2 c^2 -s ) - \nu_2 s,
\end{align*}
where $\nu_1,\nu_2 \geq 0$ are the Lagrangian multipliers.
The KKT conditions are thus
\begin{align*}
    f_i + a_i  \ell c  
+ \tfrac{1}{2} h_i  \ell^2 c^2 -s  &\leq 0, \\
s\geq 0,~~ \nu_1 \geq 0,~~ \nu_2 &\geq 0,\\
\nu_1  (f_i + a_i  \ell c  
+ \tfrac{1}{2} h_i  \ell^2 c^2 -s ) = 0, ~~\nu_2 s &= 0 ,\\
\widetilde{\lambda} - \nu_1 - \nu_2 &= 0,\\
\ell c + \nu_1 a_i \ell + \nu_1 h_i \ell^2 c &= 0. 
\end{align*}

By checking the complementary conditions, the solution to the above KKT equations has three cases, which are summarized below.
\begin{enumerate}
    \item[Case I:]  The Lagragian multiplier $\nu_2 =0$. In which case $\nu_1^\star = \widetilde{\lambda}$, $\nu_2^\star = 0$, $c^\star = -\frac{\widetilde{\lambda} a_i }{1+ \widetilde{\lambda} h_i \ell  }$,  and 
   $$s^\star = f_i + a_i \ell c^\star + \frac 1 2 h_i \ell^2 {c^\star}^2 = f_i  -\frac{\widetilde{\lambda} a_i^2 \ell}{1+ \widetilde{\lambda} h_i \ell  }  +  \frac{ h_i\widetilde{\lambda}^2 a_i^2 \ell^2}{2(1+ \widetilde{\lambda} h_i \ell )^2 },$$ 
%$s^\star = f_i + a_i \ell c^\star + \frac 1 2 h_i \ell^2 {c^\star}^2$, 
    which is feasible if $s^\star \geq 0$. The resulting objective function is $\frac 1 2 {c^\star}^2 \ell + \widetilde{\lambda}s^\star$, which is $\geq 0$.  

\item[Case II:] The Lagragian multiplier $\nu_1 =0$. In which case $\nu_1^\star = 0$, $\nu_2^\star = \widetilde{\lambda}$, $c^\star = 0$, $s^\star = 0$, which is feasible if $f_i = 0$. The objective function is $0$ in this case and the variable $w$ is unchanged since $w -w^t= c^\star x_i = 0$.

\item[Case III:] Neither Lagragian multiplier is zero. In which case there are two possible solutions for $c$ given by  $c^\star = \frac{-a_i  \pm \sqrt{a_i^2  - 2h_i  f_i}  }{h_i\ell }$, $\nu_1^\star = -\frac{ c}{a_i + h_i \ell c}$, $\nu_2^\star = \widetilde{\lambda} + \frac{ c}{a_i  + h_i \ell c}$, $s^\star =0$. Note that 
\[a_i +h_i\ell c = \pm\sqrt{a_i^2-2h_if_i}.\]
Consequently to guarantee that the Lagrangian multipliers $\nu_1$ and $\nu_2$ are non-negative, we must have
$c^\star = \frac{-a_i + \sqrt{a_i^2  - 2h_i  f_i}  }{h_i\ell }$
and in this case the objective function equals $\frac 1 2 {c^\star}^2 \ell \geq 0$.
\end{enumerate}
As a summary, if $f_i = 0$, Case II is the optimal solution. Alternatively if $f_i>0$ and
if $$\widetilde{s} = f_i  -\frac{\widetilde{\lambda} a_i^2 \ell}{1+ \widetilde{\lambda} h_i \ell  }  +  \frac{ h_i\widetilde{\lambda}^2 a_i^2 \ell^2}{2(1+ \widetilde{\lambda} h_i \ell )^2 }$$
is non-negative 
then Case I is the optimal solution. Otherwise, Case III with $c^\star = \frac{-a_i + \sqrt{a_i^2  - 2h_i  f_i}  }{h_i \ell}$ is the optimal solution.

% Upon inspection, the above KKT equations has only one solution given by
% \begin{align*}
%     \nu^\star &= \widetilde{\lambda},\\
%     c^\star &= -\frac{\widetilde{\lambda} a_i }{1+ \widetilde{\lambda} h_i \ell  },\\
%     s^\star & = f_i  -\frac{\widetilde{\lambda} a_i^2 \ell}{1+ \widetilde{\lambda} h_i \ell  }  
% +  \frac{ h_i\widetilde{\lambda}^2 a_i^2 \ell^2}{2(1+ \widetilde{\lambda} h_i \ell )^2 }.
% \end{align*}
Therefore, the optimal solution to~\eqref{eq:newtraph2ndslack_l1_unreg_glm} is then  
% \begin{align*}
%     w^{t+1} &= w^t -\frac{\widetilde{\lambda} a_i }{1+ \widetilde{\lambda} h_i \ell } x_i = w^t -\frac{\lambda a_i }{2(1-\lambda)+ \lambda h_i   }  \frac{x_i}{ \|x_i\|^2},\\
%     s^{t+1}& = f_i  -\frac{\widetilde{\lambda} a_i^2 \ell}{1+ \widetilde{\lambda} h_i \ell  }  
% +  \frac{ h_i\widetilde{\lambda}^2 a_i^2 \ell^2}{2(1+ \widetilde{\lambda} h_i \ell )^2 }\\
% &=f_i  -\frac{\lambda a_i^2 \|x_i\|^2}{2(1-\lambda)+ \lambda h_i \|x_i\|^2  }  
% +  \frac{ h_i\lambda^2 a_i^2 \|x_i\|^4}{2(2(1-\lambda)+ \lambda h_i \|x_i\|^2 )^2 }.
% \end{align*}
% To include the $s\geq 0$ constraint, we can then update $s$ as
% \begin{align*}
%     s^{t+1} = \max\left\{ f_i  -\frac{\lambda a_i^2 \|x_i\|^2}{2(1-\lambda)+ \lambda h_i \|x_i\|^2  }  
% +  \frac{ h_i\lambda^2 a_i^2 \|x_i\|^4}{2(2(1-\lambda)+ \lambda h_i \|x_i\|^2 )^2 },0 \right\}.
% \end{align*}
\begin{align*}
    w^{t+1} &= w^t + c^\star x_i ,\\
    s^{t+1} &= \max\left\{ \widetilde{s},~0 \right\},
\end{align*}
where
\begin{align*}
    c^\star = \begin{cases}
    0, \quad &\text{if } f_i = 0\\
    -\frac{\widetilde{\lambda} a_i }{1+ \widetilde{\lambda} h_i \ell  }, &\text{if } f_i>0 \text{ and }  \widetilde{s}\geq 0,\\
\frac{-a_i + \sqrt{a_i^2  - 2h_i  f_i}  }{h_i\ell}, & \text{otherwise.}
    \end{cases}
\end{align*}

\section{Additional Numerical Experiments}
\label{sec:add_exp}

\subsection{Non-convex problems}
For the non-convex experiements, 
we used the Python Package \texttt{pybenchfunction} available on github
\url{Python_Benchmark_Test_Optimization_Function_Single_Objective}.

\subsubsection{PermD$\beta^+$ is an incorrect implementation of PermD$\beta$. }
\label{sec:PermD}
 We note here that the PermD$\beta^+$ implemented  is given by 
$$  \texttt{PermD}\beta^+(x) \eqdef  \sum_{i=1}^{d} \sum_{j=1}^d \left((j^i +\beta)\left(\left(\frac{x_j}{j} \right)^i -1 \right) \right)^2.$$
which is different than the standard \texttt{PermD}$\beta$ function which is given by
$$  \texttt{PermD}\beta(x) \eqdef  \sum_{i=1}^{d} \left(\sum_{j=1}^d (j^i +\beta)\left(\left(\frac{x_j}{j} \right)^i -1 \right) \right)^2.$$
We believe this is a small mistake, which is why we have introduced the plus in $\texttt{PermD}\beta^+$ to distinguish this function from the standard \texttt{PermD}$\beta$ function. Yet, the $\texttt{PermD}\beta^+$ is still an interesting non-convex problem, and thus we have used it in our experiments despite this small alteration.
% np.sum([np.sum((j**i + self.beta) *((X/j)**i - 1))**2  for i in range(1, d+1)])

\subsubsection{The Levy N. 13 and Rosenbrock problems}
\label{secapp:nonconvextoy}
\vspace{-1mm}

Here we provide two additional experiments on the non-convex function Levy N. 13 and Rosenbrock that complement the findings in~\ref{sec:nonconvextoy}.

For the Levy N. 13 function in Figure~\ref{fig:levy} we have that again \texttt{SP2} converges in $10$ epochs to the global minima. In contrast Newton's method converges immediately to a local maxima, that can be easily seen on the surface plot of the right of Figure~\ref{fig:levy}.

The one problem where \texttt{SP2} was not the fastest was on the Rosenbrock function, see Figure~\ref{fig:rosenbrock}. Here \texttt{Newton} was the fastest, converging in under 10 epochs. But note, this problem was designed to emphasize the advantages of Newton over gradient descent.

\begin{figure}[t]
\centering
\includegraphics[width = 0.3\textwidth]{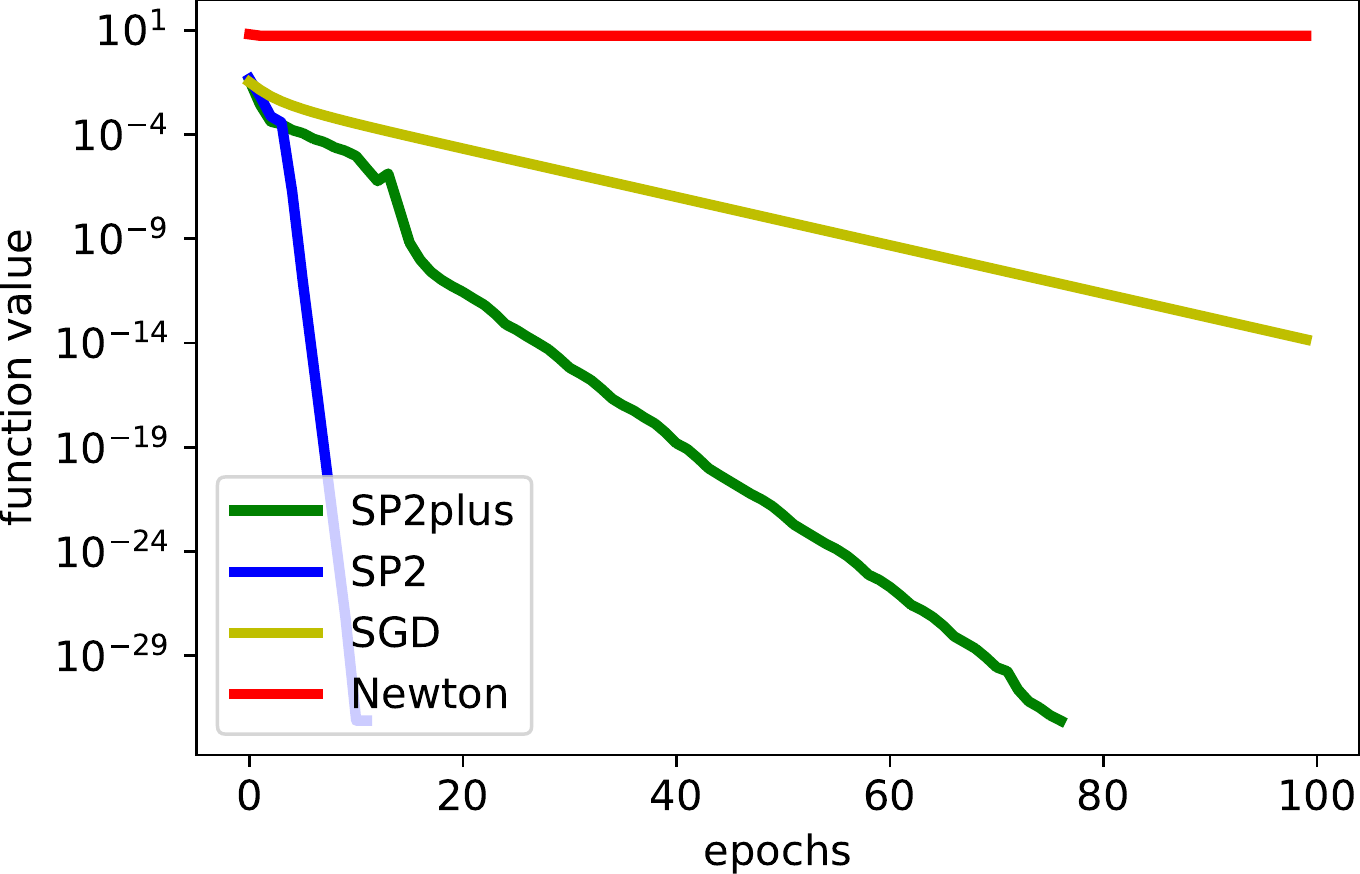}
\includegraphics[width = 0.3\textwidth]{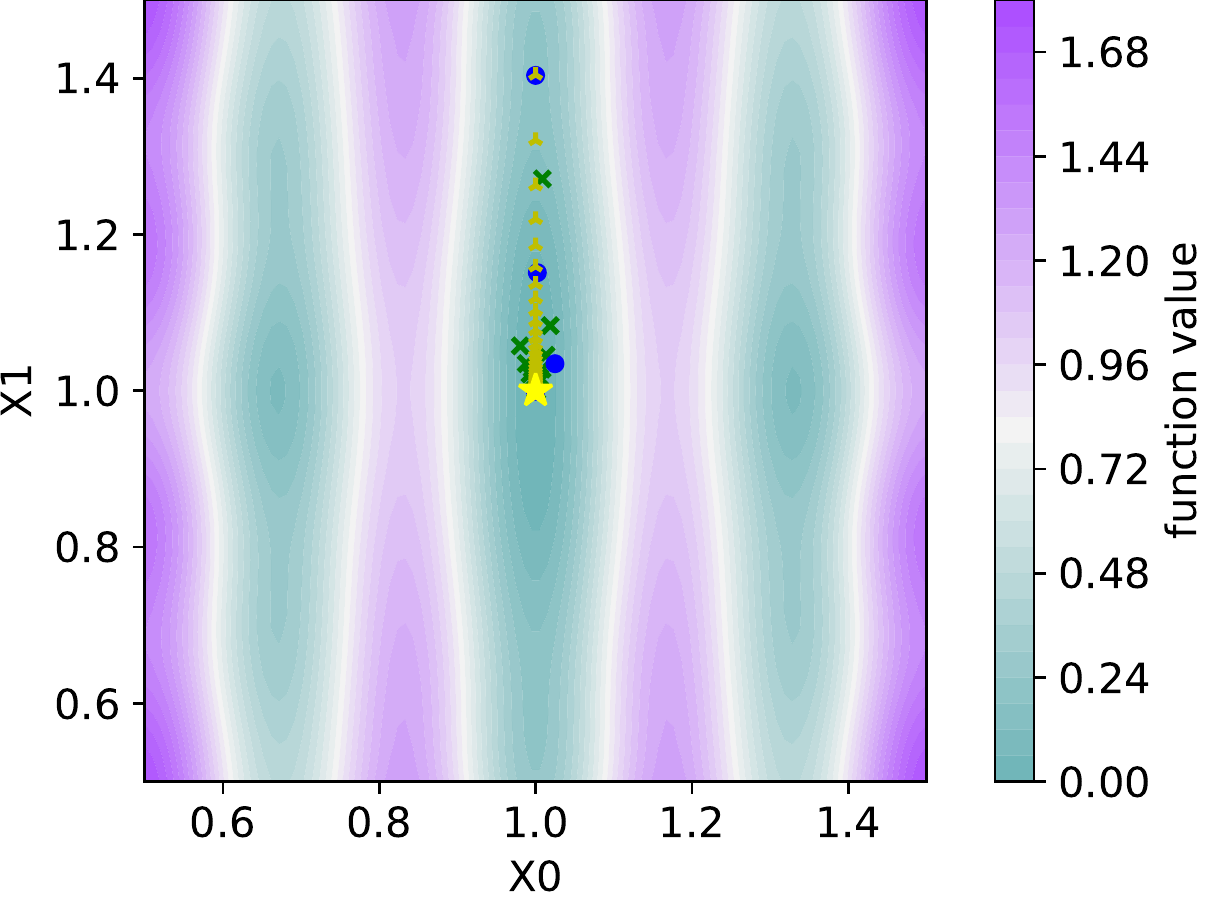}
\includegraphics[width = 0.3\textwidth]{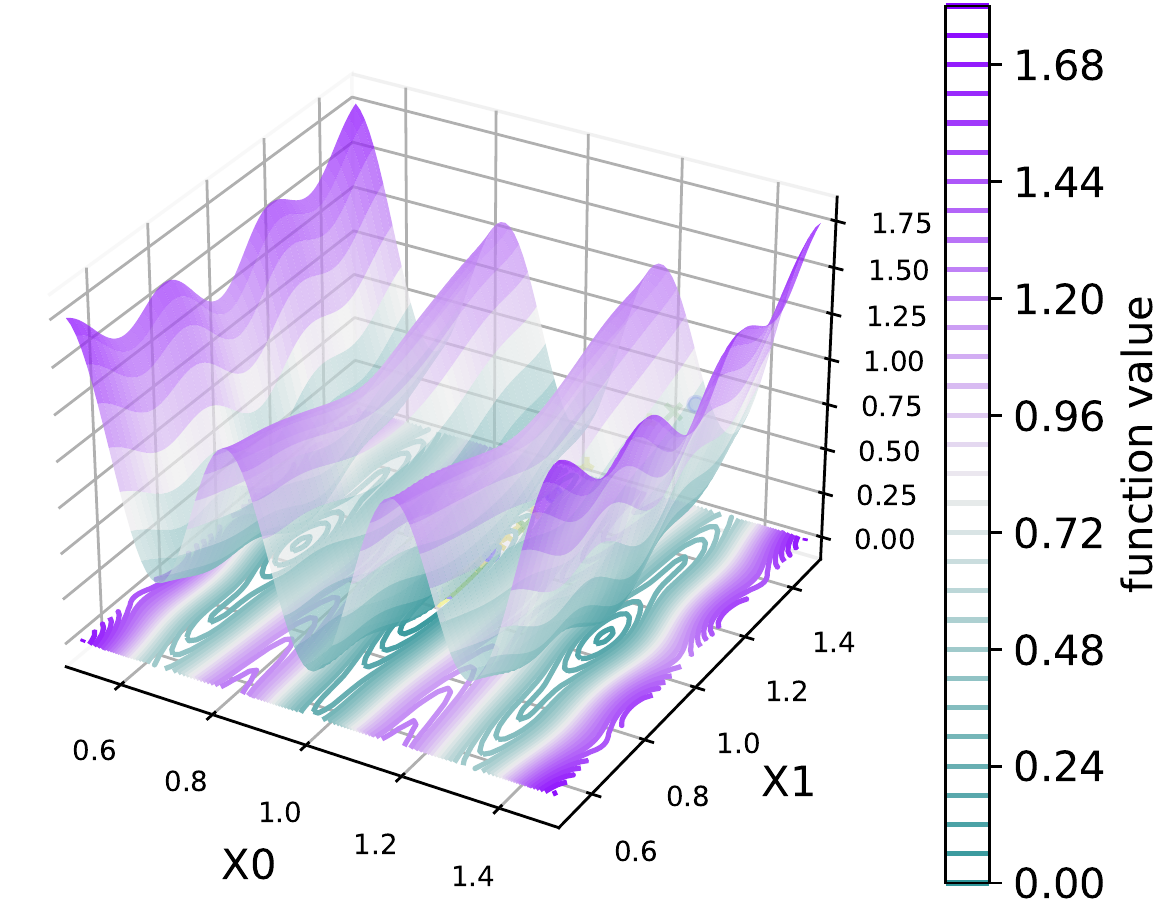}
%Comparing  \texttt{SP2$^+$}, \texttt{SP2}, \texttt{SGD}, \texttt{Newton} on 
\caption{The Levy N. 13 function where Left: we plot $f(x)$ across epochs Middle: level set plot, Right: Surface plot.  \texttt{SP2} is in \textcolor{blue}{blue }, \texttt{SP2$^+$} is in \textcolor{darkgreen}{green },  \texttt{SGD} is in \textcolor{citrine}{yellow } and  \texttt{Newton} is in \textcolor{red}{red }.}
\label{fig:levy}
\end{figure}

\begin{figure}[t]
\centering
\includegraphics[width = 0.3\textwidth]{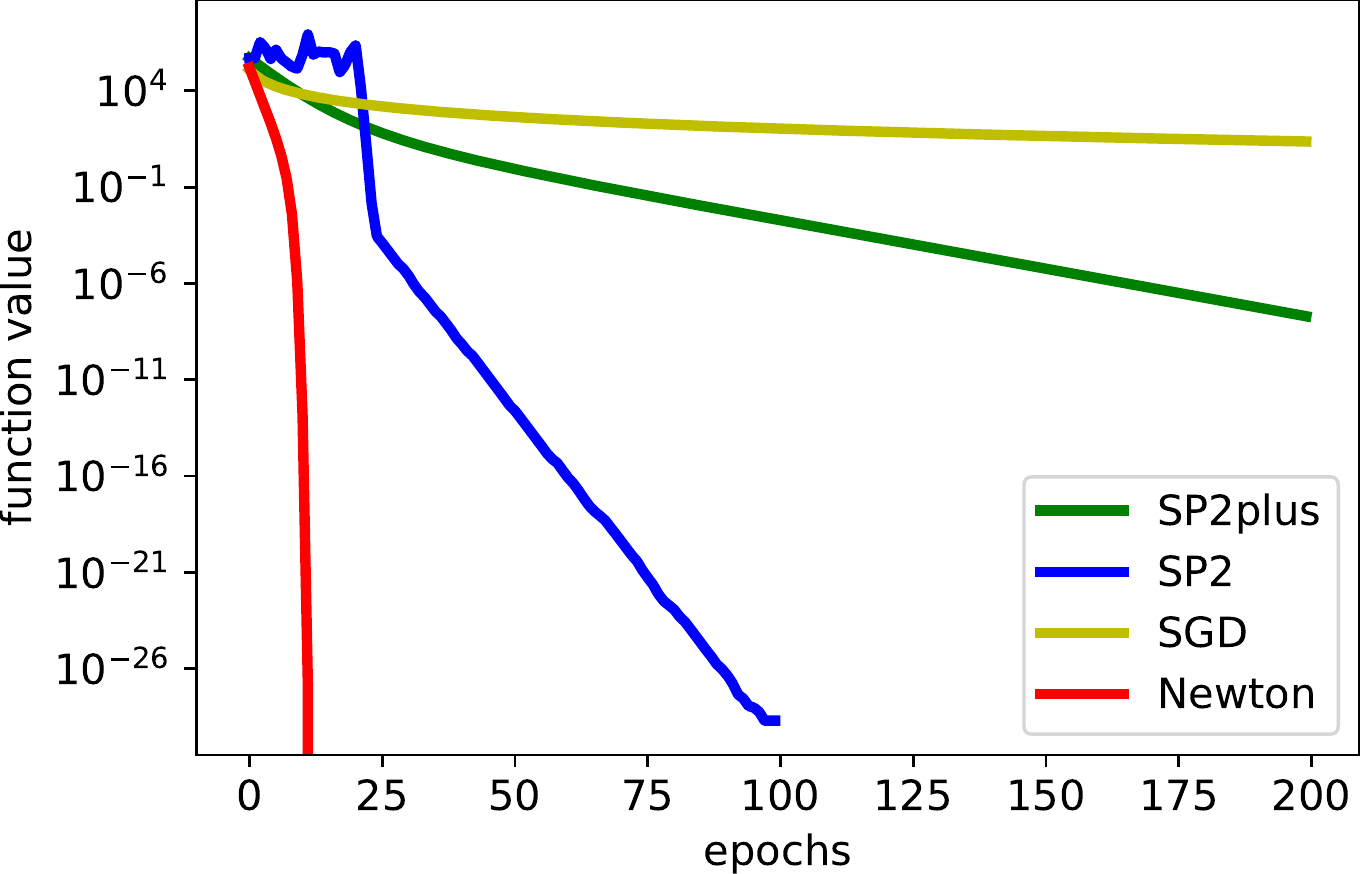}
\includegraphics[width = 0.3\textwidth]{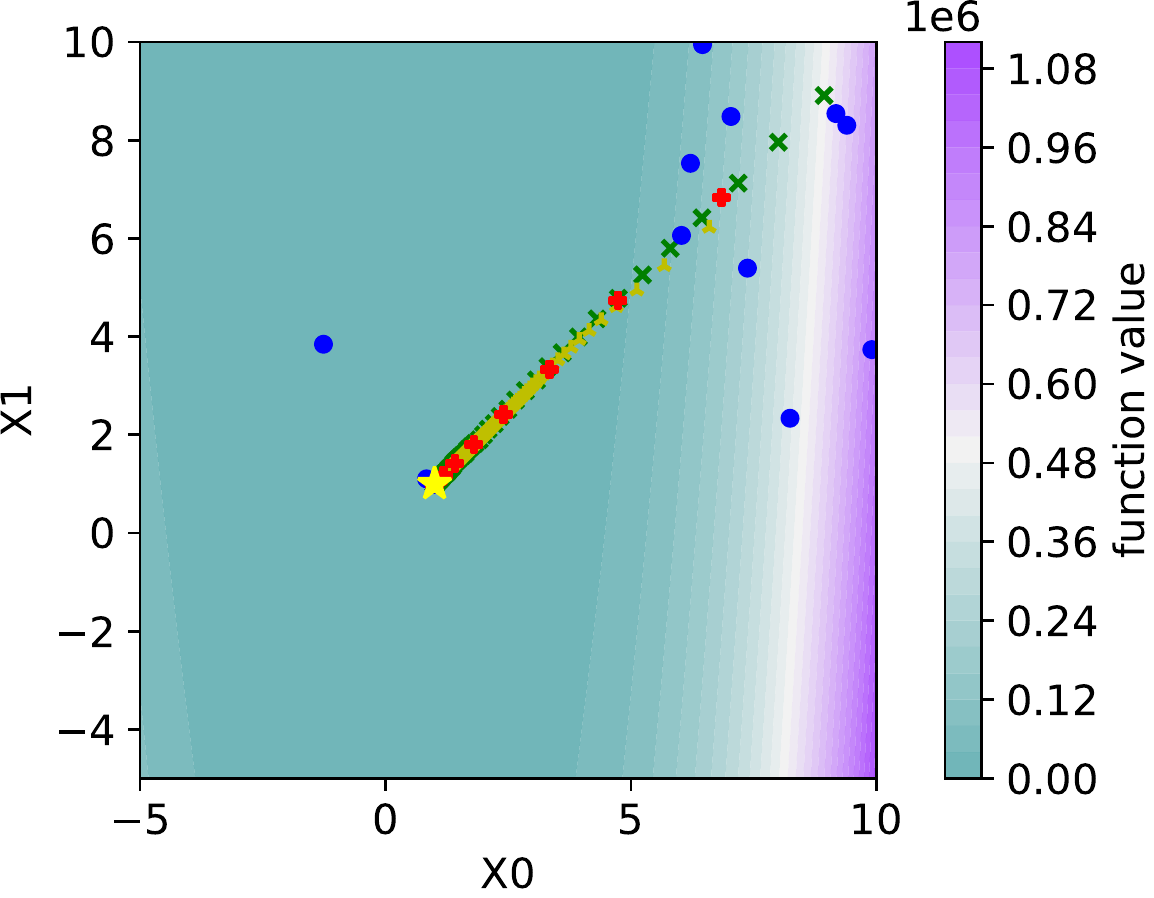}
\includegraphics[width = 0.3\textwidth]{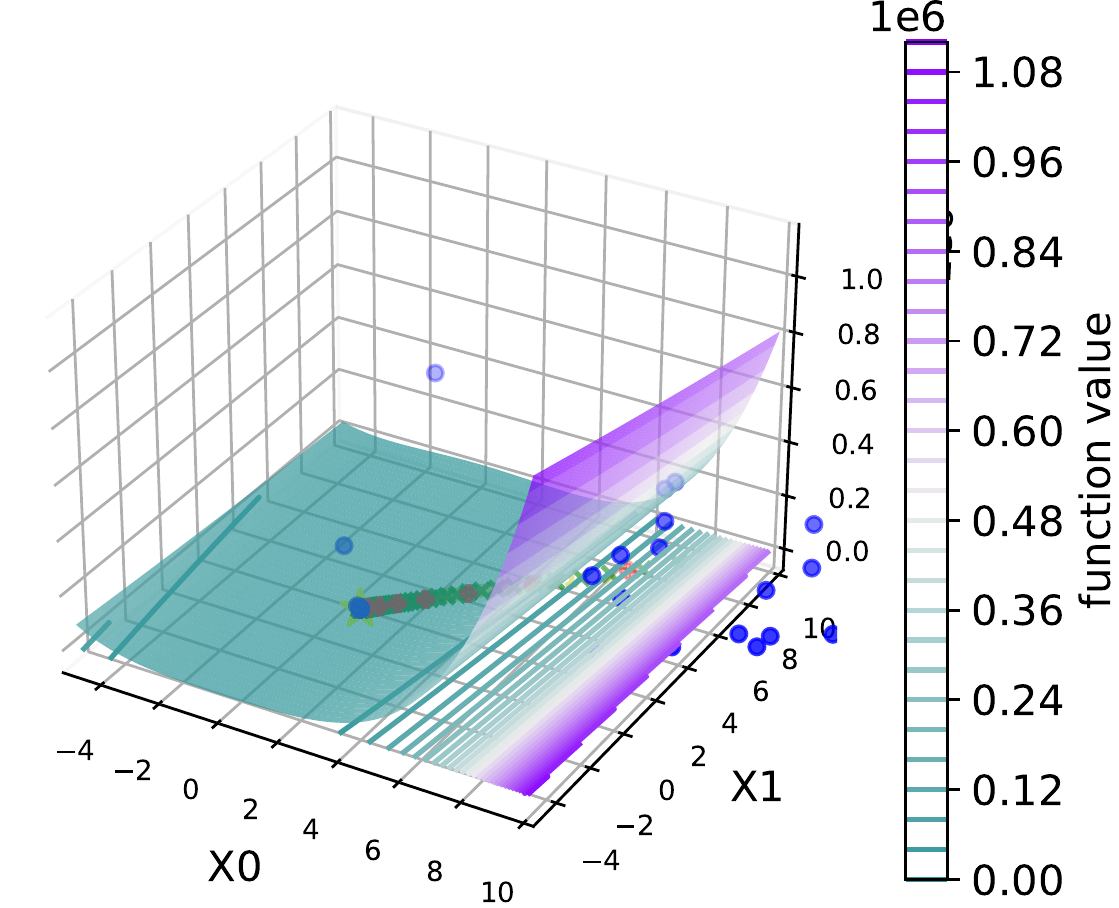}
\caption{The Rosenbrock function where Left: we plot $f(x)$ across epochs Middle: level set plot, Right: Surface plot.  \texttt{SP2} is in \textcolor{blue}{blue }, \texttt{SP2$^+$} is in \textcolor{darkgreen}{green },  \texttt{SGD} is in \textcolor{citrine}{yellow } and  \texttt{Newton} is in \textcolor{red}{red }.}
\label{fig:rosenbrock}
\end{figure}

\subsection{Additional Convex Experiments}

We set the desired tolerance for each algorithm to $0.01$, and set the maximum number of epochs for each algorithm to $200$ in colon-cancer and 30 in mushrooms. 
To choose an optimal slack parameter $\lambda$ for \texttt{SP2L2}$^+$, \texttt{SP2L1}$^+$, and \texttt{SP2max}$^+$, we test these three methods on a uniform grid $\lambda \in \{0.1,0.2, \ldots, 0.9\}$ with $\sigma = 0.001$. The gradient norm and loss evaluated at each epoch are presented in Figures~\ref{fig:colon_grad_diff_lamb}-\ref{fig:mush_loss_diff_lamb} (see Appendix~\ref{sec:add_exp}). It can be seen that \texttt{SP2L2}$^+$ performs best when $\lambda = 0.9$ in colon-cancer and $\lambda = 0.1$ in mushrooms, \texttt{SP2L1}$^+$ and \texttt{SP2max}$^+$ perform best when $\lambda = 0.1$ in both data sets. Therefore, we set $\lambda = 0.9$ for \texttt{SP2L2}$^+$ in colon-cancer and fix $\lambda = 0.1$ in other cases.

Under the same setting as in Section~\ref{sec:convexexp}, we also compare the \texttt{SP2max} and \texttt{SP2max}$^+$ methods on a  grid $\lambda =[0.001~ 0.01:0.01:0.05]$ with $\sigma = 0$. The gradient norm and loss evaluated at each epoch are presented in Figures~\ref{fig:colon_grad_diff_lamb_max}-\ref{fig:colon_loss_diff_lamb_max} (see Appendix~\ref{sec:add_exp}). As we observe, \texttt{SP2max}$^+$ always outperforms the \texttt{SP2max} method.

\begin{figure}[t]
\begin{minipage}{0.32\linewidth}
\centering
\includegraphics[width=1.7in]{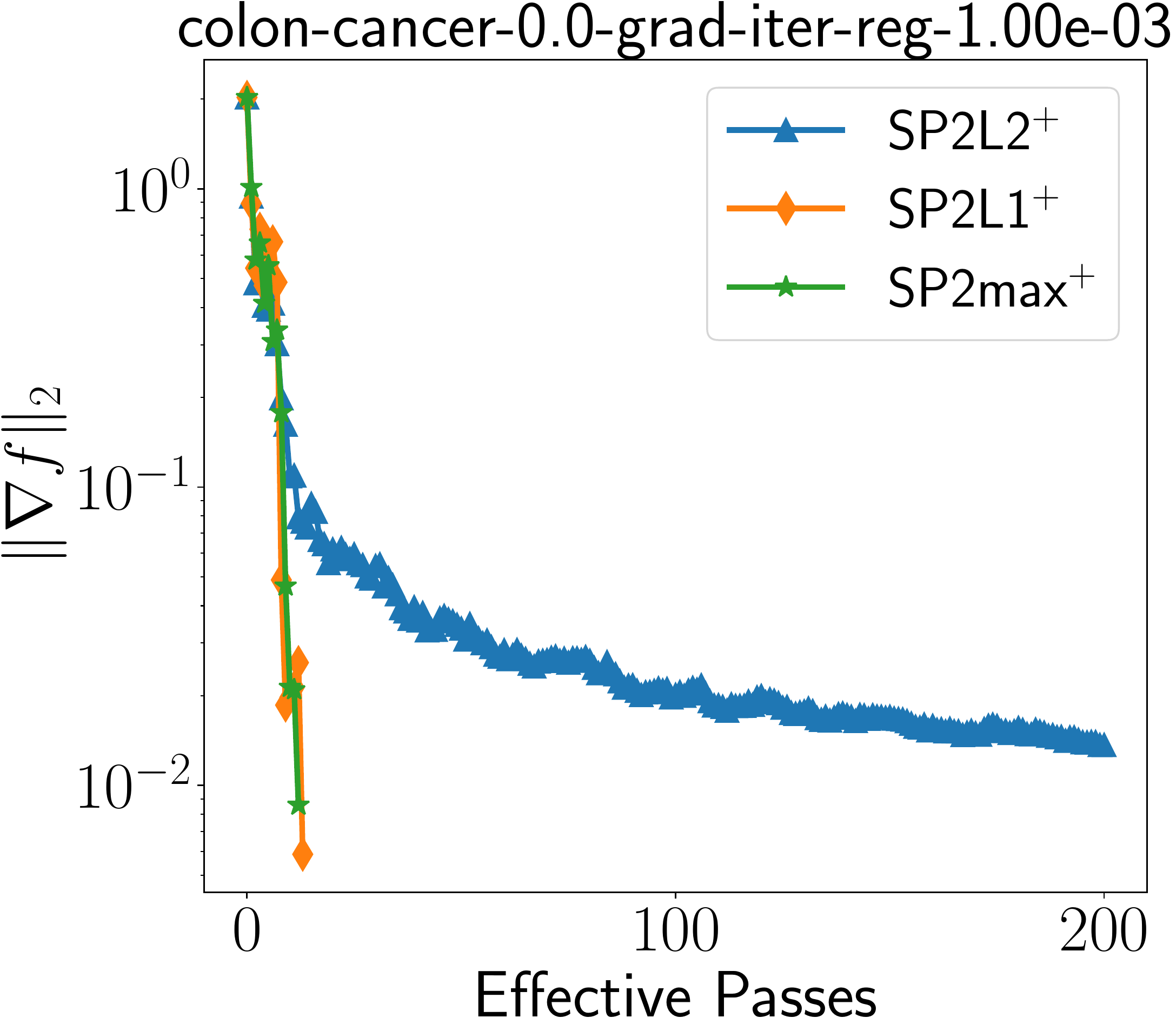}
\centerline{\small{(a) $\lambda = 0.1$}}
\end{minipage}
\hfill
\begin{minipage}{0.32\linewidth}
\centering
\includegraphics[width=1.7in]{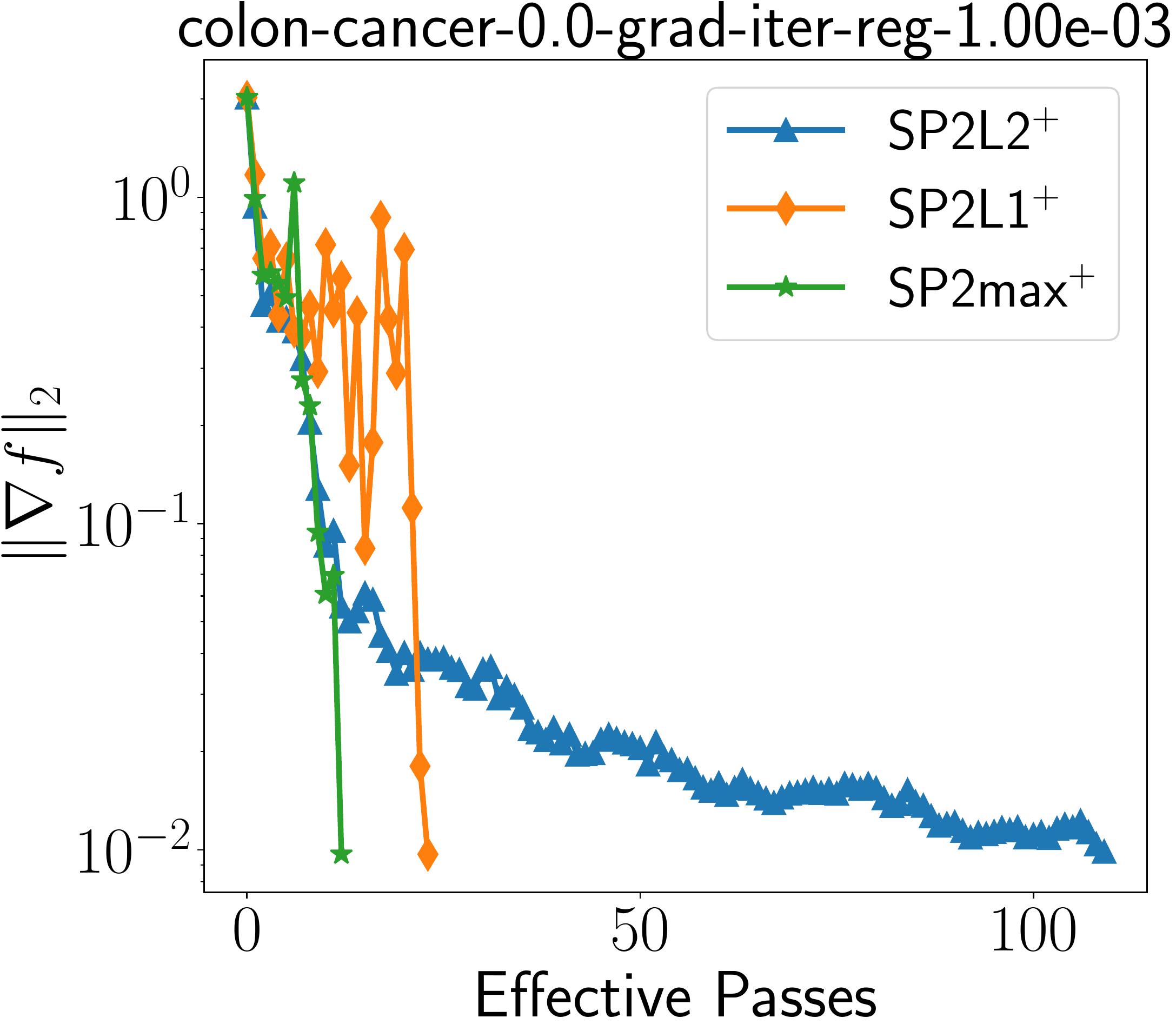}
\centerline{\small{(b) $\lambda = 0.2$}}
\end{minipage}
\hfill
\begin{minipage}{0.32\linewidth}
\centering
\includegraphics[width=1.7in]{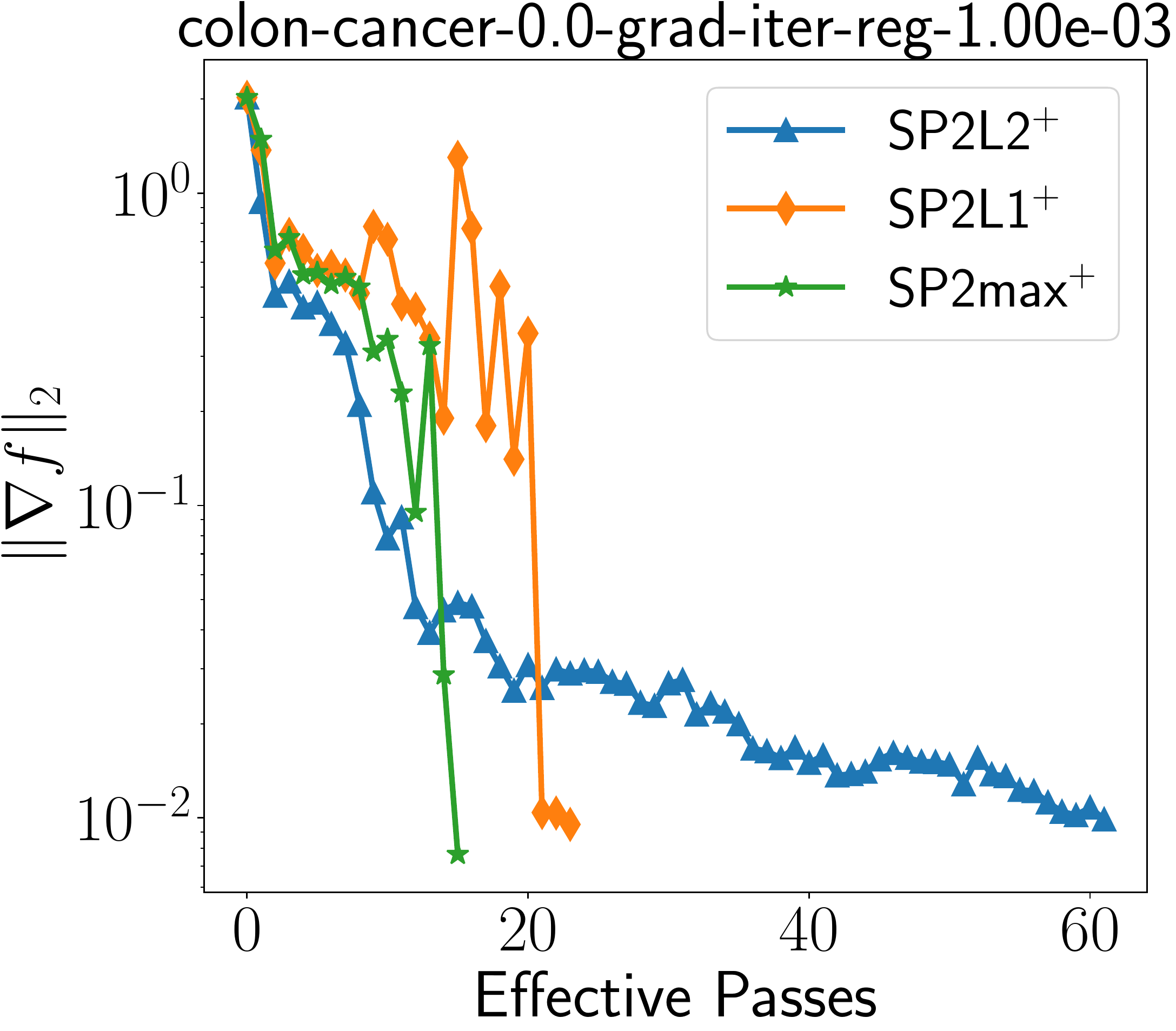}
\centerline{\small{(c) $\lambda = 0.3$}}
\end{minipage}
\\
\begin{minipage}{0.32\linewidth}
\centering
\includegraphics[width=1.7in]{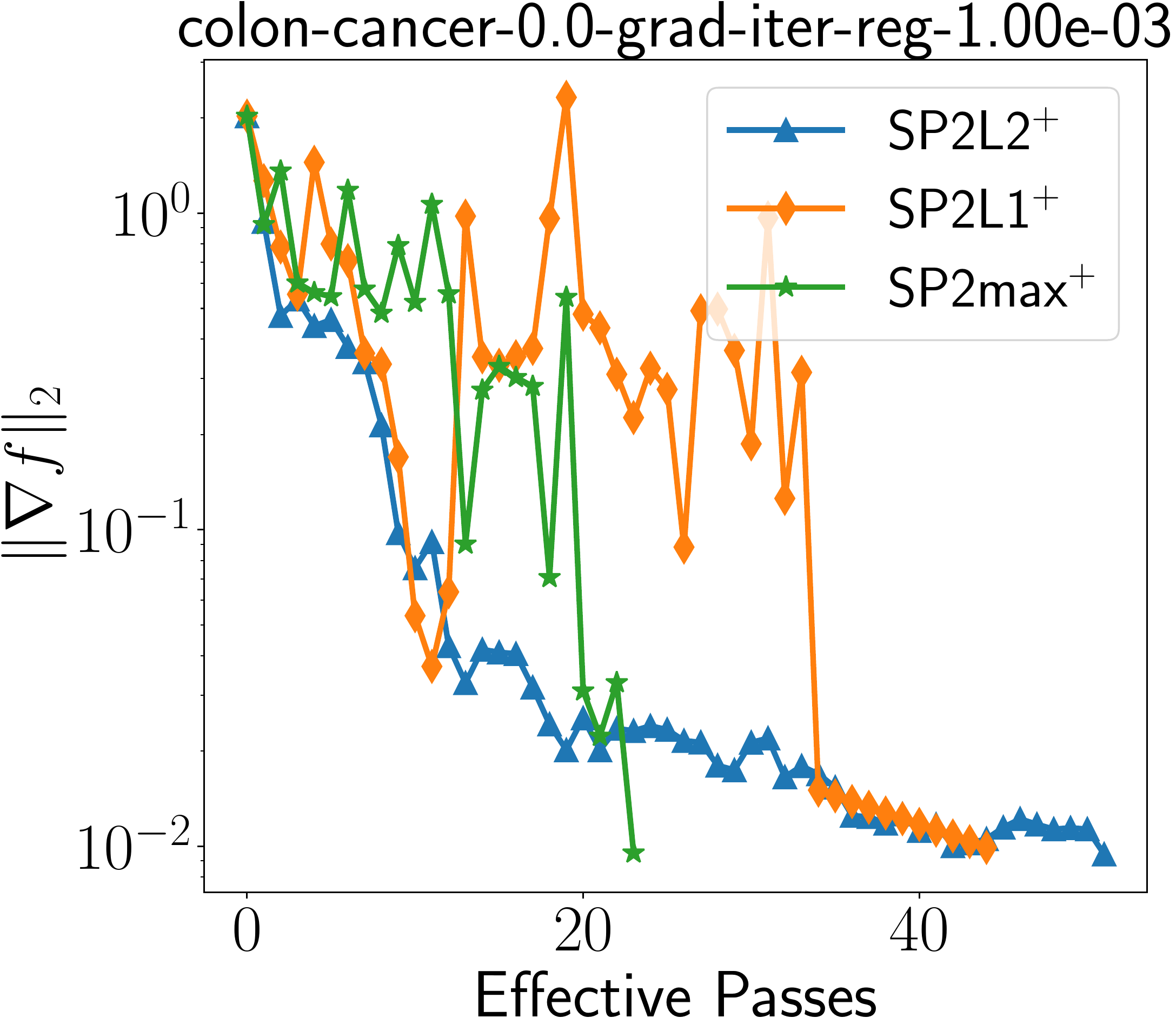}
\centerline{\small{(d) $\lambda = 0.4$}}
\end{minipage}
\hfill
\begin{minipage}{0.32\linewidth}
\centering
\includegraphics[width=1.7in]{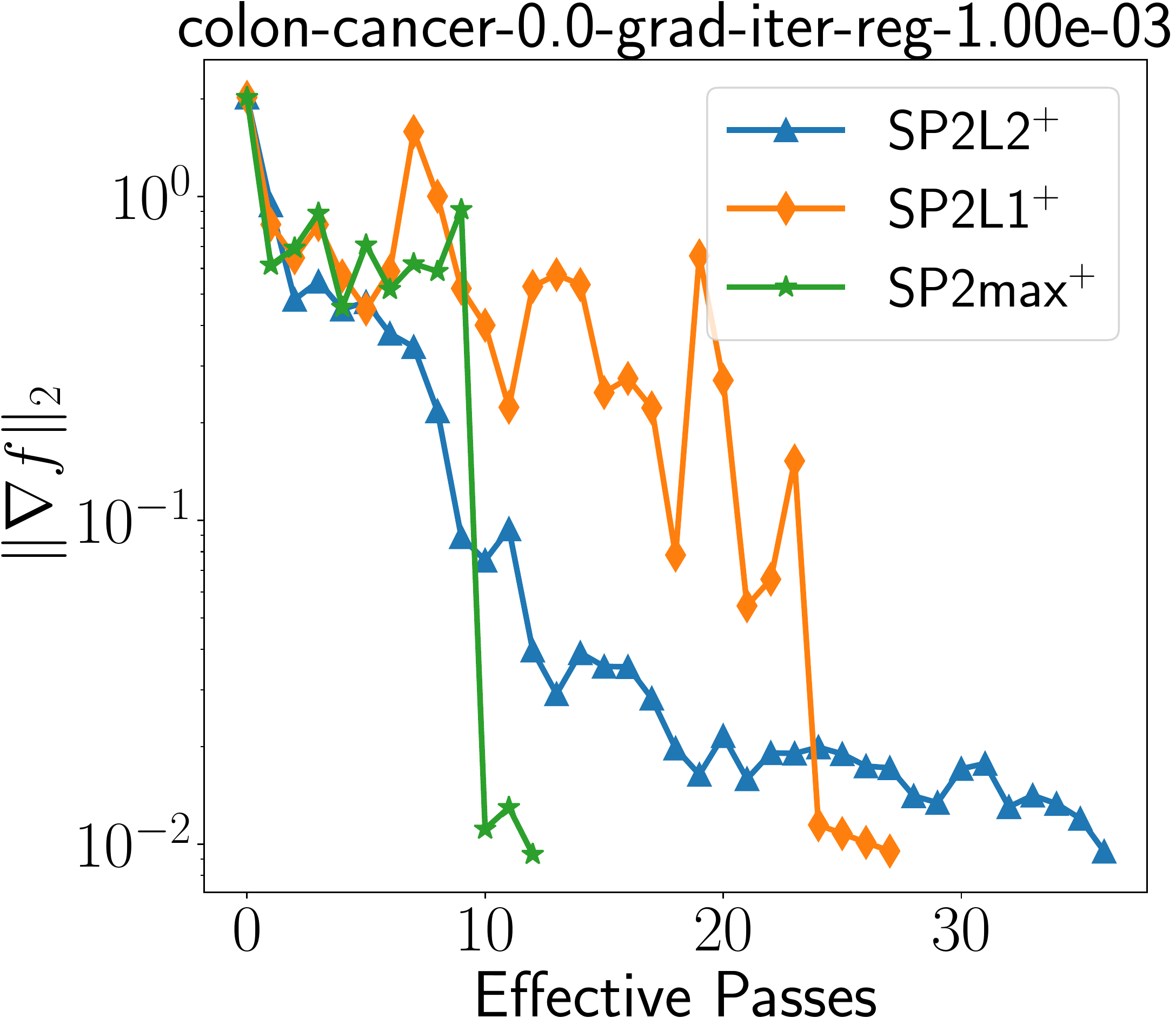}
\centerline{\small{(e) $\lambda = 0.5$}}
\end{minipage}
\hfill
\begin{minipage}{0.32\linewidth}
\centering
\includegraphics[width=1.7in]{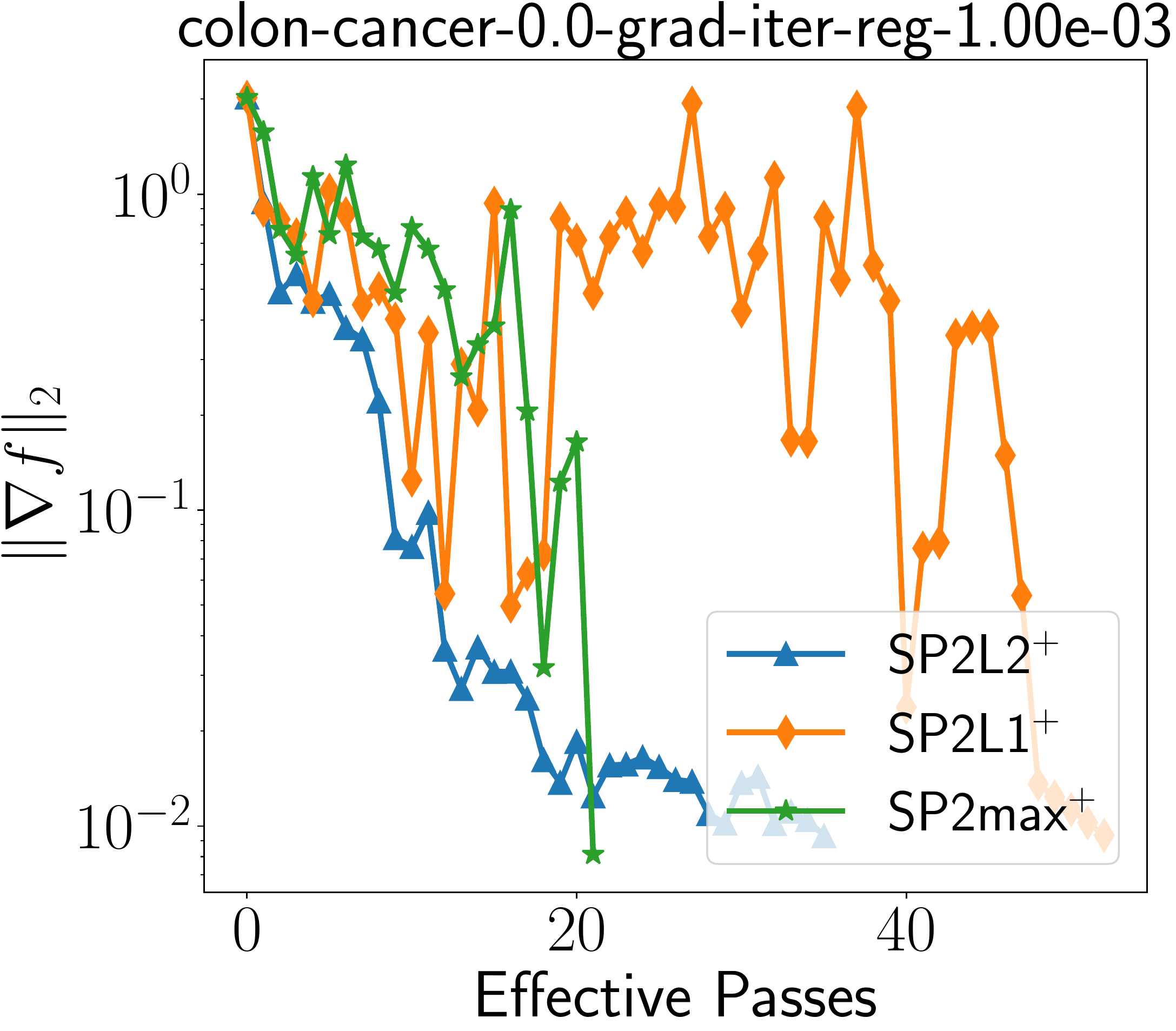}
\centerline{\small{(f) $\lambda = 0.6$}}
\end{minipage}
\\
\begin{minipage}{0.32\linewidth}
\centering
\includegraphics[width=1.7in]{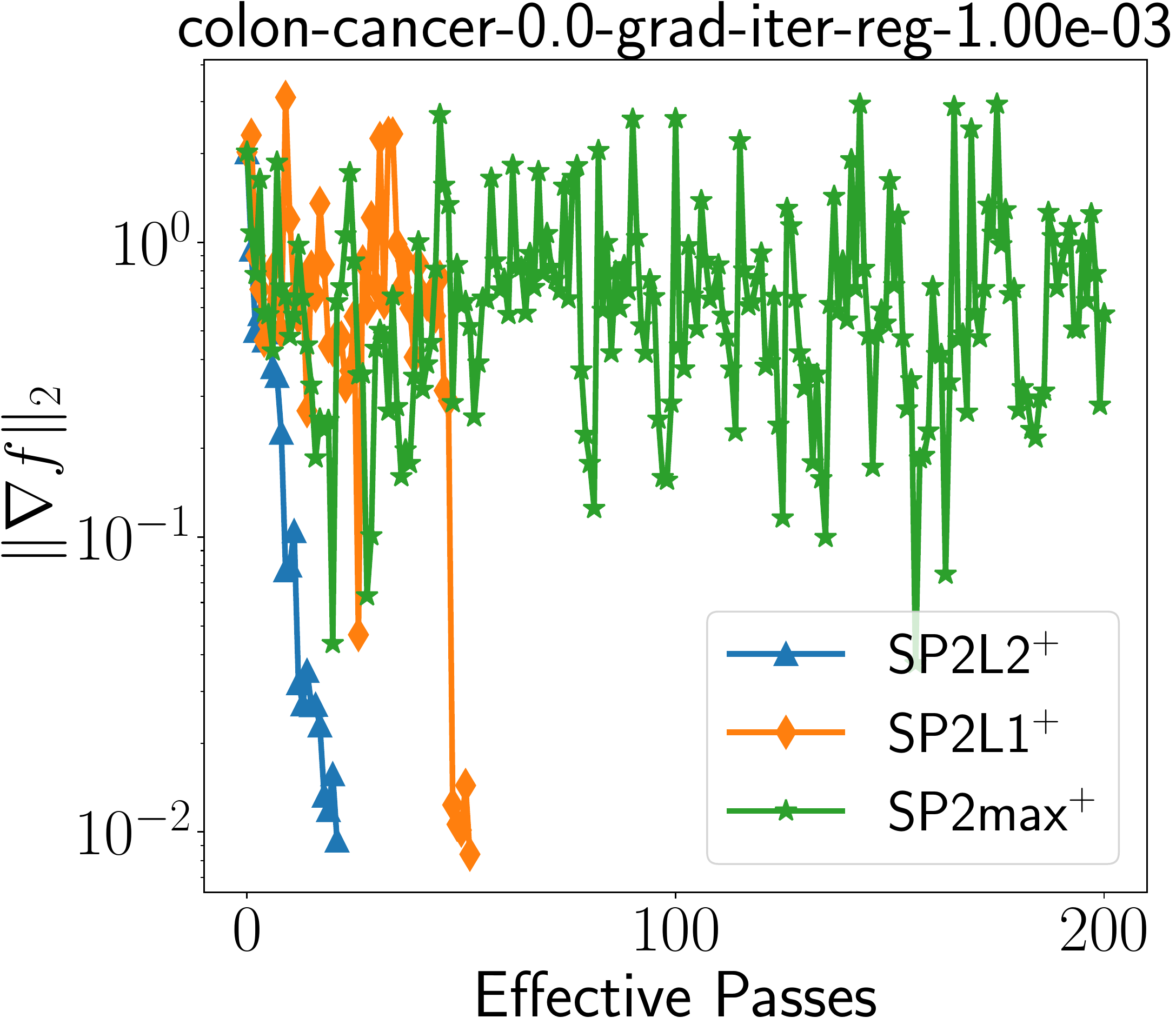}
\centerline{\small{(g) $\lambda = 0.7$}}
\end{minipage}
\hfill
\begin{minipage}{0.32\linewidth}
\centering
\includegraphics[width=1.7in]{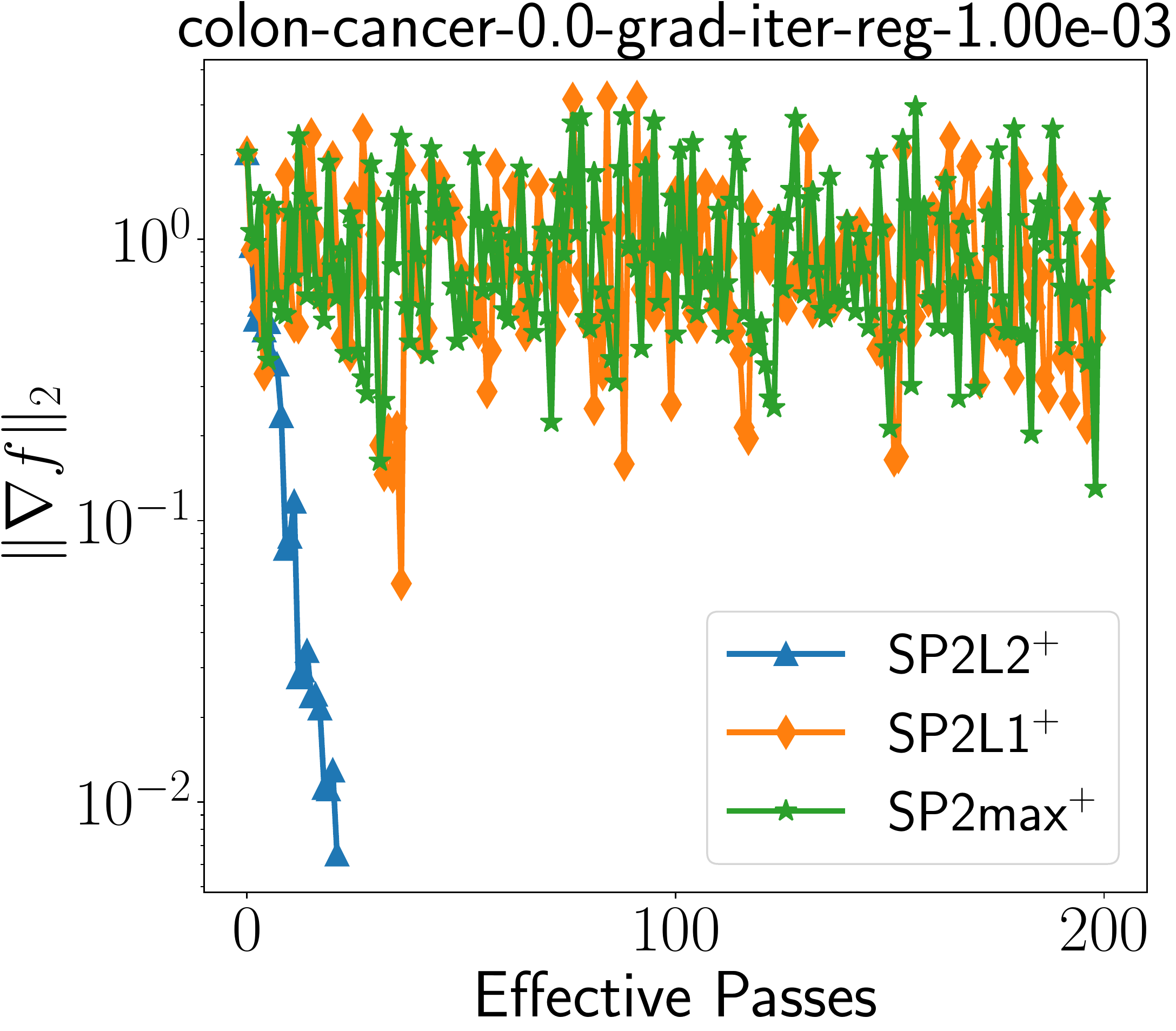}
\centerline{\small{(h) $\lambda = 0.8$}}
\end{minipage}
\hfill
\begin{minipage}{0.32\linewidth}
\centering
\includegraphics[width=1.7in]{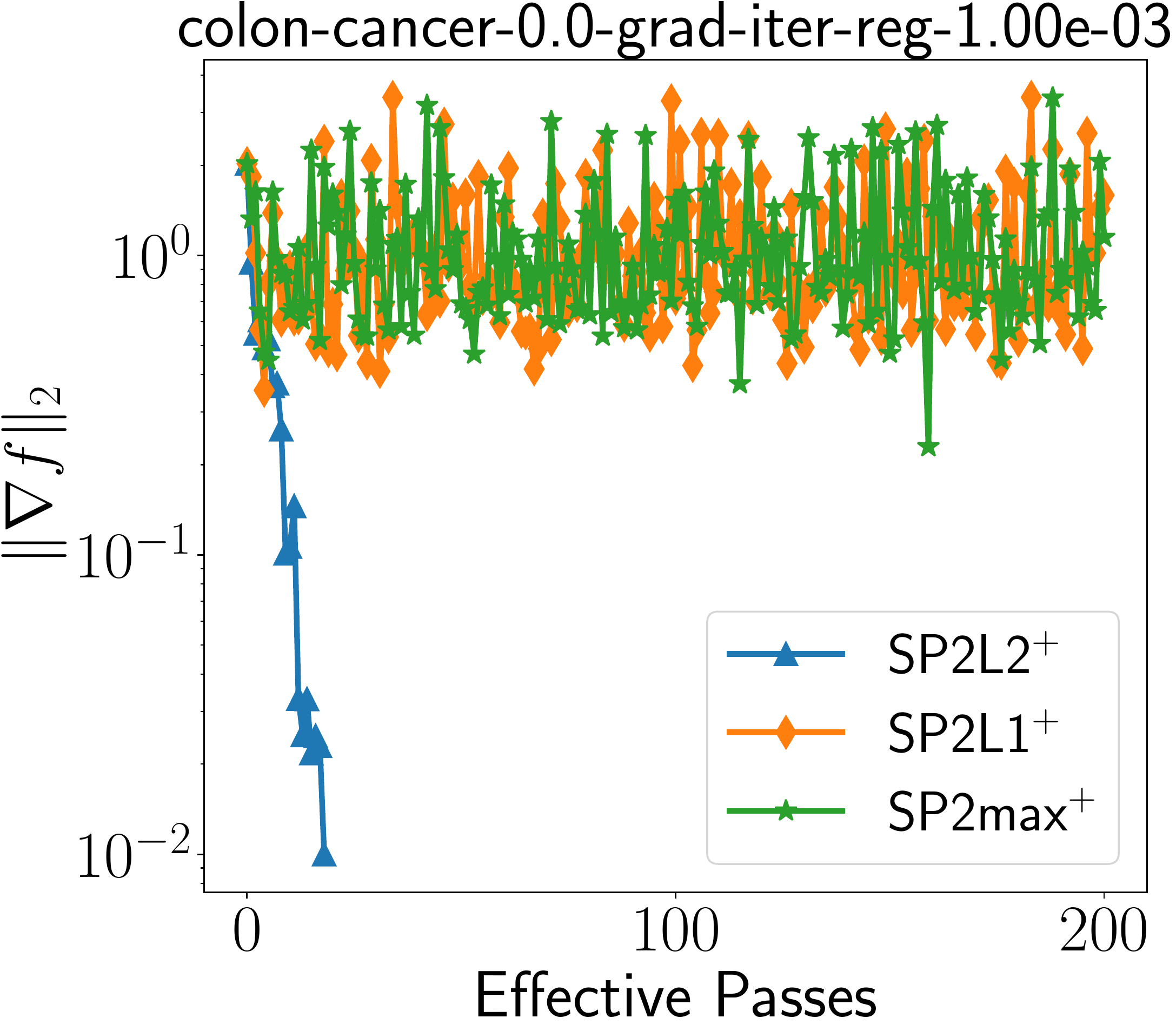}
\centerline{\small{(i) $\lambda = 0.9$}}
\end{minipage}
\caption{Colon-cancer: gradient norm at each epoch with different $\lambda$. }
\label{fig:colon_grad_diff_lamb}
\end{figure}

\begin{figure}[t]
\begin{minipage}{0.32\linewidth}
\centering
\includegraphics[width=1.7in]{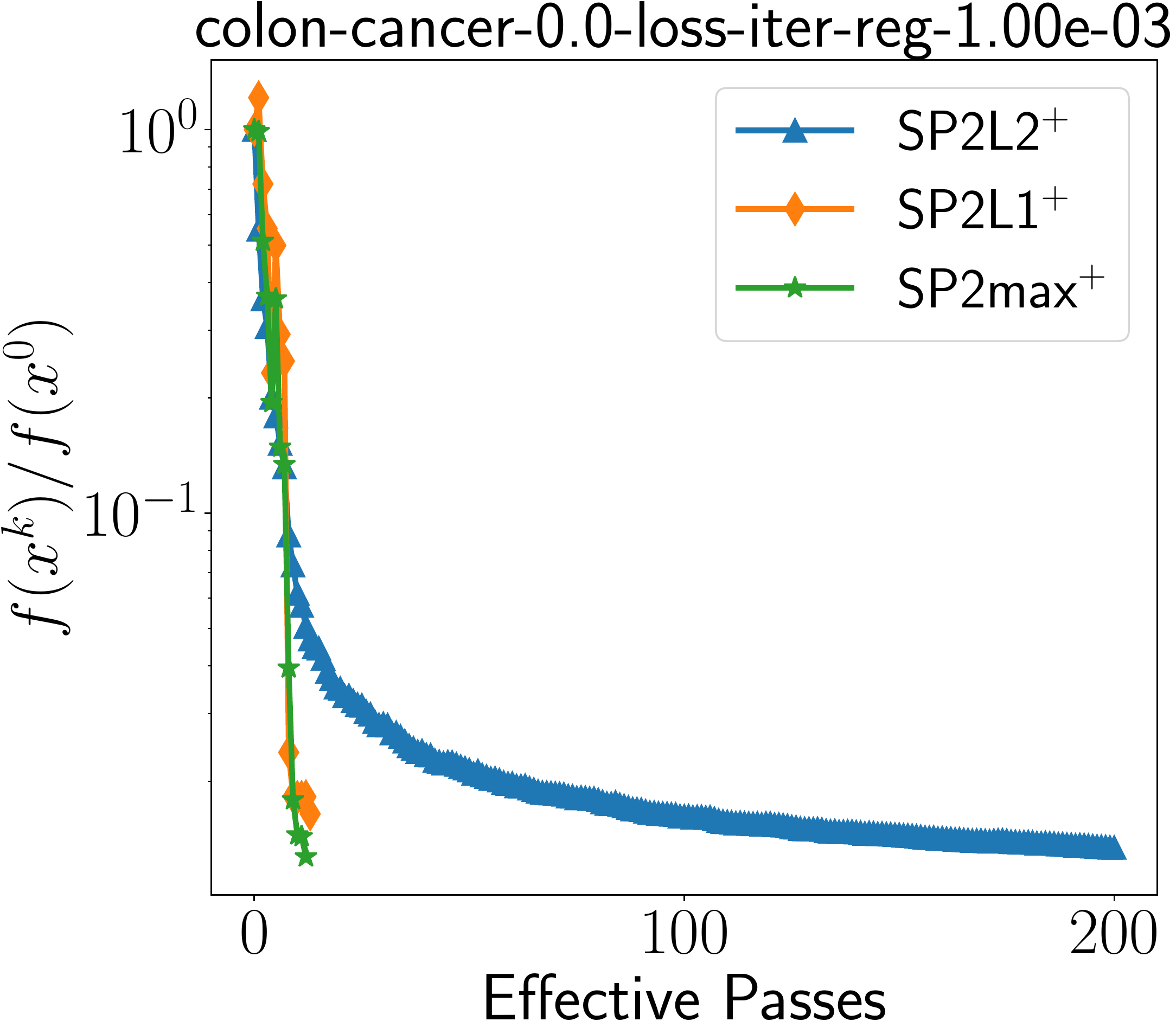}
\centerline{\small{(a) $\lambda = 0.1$}}
\end{minipage}
\hfill
\begin{minipage}{0.32\linewidth}
\centering
\includegraphics[width=1.7in]{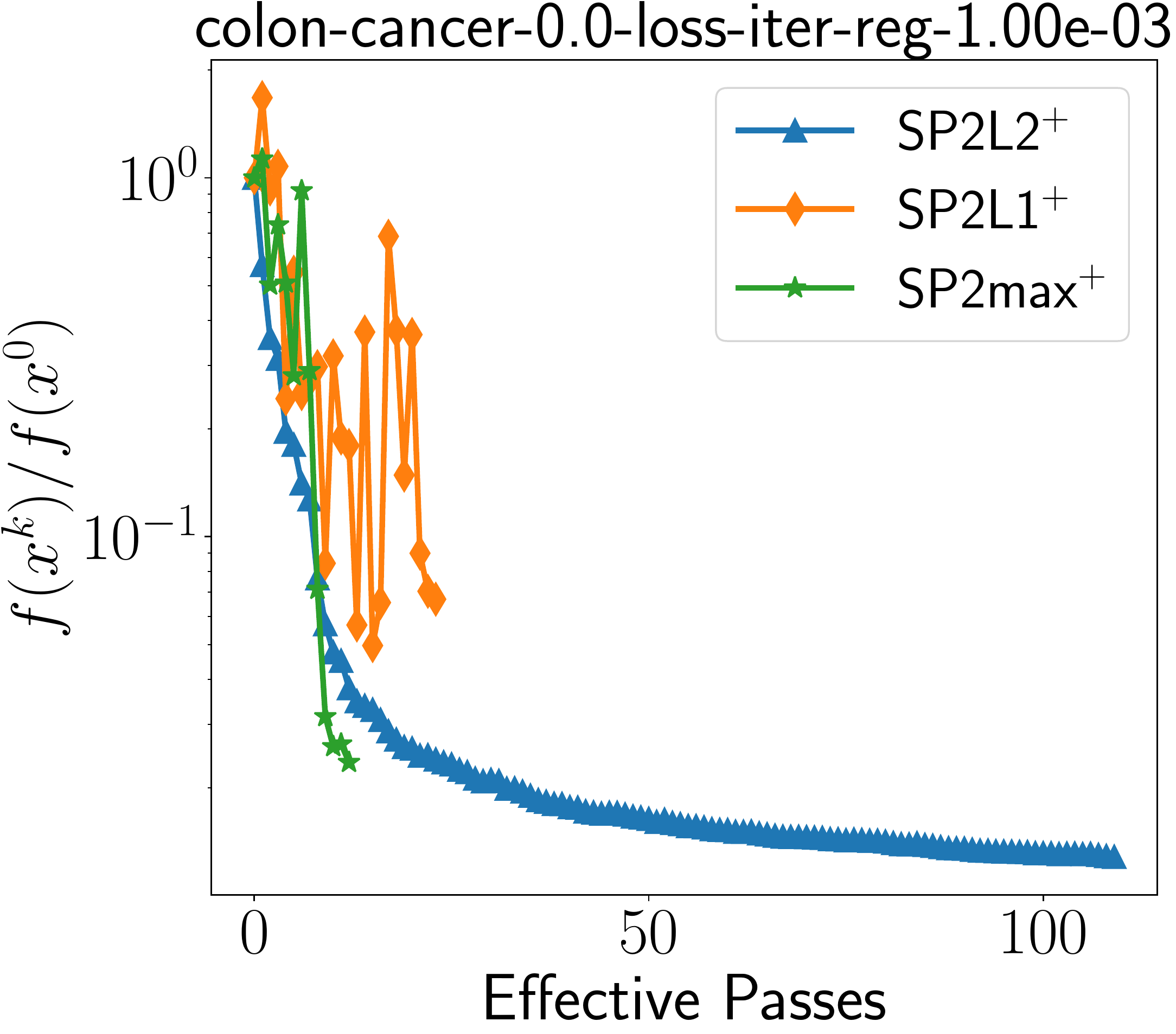}
\centerline{\small{(b) $\lambda = 0.2$}}
\end{minipage}
\hfill
\begin{minipage}{0.32\linewidth}
\centering
\includegraphics[width=1.7in]{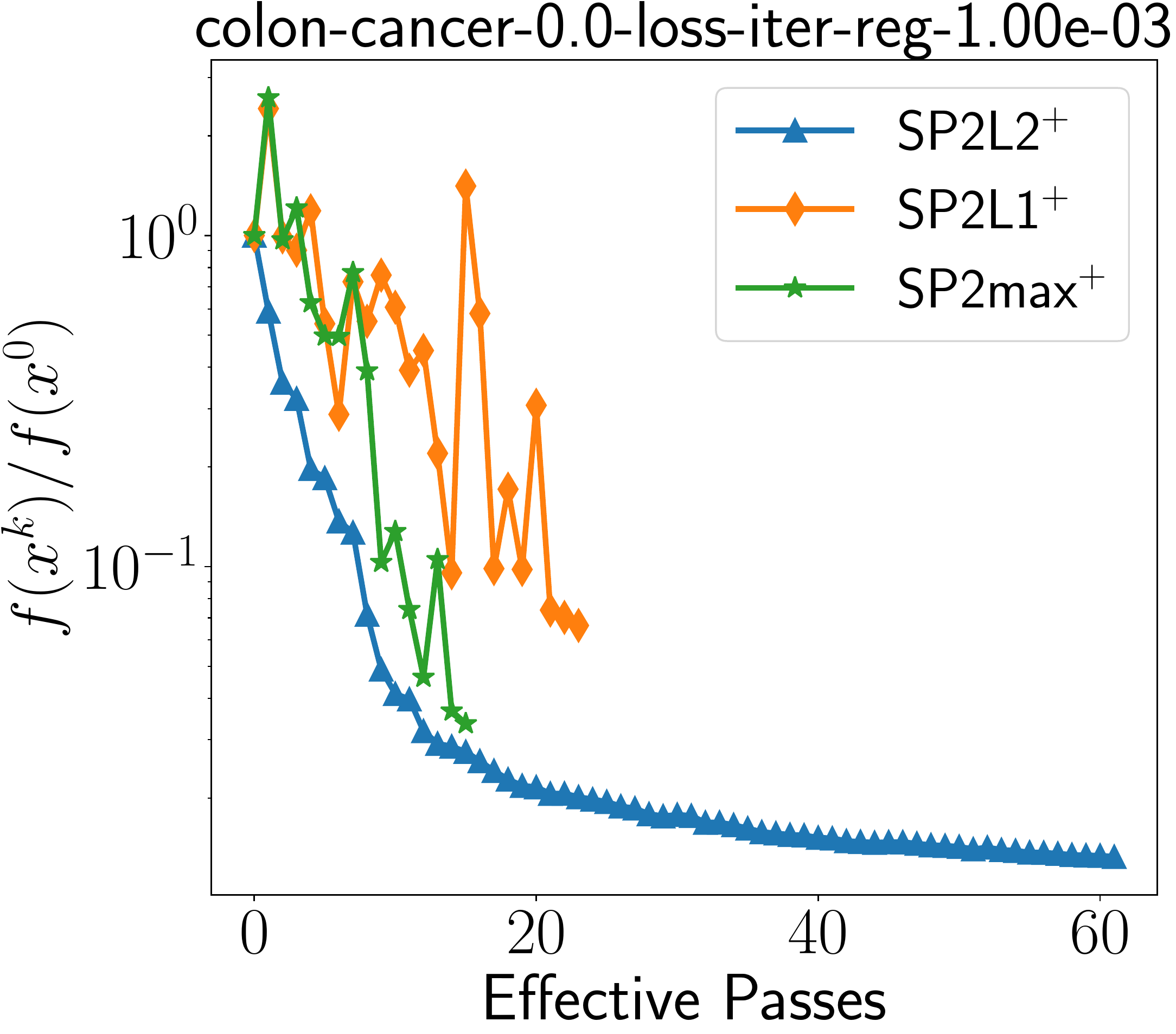}
\centerline{\small{(c) $\lambda = 0.3$}}
\end{minipage}
\\
\begin{minipage}{0.32\linewidth}
\centering
\includegraphics[width=1.7in]{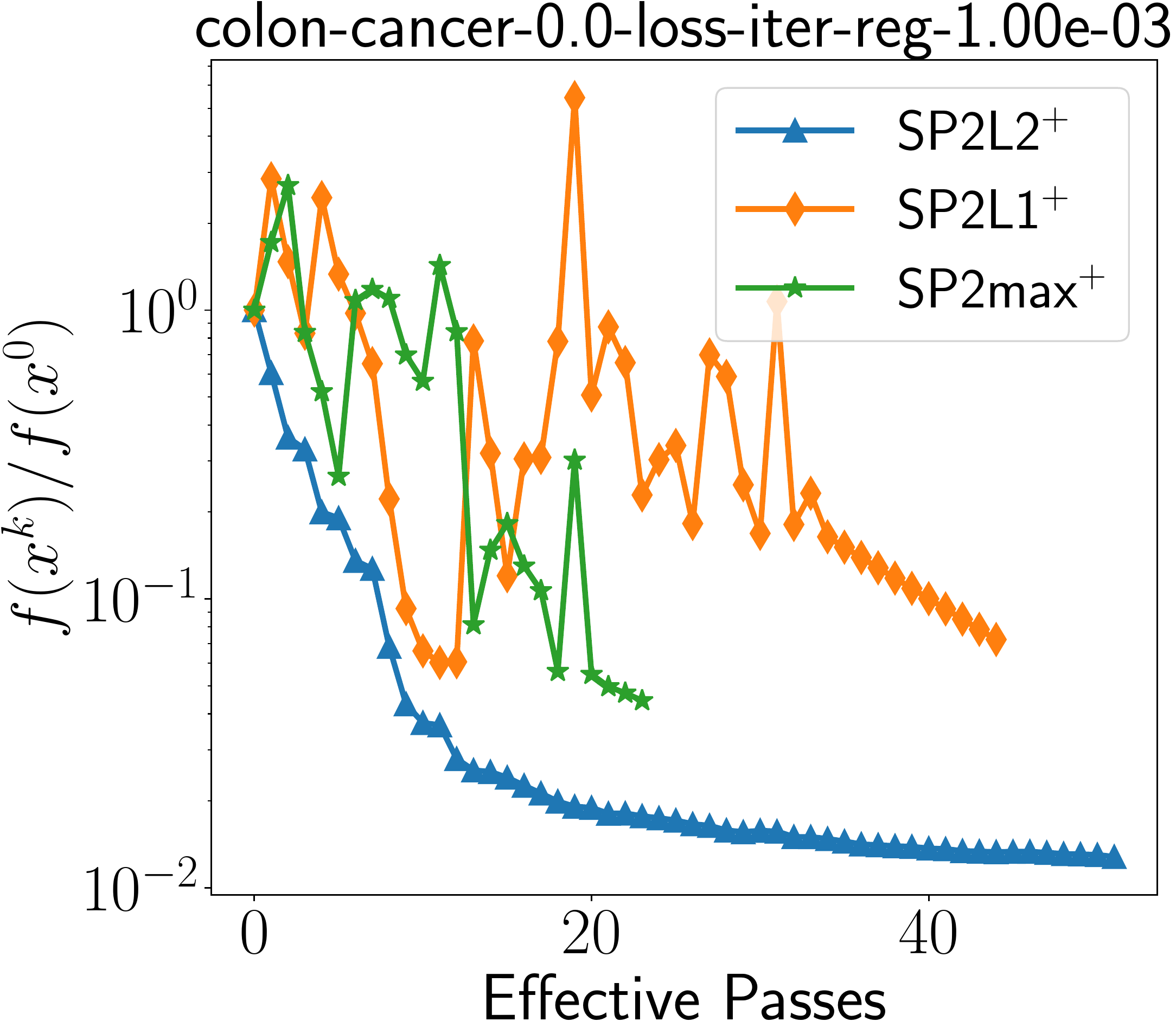}
\centerline{\small{(d) $\lambda = 0.4$}}
\end{minipage}
\hfill
\begin{minipage}{0.32\linewidth}
\centering
\includegraphics[width=1.7in]{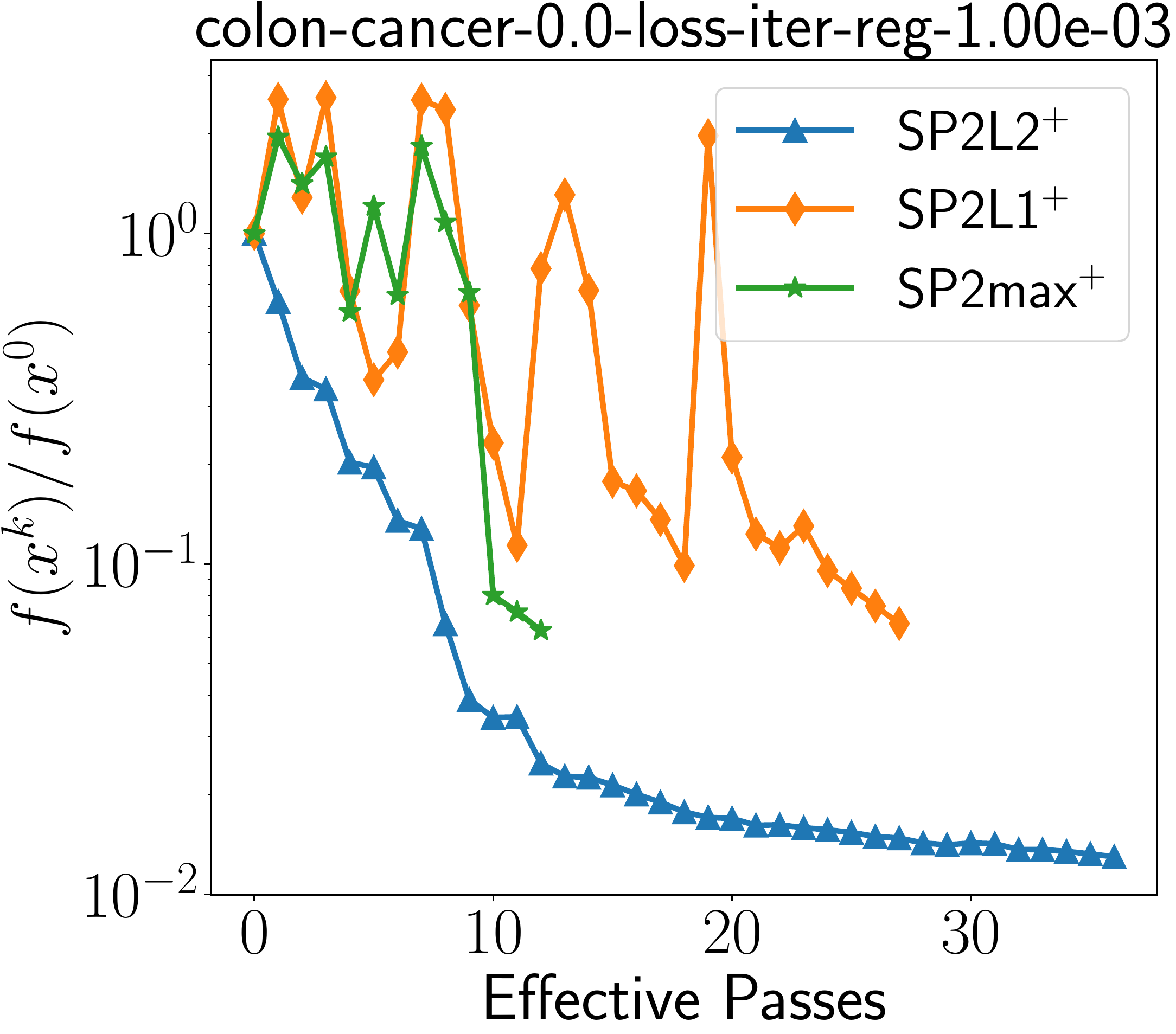}
\centerline{\small{(e) $\lambda = 0.5$}}
\end{minipage}
\hfill
\begin{minipage}{0.32\linewidth}
\centering
\includegraphics[width=1.7in]{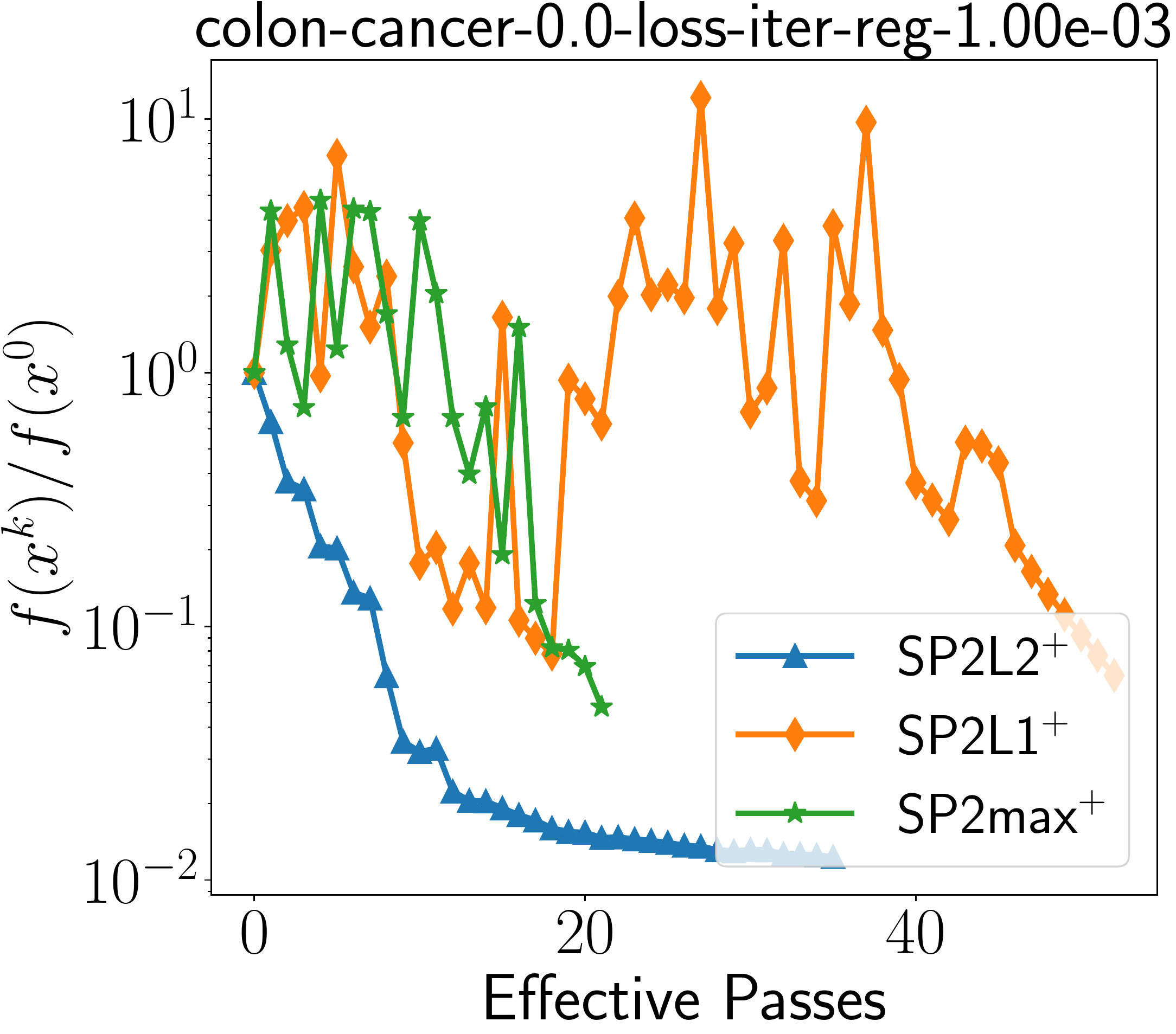}
\centerline{\small{(f) $\lambda = 0.6$}}
\end{minipage}
\\
\begin{minipage}{0.32\linewidth}
\centering
\includegraphics[width=1.7in]{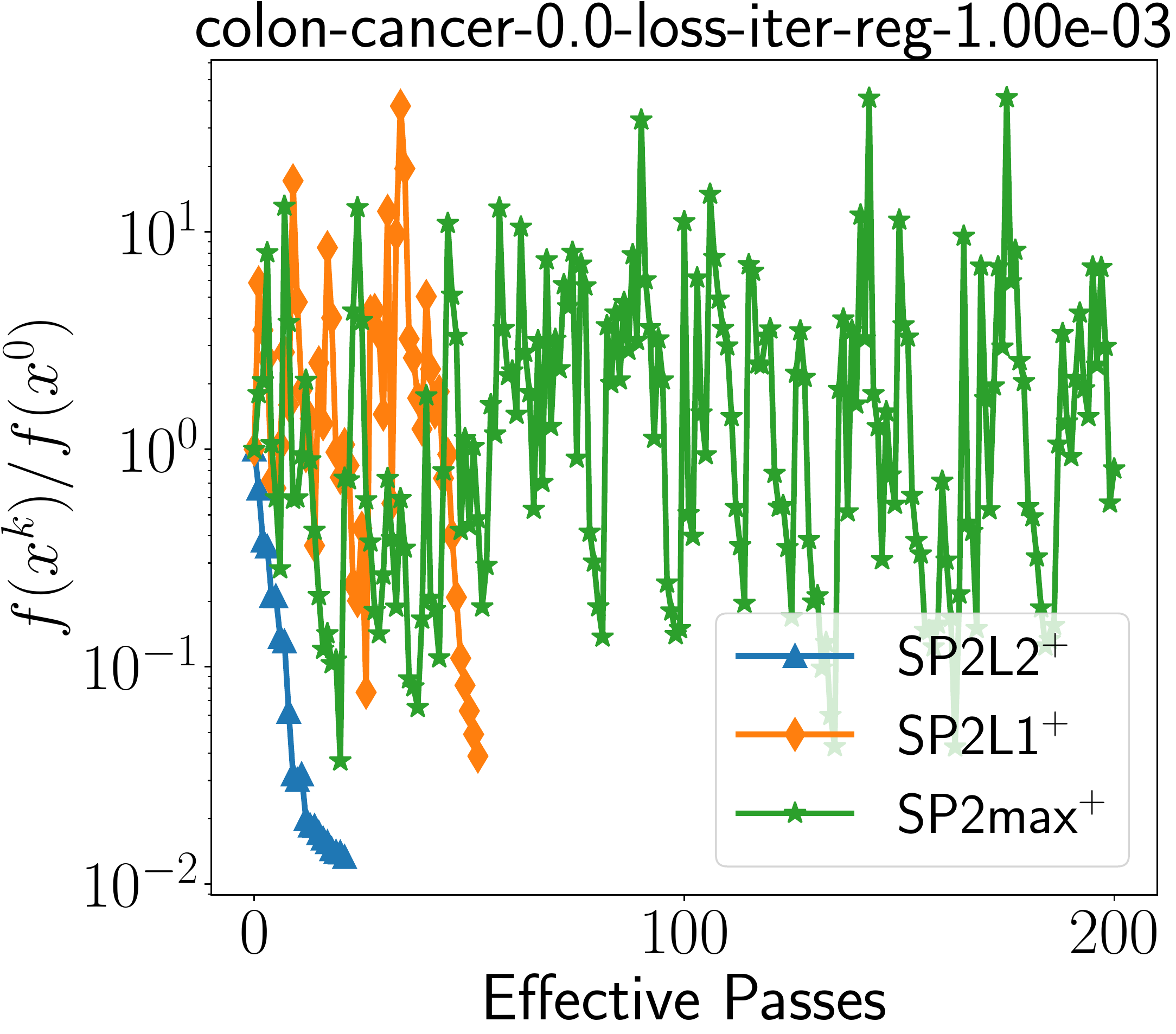}
\centerline{\small{(g) $\lambda = 0.7$}}
\end{minipage}
\hfill
\begin{minipage}{0.32\linewidth}
\centering
\includegraphics[width=1.7in]{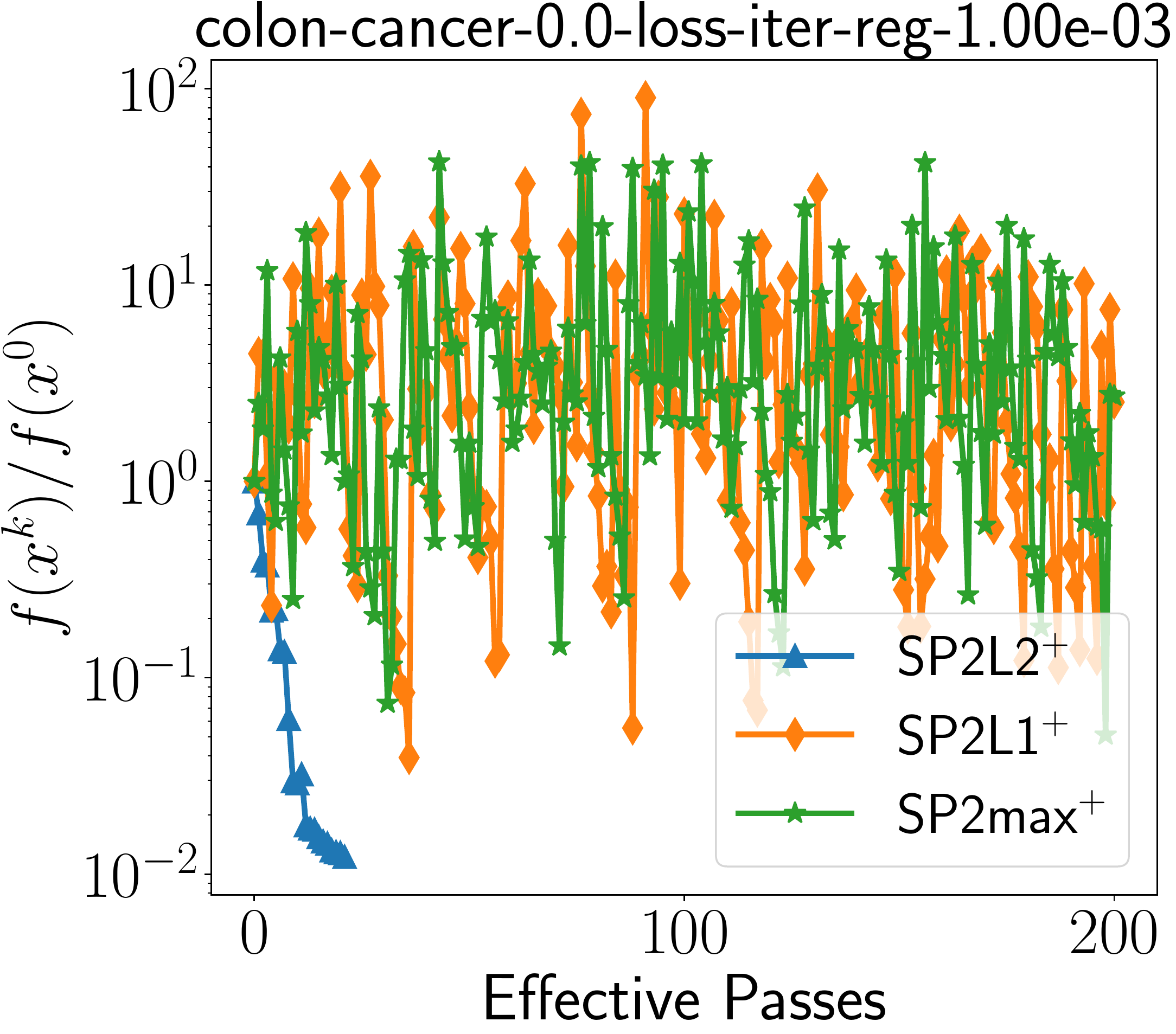}
\centerline{\small{(h) $\lambda = 0.8$}}
\end{minipage}
\hfill
\begin{minipage}{0.32\linewidth}
\centering
\includegraphics[width=1.7in]{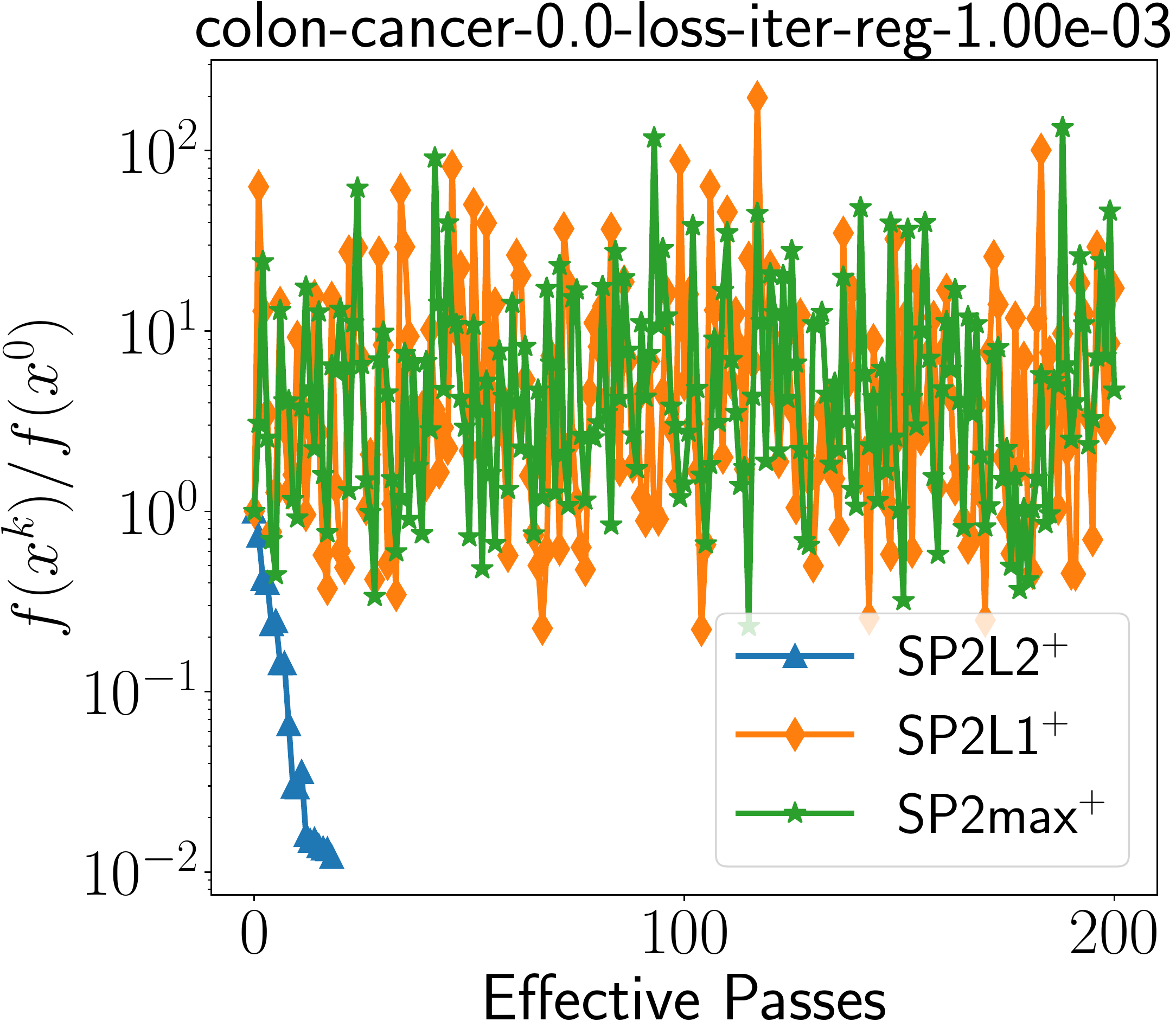}
\centerline{\small{(i) $\lambda = 0.9$}}
\end{minipage}
\caption{Colon-cancer: loss at each epoch with different $\lambda$. }
\label{fig:colon_loss_diff_lamb}
\end{figure}

% ====== mushrooms ======

\begin{figure}[t]
\begin{minipage}{0.32\linewidth}
\centering
\includegraphics[width=1.7in]{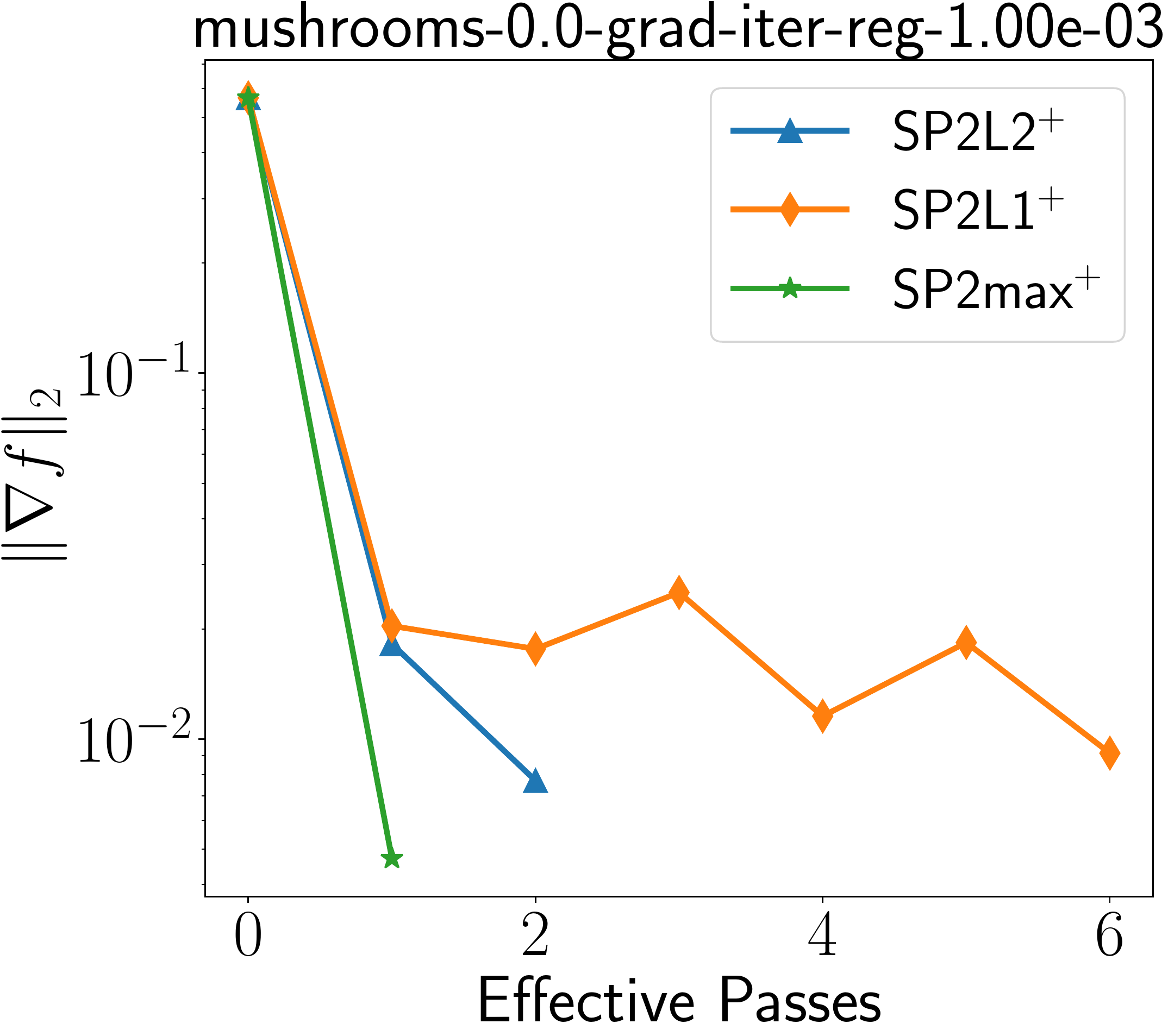}
\centerline{\small{(a) $\lambda = 0.1$}}
\end{minipage}
\hfill
\begin{minipage}{0.32\linewidth}
\centering
\includegraphics[width=1.7in]{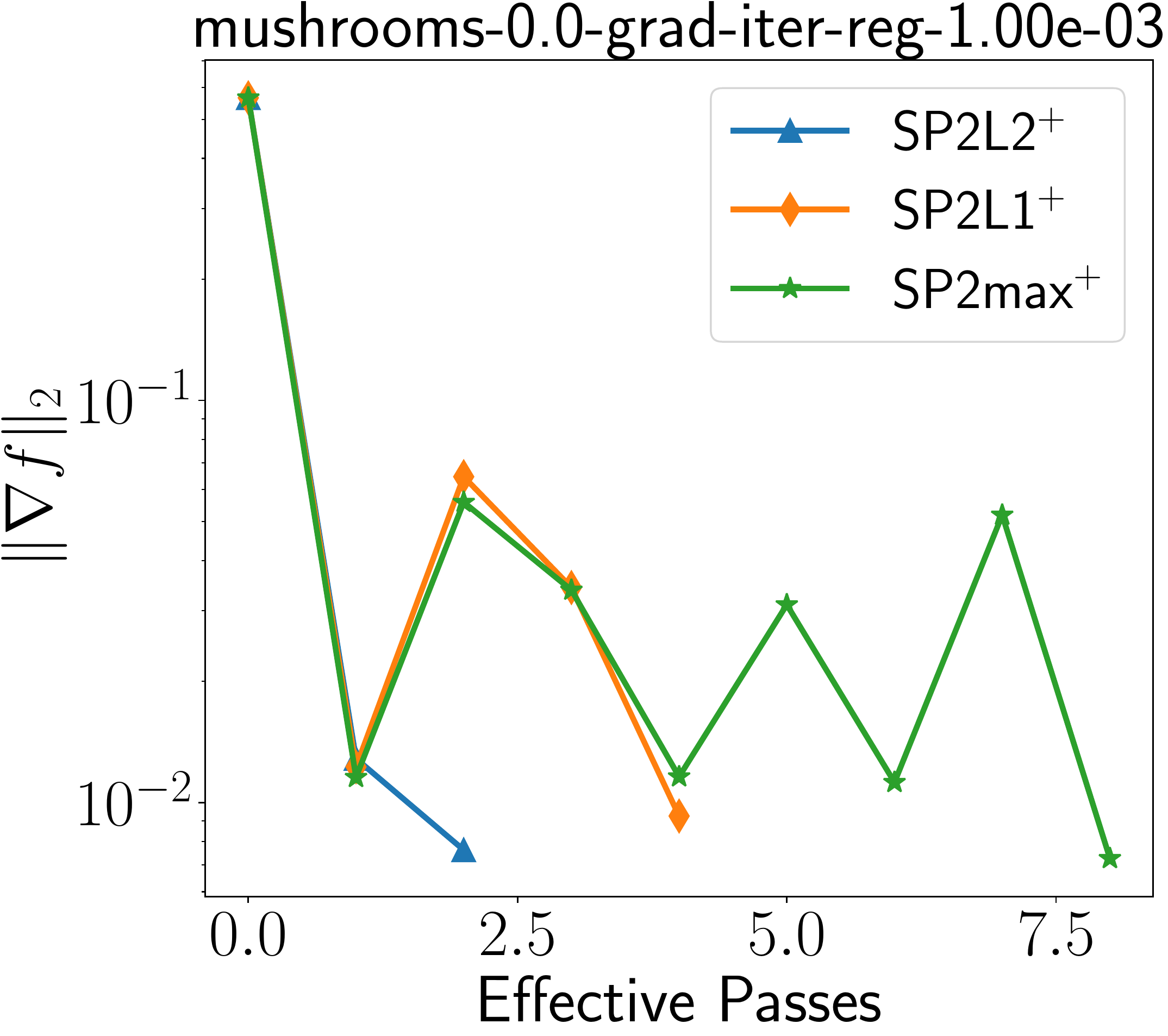}
\centerline{\small{(b) $\lambda = 0.2$}}
\end{minipage}
\hfill
\begin{minipage}{0.32\linewidth}
\centering
\includegraphics[width=1.7in]{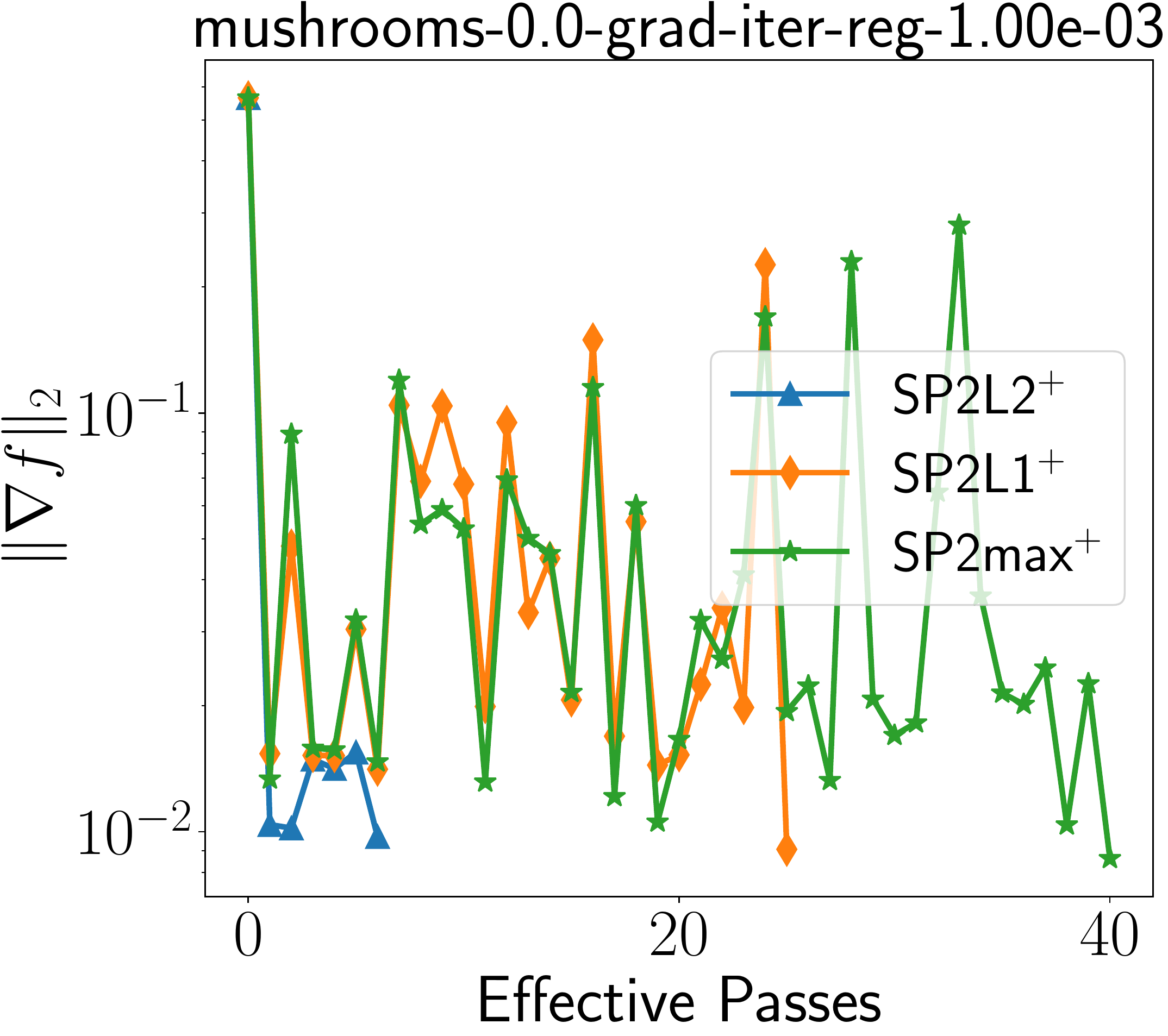}
\centerline{\small{(c) $\lambda = 0.3$}}
\end{minipage}
\\
\begin{minipage}{0.32\linewidth}
\centering
\includegraphics[width=1.7in]{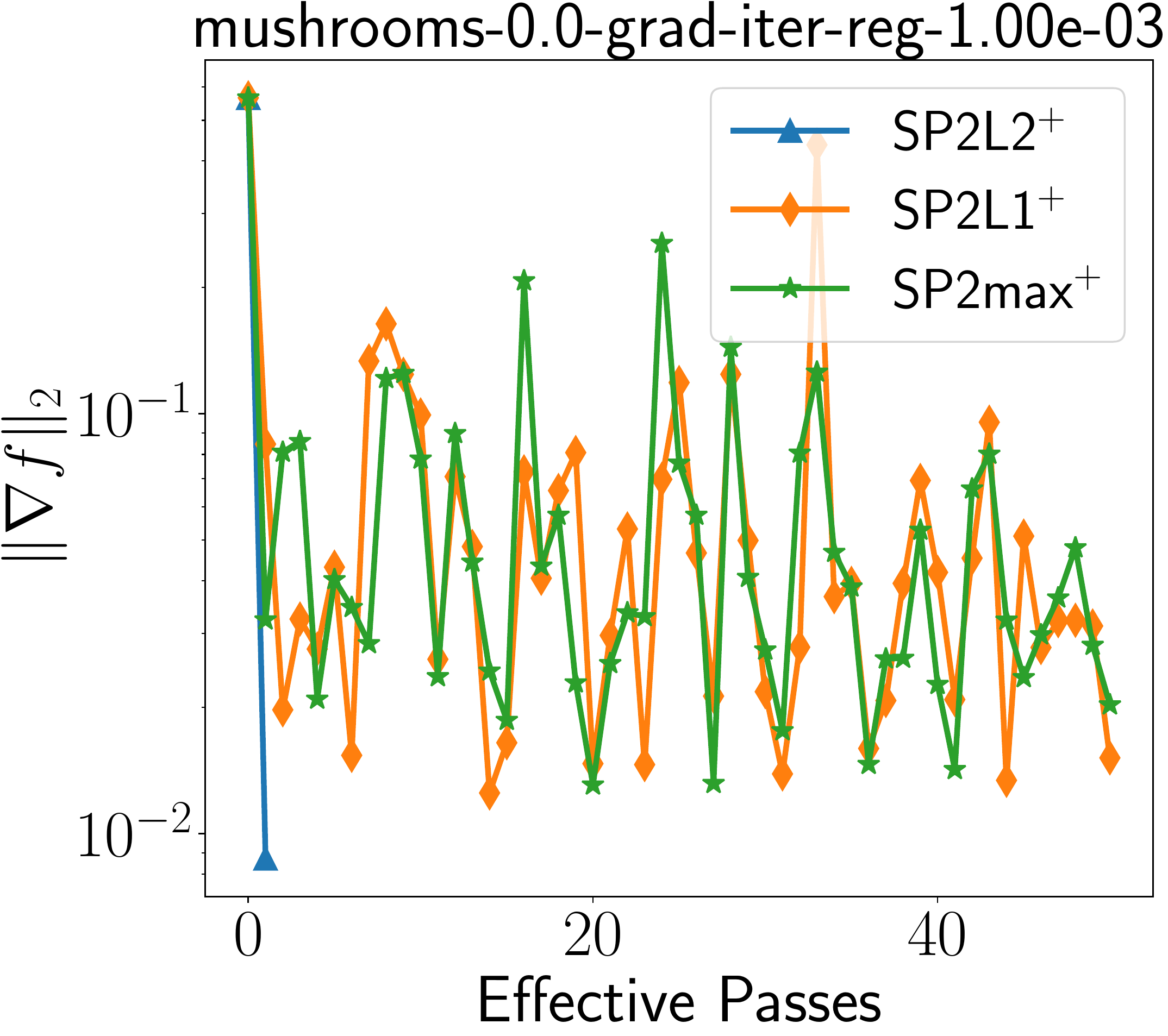}
\centerline{\small{(d) $\lambda = 0.4$}}
\end{minipage}
\hfill
\begin{minipage}{0.32\linewidth}
\centering
\includegraphics[width=1.7in]{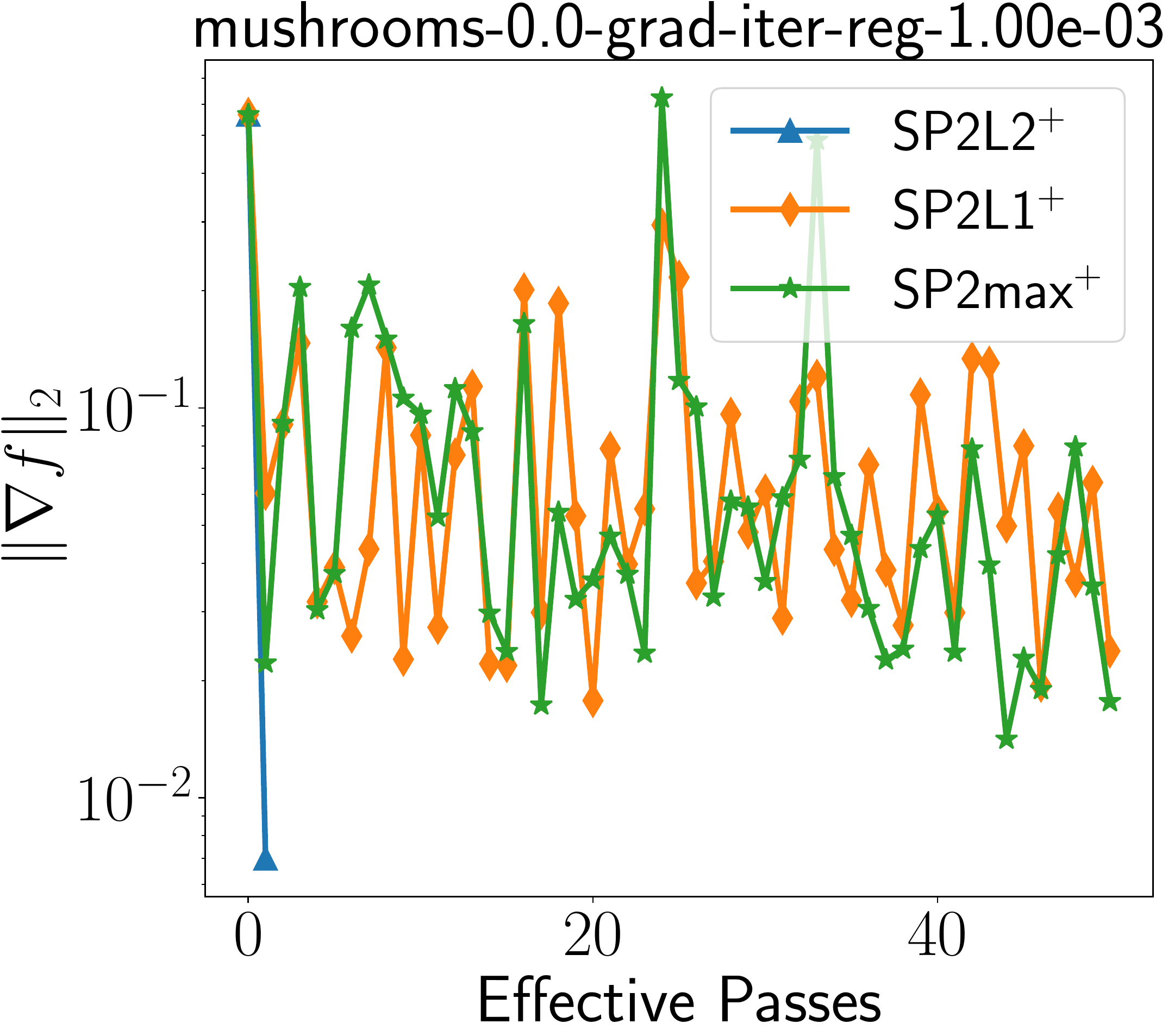}
\centerline{\small{(e) $\lambda = 0.5$}}
\end{minipage}
\hfill
\begin{minipage}{0.32\linewidth}
\centering
\includegraphics[width=1.7in]{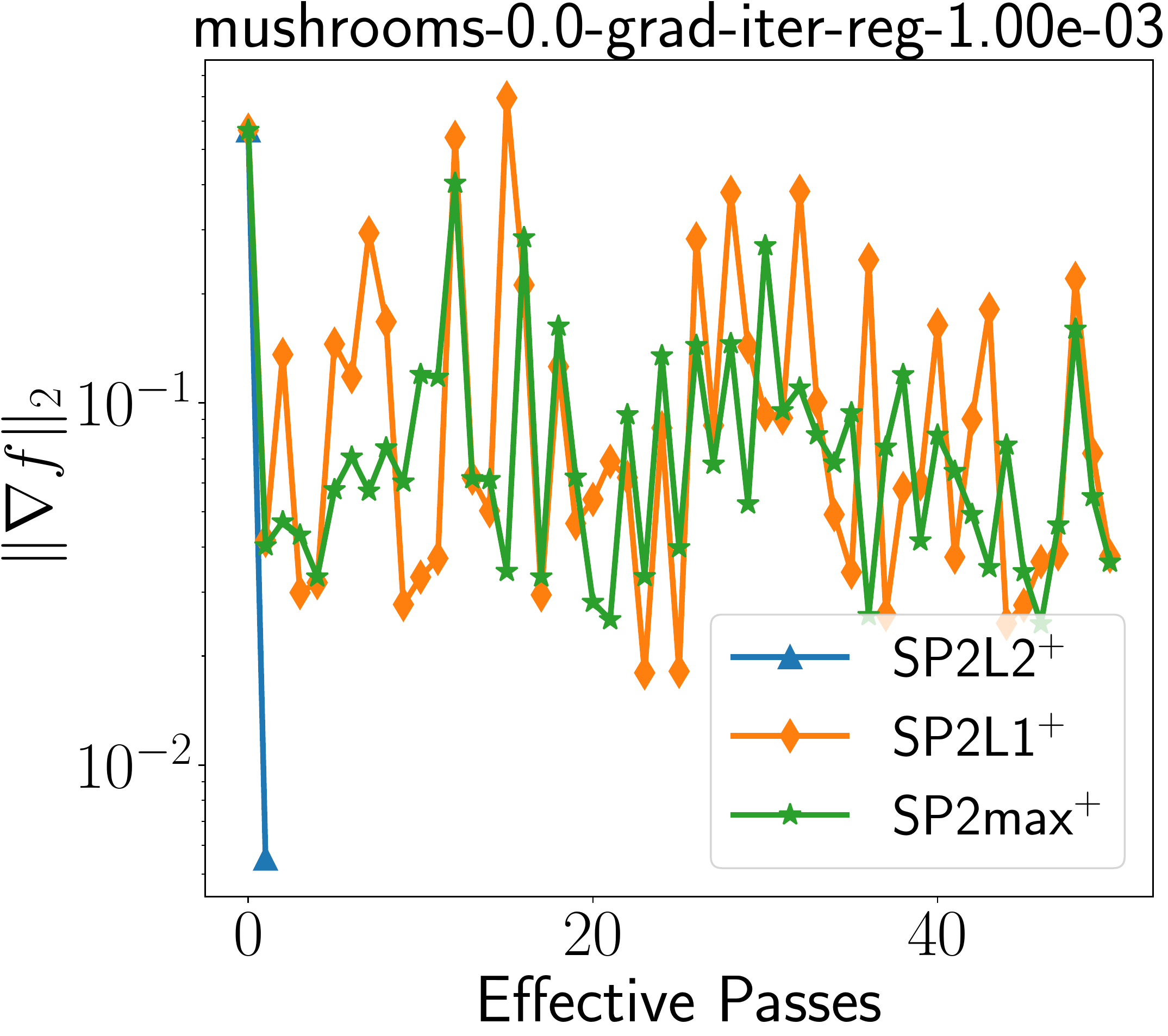}
\centerline{\small{(f) $\lambda = 0.6$}}
\end{minipage}
\\
\begin{minipage}{0.32\linewidth}
\centering
\includegraphics[width=1.7in]{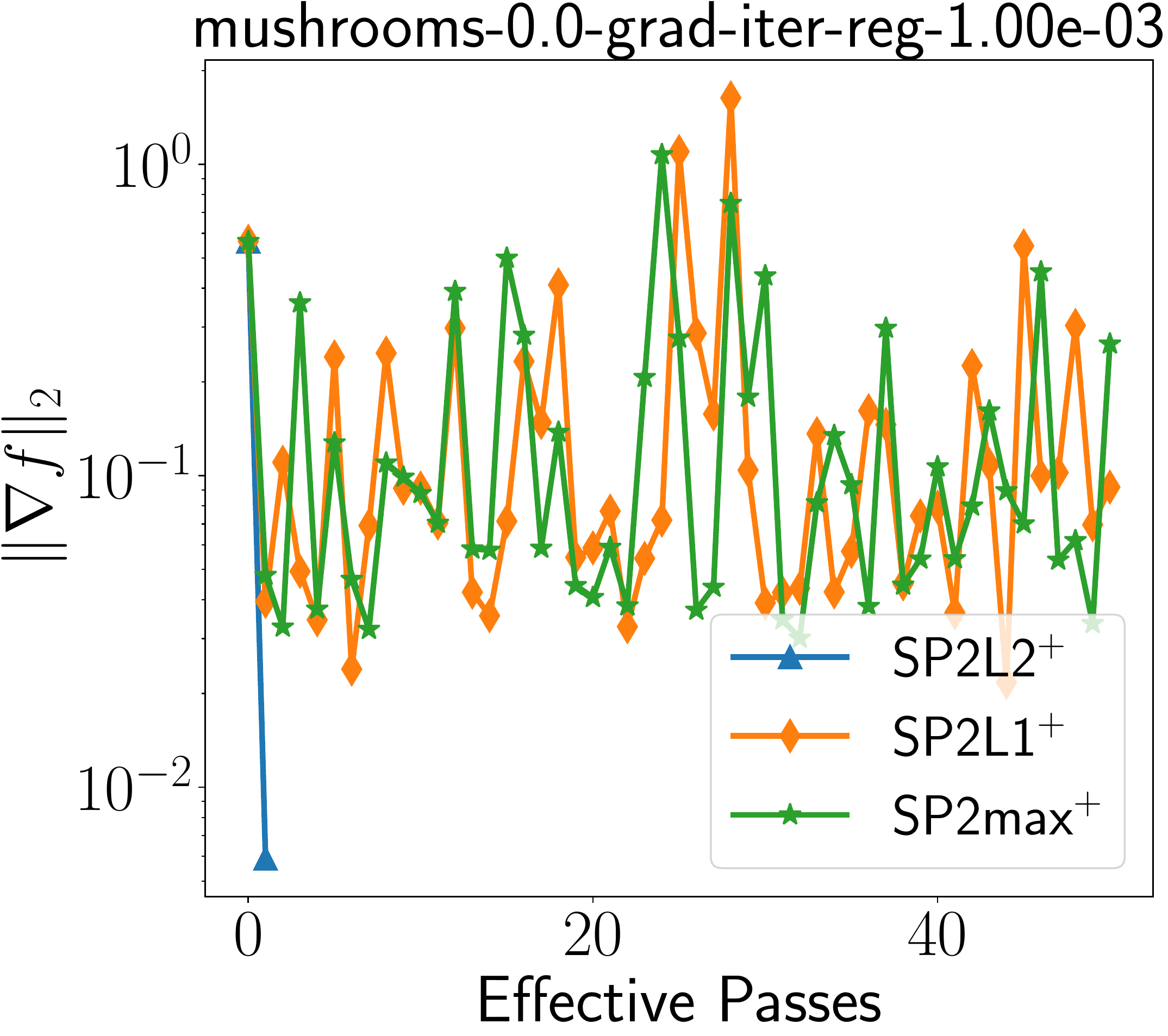}
\centerline{\small{(g) $\lambda = 0.7$}}
\end{minipage}
\hfill
\begin{minipage}{0.32\linewidth}
\centering
\includegraphics[width=1.7in]{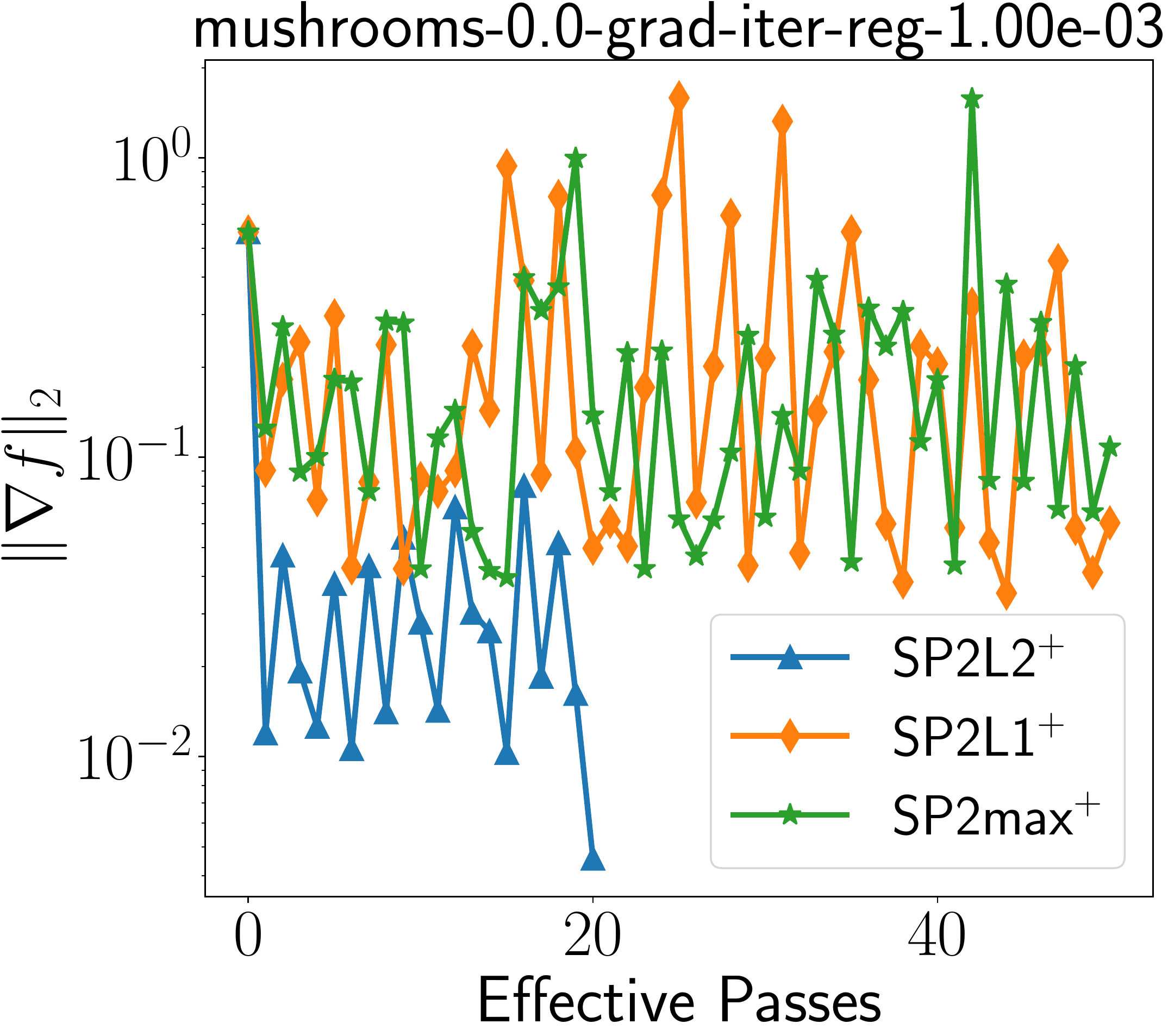}
\centerline{\small{(h) $\lambda = 0.8$}}
\end{minipage}
\hfill
\begin{minipage}{0.32\linewidth}
\centering
\includegraphics[width=1.7in]{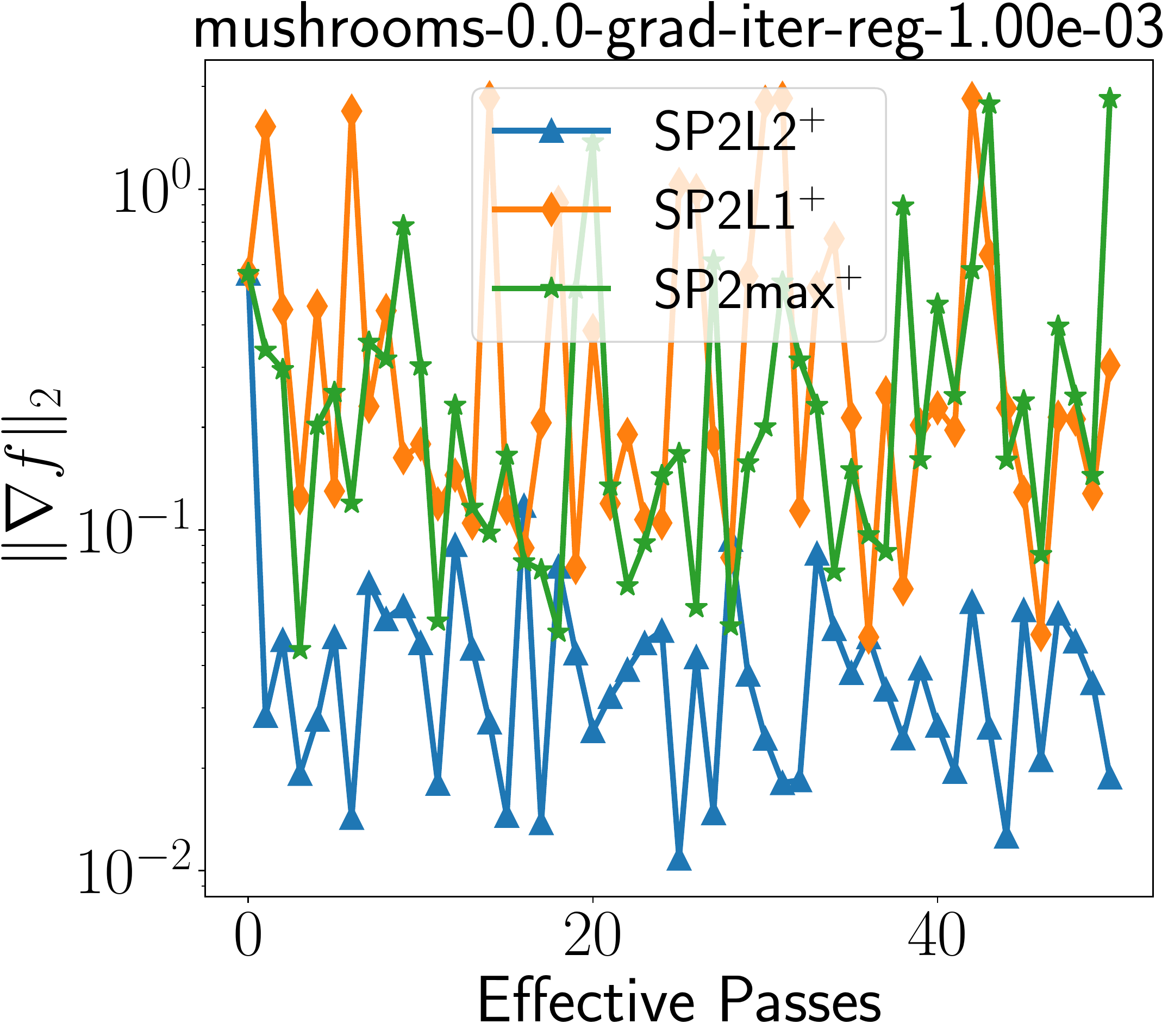}
\centerline{\small{(i) $\lambda = 0.9$}}
\end{minipage}
\caption{Mushrooms: gradient norm at each epoch with different $\lambda$. }
\label{fig:mush_grad_diff_lamb}
\end{figure}

\begin{figure}[t]
\begin{minipage}{0.32\linewidth}
\centering
\includegraphics[width=1.7in]{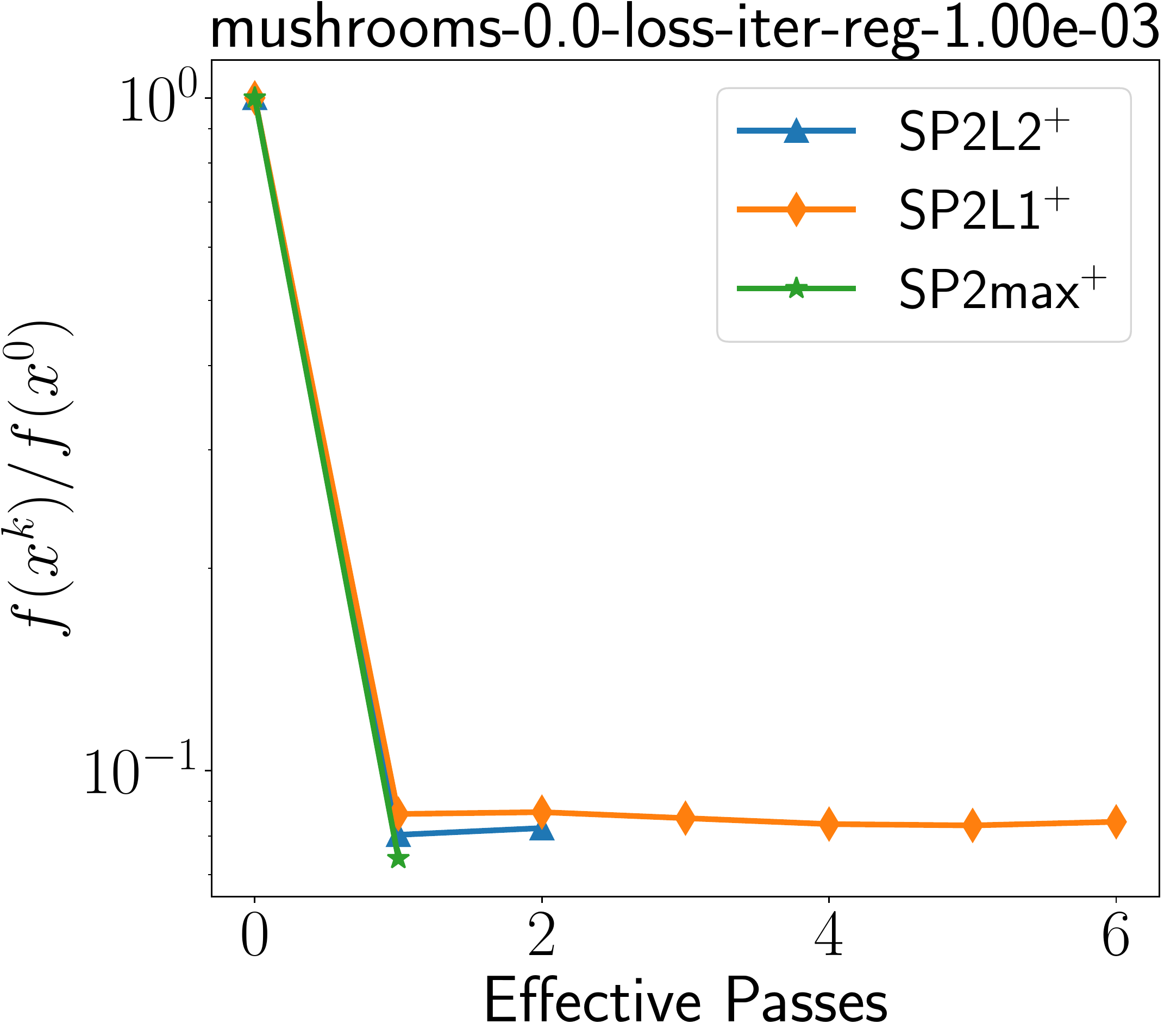}
\centerline{\small{(a) $\lambda = 0.1$}}
\end{minipage}
\hfill
\begin{minipage}{0.32\linewidth}
\centering
\includegraphics[width=1.7in]{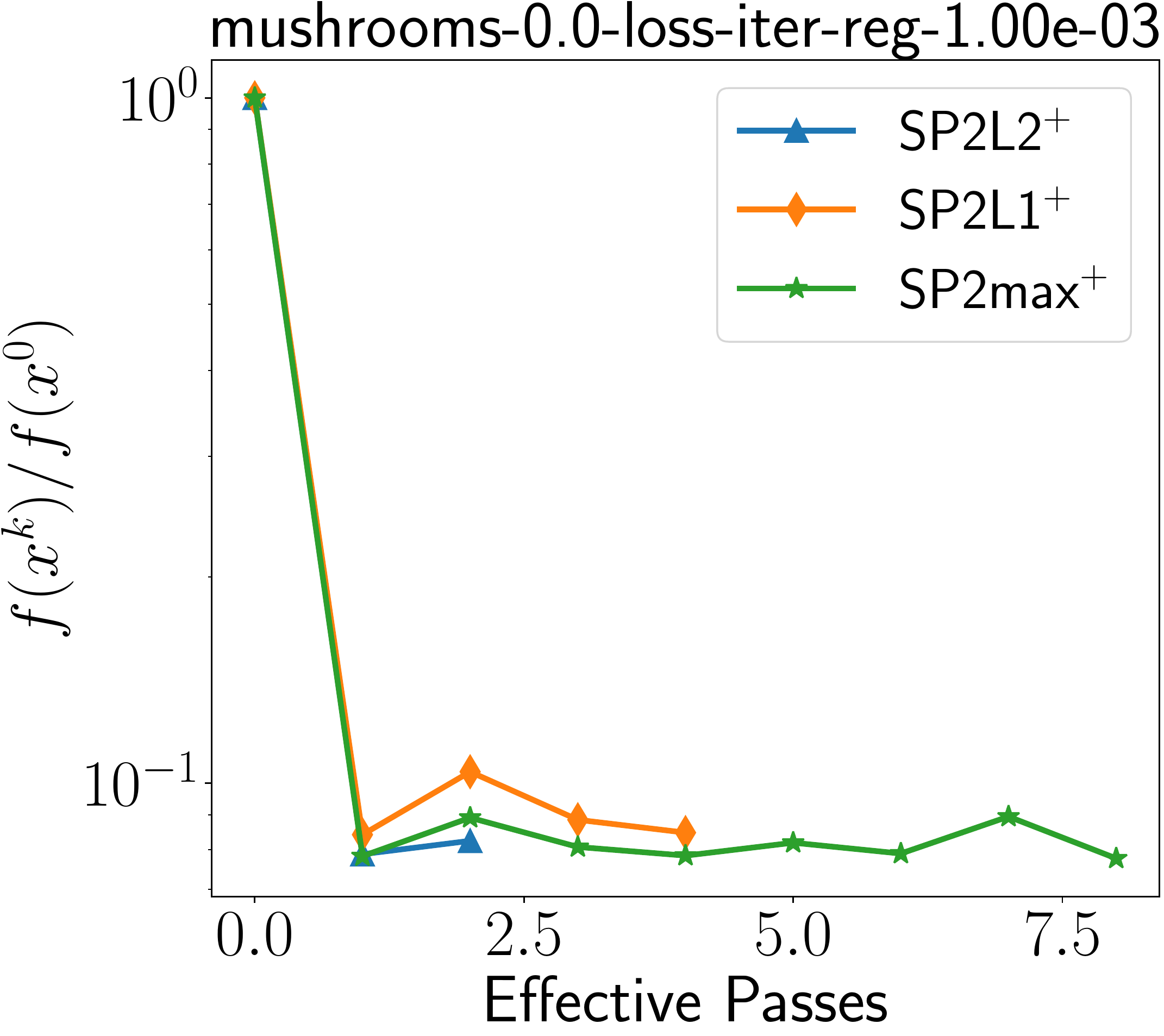}
\centerline{\small{(b) $\lambda = 0.2$}}
\end{minipage}
\hfill
\begin{minipage}{0.32\linewidth}
\centering
\includegraphics[width=1.7in]{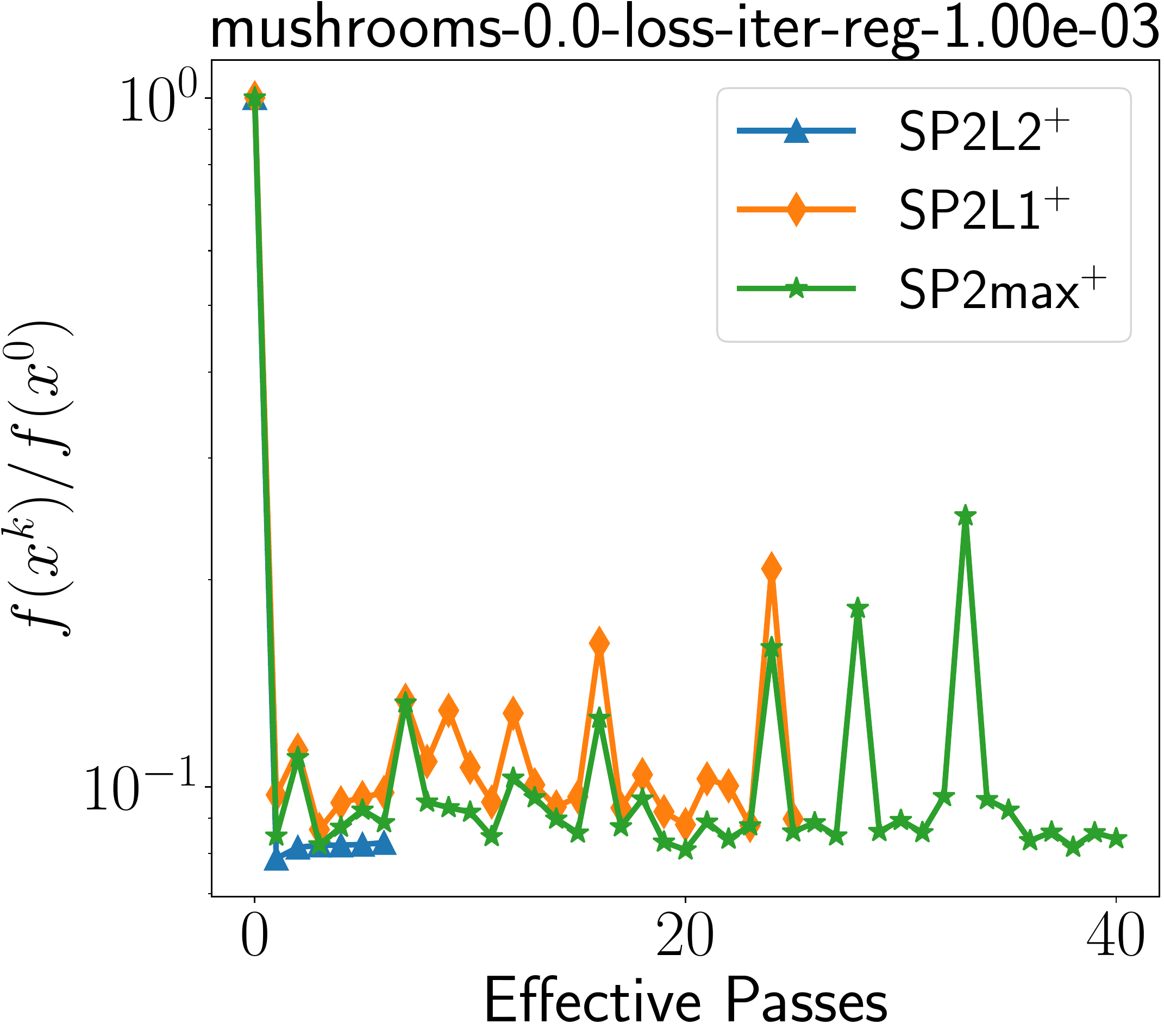}
\centerline{\small{(c) $\lambda = 0.3$}}
\end{minipage}
\\
\begin{minipage}{0.32\linewidth}
\centering
\includegraphics[width=1.7in]{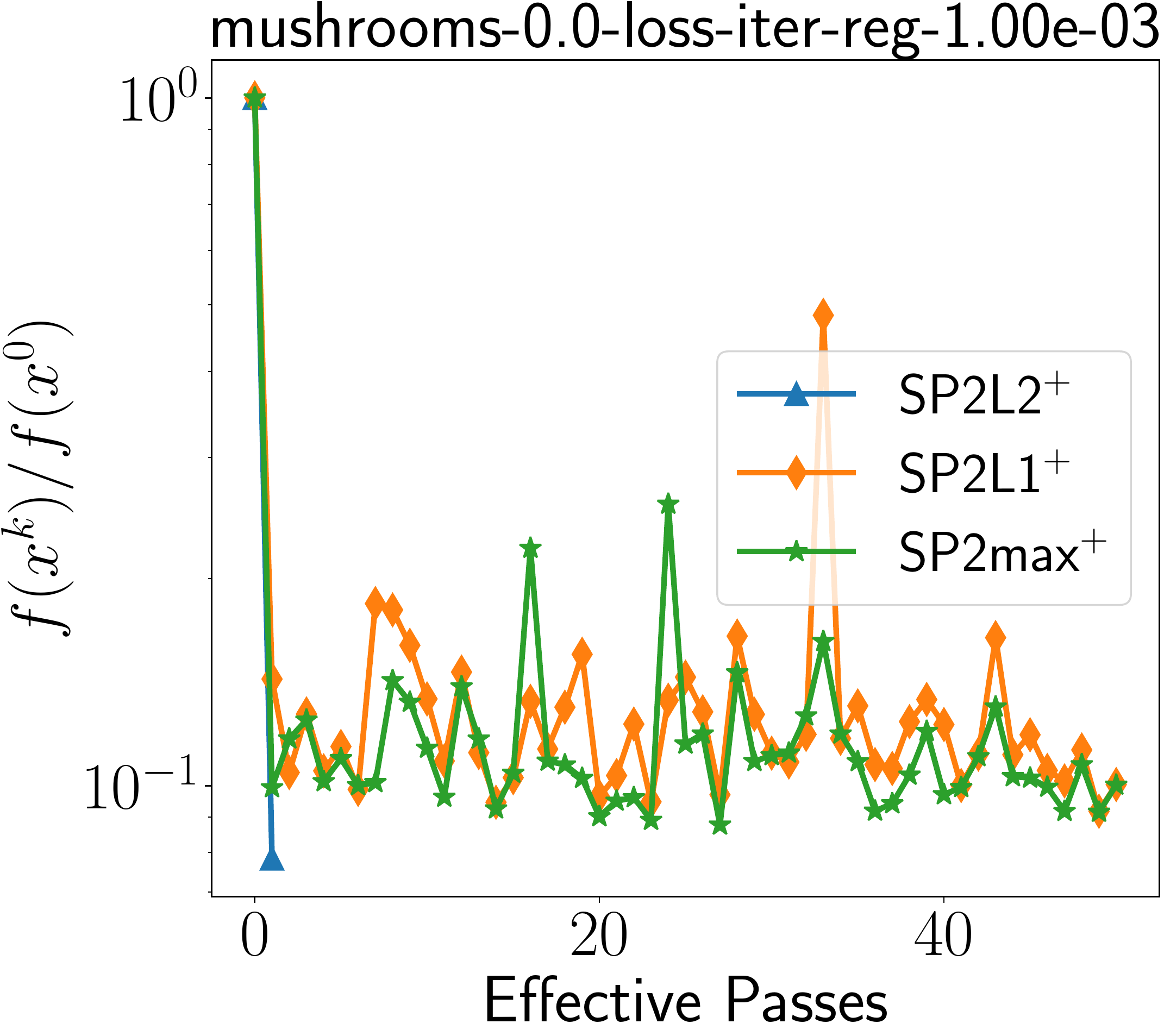}
\centerline{\small{(d) $\lambda = 0.4$}}
\end{minipage}
\hfill
\begin{minipage}{0.32\linewidth}
\centering
\includegraphics[width=1.7in]{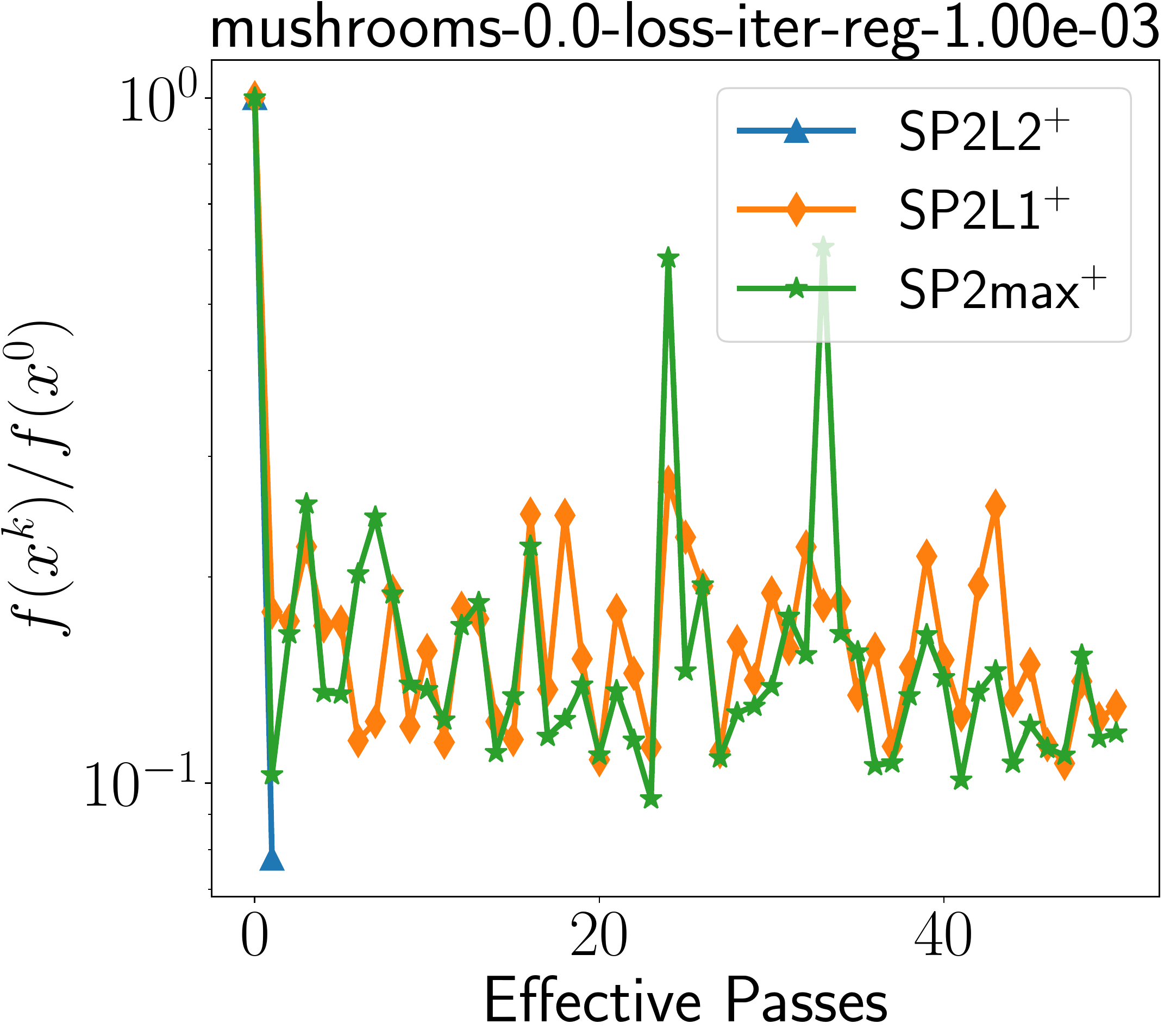}
\centerline{\small{(e) $\lambda = 0.5$}}
\end{minipage}
\hfill
\begin{minipage}{0.32\linewidth}
\centering
\includegraphics[width=1.7in]{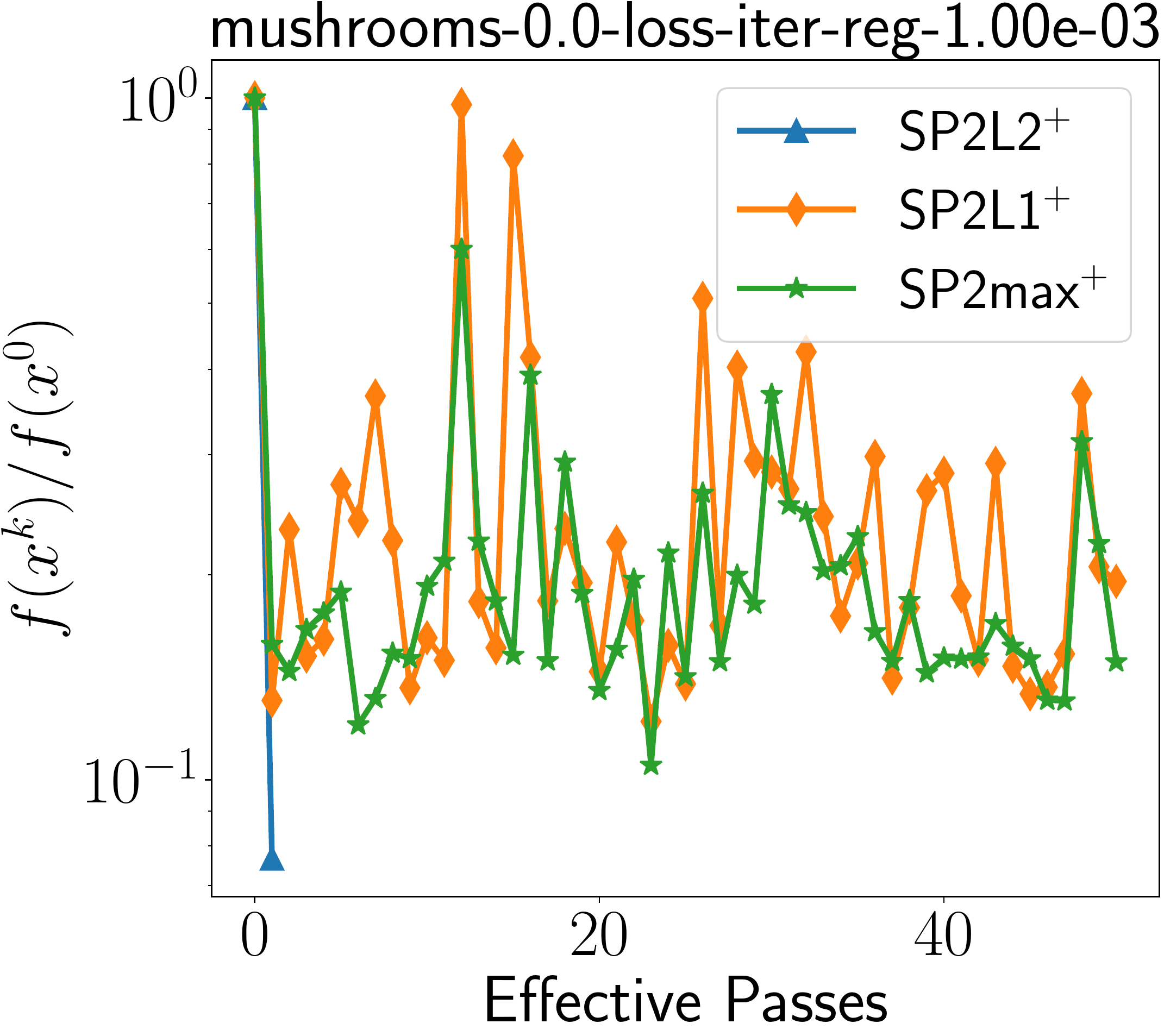}
\centerline{\small{(f) $\lambda = 0.6$}}
\end{minipage}
\\
\begin{minipage}{0.32\linewidth}
\centering
\includegraphics[width=1.7in]{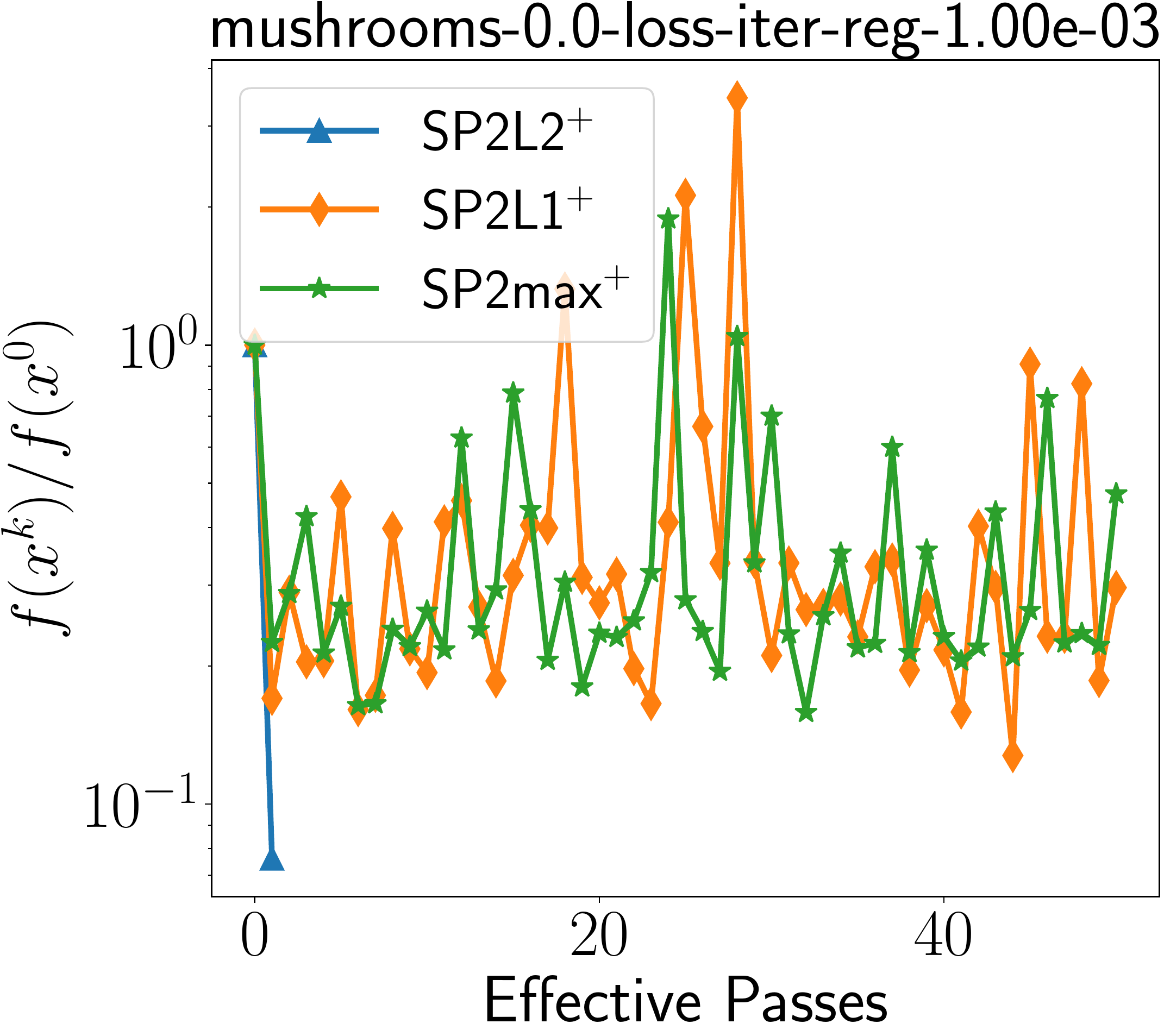}
\centerline{\small{(g) $\lambda = 0.7$}}
\end{minipage}
\hfill
\begin{minipage}{0.32\linewidth}
\centering
\includegraphics[width=1.7in]{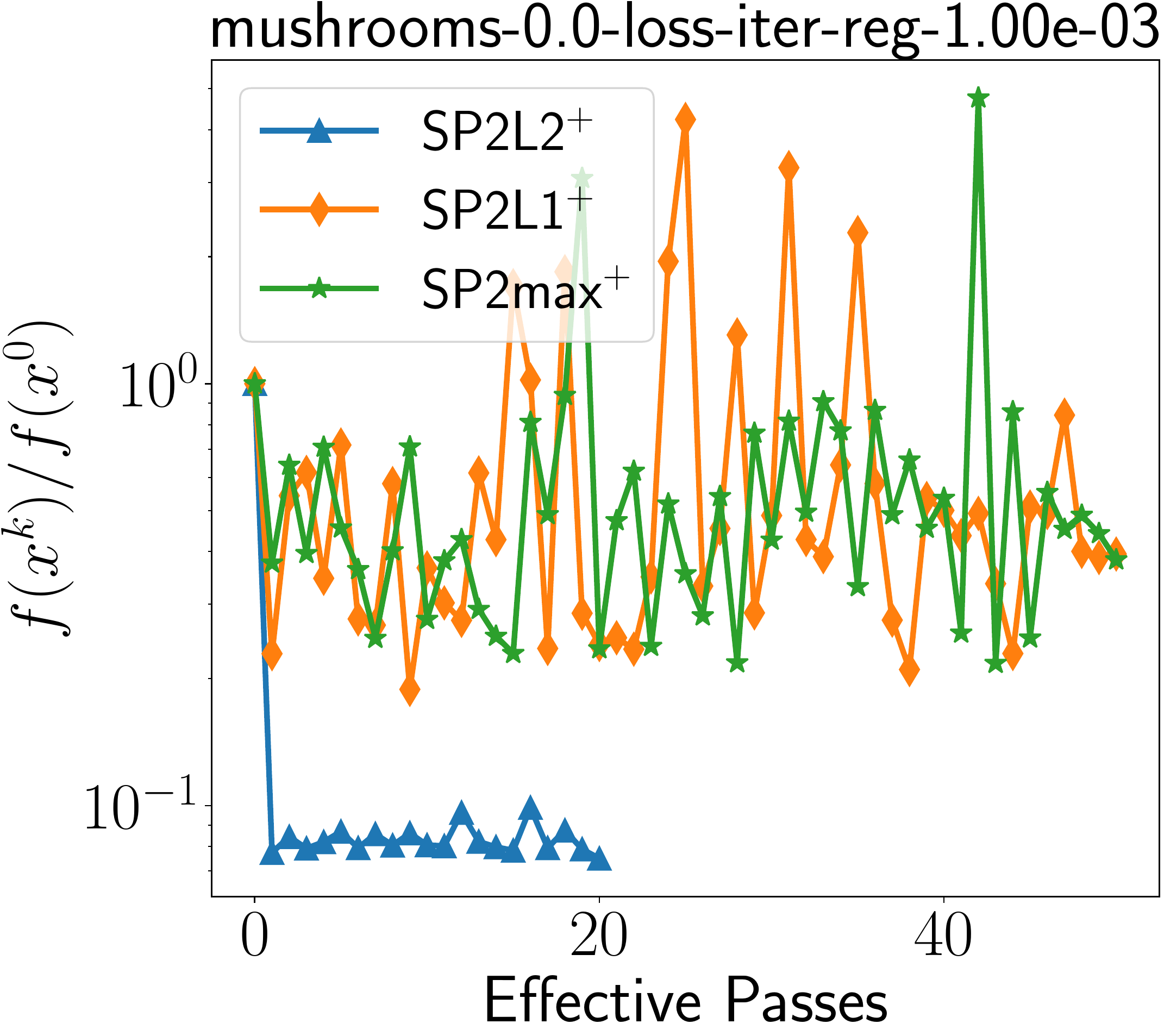}
\centerline{\small{(h) $\lambda = 0.8$}}
\end{minipage}
\hfill
\begin{minipage}{0.32\linewidth}
\centering
\includegraphics[width=1.7in]{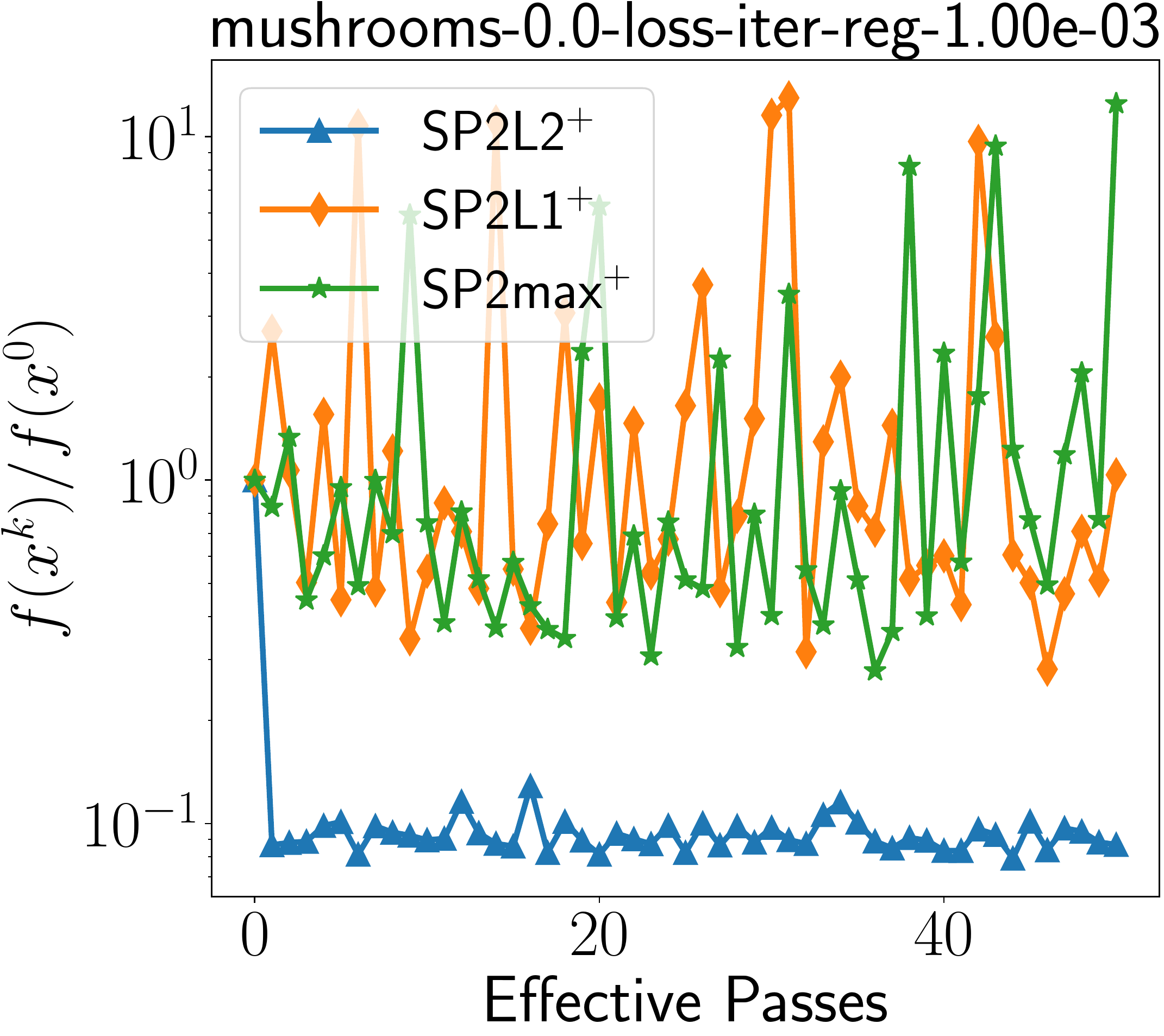}
\centerline{\small{(i) $\lambda = 0.9$}}
\end{minipage}
\caption{Mushrooms: loss at each epoch with different $\lambda$. }
\label{fig:mush_loss_diff_lamb}
\end{figure}

\begin{figure}[t]
\begin{minipage}{0.32\linewidth}
\centering
\includegraphics[width=1.7in]{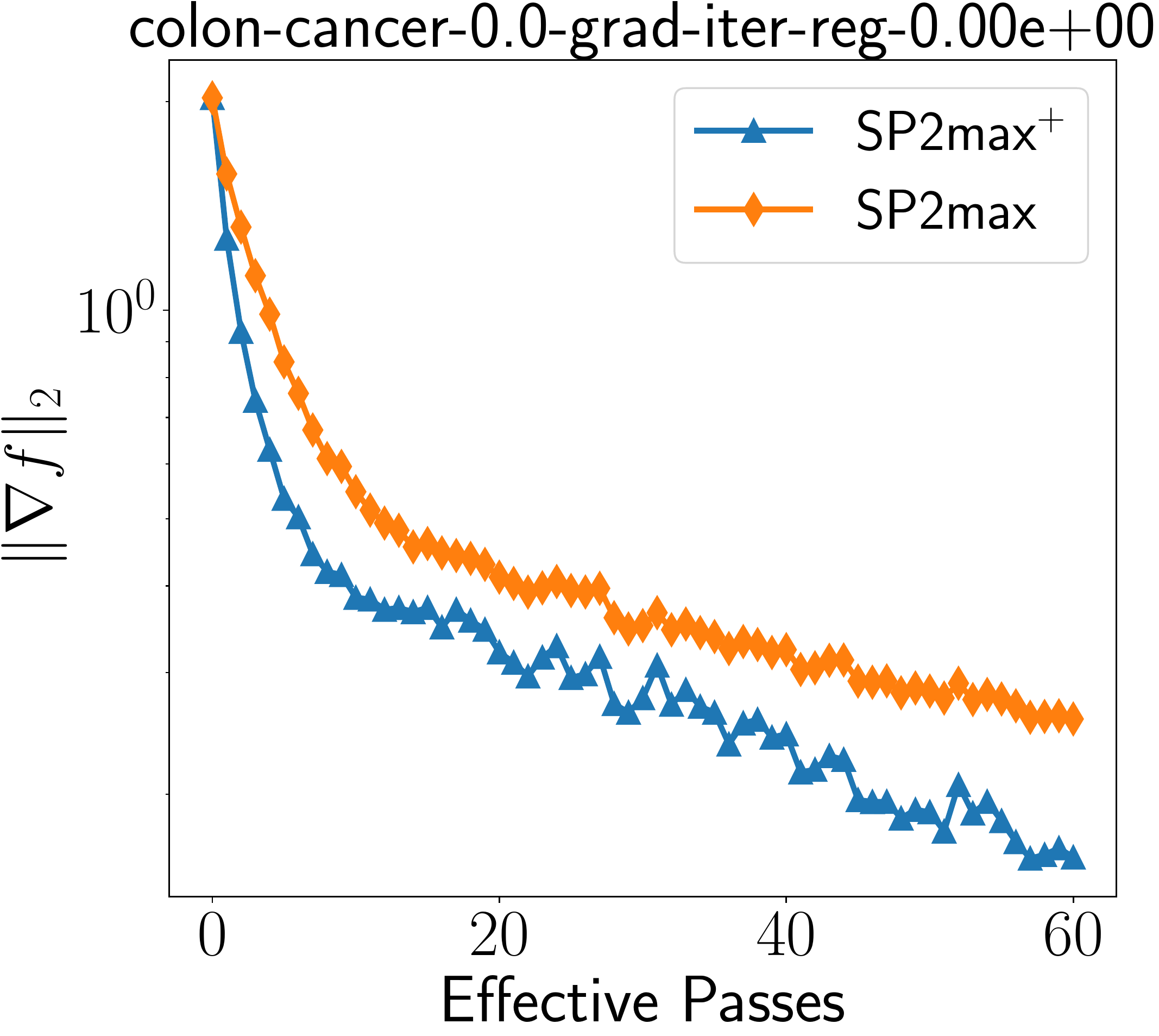}
\centerline{\small{(a) $\lambda = 0.001$}}
\end{minipage}
\hfill
\begin{minipage}{0.32\linewidth}
\centering
\includegraphics[width=1.7in]{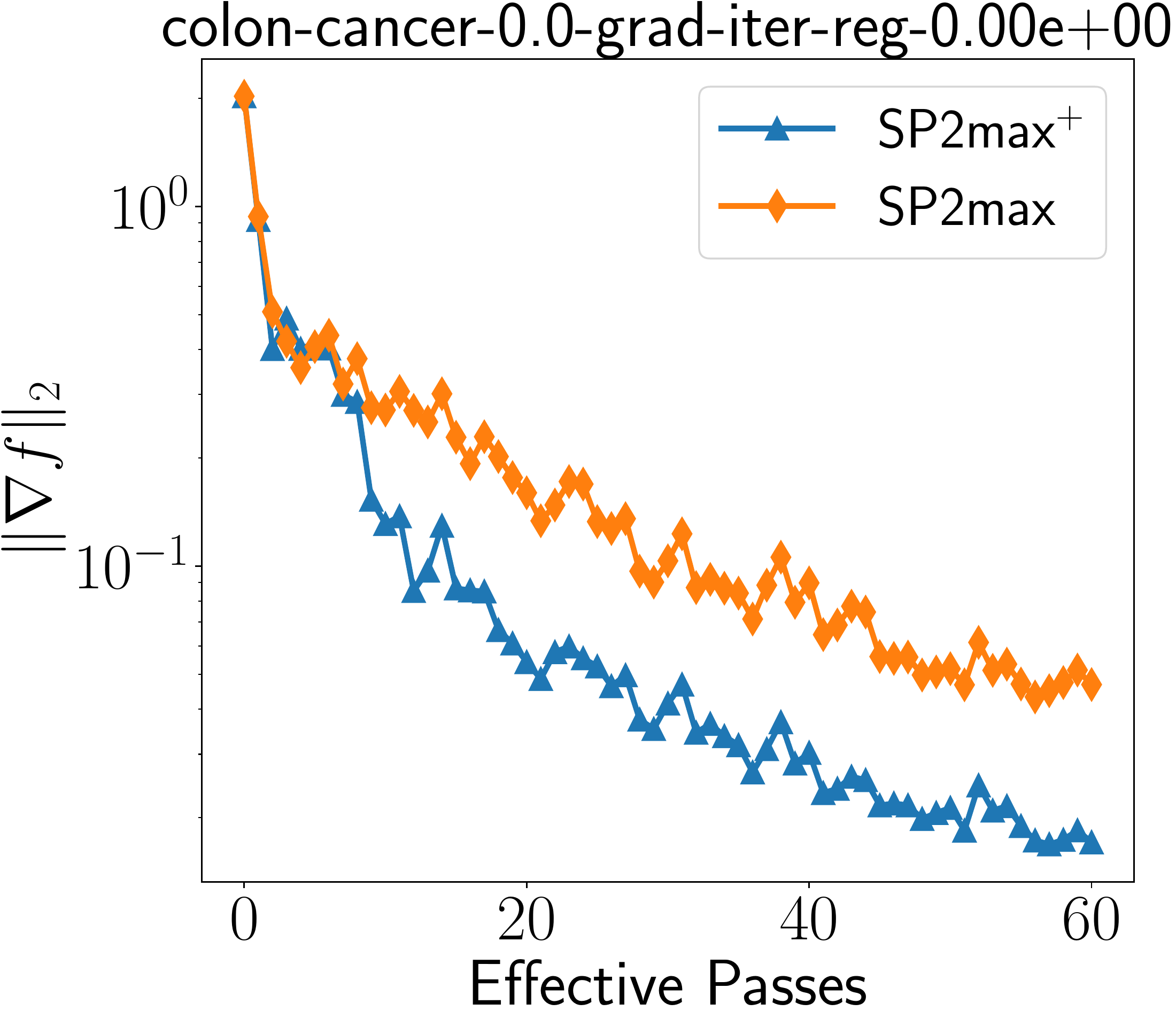}
\centerline{\small{(b) $\lambda = 0.01$}}
\end{minipage}
\hfill
\begin{minipage}{0.32\linewidth}
\centering
\includegraphics[width=1.7in]{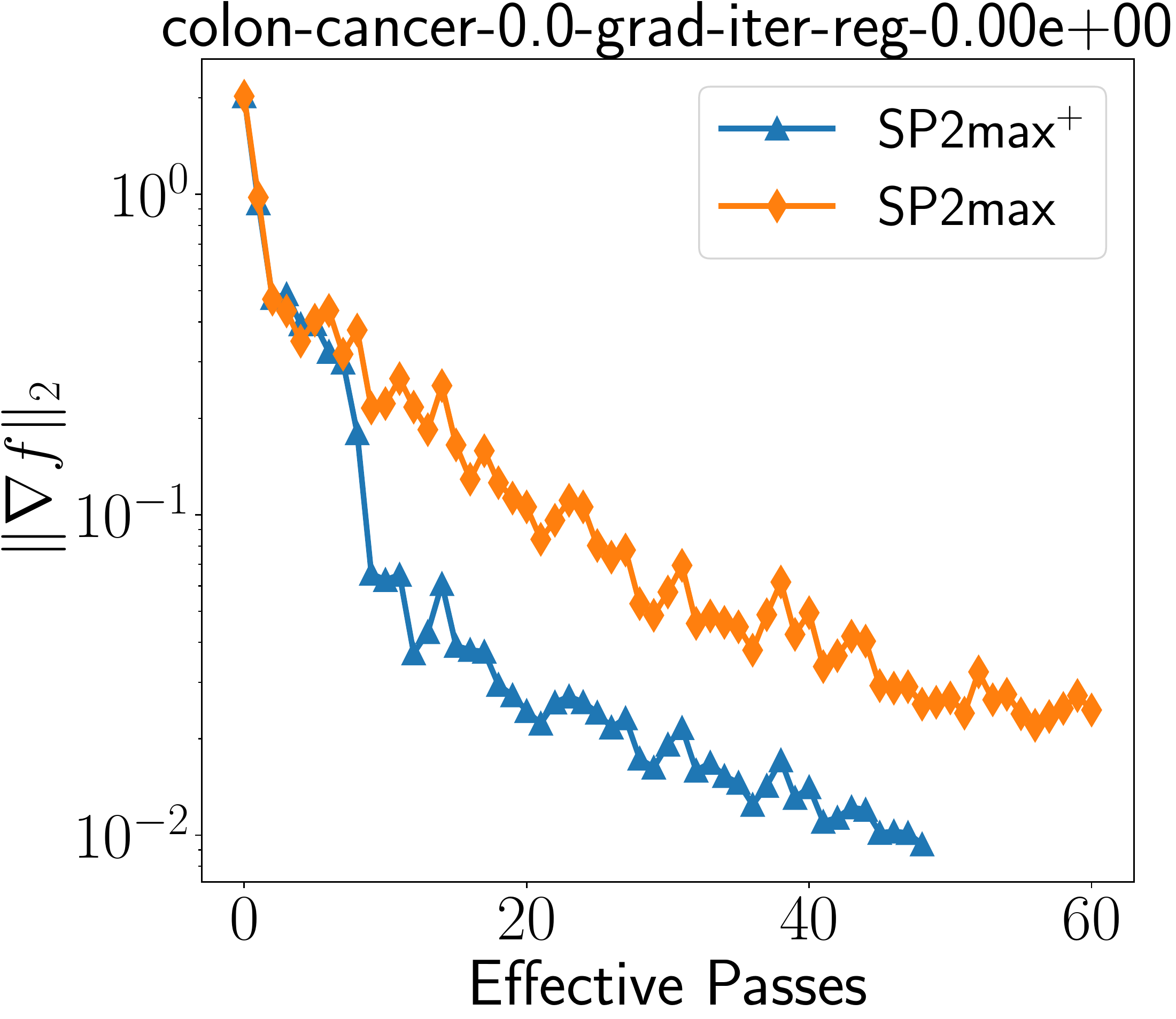}
\centerline{\small{(c) $\lambda = 0.02$}}
\end{minipage}
\\
\begin{minipage}{0.32\linewidth}
\centering
\includegraphics[width=1.7in]{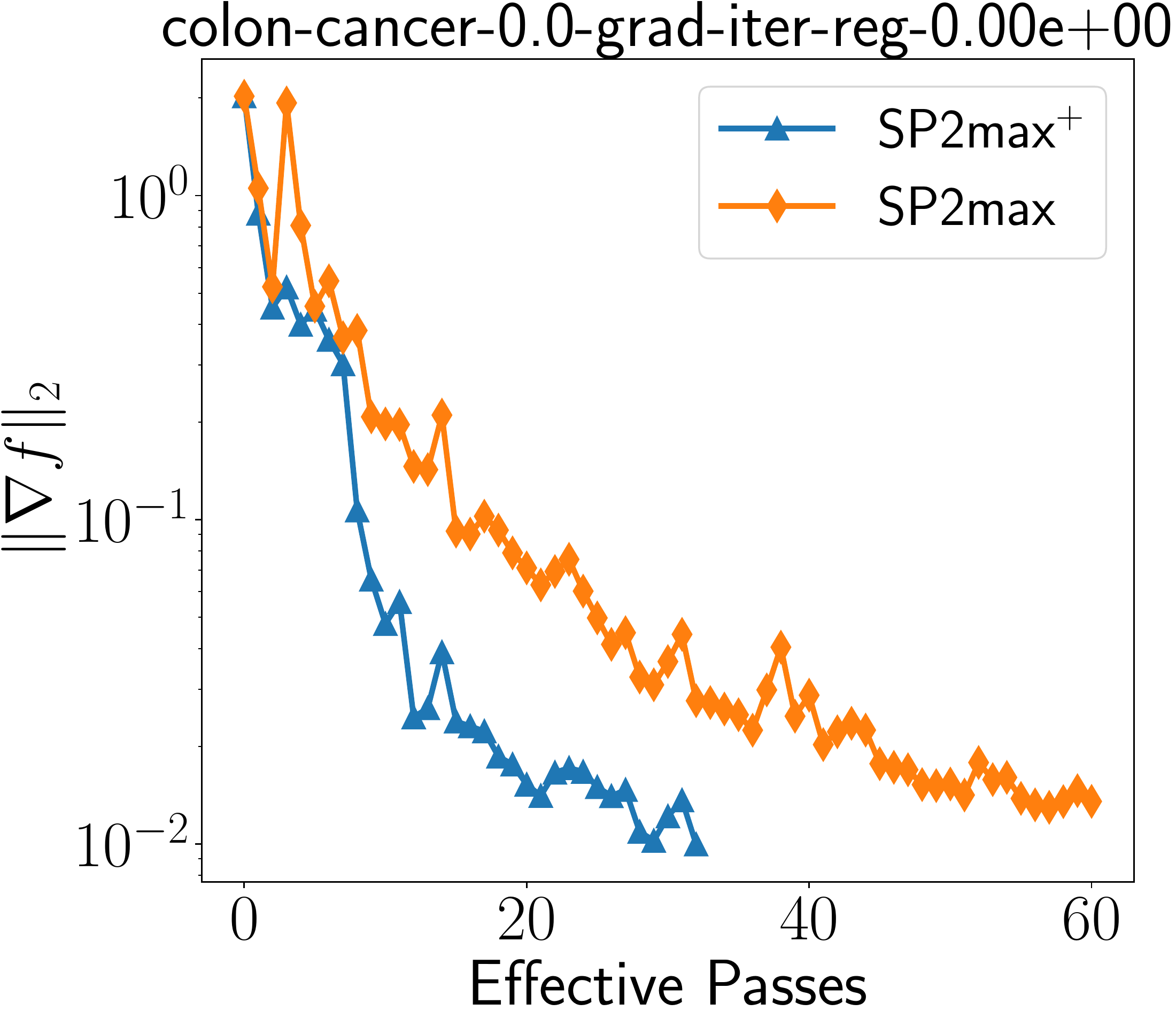}
\centerline{\small{(d) $\lambda = 0.03$}}
\end{minipage}
\hfill
\begin{minipage}{0.32\linewidth}
\centering
\includegraphics[width=1.7in]{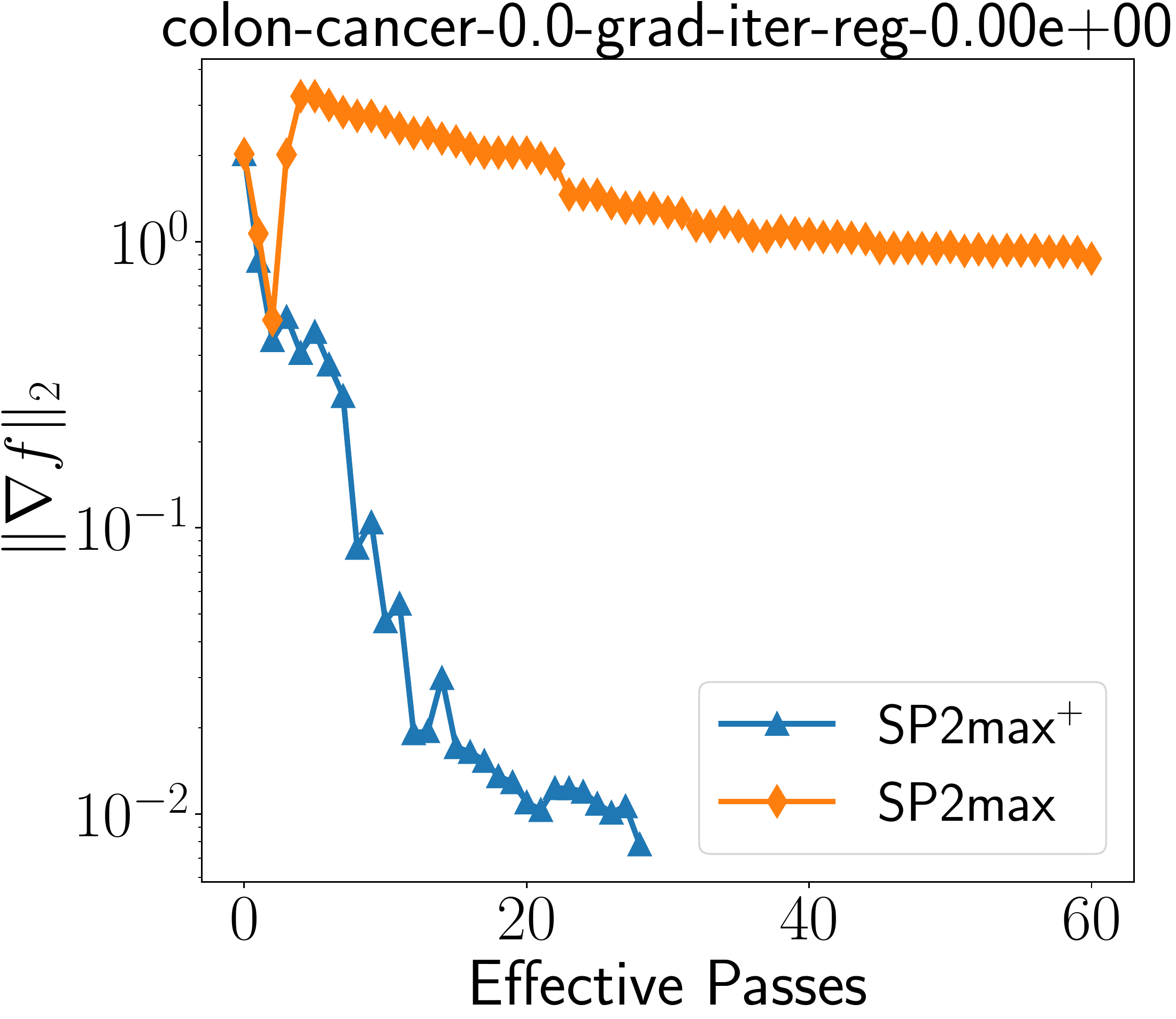}
\centerline{\small{(e) $\lambda = 0.04$}}
\end{minipage}
\hfill
\begin{minipage}{0.32\linewidth}
\centering
\includegraphics[width=1.7in]{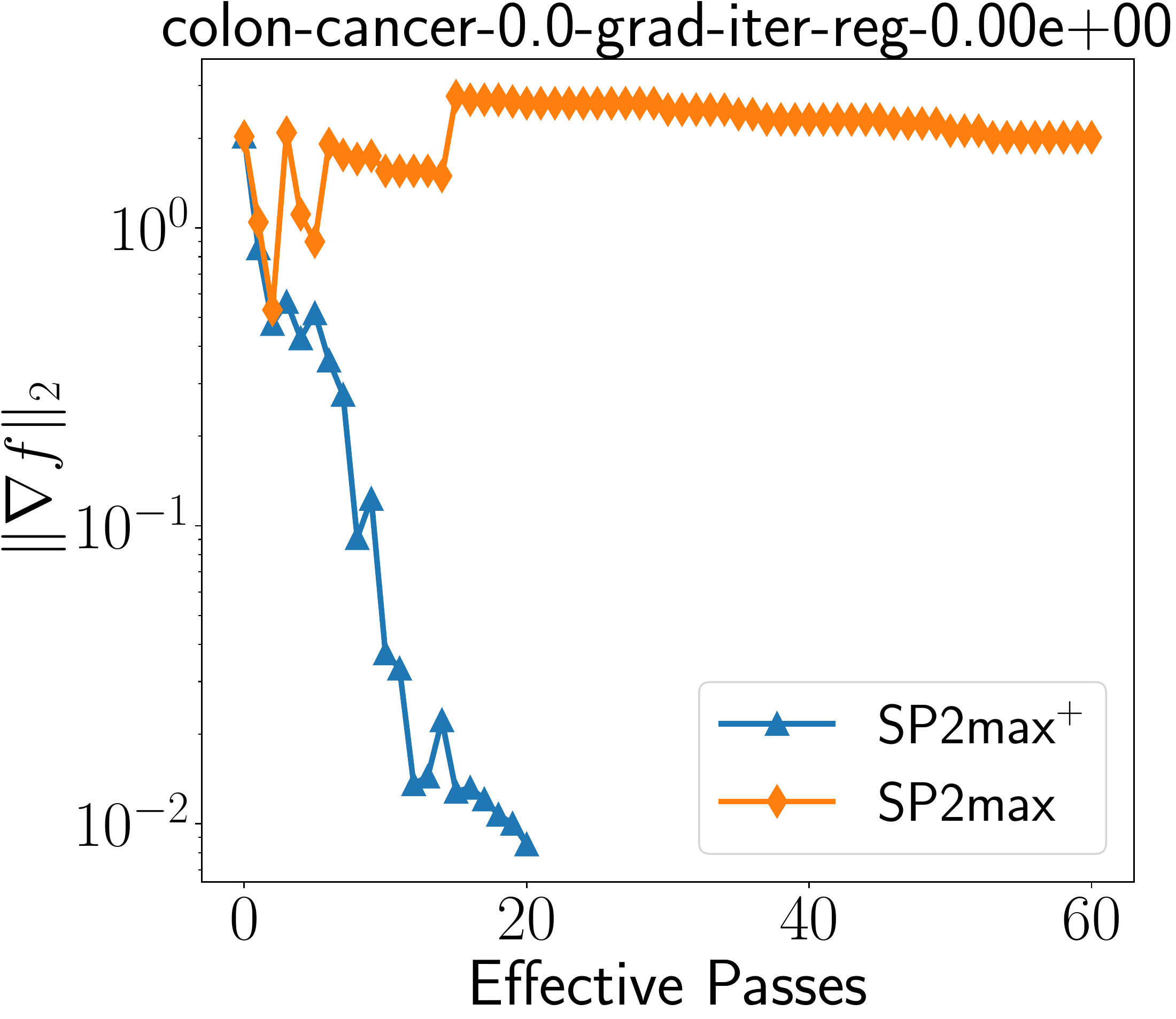}
\centerline{\small{(f) $\lambda = 0.05$}}
\end{minipage}
\caption{Colon-cancer: gradient norm at each epoch with different $\lambda$. }
\label{fig:colon_grad_diff_lamb_max}
\end{figure}

\begin{figure}[t]
\begin{minipage}{0.32\linewidth}
\centering
\includegraphics[width=1.7in]{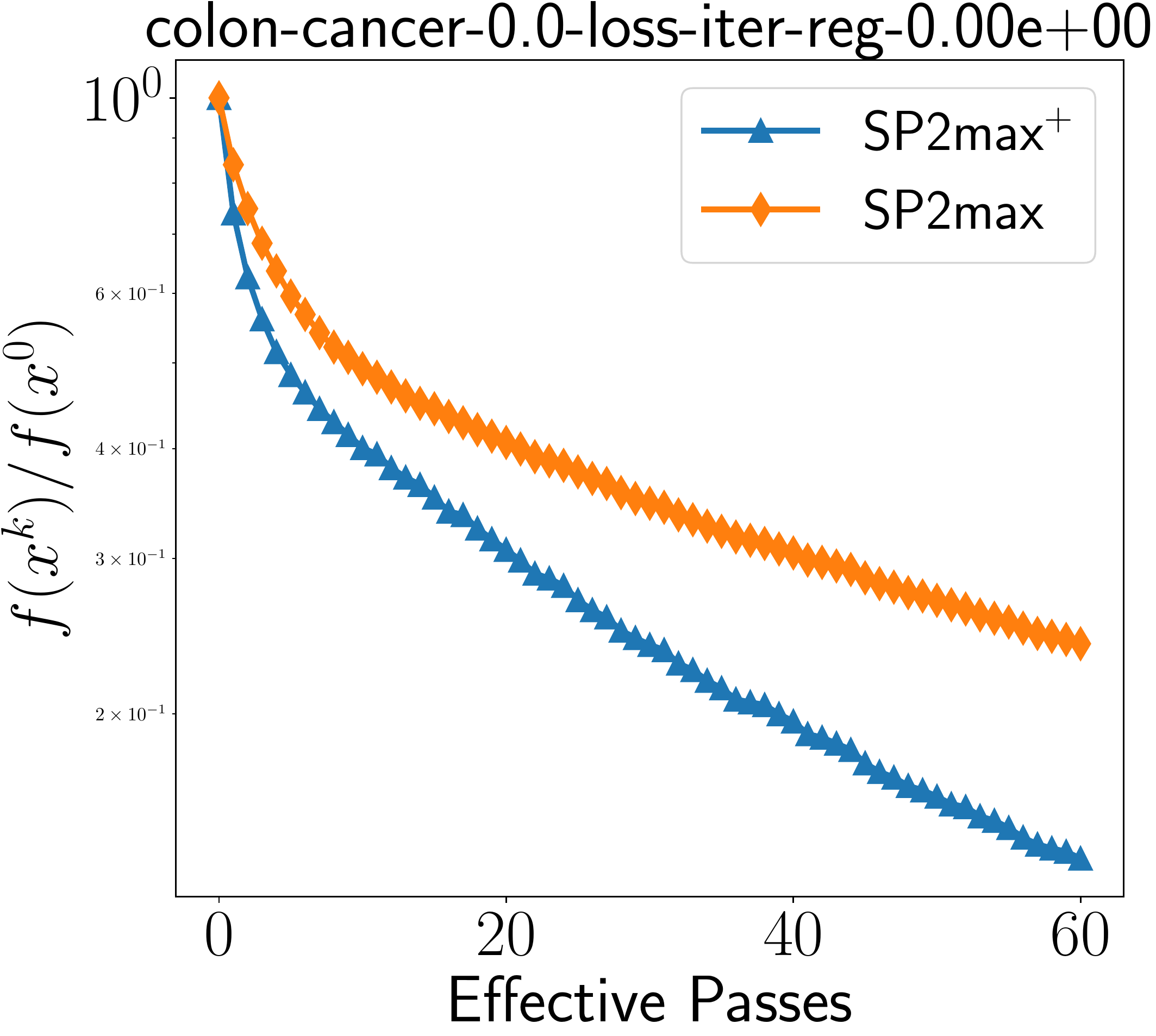}
\centerline{\small{(a) $\lambda = 0.001$}}
\end{minipage}
\hfill
\begin{minipage}{0.32\linewidth}
\centering
\includegraphics[width=1.7in]{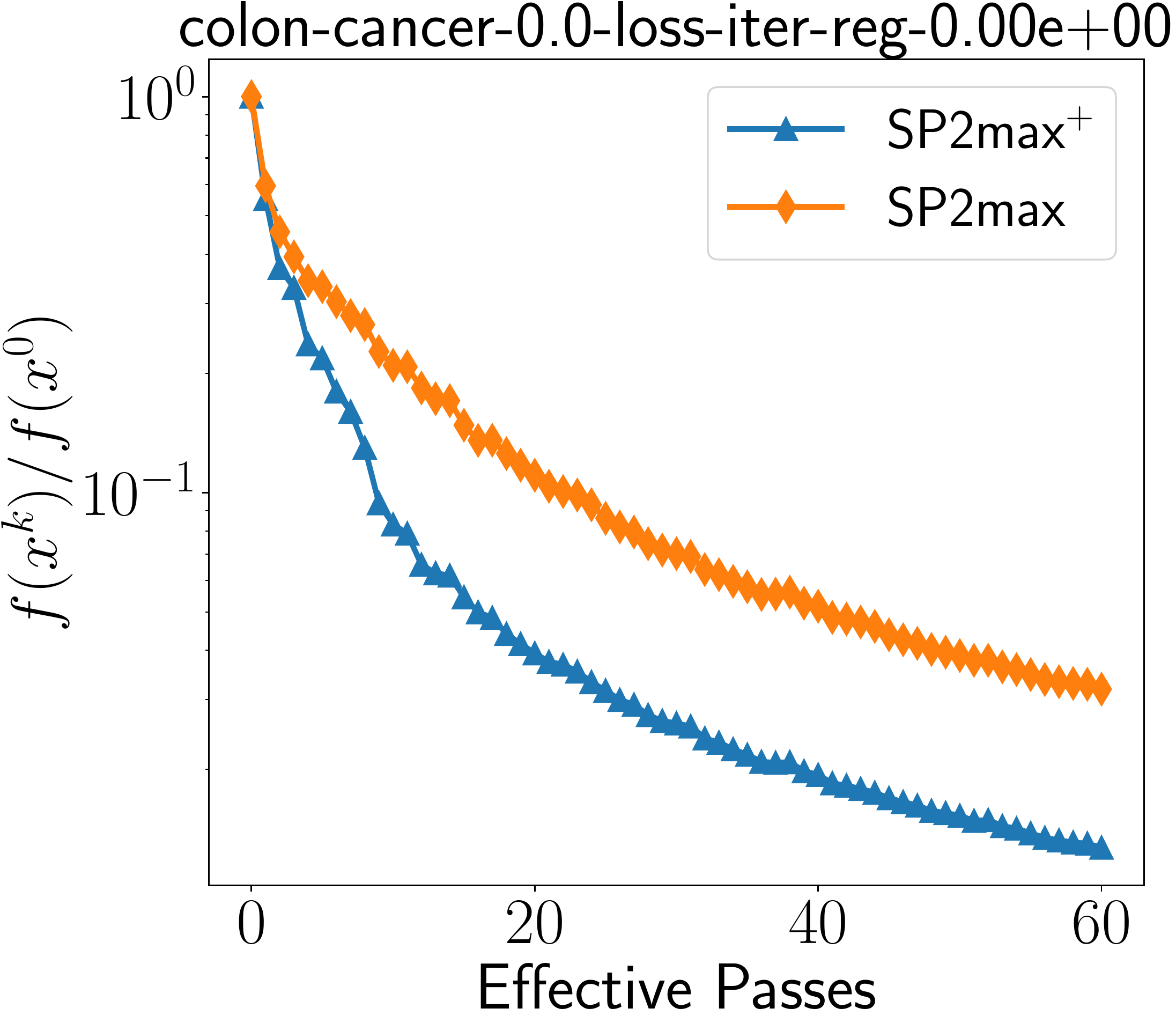}
\centerline{\small{(b) $\lambda = 0.01$}}
\end{minipage}
\hfill
\begin{minipage}{0.32\linewidth}
\centering
\includegraphics[width=1.7in]{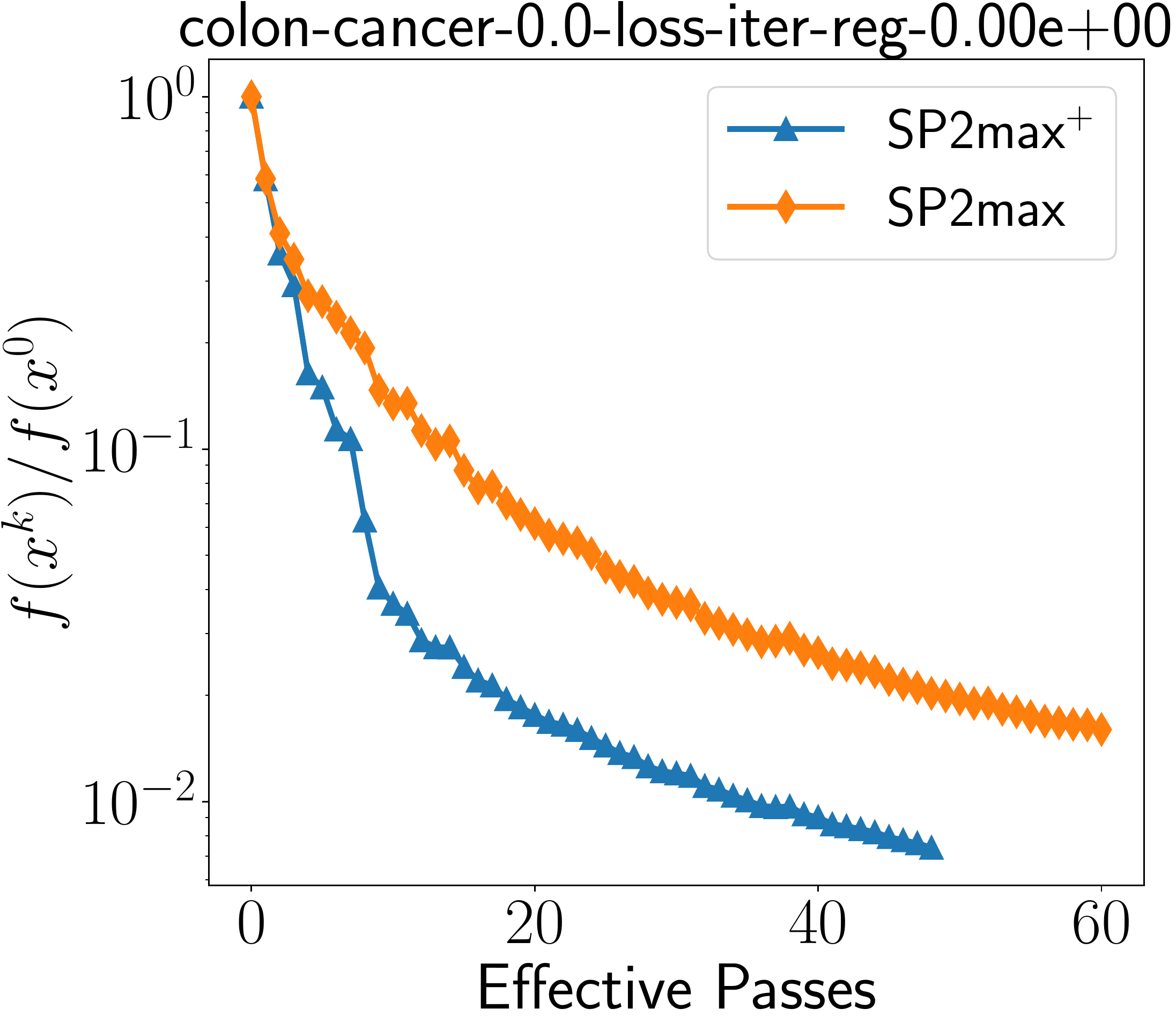}
\centerline{\small{(c) $\lambda = 0.02$}}
\end{minipage}
\\
\begin{minipage}{0.32\linewidth}
\centering
\includegraphics[width=1.7in]{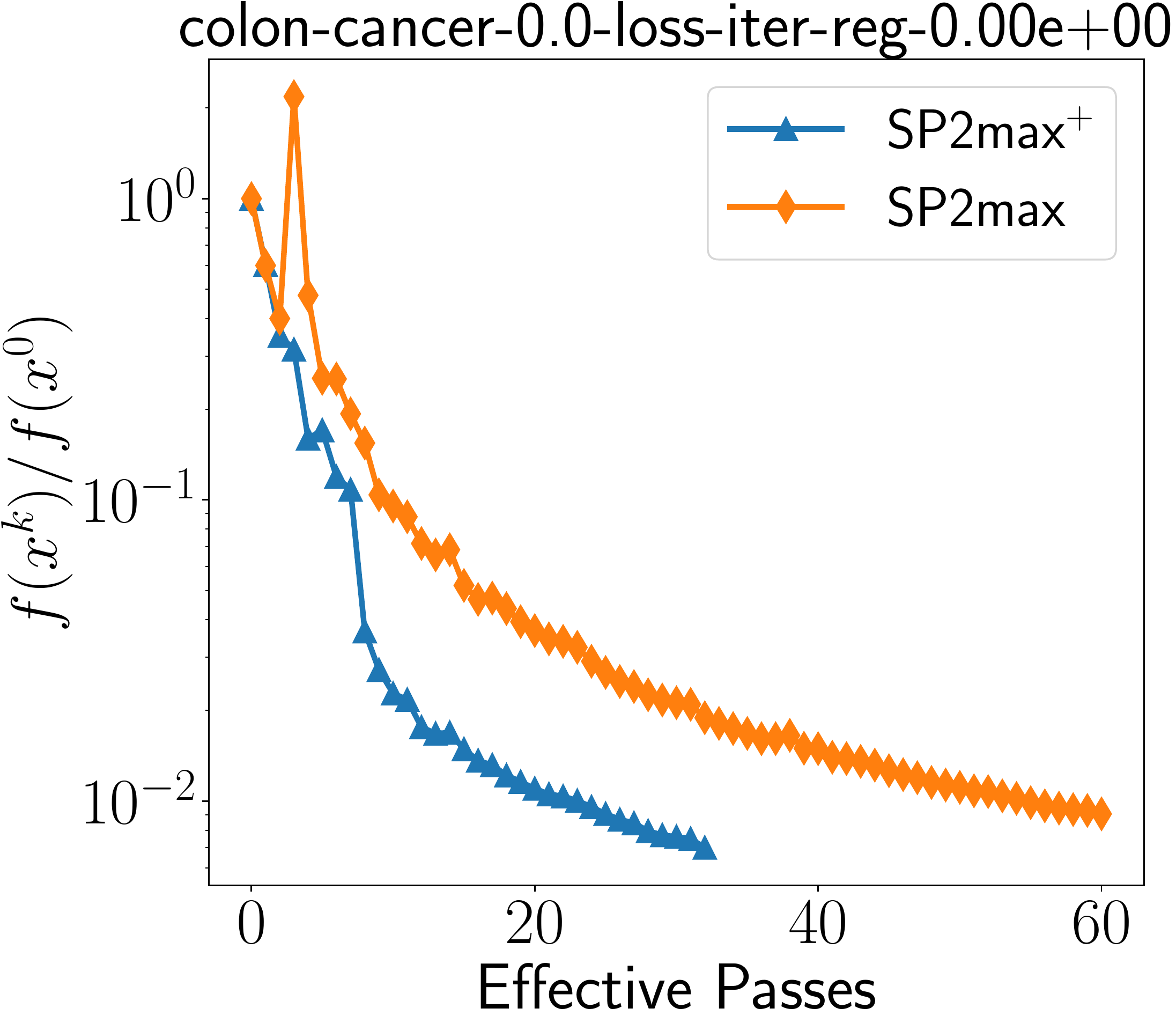}
\centerline{\small{(d) $\lambda = 0.03$}}
\end{minipage}
\hfill
\begin{minipage}{0.32\linewidth}
\centering
\includegraphics[width=1.7in]{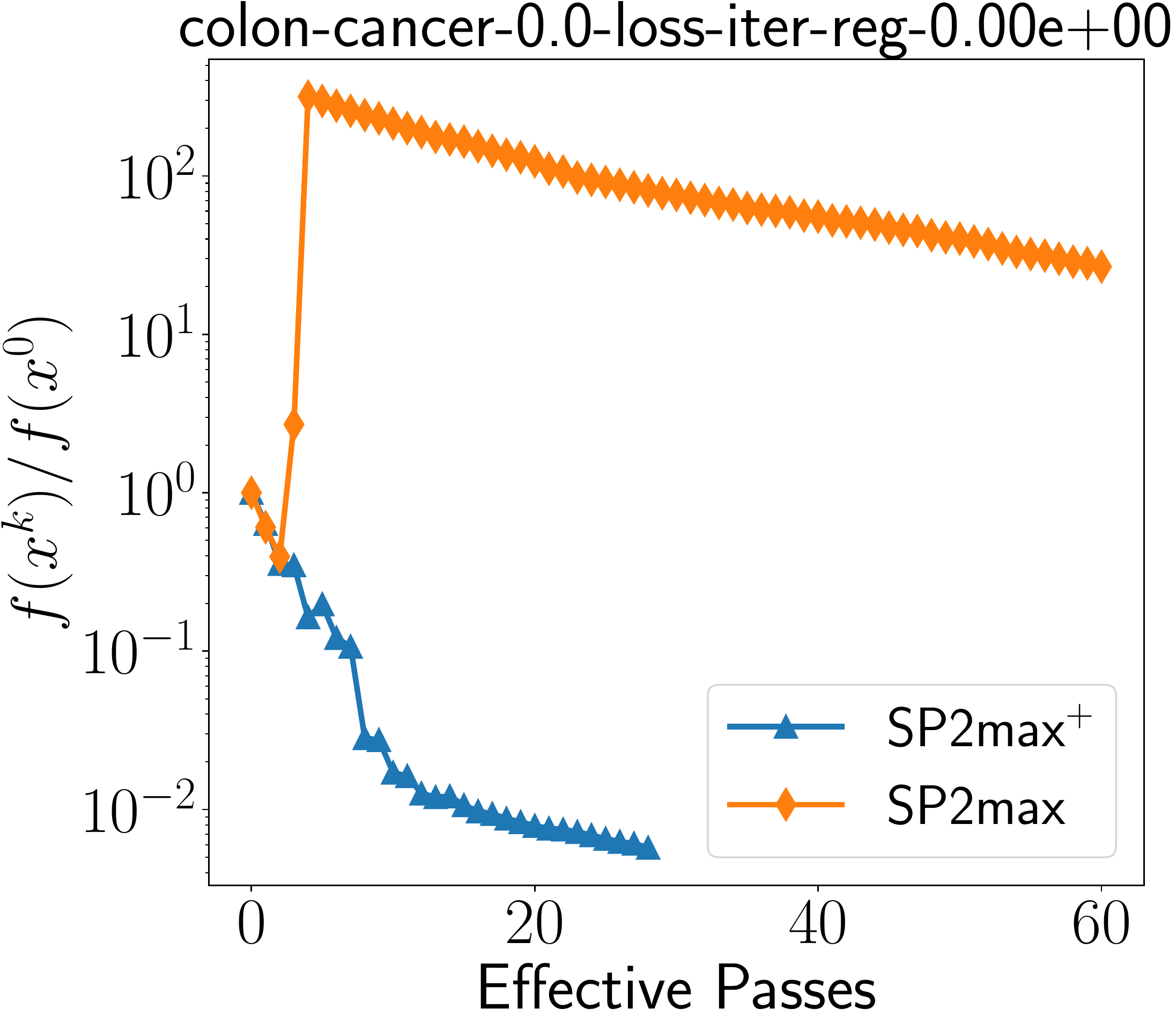}
\centerline{\small{(e) $\lambda = 0.04$}}
\end{minipage}
\hfill
\begin{minipage}{0.32\linewidth}
\centering
\includegraphics[width=1.7in]{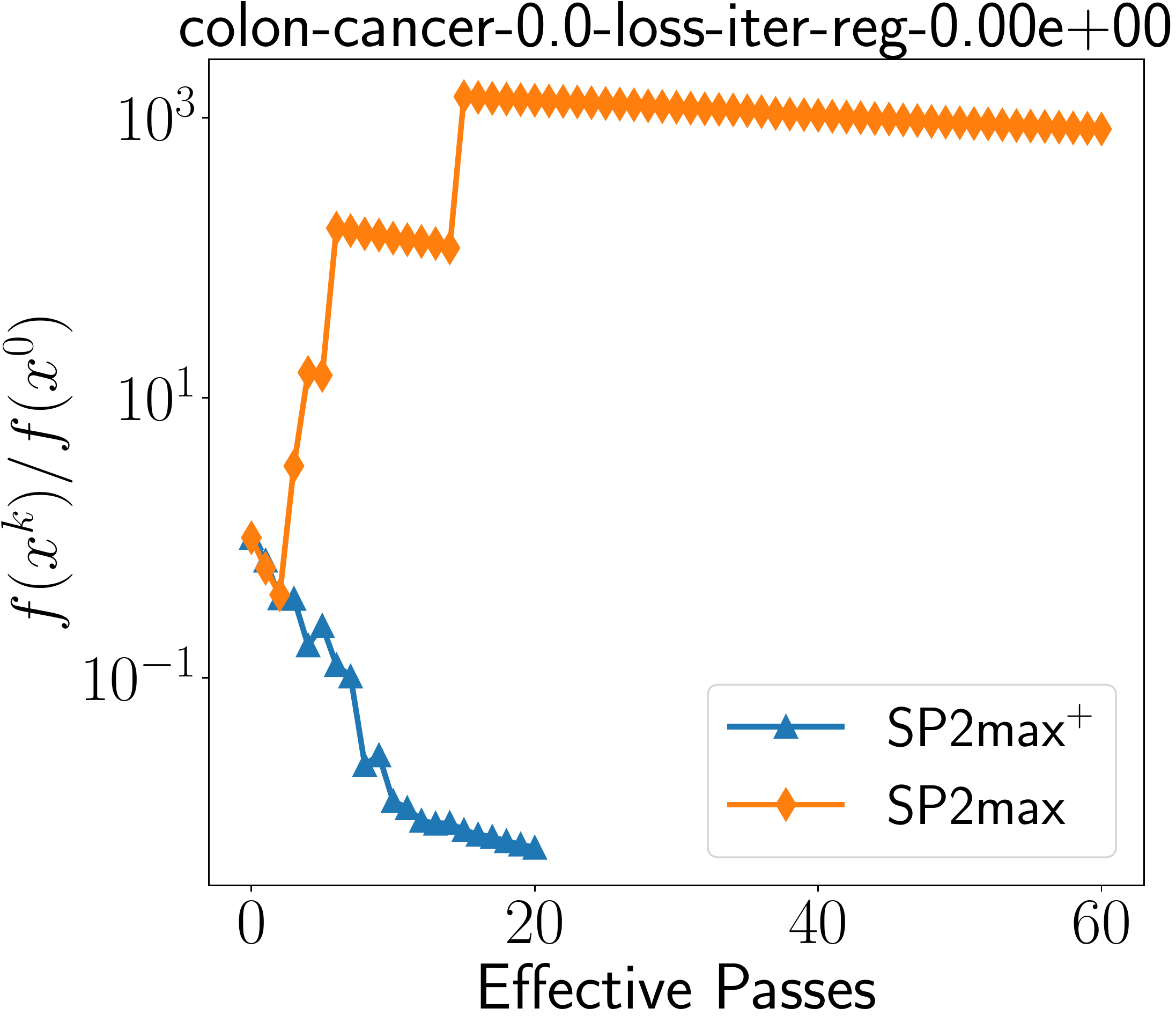}
\centerline{\small{(f) $\lambda = 0.05$}}
\end{minipage}
\caption{Colon-cancer: loss at each epoch with different $\lambda$. }
\label{fig:colon_loss_diff_lamb_max}
\end{figure}

\begin{figure}[t]
\begin{minipage}{0.32\linewidth}
\centering
\includegraphics[width=1.7in]{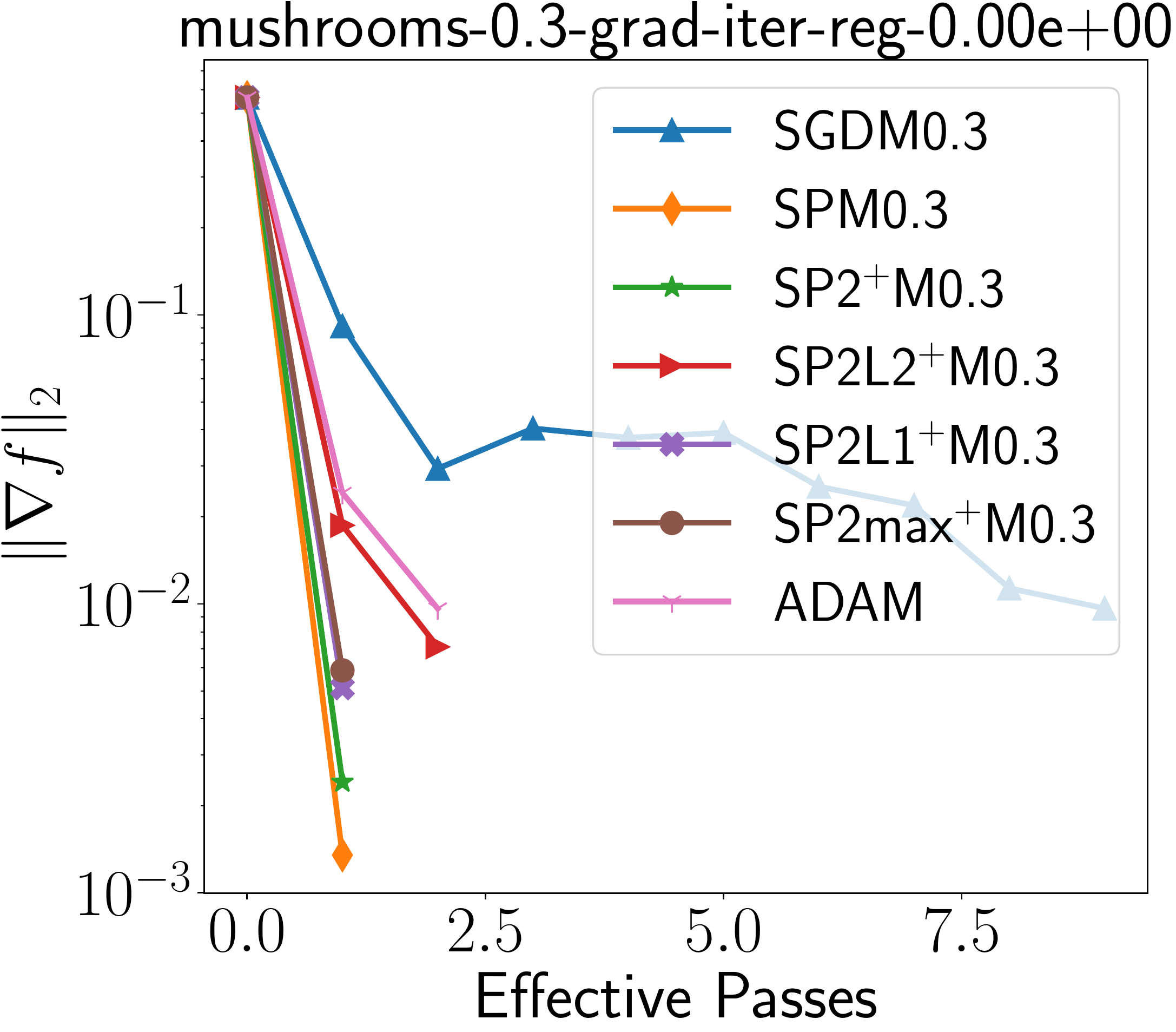}
%\centerline{\small{(a)}}
\end{minipage}
\hfill
\begin{minipage}{0.32\linewidth}
\centering
\includegraphics[width=1.7in]{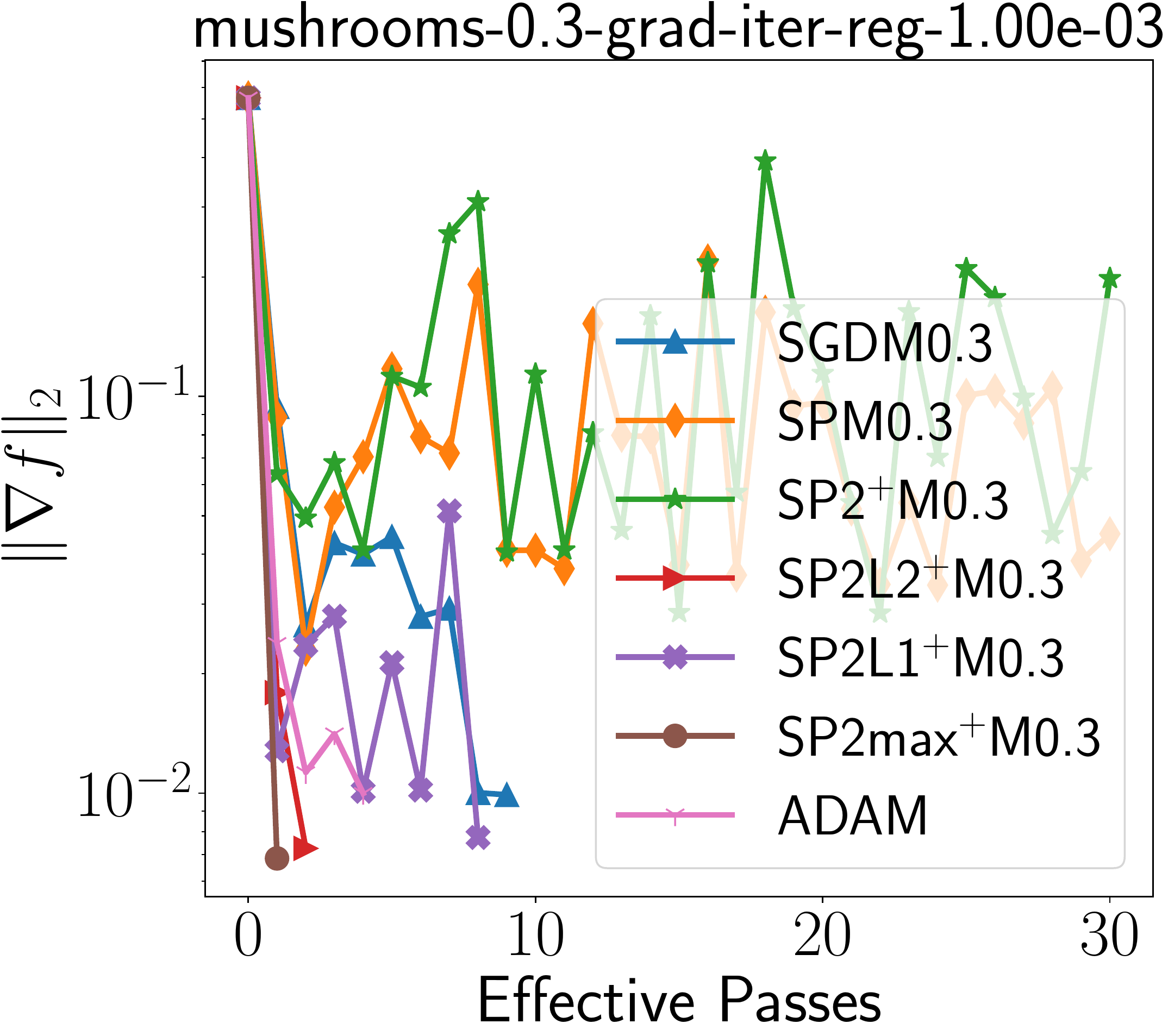}
%\centerline{\small{(a)}}
\end{minipage}
\hfill
\begin{minipage}{0.32\linewidth}
\centering
\includegraphics[width=1.7in]{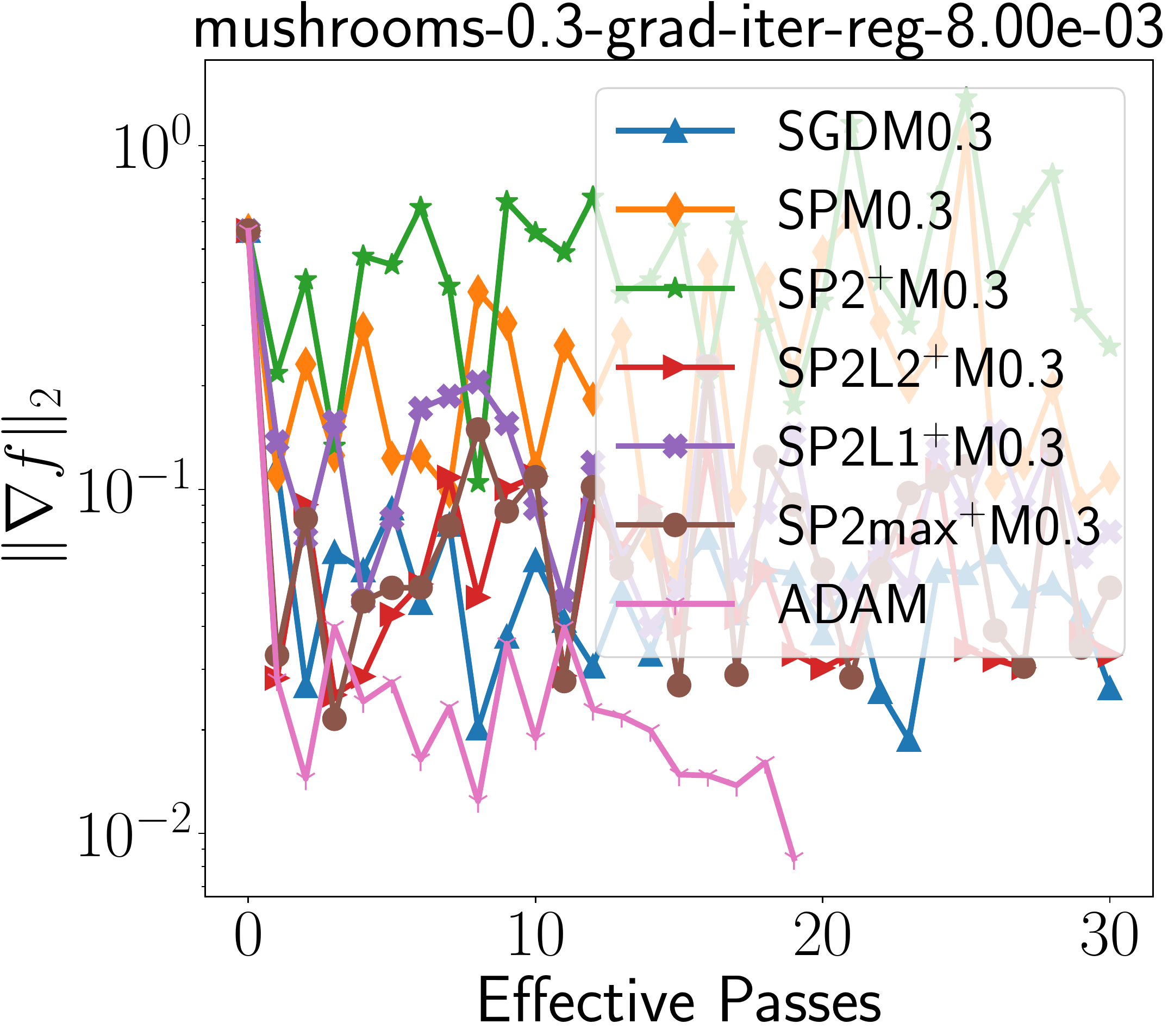}
%\centerline{\small{(b)}}
\end{minipage}
\hfill
\begin{minipage}{0.32\linewidth}
\centering
\includegraphics[width=1.7in]{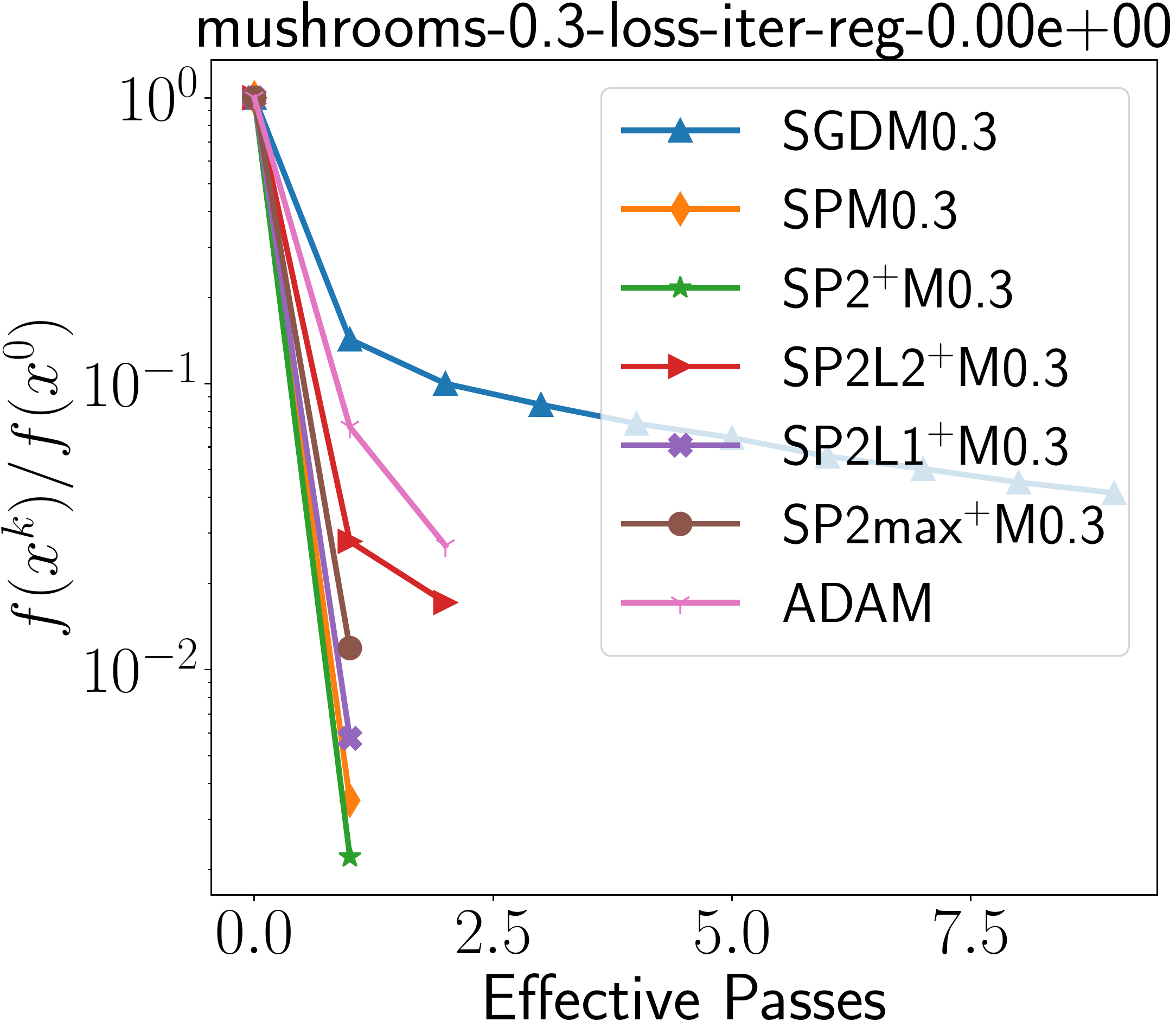}
%\centerline{\small{(c)}}
\end{minipage}
\hfill
\begin{minipage}{0.32\linewidth}
\centering
\includegraphics[width=1.7in]{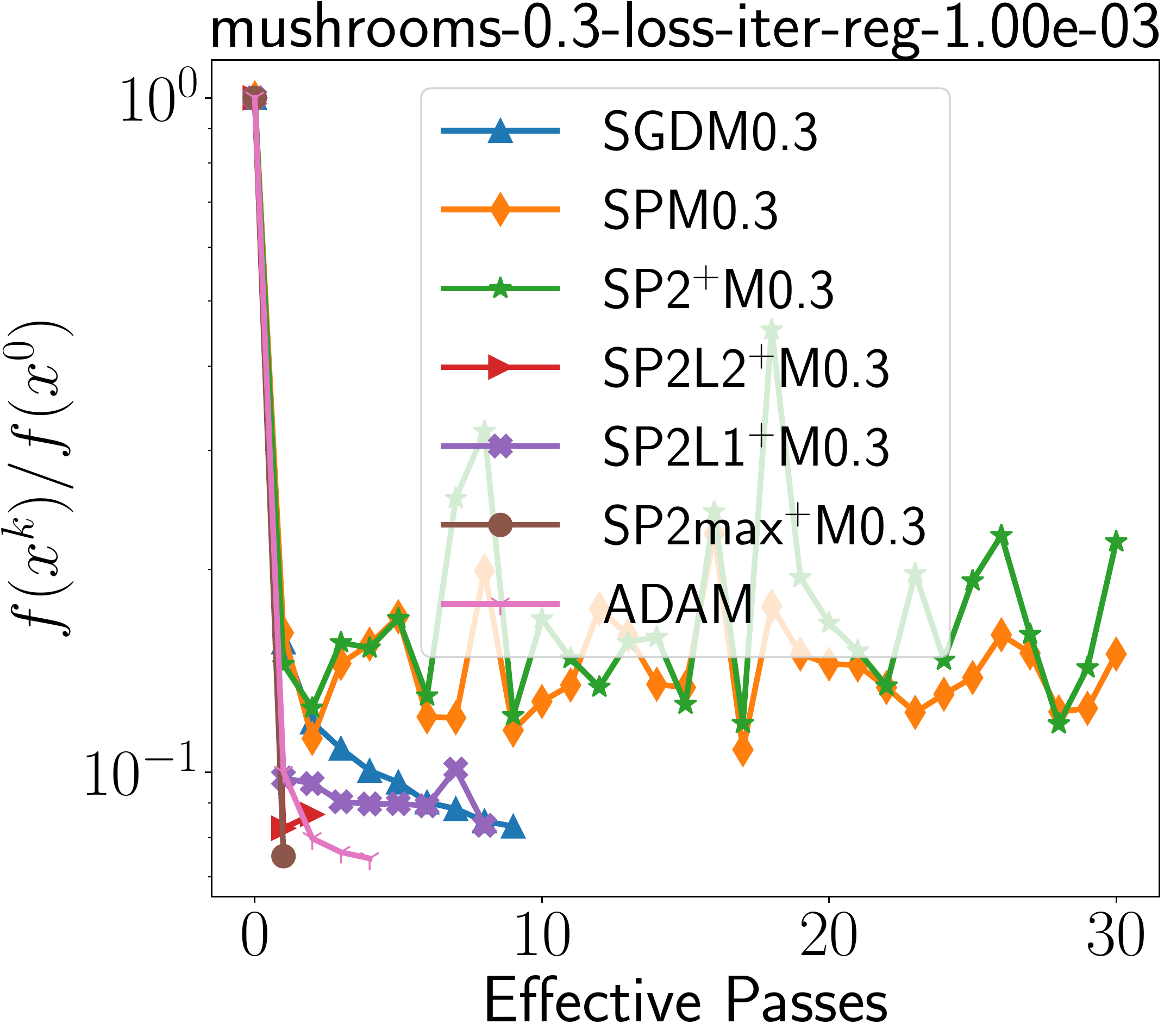}
%\centerline{\small{(c)}}
\end{minipage}
\hfill
\begin{minipage}{0.32\linewidth}
\centering
\includegraphics[width=1.7in]{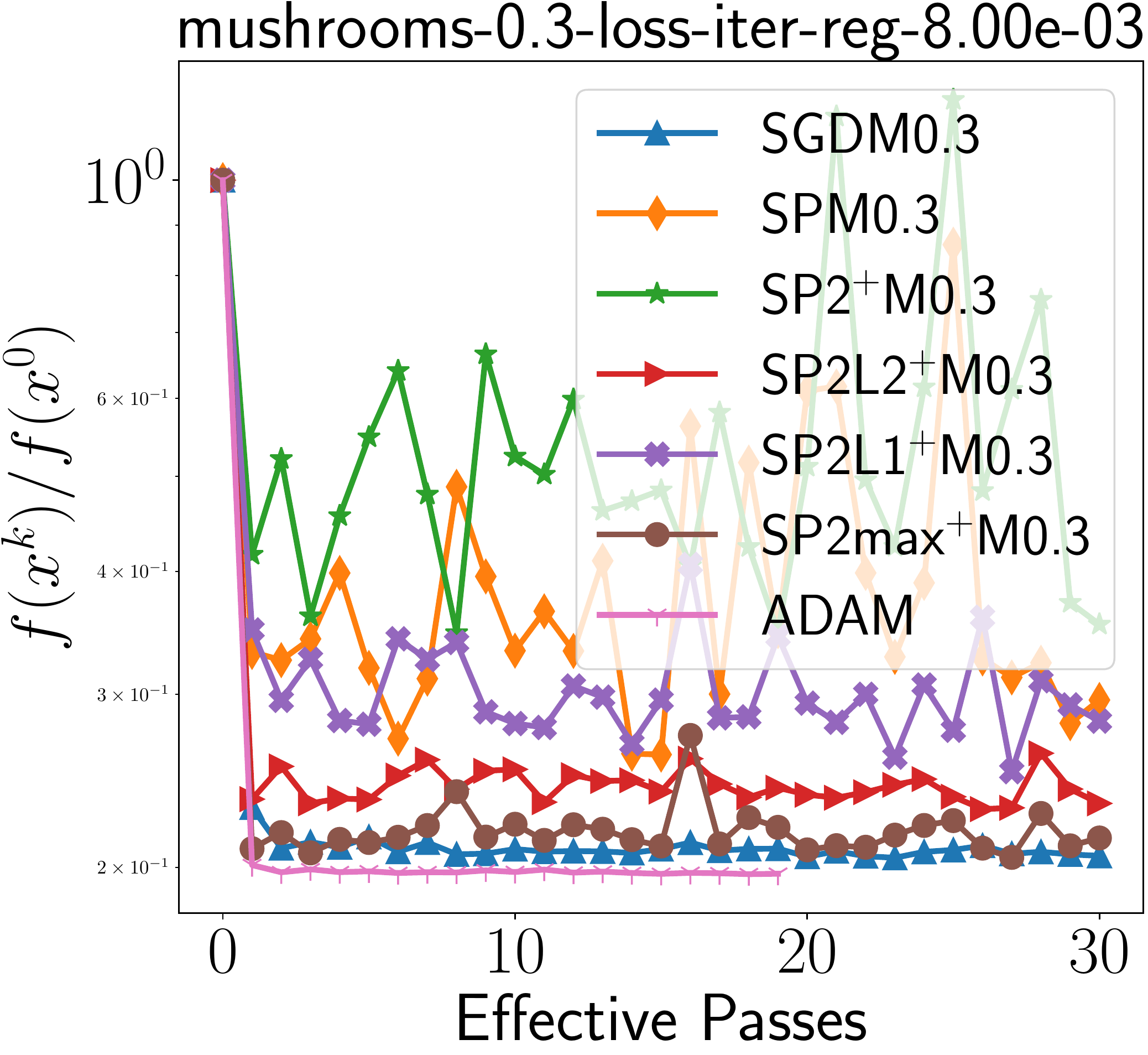}
%\centerline{\small{(d)}}
\end{minipage}
\caption{Mushrooms: gradient norm and loss at each epoch with momentum being $0.3$.}
\label{fig:mush_Gradnorm_loss_M05}
\end{figure}

\begin{figure}[t]
\begin{minipage}{0.45\linewidth}
\centering
\includegraphics[width=1.7in]{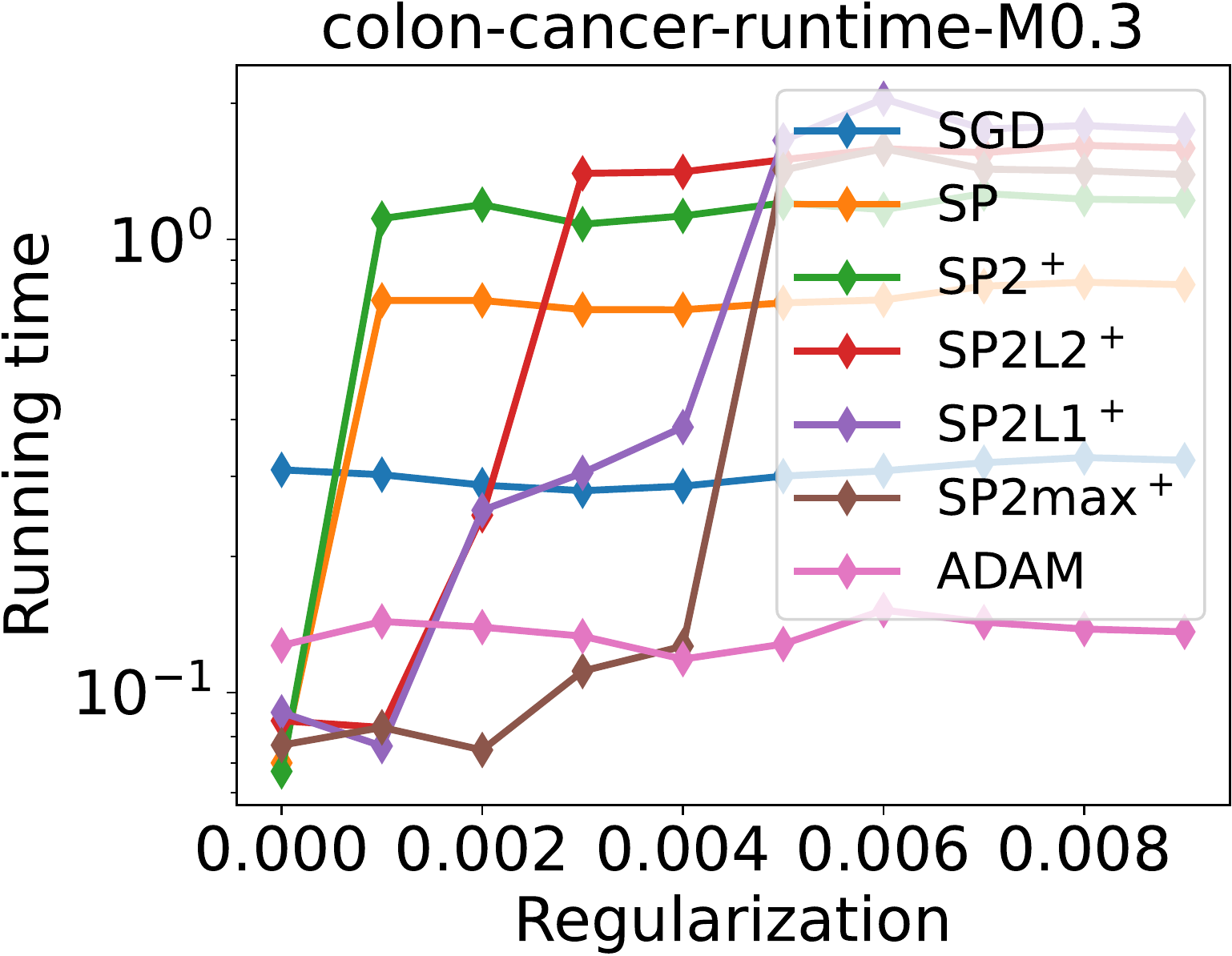}
%\centerline{\small{(a)}}
\end{minipage}
\hfill
\begin{minipage}{0.5\linewidth}
\centering
\includegraphics[width=1.7in]{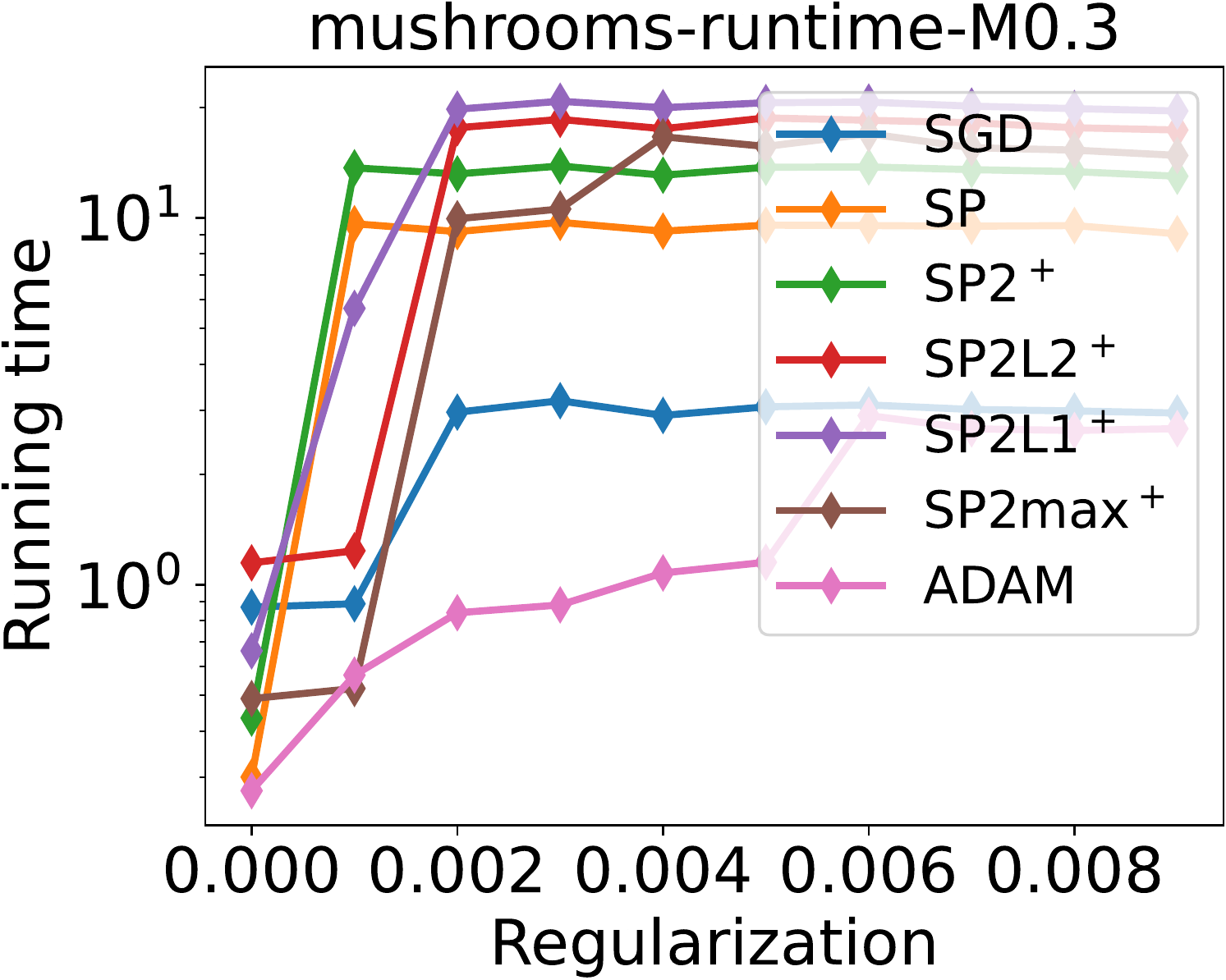}
%\centerline{\small{(b)}}
\end{minipage}
\caption{Running time in seconds.}
\label{fig:runtime}
\end{figure}

\end{document}